\def\papersize    {b5paper} 
\def\showtrims    {false} 
\def\confidential {false} 

\RequirePackage[l2tabu,orthodox]{nag} 
\newcommand{\papersizeswitch}[3]{\ifnum\strcmp{\papersize}{#1}=0#2\else#3\fi}
\papersizeswitch{b5paper}{\def\classfontsize{10pt}}{\def\classfontsize{12pt}}
\documentclass[\classfontsize,\papersize,twoside,showtrims,extrafontsizes]{memoir}
\showtrimsoff
\papersizeswitch{b5paper}{
    \pagebv
    \setlrmarginsandblock{26mm}{20mm}{*}
    \setulmarginsandblock{35mm}{30mm}{*}
    \setheadfoot{8mm}{10mm}
    \setlength{\headsep}{7mm}
    \setlength{\marginparwidth}{18mm}
    \setlength{\marginparsep}{2mm}
}{
    \papersizeswitch{a4paper}{
        \pageaiv
        \setlength{\trimtop}{0pt}
        \setlength{\trimedge}{\stockwidth}
        \addtolength{\trimedge}{-\paperwidth}
        \settypeblocksize{634pt}{448.13pt}{*}
        \setulmargins{4cm}{*}{*}
        \setlrmargins{*}{*}{0.66}
        \setmarginnotes{17pt}{51pt}{\onelineskip}
        \setheadfoot{\onelineskip}{2\onelineskip}
        \setheaderspaces{*}{2\onelineskip}{*}
    }{}
}

\ifnum\strcmp{\showtrims}{true}=0
    \showtrimson
    \papersizeswitch{b5paper}{\stockaiv}{\stockaiii}
    \trimLmarks
    
    \renewcommand*{\tmarktl}{%
      \begin{picture}(0,0)
        \unitlength 1mm
        \thinlines
        \put(-2,0){\line(-1,0){18}}
        \put(0,2){\line(0,1){18}}
        \put(3,15){\normalfont\ttfamily\fontsize{8bp}{10bp}\selectfont\jobname\ \
          \today\ \ 
          \printtime\ \ 
          Page \thepage}
      \end{picture}}

\fi

\checkandfixthelayout                 
\sideparmargin{outer}                 

\usepackage{microtype}
\usepackage{mathtools}
\usepackage{listings}                 
\usepackage{tikz}

\usepackage{amsmath,dsfont,amsfonts,braket}
\usepackage{physics}
\usepackage[shortlabels]{enumitem}

\usepackage{bm}
\usepackage{booktabs}       
\usepackage{amsfonts}       
\usepackage{nicefrac}       
\usepackage{microtype}      
\usepackage{amssymb}        
\usepackage{graphicx}
\usepackage{subcaption}

\usepackage{amsmath}
\usepackage{algorithmic}
\usepackage{amsthm}
\usepackage{pdfpages}
\usepackage{fancyhdr}

\usepackage{amsmath}
\usepackage[normalem]{ulem}

\useunder{\uline}{\ul}{}
\DeclareMathOperator*{\argmin}{arg\,min}

\usepackage{cite}
\usepackage{acronym}
\usepackage[hang,flushmargin]{footmisc} 

\usepackage{titlesec}
\usepackage{fancyhdr}

\usepackage[hyphens]{url}             
\usepackage[unicode=false,psdextra]{hyperref}                 

\renewcommand{\listfigurename}{List of plots}
\renewcommand{\listtablename}{Tables}

\usepackage{pdfpages}
\usepackage{graphicx}                 
\usepackage[font=footnotesize,labelformat=simple]{subcaption}

\usepackage{xcolor,colortbl}        
\usepackage{eso-pic}                  
\usepackage{preamble/dtucolors}
\graphicspath{{graphics/}}

\fancypagestyle{noheader}{
    \fancyhead{}         
    \fancyfoot[C]{\thepage}  
}

\usepackage{polyglossia}    
\setdefaultlanguage{english}
\usepackage{csquotes}       

\usepackage[nodayofweek]{datetime}
\usepackage[super]{nth}
\newdateformat{mydate}{\nth{\THEDAY}{ }\monthname[\THEMONTH] \THEYEAR} 

\usepackage{flafter}  
\usepackage[noabbrev,nameinlink,capitalise]{cleveref} 
\DeclareCaptionFont{dtu}{\normalsize\sffamily\selectfont} 
\usepackage[labelfont={dtu,bf},labelsep=period]{caption}
\captionnamefont{\bfseries}
\subcaptionlabelfont{\bfseries}
\newsubfloat{figure}
\newsubfloat{table}
\crefformat{equation}{(#2#1#3)}
\crefrangeformat{equation}{(#3#1#4) to~(#5#2#6)}
\crefmultiformat{equation}{(#2#1#3)}%
{ and~(#2#1#3)}{, (#2#1#3)}{ and~(#2#1#3)}

\titleformat{\part}[display]{\filcenter\chapnamefont\fontsize{42pt}{0pt}\selectfont\setlength{\parskip}{-1.5cm}}{\vspace{4cm}\color{dtugray}{\chapnamefont\fontsize{36pt}{0pt}\selectfont}
\partname\chapnamefont\color{dtured}\fontsize{40pt}{0pt}\selectfont\hspace{.3em}\thepart}{4ex}{}
\titlespacing{\part}{0pt}{0pt}{0pt}

\makeatletter
\makechapterstyle{mychapterstyle}{
    \chapterstyle{default}
    \def\format{\normalfont\sffamily}
    \setlength\beforechapskip{0mm}
    \renewcommand*{\chapnamefont}{\format\HUGE}

    \renewcommand*{\printchaptername}{\chapnamefont\MakeUppercase{\@chapapp}}
    \patchcommand{\printchaptername}{\begingroup\color{dtugray}}{\endgroup}
    
    \patchcommand{\printchapternum}{\begingroup\color{dtured}}{\endgroup}
    
    \setlength\midchapskip{1ex}

}
\makeatother
\chapterstyle{mychapterstyle}

\let\appendixpagenameorig\appendixpagename
\renewcommand{\appendixpagename}{\vspace{-2cm}\normalfont\sffamily\format\fontsize{42pt}{0pt}\selectfont\appendixpagenameorig} 

\newcounter{protocol}

\def\hffont{\sffamily\small}
\makepagestyle{myruled}
\makeheadrule{myruled}{\textwidth}{\normalrulethickness}
\makeevenhead{myruled}{\hffont\thepage}{}{\hffont\leftmark}
\makeoddhead{myruled}{\hffont\rightmark}{}{\hffont\thepage}
\makeevenfoot{myruled}{}{}{}
\makeoddfoot{myruled}{}{}{}
\makepagestyle{appruled}
\makeheadrule{appruled}{\textwidth}{\normalrulethickness}
\makeevenhead{appruled}{\hffont\thepage}{}{\hffont\appendixname~\leftmark}
\makeoddhead{appruled}{\hffont\appendixname~\rightmark}{}{\hffont\thepage}
\makeevenfoot{appruled}{}{}{}
\makeoddfoot{appruled}{}{}{}
\makepsmarks{myruled}{
    \nouppercaseheads
    \createmark{chapter}{both}{shownumber}{}{\space}
    \createmark{section}{right}{shownumber}{}{\space}
    \createplainmark{toc}{both}{\contentsname}
    \createplainmark{lof}{both}{\listfigurename}
    \createplainmark{lot}{both}{\listtablename}
    \createplainmark{bib}{both}{\bibname}
    \createplainmark{index}{both}{\indexname}
    \createplainmark{glossary}{both}{\glossaryname}
}
\copypagestyle{cleared}{myruled}      
\makeevenhead{cleared}{\hffont\thepage}{}{} 
\makeevenfoot{plain}{}{}{}            
\makeoddfoot{plain}{}{}{}             

\setsecheadstyle              {\huge\sffamily\raggedright}
\setsubsecheadstyle           {\LARGE\sffamily\raggedright}
\setsubsubsecheadstyle        {\Large\sffamily\raggedright}

\hypersetup{
    pdfdisplaydoctitle,
    bookmarksnumbered=true,
    bookmarksopen,
    breaklinks,
    linktoc=all,
    plainpages=false,
    unicode=true,
    colorlinks=false,
    citebordercolor=dtured,           
    filebordercolor=dtured,           
    linkbordercolor=dtured,           
    urlbordercolor=blue,               
    hidelinks,                        
}

\makeatletter
\renewcommand{\@memb@bchap}{%
  \ifnobibintoc\else
    \phantomsection
    \addcontentsline{toc}{part}{\bibname}%
  \fi
  \chapter*{\bibname}%
  \bibmark
  \prebibhook
}
\let\oldtableofcontents\tableofcontents
\newcommand{\newtableofcontents}{
    \@ifstar{\oldtableofcontents*}{
        \phantomsection\addcontentsline{toc}{chapter}{\contentsname}\oldtableofcontents*}}
\let\tableofcontents\newtableofcontents
\makeatother

\newcommand{\confidentialbox}[1]{
    \put(0,0){\parbox[b][\paperheight]{\paperwidth}{
        \begin{vplace}
            \centering
            \scalebox{1.3}{
                \begin{tikzpicture}
                    \node[very thick,draw=red!#1,color=red!#1,
                          rounded corners=2pt,inner sep=8pt,rotate=-20]
                          {\sffamily \HUGE \MakeUppercase{Confidential}};
                \end{tikzpicture}
            }
        \end{vplace}
    }}
}

\newcommand{\prefrontmatter}{
    \pagenumbering{alph}
    \ifnum\strcmp{\confidential}{true}=0
        \AddToShipoutPictureBG{\confidentialbox{10}}   
        \AddToShipoutPictureFG*{\confidentialbox{100}} 
    \fi
}

\newcommand{\frieze}{%
    \AddToShipoutPicture*{
        \put(0,0){
            \parbox[b][\paperheight]{\paperwidth}{%
                \includegraphics[trim=130mm 0 0 0,width=0.9\textwidth]{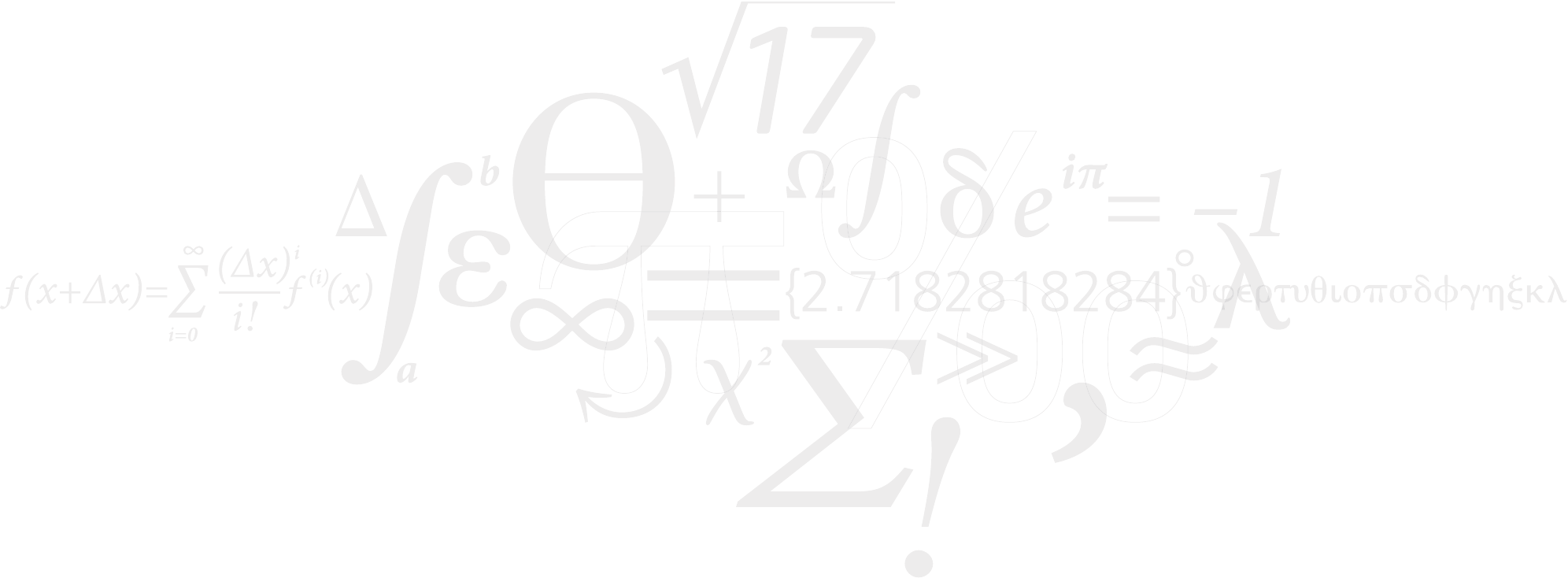}
                \vspace*{2.5cm}
            }
        }
    }
}

\makeatletter
\patchcmd{\leavespergathering}{\ifnum\@memcnta<\tw@}{\ifnum\@memcnta<\@ne}{
    \leavespergathering{1}
    \patchcmd{\@memensuresigpages}{\repeat}{\repeat\frieze}{}{}
}{}
\makeatother

\makeatletter
\def\MT@is@composite#1#2\relax{%
  \ifx\\#2\\\else
    \expandafter\def\expandafter\MT@char\expandafter{\csname\expandafter
                    \string\csname\MT@encoding\endcsname
                    \MT@detokenize@n{#1}-\MT@detokenize@n{#2}\endcsname}%
    \ifx\UnicodeEncodingName\@undefined\else
      \expandafter\expandafter\expandafter\MT@is@uni@comp\MT@char\iffontchar\else\fi\relax
    \fi
    \expandafter\expandafter\expandafter\MT@is@letter\MT@char\relax\relax
    \ifnum\MT@char@ < \z@
      \ifMT@xunicode
        \edef\MT@char{\MT@exp@two@c\MT@strip@prefix\meaning\MT@char>\relax}%
          \expandafter\MT@exp@two@c\expandafter\MT@is@charx\expandafter
            \MT@char\MT@charxstring\relax\relax\relax\relax\relax
      \fi
    \fi
  \fi
}
\def\MT@is@uni@comp#1\iffontchar#2\else#3\fi\relax{%
  \ifx\\#2\\\else\edef\MT@char{\iffontchar#2\fi}\fi
}
\makeatother

\usepackage{fontspec}

\usepackage{blindtext}

\usepackage{pdfpages}

\begin{document}
\prefrontmatter
\thispagestyle{empty}             
\calccentering{\unitlength}
\begin{adjustwidth*}{\unitlength}{-\unitlength}
    \begin{adjustwidth}{-0.5cm}{-0.5cm}
        \sffamily
        \begin{flushright}
            Ph.D. Thesis\\*[0cm]
            September 2023 \\
        \end{flushright}
        \vspace*{\fill}
        \noindent
        \includegraphics[width=0.6\textwidth]{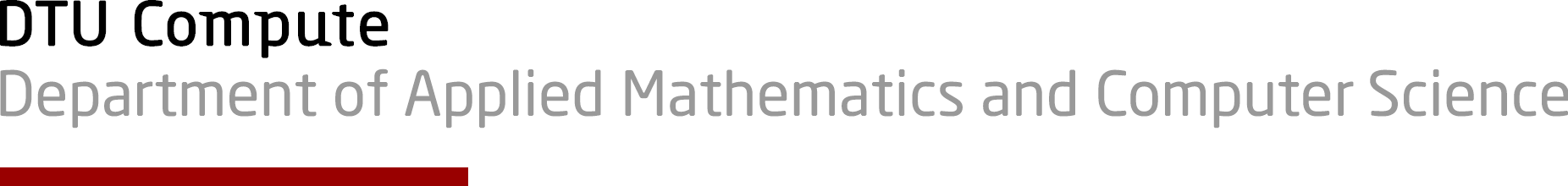}\\*[0.2cm]
        \HUGE \textbf{Machine Learning\\ for Static and Single-Event Dynamic\\ Complex Network Analysis \\     
        }\\*[0.6cm]
        \parbox[b]{0.5\linewidth}{
        \huge Nikolaos Nakis\\*[1.2cm]
        } 

        \hfill\includegraphics[scale=0.065]{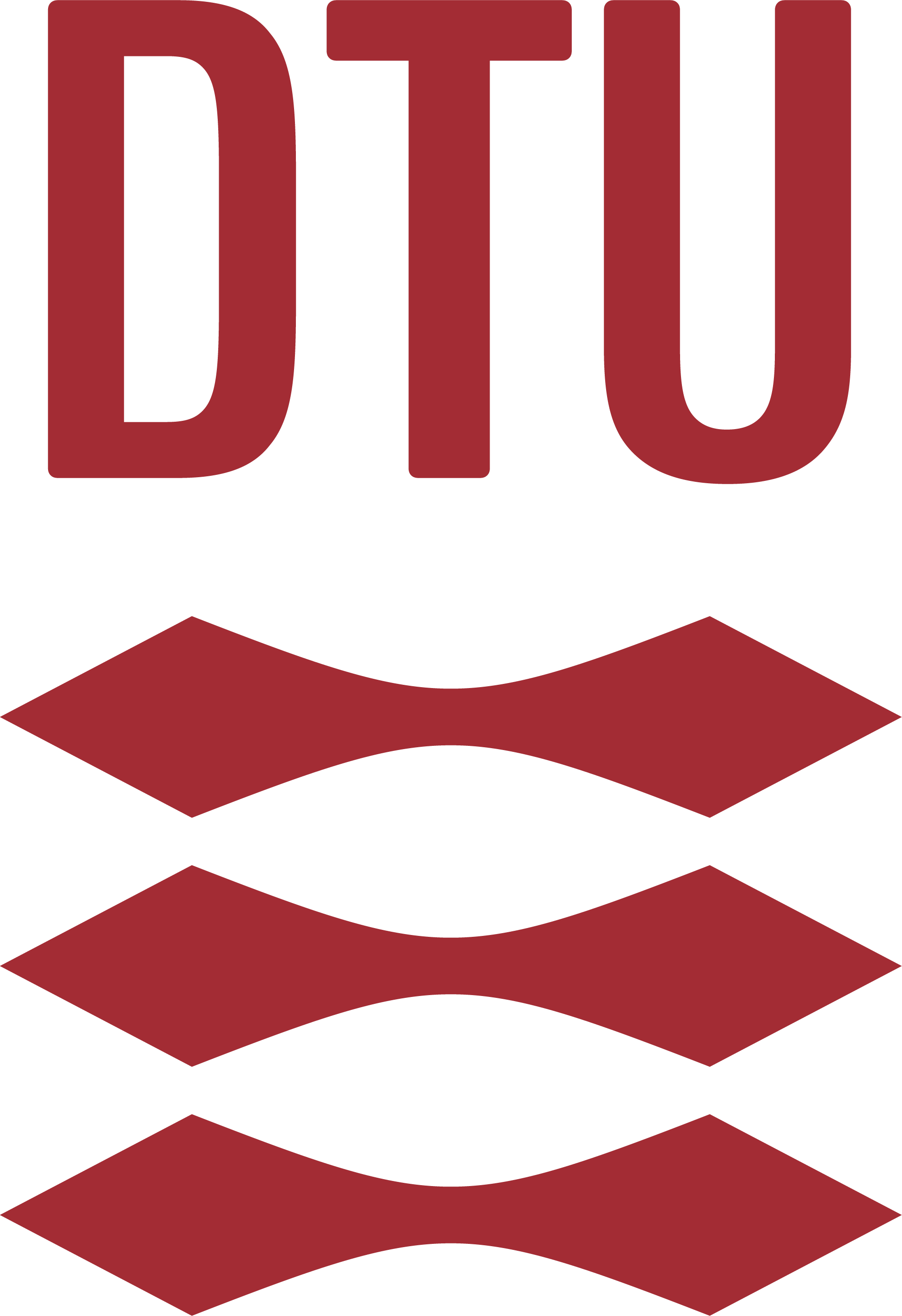}
    \end{adjustwidth}
\end{adjustwidth*}
\normalfont
\normalsize

\cleartoevenpage
\thispagestyle{empty} 
\frieze
\noindent
\sffamily

\begin{center}
    Ph.D. Degree program  \\*[0.2cm]
    Technical University of Denmark (DTU), \\
    \textit{Department of Applied Mathematics and Computer Science} \\*[0.2cm]
    
\end{center}

\large
\vspace*{\fill}
\hspace{-0.52cm}\textbf{Supervisor:} Prof. Morten Mørup (DTU)\\
\textbf{Co-supervisor:} Prof. Sune Lehmann (DTU)
\vspace{5cm}

\small
\hspace*{\fill}\textbf{DTU Compute}

\hspace*{\fill}\textbf{Department of Applied Mathematics and Computer Science}

\hspace*{\fill}\textbf{Technical University of Denmark}


\hspace*{\fill}Building 321

\hspace*{\fill}2800 Kongens Lyngby, Denmark

\normalsize
\normalfont
\vspace*{2cm}

\clearforchapter

\newtheorem{definition}{Definition}
\newtheorem{theorem}{Theorem}[section]
\newtheorem{lemma}[theorem]{Lemma}
\newtheorem{corollary}[theorem]{Corollary}
\newtheorem{innercustomlemma}{Lemma}
\newenvironment{customlemma}[1]

\frontmatter
\chapter{Summary (English)}

Networks are prevalent data structures that naturally express complex systems. They emerge across a multitude of scientific domains, including physics, sociology, the science of science, biology, neuroscience, and more. In these disciplines, networks illustrate diverse interactions and systems: spin glasses in physics, social connections in sociology, academic collaborations, protein-to-protein interactions in biology, and both structural and functional brain connectivity in neuroscience, to name a few. Due to their complexity and inherently high-dimensional discrete nature, accurately characterizing network structures is both non-trivial and challenging. In recent years, Graph Representation Learning (\textsc{GRL}) has achieved remarkable success in the study of networks, establishing itself as the leading method for network analysis. In general, \textsc{GRL} aims to create a function that can successfully map a network to a low-dimensional latent space through a learning process. Such a mapping defines representations that can be very useful for conducting various downstream tasks, and importantly for helping us to further our understanding of complex networks and their underlying structures.
\\

\noindent
The primary objective of this thesis is to develop novel algorithmic approaches for Graph Representation Learning of static and single-event dynamic networks. In such a direction, we focus on the family of Latent Space Models, and more specifically on the Latent Distance Model which naturally conveys important network characteristics such as homophily, transitivity, and the balance theory. Furthermore, this thesis aims to create structural-aware network representations, which lead to hierarchical expressions of network structure, community characterization, the identification of extreme profiles in networks, and impact dynamics quantification in temporal networks. Crucially, the methods presented are designed to define unified learning processes, eliminating the need for heuristics and multi-stage processes like post-processing steps. Our aim is to delve into a journey towards unified network embeddings that are both comprehensive and powerful, capable of characterizing network structures and adeptly handling the diverse tasks that graph analysis offers.

\chapter{Summary (Danish)}

Netværk er almindelige datastrukturer, der naturligt udtrykker komplekse systemer. De opstår på tværs af mange videnskabelige domæner, herunder fysik, sociologi, videnskabens videnskab, biologi, neurovidenskab og mere. I disse discipliner illustrerer netværk forskellige interaktioner og systemer: spin-glas i fysik, sociale forbindelser i sociologi, akademisk samarbejde, protein-til-protein interaktioner i biologi, samt både strukturel og funktionel hjerneforbindelse i neurovidenskab, for blot at nævne nogle få. På grund af deres kompleksitet og iboende høj-dimensionale diskrete natur er nøjagtig karakterisering af netværksstrukturer både ikke-trivielt og udfordrende. I de seneste år har Graf Representation Læring (\textsc{GRL}) opnået bemærkelsesværdig succes i studiet af netværk, og har etableret sig som den førende metode til netværksanalyse. Generelt sigter \textsc{GRL} mod at skabe en funktion, der med succes kan kortlægge et netværk til et lav-dimensionalt latent rum gennem en læringsproces. En sådan kortlægning definerer repræsentationer, der kan være meget nyttige til at udføre forskellige efterfølgende opgaver, og vigtigt for at hjælpe os med yderligere at forstå komplekse netværk og deres underliggende strukturer.
\\

\noindent
Hovedformålet med denne afhandling er at udvikle nye algoritmiske tilgange til Graf Representation Læring af statiske og enkeltbegivenheds dynamiske netværk. I denne retning fokuserer vi på familien af Latente Rummodeller, og mere specifikt på den Latente Afstandsmodel, som naturligt formidler vigtige netværksegenskaber såsom homofili, transitivitet og balance teorien. Yderligere sigter denne afhandling mod at skabe strukturbevidste netværksrepræsentationer, hvilket fører til hierarkiske udtryk af netværksstruktur, fællesskabskarakterisering, identifikation af ekstreme profiler i netværk og kvantificering af påvirkningsdynamik i tidsmæssige netværk. Afgørende er de præsenterede metoder designet til at definere ensartede læringsprocesser, hvilket eliminerer behovet for heuristikker og flertrinsprocesser som efterbehandlings trin. Vores mål er at dykke ned i en rejse mod ensartede netværksindlejringer, der er både omfattende og kraftfulde, i stand til at karakterisere netværksstrukturer og dygtigt håndtere de forskelligartede opgaver, som grafanalyse tilbyder.

\chapter{Preface}
This Ph.D. thesis, entitled \textit{Machine Learning for Static and Single-Event Dynamic Complex Network Analysis}, was prepared in the Section for Cognitive Systems (CogSys) within the Department of Applied Mathematics and Computer Science at the Technical University of Denmark (DTU). It is submitted in partial fulfillment of the requirements for obtaining a Ph.D. in Applied Mathematics and Computer Science from DTU.
\\

\noindent
The Ph.D. project was supervised by Prof. Morten Mørup and co-supervised by Prof. Sune Lehmann while it was generously financed by the Independent Research Fund Denmark [grant number: 0136-00315B]. The Ph.D. project was carried out at DTU during the period September 2020 - September 2023, except for a five-month external stay at the University of Umeå under the supervision of Prof. Martin Rosvall, and a three-month external stay at Yale University under the supervision of Sterling Prof. Nicholas Christakis. 
\\

\noindent
The thesis comprises five research papers focused on the topic of Graph Representation Learning.

\vfill

{
\centering
    Kongens Lyngby, \mydate\today\\[.5cm]
    \includegraphics[scale=0.3]{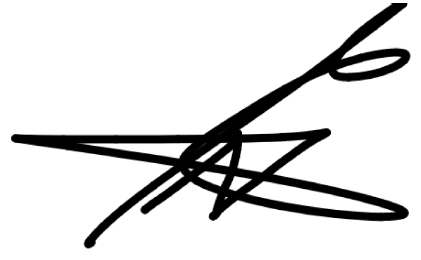}\\[.5cm]
\begin{center}
    Nikolaos Nakis
\end{center}
}
\chapter{Acknowledgements}
First and foremost, I would like to thank my supervisor, Prof. Morten Mørup. I consider myself extremely fortunate to have had such a talented and unique researcher guide me through my Ph.D. studies and act as my mentor. From my very first day on this journey, Morten has been a constant pillar of support, standing by me through the 'good' and 'bad' surprises our research presented. I will forever be grateful for all the things I have learned alongside him, which will stay with me for the rest of my career. It's largely thanks to him that I decided to pursue a career in research.
\\

\noindent
I am eternally grateful to Dr. Abdulkadir \c{C}elikkanat. Kadir has been a consistent source of inspiration throughout my Ph.D. journey. His natural talent for research, his exceptional work ethic, and his kind personality have all guided me toward becoming a better researcher and, more broadly, a better person. I owe him immense gratitude for all the support and guidance he has provided these past years.
\\

\noindent
I would like to thank my co-supervisor, Prof. Sune Lehman. Sune is not only one of the most outstanding researchers I have had the pleasure of working with but also one of the coolest. Over the past few years, he has consistently been a source of inspiration, brimming with amazing and endless research ideas. His guidance has significantly broadened my research horizons and shaped my way of thinking.
\\

\noindent
Throughout this journey, I was fortunate to encounter exceptional individuals and researchers who I now proudly call friends. Tremendous thanks to Davide and Agatha; spending time with you both has helped me navigate the lows and celebrate the highs of these past years,  you have been my Polish-Italian hybrid family in Copenhagen. In the same spirit, my deep appreciation goes to Peter, Germans, Mohammed (Mo), George, Giorgio, Lasse, Louis, Silvia, and Rui. Thank you all for everything.
\\

\noindent
I have had the pleasure to be a part of the CogSys research family. I would like to thank all of my colleagues at DTU for becoming a part of my daily routine these past years.
\\

\noindent
A special thanks to my dearest friend George our friendship united us also in our academic choices and since the day we met in Denmark our life paths have been very similar. Thank you for all the discussions, laughs, and frustrations we shared during our Ph.D. journeys; they truly helped me reach where I am today.
\\

\noindent
I cannot find the words to express my gratitude to you, Anastasia. From the first day we met in Denmark, you have consistently been there for me, guiding and inspiring me with your passion for research and life. I owe part of who I've become today to you, and I'm deeply thankful for all the help you've generously provided, enabling me to reach this milestone.
\\

\noindent
Lastly, I would like to thank my family and friends. Especially, to my parents Grigoris and Aristi, I would not be here without you. Thank you for everything.

\chapter{List of publications}

Contributions included in the thesis:

\vspace{.2cm}

\begin{enumerate}

    \item N. Nakis, A. Çelikkanat, S. Lehmann and M. Mørup, "A Hierarchical Block Distance Model for Ultra Low-Dimensional Graph Representations," in \textit{IEEE Transactions on Knowledge and Data Engineering}, \url{https://doi.org/10.1109/TKDE.2023.3304344}, 2023.
    
    \item N. Nakis, A. Çelikkanat, and M. Mørup. "HM-LDM: A Hybrid-Membership Latent Distance Model". In: \textit{Complex Networks and Their Applications XI}. COMPLEX NETWORKS 2016 2022. Studies in Computational Intelligence, vol 1077. Springer, Cham. \url{https://doi.org/10.1007/978-3-031-21127-0_29}
    
    \item N. Nakis, A. Çelikkanat, L. Boucherie, C. Djurhuus, F. Burmester, D. M. Holmelund, M. Frolcová, and M. Mørup. "Characterizing Polarization in Social Networks using the Signed Relational Latent Distance Model". \textit{Proceedings of The 26th International Conference on Artificial Intelligence and Statistics}, PMLR 206:11489-11505, 2023.

    \item N. Nakis, A. Çelikkanat, and M. Mørup. "A Hybrid Membership Latent Distance Model for Unsigned and Signed Integer Weighted Networks". In: \textit{Advances in Complex Systems}. \url{https://doi.org/10.1142/S0219525923400027}, 2023.

\item N. Nakis, A. Çelikkanat, and M. Mørup. "Time to Cite: Modeling Citation Networks using the Dynamic Impact Single-Event Embedding Model". \textit{Proceedings of The 27th International Conference on Artificial Intelligence and Statistics}, PMLR 238:1882-1890, 2024.

\end{enumerate}

\vspace{.5cm}

\noindent
Contributions not included in the thesis:

\vspace{.2cm}

\begin{enumerate}[resume]

   \item A. Çelikkanat, N. Nakis, and M. Mørup, “Piecewise-velocity model for learning continuous-time dynamic node representations,” in \textit{Proceedings of the First
Learning on Graphs Conference} (B. Rieck and R. Pascanu, eds.), vol. 198 of
Proceedings of Machine Learning Research, pp. 36:1–36:21, PMLR, 09–12 Dec
2022.

\item A. Çelikkanat, N. Nakis, and M. Mørup. Continuous-time Graph Representation with Sequential Survival Process. \textit{In Proceedings of the Thirty-Eighth AAAI Conference on Artificial Intelligence}, 2024.

\end{enumerate}


\newpage

\noindent
Work related to these papers was presented at the following conferences:
    
\vspace{.5cm}

\hspace{.5em} (*) indicates the presenting author

\begin{enumerate}

    \item N. Nakis*, A. Çelikkanat, and M. Mørup. "HM-LDM: A Hybrid-Membership Latent Distance Model". \textit{The 11th International Conference on Complex Networks and their Applications}, Palermo (Italy), Nov 2022, \textcolor{blue}{Oral Presentation}.

    \item A. Çelikkanat*, N. Nakis, and M. Mørup. Piecewise-Velocity Model for Learning Continuous-time Dynamic Node Representations”. \textit{The 1st Learning on Graphs Conference}, Online, Dec 2022, \textcolor{blue}{Poster Presentation}.

    \item N. Nakis, A. Çelikkanat, L. Boucherie, C. Djurhuus*, F. Burmester*, D. M. Holmelund*, M. Frolcová, and M. Mørup. "Characterizing Polarization in Social Networks using the Signed Relational Latent Distance Model". \textit{The 25th International Conference on Artificial Intelligence and Statistics}, Valencia (Spain), March 2023, \textcolor{blue}{Poster Presentation}.

    \item N. Nakis*, A. Çelikkanat, L. Boucherie, and M. Mørup. "Characterizing Polarization in Social Networks using Archetypal Analysis". \textit{The 9th International Conference on Computational Social Science}, Copenhagen (Denmark), Jul 2023, \textcolor{blue}{Oral Presentation}.

    \item  A. Çelikkanat*, N. Nakis, and M. Mørup. "Piecewise-Velocity Model for Learning Continuous-time Dynamic Node Representations". \textit{The 9th International Conference on Computational Social Science}, Copenhagen (Denmark), Jul 2023, \textcolor{blue}{Oral Presentation}.

\end{enumerate}
\chapter{List of Symbols}\label{ch:table_of_symbols}

\begin{center}
\resizebox{1\textwidth}{!}{
\begin{tabular}{ccl}
\textbf{Symbol} &~~~~& \textbf{Description} \\
$\mathcal{G}$ &~~~~& Graph \\
$\mathcal{V}$ &~~~~& Vertex set \\
$\mathcal{E}$ &~~~~& Edge set \\
$\mathcal{E}^{+}$ &~~~~& Positive edge set \\
$\mathcal{E}^{-}$ &~~~~& Negative edge set \\
$N$ &~~~~& Number of nodes \\
$D$ &~~~~& Dimension size \\
$\gamma_i,\beta_i,\psi_i,\alpha_i$ &~~~~& Bias terms of node $i$ \\
$\mathbf{w}_i,\mathbf{z}_i$ &~~~~& Latent embeddings for node $i$ \\
$\lambda_{ij}$ &~~~~& Poisson rate (intensity) of node pair $(i,j)$\\
 $\lambda^+_{ij}$ &~~~~& Positive interaction Poisson rate (intensity) of node pair $(i,j)$ of the Skellam distribution\\
$\lambda^-_{ij}$ &~~~~& Negative interaction Poisson rate (intensity) of node pair $(i,j)$ of the Skellam distribution\\
$\mathcal{I}_{|y|}$ &~~~~& Modified Bessel function of the first kind and order $|y|$\\
$\delta$ &~~~~& Simplex side length with $\delta \in \mathbb{R}_+$ \\
$p$ &~~~~& Power of the $\ell_2$ norm with $p \in\{1,2\}$\\
$\Delta^{D}$ &~~~~& The standard $D-$simplex\\
$\bm{\Lambda}$ &~~~~& Eigenmodel non-negative relational matrix\\
$\mathbf{A}$ &~~~~& The matrix containing the archetypes (extreme points of the convex hull) \\
$\bm{\mu}_i$ &~~~~& cluster centroid vector\\
$\mu$ &~~~~& mean value of a distribution\\
$\sigma$ &~~~~& standard deviation of a distribution\\
$T$ &~~~~& single-event network timeline\\
\end{tabular}
}
\end{center}
\clearforchapter
\phantomsection 
\addcontentsline{toc}{part}{Contents} 
\tableofcontents*
\clearforchapter
\mainmatter

\part{Introduction and methods}

\chapter{Introduction}
\section{Networks}

Networks are widespread data structures and represent the most natural means of expressing complex systems. They appear across various scientific domains, encompassing fields such as physics, sociology, science of science, biology, and more. Within these disciplines, networks are used to describe a multitude of interactions and systems, such as spin glasses in physics, friendship interactions in sociology, scholarly collaborations in academia, protein-to-protein interactions in biology, and structural and functional brain connectivity in neuroscience, among many others \cite{newman}. Given their complexity and high-dimensional discrete nature, accurately characterizing the structure of networks is regarded as a non-trivial and challenging task. Scientists employ various graph analysis tools to examine these networks, seeking to gain insights into their underlying structures. These tools are used for several downstream tasks, including link/relation prediction \cite{libennowel-cikm03}, node classification and clustering \cite{srl2007,node2vec-kdd16}, and community detection \cite{FORTUNATO201075,com_detection}. Moreover, the importance of network analysis extends beyond scientific research, influencing practical applications in industries like telecommunications, transportation, healthcare, and finance. Whether optimizing routes in a transportation network, understanding the spread of diseases within a population, or detecting fraudulent activities within a financial system, network analysis plays a pivotal role. The methodologies and techniques developed in network analysis continue to evolve, pushing the boundaries of our understanding and application of complex systems. The synergy between theoretical development and practical application ensures that network analysis remains an integral and dynamic field of study, connecting diverse domains and contributing to advancements across a broad spectrum of disciplines.

\subsection{Network science}
Towards advancing our understanding of networks, network science emerged. A multidisciplinary field that focuses on the study of complex networks, searching characterizations for their structure, behavior, evolution, and function. It incorporates concepts and methodologies from areas such as physics, mathematics, computer science, biology, sociology, and economics, allowing for a comprehensive analysis and understanding of various kinds of networks \cite{Vespignani2018TwentyYO}. Structural analysis for complex systems focuses on the examination of the node and edge properties of a given network, including statistics such as degree distribution, clustering coefficients, and community structures. Seminal work in this area includes the studies of small-world networks \cite{Watts1998CollectiveDO} and scale-free networks \cite{Baraba_si_1999}. Temporal network analysis under dynamic processes aims to further our understanding of how networks evolve and change over time, and how processes such as information spreading, disease propagation, and social influence operate on such structures. Classic models like the Erdős–Rényi model \cite{Erdos2022OnRG} modeling the evolution of a random network and the SIR model \cite{SIR} for disease spreading were the pioneering works. Analysis of multilayer complex networks investigates networks that have multiple types or layers of connections. This area explores how different layers interact with each other and contribute to the overall behavior of the network. Works like the model of interconnected networks \cite{Buldyrev2010Catastrophic} have deepened our understanding of these complex systems. Important directions also include network visualization and data mining where researchers utilize computational and visualization tools to analyze large-scale network data. This includes discovering patterns, anomalies, and community structures within big data sets. Identifying influential spreaders in complex networks \cite{Kitsak2010Identification} and designing community detection algorithms \cite{Girvan2002Community} are prominent examples of data mining in network science. Being a multidisciplinary field, network science defines multiple applications across various fields (for a comprehensive overview please see \cite{barabasi2016network}). Through a combination of mathematical modeling, computational analysis, and empirical study, network science continues to offer profound insights into the structures and dynamics that underlie diverse systems.

\subsection{Classical methods for network analysis}

The very first approaches to studying networks focused on node properties and node-level statistics. These included various centrality measures \cite{betweeness}, like node degree, eigenvector centrality \cite{Newman2016}, and the clustering coefficient \cite{Watts1998CollectiveDO}, to name a few. Centrality measures focus on expressing different formulations for the importance of nodes in a network, able to capture various properties. Apart from centrality measures, multiple metrics of similarity in terms of the node neighborhood overlap have also been proposed and extensively studied. These include local overlap measures \cite{L__2011} like the Jaccard overlap, the Sorensen index, and the Adamic-Adar index which express different functions over each node's local neighborhood and the common neighbors that two nodes share. Such metrics define node similarity while accounting for the node degree biases with variations on the expression of importance that each common neighbor provides to the metric. For example, the Adamic-Adar index gives higher importance to connections with lower-degree nodes, as they are regarded as more informative than connections with high-degree nodes. Similarly, various global overlap measures have also been proposed which also take into account the global network structure (rather than only the local neighborhood). Popular choices include the Katz Index \cite{Katz1953} which counts over the number of paths of all lengths between a pair of nodes, the Leicht, Holme, and Newman similarity \cite{Leicht_2006} correcting over the Katz Index to account for degree biases by normalizing with the number of expected paths between two nodes, and various random walk methods such as the PageRank \cite{Page1999ThePC} algorithm, expressing stationary probabilities that a random walk starting at node $i$ has to visit node $j$ at some point. Local and global measures express multiple important characteristics of networks and often provide competitive performance in the downstream tasks even against models with advanced learning procedures \cite{deepwalk-perozzi14} but express limitations due to their heuristic nature.

The early algorithmic attempts towards obtaining graph representations relied on spectral-decomposition approaches under dimensionality reduction frameworks defining an approximative expression of the Laplacian or adjacency matrices \cite{868688,10.5555/2980539.2980649}. Classical examples include the Isomap algorithm \cite{Tenenbaum2319} which uses Multidimensional Scaling (MDS) \cite{KruskalWish1978} in order to translate the k-nearest neighbors-based geodesic distances between nodes into a lower dimensional Euclidean space. Another type of well-known methods are the Laplacian Eigenmaps \cite{Wang2012,Chung:1997,articleNiyogi}, where node embeddings are defined as the k-smallest eigenvectors of the normalized graph Laplacian. Laplacian matrices have found lots of applications in graph analysis due to the rich cut information \cite{868688} they carry. Matrix factorization approaches have been studied extensively and in-depth due to their simplicity, with a lot of linear and non-linear variants \cite{NIPS2003_d69116f8,10.1145/1553374.1553494,inproceedings}. In contrast to their many advantages, these methods can be expensive when analyzing large networks due to the computational cost that the matrix decomposition enforces, a study trying to solve the scalability issues via graph partition and parallel computation was proposed in \cite{40839}.

\section{Graph Representation Learning}
Towards the understanding of networks, Graph Representation Learning (\textsc{GRL}) \cite{GRL_HAM,GRL-survey-ieeebigdata20} has found incredible success in the past years, regarded as the foremost method for network analysis. \textsc{GRL} has been so popular since it is composed of approaches outperforming significantly the prior classical methods in the downstream tasks. Classic methods usually rely on graph kernels and graph statistics including various graph centrality measures \cite{GRL_HAM}. Unlike \textsc{GRL}, conventional algorithms exhibit restricted flexibility and capacity as they employ node and graph-level statistics, requiring meticulous heuristic design and often resulting in high time and space complexity \cite{GRL_HAM}. The main goal of \textsc{GRL} is to construct a function that defines a mapping of the network into a low-dimensional (usually Euclidean) latent space through a learning process. Specifically, such a projection has to translate node, edge or even graph similarity of network(s) into similarity in the latent space, i.e., by positioning related nodes, edges, or graph representations close in proximity in the latent space \cite{survey_hamilton_leskovec}. The main focus of \textsc{GRL} lies in learning continuous vector representations for individual nodes, edges, or graphs in the graph-defining embeddings. Node representations can be used for tasks like node classification, link prediction, and clustering. Influential methods in this area include DeepWalk \cite{deepwalk-perozzi14}, Node2Vec \cite{node2vec-kdd16}, and LINE \cite{line}. Edge embeddings are most often used to predict or infer missing links within a network. Techniques such as GraphSAGE \cite{hamilton2017inductive} and SEAL \cite{zhang2018link} have contributed significantly to this aspect of \textsc{GRL}. Graph embeddings focus on learning a comprehensive representation of the entire graph, which can be employed in graph classification or similarity computation and is considered one of the most complex \textsc{GRL} areas. Graph Kernels \cite{JMLR:v12:shervashidze11a} and Graph Neural Networks (GNNs) \cite{Scarselli2009TheGN} have shown great success in this domain. A lot of attention has been given to Graph Neural Networks (GNNs) which define deep learning models specifically designed to operate on graph data. Convolutional Graph Neural Networks (GCNs) \cite{kipf2017semisupervised} and Graph Attention Networks (GATs) \cite{veličković2018graph} are prominent examples. Importantly, GNNs focus on combining structural information with node features, labels, and other metadata to enrich the learning process. This has been shown to improve results in tasks like node classification, as seen in methods like HAN \cite{zheng2020heterogeneoustemporal}. A prominent requirement in \textsc{GRL} are scalability and efficiency aspects. In particular, a major aim is the developing of algorithms and techniques that can scale to large graphs while maintaining computational efficiency which we will also address in this thesis. Techniques like GraphSAGE \cite{hamilton2017inductive} and FastGCN \cite{chen2018fastgcn} address these challenges. Lastly, \textsc{GRL} aims to incorporate higher-order proximity unlike most traditional approaches; accounting for both direct and indirect relationships among nodes to provide richer and more nuanced embeddings \cite{netmf-wsdm18}.

\subsection{Matrix decomposition methods}
A notable category for \textsc{GRL} relies on matrix decomposition techniques \cite{netmf-wsdm18, netsmf-www2019,cao2015grarep, GRL-survey-ieeebigdata20} where node representations are obtained by the decomposition of a target matrix, constructed in such a way as to convey nodal proximity information, and can potentially include both first and higher-order adjacency information \cite{HOPE-kdd16,netmf-wsdm18}. The core idea lies in the assumption that the defined target matrix can be represented by a small number of latent factors that constitute the node embeddings. The main drawback of such methods is their space and time complexities since they usually lead to quadratic dependencies with the number of nodes in the graph. Recent studies aimed to address such computational challenges through techniques like matrix sparsification tools, hierarchical representations, or fast hashing schemes \cite{prone-ijai19, netsmf-www2019, louvainNE-wsdm20, harp-aaai18, randne-icdm18}.

\subsection{Random walk methods}
Pioneering \textsc{GRL} approaches drew inspiration from Natural Language Processing (NLP) \cite{NLP} and employed random walks to generate node sequences analogous to sentences in NLP \cite{deepwalk-perozzi14, node2vec-kdd16, expon_fam_emb, line}. Specifically, these works leveraged the Skip-Gram algorithm \cite{MSCC+13}, to acquire node representations \cite{deepwalk-perozzi14, node2vec-kdd16, line, expon_fam_emb, biasedwalk} by optimizing the co-occurrence probability for node pairs based on their distances obtained through the walks. The initial representatives of the random walk-based methods are DeepWalk \cite{deepwalk-perozzi14} and Node2Vec \cite{node2vec-kdd16}. As an extension to the DeepWalk procedure, Node2Vec introduced a global walk bias which allowed the use of both Breadth-First Sampling and Depth-First Sampling during training in order to learn both local and global graph structures. Random walk methods are actually closely related to matrix factorization approaches as shown in \cite{netmf-wsdm18}. Lastly, multiple random walk-based methods \cite{ppne,ddrw,tridnr} combine the graph structure with additional node labels and attributes to achieve more informative embeddings and take advantage of fruitful nodal meta-data alongside the network structural information.

\subsection{Deep learning based methods} 
Relatively recent pioneering works \cite{graphsage_hamilton} have extended \textsc{GRL} to the deep learning theory, giving rise to Graph Neural Networks (GNN). Essentially, GNNs perform iterative message-passing extending convolution operations to graphs demonstrating remarkable performance by integrating node attributes and network structure during the embedding learning process. Convolution operations are defined over the local neighborhood of nodes while embedding aggregation is adopted to generalize the local structure and learn higher-level proximity. Graph Convolutional Neural Networks (GCNN) thereby scale linearly in the number of graph edges. Several examples of the representative power and success of GCNNs are given in \cite{NIPS2015_f9be311e,DBLP:journals/corr/HenaffBL15} and \cite{bruna2014spectral}. One of their limitations is usually the necessity for node features or else meta-data to avoid the over-smoothing pitfall hampering performance \cite{kipf2017semisupervised} when the GNN model defines deep architectures.


\subsection{Latent space models} 

Latent Space Models (\textsc{LSM}) for the representation of graphs have been quite popular over the past years \cite{past1,past2,past3,past4,past5,expl2,LSM_geo}, especially for social networks analysis \cite{recent_lsm1,recent_LSM2} facilitating community extraction \cite{hmldm} and characterization of network polarization \cite{slim}. \textsc{LSM}s utilize the generalized linear model framework to obtain informative latent node embeddings while preserving network characteristics. The choice of latent effects in modeling the link probabilities between the nodes leads to different expressive capabilities characterizing network structure. In particular, in the Latent Distance Model (\textsc{LDM}) \cite{exp1} nodes are placed closer in the latent space if they are similar or vice-versa. \textsc{LDM} obeys the triangle inequality and thus naturally represents transitivity \cite{hom1,hom0} (\textit{"a friend of a friend is a friend"}) and network homophily \cite{hom2,hom3} (\textit{a tendency where similar nodes are more likely to connect to each other than dissimilar ones}). Homophily is a very well-known and well-studied effect appearing in social networks \cite{hom1,hom2,hom3} and essentially describes the tendency for people to form connections with those that share similarities with themselves. Similarities can be drawn from meta-data (observed node attributes) and may refer to shared demographic properties, political opinions, etc. Homophily has been observed among a broad range of collaborations (see \cite{hom0} for a complete overview). Homophily can also be accounted for based on the unobserved attributes as defined by the \textsc{LDM} as shown in \cite{KRIVITSKY2009204}. Homophily explains prominent patterns as expressed in social networks in terms of transitivity, as well as, balance theory (“the enemy of my friend is an enemy”) \cite{balance_theory}.

\begin{figure*}[!t]
\centering
\includegraphics[width=\textwidth]{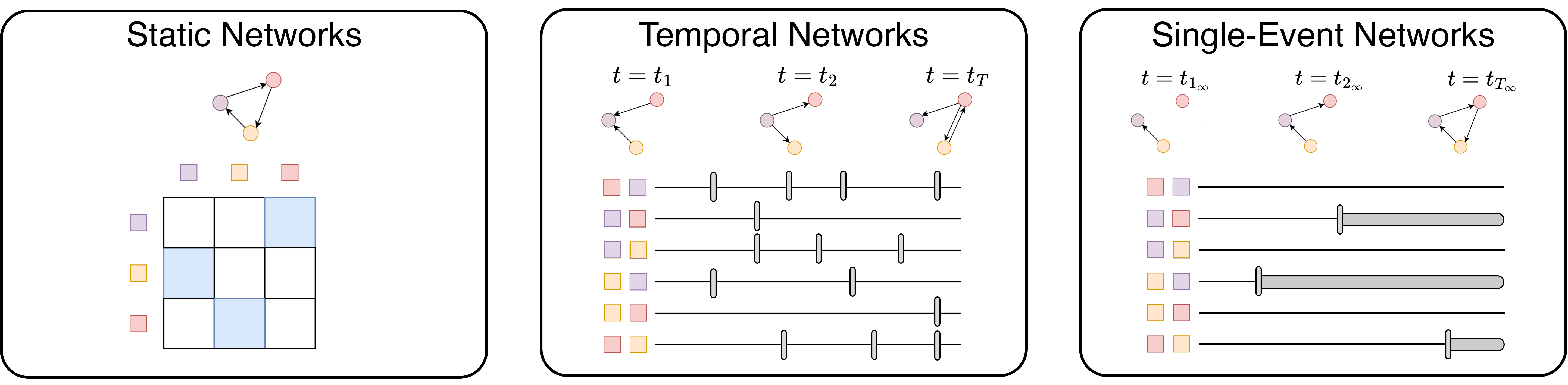}
\caption{Examples of three different types of networks based on their temporal structure. Round points represent network nodes, square points make up the corresponding colored node dyads, arrows represent directed relationships between two nodes, vertical lines represent events, and black lines are the timelines while grey bold lines show that a link (event) appeared once and cannot be observed again. \textit{Left panel:} Static networks where links occur once and there is no temporal information available. \textit{Middle panel:} Temporal networks where links are events in time and can be observed multiple times along the timeline. \textit{Right panel:} Single-event networks where links appear in a temporal manner but they can occur once, defining edges as single events. }\label{fig:sen}
\end{figure*}

\section{Graph Representation Learning for temporal networks}
So far, we have treated networks as static in time. Many networks evolve through time and are liable to modifications in structure with newly arriving nodes or emerging connections. \textsc{GRL} methods have primarily addressed static networks, in other words, a snapshot of the networks at a specific time. However, recent years have seen increasing efforts toward modeling dynamic complex networks, see also \cite{recent_LSM2} for a review. Whereas most approaches have concentrated their attention on discrete-time temporal networks, which have built upon a collection of time-stamped networks \cite{discrete_1,discrete_3,discrete_4,discrete_5,discrete_6} modeling of networks in continuous-time has also been studied \cite{hawkes_1,hawkes_2,hawkes_3,fan2021continuous}. These approaches have been based on latent class \cite{ishiguro2010dynamic, hawkes_1,hawkes_2,hawkes_3,herlau2013modeling} and latent feature modeling approaches \cite{heaukulani2013dynamic, durante2014bayesian, durante2016locally, fan2021continuous, recent_LSM2, past2}, including advanced dynamic graph neural network representations \cite{dyrep, gnn_1}. 

In this thesis, we will focus on a special class of dynamic networks characterized by a single event occurring between dyads, which we denote as single-event networks (\textsc{SEN}). I.e., links occur only once. Specifically, we aim to analyze citation networks which are a prominent example of \textsc{SEN}, as edges can appear only once at the time of the paper's publication. However, neither of the existing dynamic network modeling approaches explicitly account for \textsc{SEN}s. Whereas continuous-time modeling approaches are designed for multiple events, thereby easily over-parameterize such highly sparse networks, static networks can easily be applied to such networks by disregarding the temporal structure but thereby potentially miss important structural information given by the event time. Despite these limitations, to the best of our knowledge, existing generative dynamic network modeling approaches do not explicitly account for single-event occurrences. In Figure \ref{fig:sen}, we provide an example of three cases of networks that define static, traditional event-based dynamic networks as well as single-event networks. We here observe how static networks are completely blind to the temporal information that single-event networks capture while it is also evident that they differ from traditional temporal networks where each dyad can have multiple events across time.

\section{Graph Representation Learning and network science} 

Network Science and \textsc{GRL} are two disciplines that both operate in the realm of graph-structured data but focus on different aspects and utilize distinct methodologies. Network Science focuses on the introduction of heuristics allowing for the analysis of network topology, clustering, centrality, community detection, and network dynamics. In contrast, \textsc{GRL} focuses on converting graph-structured data into (usually) continuous vector representations obtained via a learning procedure, and most importantly it supports predictive modeling. In summary, while both Network Science and \textsc{GRL} operate on graph-structured data, they differ in their objectives, methodologies, applications, data focus, and interdisciplinary roots. Network Science is more concerned with the analysis and understanding of complex networks, whereas \textsc{GRL} focuses on learning representations to facilitate predictive modeling and machine learning tasks. The main focus of this thesis will be the development of efficient \textsc{GRL} methods capable of the analysis of large-scale graphs.

\section{Contributions}

The central aim of this thesis is the development of novel algorithmic approaches for Graph Representation Learning by utilizing the Euclidean distance metric, under the Latent Distance Model formulation \cite{exp1}. As a result, the proposed frameworks will obey the triangle inequality and naturally represent homophily and transitivity in the latent space, modeling high-order node proximity. It will be evident that such a choice leads to ultra-low dimensional graph representations, showcasing surprisingly superior performance when compared to multiple state-of-the-art models. Furthermore, the thesis will focus on structural-aware embeddings leading to hierarchical expressions, community characterization, and the discovery of extreme profiles in networks. Importantly the various presented methods will define unified learning processes avoiding heuristics and multi-stage processes (e.g. post-processing steps). We will focus on a journey seeking unified network embeddings sufficient and powerful enough to characterize the structure, as well as, to successfully perform the multiple and different tasks graph analysis has to offer. This comes contrary to most state-of-the-art studies where models are designed to perform one or two tasks maximum without requiring two-level strategies and procedures for performing additional tasks. Our efforts will initially focus on positive integer weighted graphs, later we will extend the analysis to signed integer weighted graphs, as well as, single-event temporal networks.

 The main contributions of the work can be summarized as follows:
\begin{itemize}
    \item We introduce novel representation learning models over graphs for the study of both signed and unsigned, as well as, single-event networks.
    
  \item We learn node embeddings capable of extracting the hierarchical structure of the network at different scales, accounting for the community discovery, and extracting distinct profiles of networks. 

  \item We, for the first time, define a likelihood function capable of the principled analysis of single-event temporal networks, while also quantifying nodal impact through modeling node receiving edge dynamics.
  
  \item We account for the computational costs of modern networks that can have millions or even billions of nodes. For that, we rely on accurate linearithmic approximations of the likelihood, unbiased random sampling procedures, and case-control inferences.
    
    \item We consider ultra low-dimensional node embeddings that are learned for moderate and large-scale networks and show high performance in all of the considered downstream tasks.
    
   \item We further highlight how the inferred hierarchical organization, community extraction, archetypal characterization, impact quantification, and low-dimensional representations can facilitate the visualization of network structures with high accuracy without requiring additional post-processing tools.
    
   \item The proposed frameworks are extended to the case of bipartite networks, where characterization of structure, hierarchical representations, community detection, and archetype extraction are considered arduous tasks, especially for signed networks.
    
    \item Extensive experimental evaluations demonstrate that the proposed approaches generally surpass widely adapted baseline methods in node classification, link prediction, and network reconstruction tasks.

 \item Importantly, the proposed frameworks define optimal embeddings that are characterized by the most consistent performance across different downstream tasks when compared to various prominent baselines.

 \item Lastly, this thesis aims to fill a missing part in the \textsc{GRL} literature which is to extensively position and benchmark the performance of Latent Distance Models for Graph Representation Learning against state-of-the-art baselines, showcasing their superior performance in multiple settings.

\end{itemize}

\section{Organization}

In \textsc{Part I}, we started with a brief introduction to Graph Representation Learning and network analysis as championed in the past years, followed by a methods chapter introducing the developed frameworks to be used throughout this thesis. We then continue with \textsc{Part II} of the thesis which addresses scalability, hierarchical representations, and community extraction in positive integer weighted networks. We showcase the performance of the proposed frameworks in downstream tasks utilizing ultra-low dimensions and importantly extend such analysis to positive integer-weighted bipartite networks. Afterward, \textsc{Part III} focuses on the analysis of signed integer weighted networks utilizing for the first time the Skellam distribution in network analysis. Specifically, we will propose frameworks able to characterize network polarization and discover extreme profiles and distinct aspects as present in networks. This will come as a generalization of Archetypal Analysis (AA) to relational data under both a defined latent space constrained to the convex hull of the latent embeddings, as well as, a Minimum Volume (MV) approach over the latent space. \textsc{Part IV} aims at the analysis of single-event temporal networks, combining an Inhomogeneous Poisson Point Process and the Latent Distance Model, defining a Single-Event Poisson Process. Lastly, \textsc{Part V} provides a discussion chapter based on the introduced theory, experiments, and results while the conclusion chapter concludes this thesis.

More detailed the thesis will follow a structure as presented below:

\begin{itemize}
    \item \textbf{\textsc{Part I}: Introduction and methods}
    \begin{itemize}
        \item \textcolor{blue}{\textsc{Chapter 1}}: Introduction. {\small This chapter provides a brief introduction to Graph Representation Learning and network analysis as championed in the past years.}
        \item \textcolor{blue}{\textsc{Chapter 2}}: Methods. {\small This chapter introduces the proposed methods for Graph Representation Learning to be used throughout this thesis.}
    \end{itemize}

    \item \textbf{\textsc{Part II}: Graph Representation Learning of positive integer weighted networks}
    \begin{itemize}
        \item \textcolor{blue}{\textsc{Chapter 3}}: \textit{"A Hierarchical Block Distance Model for Ultra Low-Dimensional Graph Representations"}. {\small This chapter is based on the original paper \cite{hbdm} currently accepted for publication by the journal \textit{"IEEE Transactions on Knowledge and Data Engineering"}.}
        \newline

        \item \textcolor{blue}{\textsc{Chapter 4}}: \textit{"HM-LDM: A Hybrid-Membership Latent Distance Model"}. {\small This chapter is based on the original paper \cite{hmldm} published in the proceedings of \textit{ "The $11^{th}$ International Conference on Complex Networks and their Applications"}.}
        \newline

                
    \end{itemize}

    \item \textbf{\textsc{Part III}: Graph Representation Learning of signed integer weighted networks}
    \begin{itemize}
        \item \textcolor{blue}{\textsc{Chapter 5}}: \textit{"Characterizing Polarization in Social Networks using the Signed Relational Latent Distance Model"}. {\small This chapter is based on the original paper \cite{hmldm} published in the proceedings of \textit{ "The $25^{th}$ International Conference on Artificial Intelligence and Statistics, AISTATS"}.}
        \newline


        \item \textcolor{blue}{\textsc{Chapter 6}}: \textit{"A Hybrid Membership Latent Distance Model for Signed Integer Weighted Networks"}. {\small This chapter is based on the original paper \cite{shmldm} published in the journal of \textit{"Advances in Complex Systems"}.}
        \newline

                
    \end{itemize}

    \item  \textbf{\textsc{Part IV:} Graph Representation Learning of Single-event temporal networks}: 
    \begin{itemize}
         \item \textcolor{blue}{\textsc{Chapter 6}}: \textit{"Time to Cite: Modeling Citation Networks using the Dynamic Impact Single-Event Embedding Model"}. {\small Preprint.}
        \newline


    \end{itemize}

    \item  \textbf{\textsc{Part V} Discussion and conclusion}: 
    \begin{itemize}
        \item \textcolor{blue}{\textsc{Chapter 8}}: Discussion. {\small This chapter discusses the general results, limitations, and future work of the thesis topic.}
        \item \textcolor{blue}{\textsc{Chapter 9}}: Conclusion. {\small This chapter concludes the thesis.}

    \end{itemize}

\end{itemize}


\section{Reproducibility and code release}

To enhance openness and reproducibility, the source code for all contributions presented in this thesis is publicly available and can be accessed in the following repositories:

\begin{itemize}
    \item   "A Hierarchical Block Distance Model for Ultra Low-Dimensional Graph Representations": \href{https://github.com/Nicknakis/HBDM}{\textcolor{blue}{\textit{github.com/Nicknakis/HBDM}}}.
    
\item "HM-LDM: A Hybrid-Membership Latent Distance Model":\\ \href{https://github.com/Nicknakis/HM-LDM}{\textcolor{blue}{\textit{github.com/Nicknakis/HM-LDM}}}.

\item "Characterizing Polarization in Social Networks using the Signed Relational Latent Distance Model": \href{https://github.com/Nicknakis/SLIM_RAA}{\textcolor{blue}{\textit{github.com/Nicknakis/SLIM\_RAA}}}.

 \item "A Hybrid Membership Latent Distance Model for Unsigned and Signed Integer Weighted Networks": \href{https://github.com/Nicknakis/HM-LDM}{\textcolor{blue}{\textit{github.com/Nicknakis/HM-LDM}}}.
\end{itemize}

\chapter{Methods}\label{trans_hom}

\section{Notation}

Before we start here we will briefly discuss the notation followed throughout the paper. We denote scalar values as lower-case and non-bold letters $\{x\}$, vectors are represented with lower-case bold letters $\{\bm{x}\}$, and matrices by upper-case bold letters $\{\bm{X}\}$. Single subscripts in lower-case bold letters $\{\bm{x}_i\}$ represent the $i'th$ vector while double subscripts in matrices $\{\bm{X}_{ij}\}$ denote the $i'th$ row and $j'th$ column single element of matrix $\{\bm{X}\}$. 

\section{What is a graph?}
We will now provide a more formal definition of what we will consider a graph or a network, while we will use both of these terms interchangeably. Let now $\mathcal{G}=(\mathcal{V},\mathcal{E})$ define a graph with $\mathcal{V}$ being the vertex/node set and $\mathcal{E} \subseteq \mathcal{V}\times \mathcal{V}$ the edge set. We define an edge $(i,j)  \in  \mathcal{E}$ as the directed relationship having as source node $i \in \mathcal{V}$ and node $j \in \mathcal{V}$ as the target. In the following, we will focus on simple graphs with no loops and no multiple edges, meaning that there is no edge with the same source and target node and that every edge is uniquely defined. We will represent a graph by its adjacency matrix $\mathbf{Y}_{N \times N}=\left(y_{i,j}\right)$ where $y_{i,j} =0$ if the pair $(i,j) \not\in \mathcal{E}$ otherwise it will be a non-zero value  $y_{i,j} \neq 0$ for all $ 1\leq i\leq N := |\mathcal{V}|$, and $ 1\leq j\leq N := |\mathcal{V}|$ . The most trivial case considers binary graphs where $y_{i,j} =1$ if $(i,j) \in \mathcal{E}$, and $y_{i,j} =0$ otherwise. We will characterize a graph as undirected when there are no directional relationships between the vertices, meaning that the edges do not have an inherent direction or arrow associated with them. In the undirected case, the adjacency matrix is symmetric, i.e. $\mathbf{Y}=\mathbf{Y}^\intercal$. In this thesis, we will initially focus on binary undirected graphs and later generalize to signed integer-weighted networks, assuming that the edge weights or the entries of the adjacency matrix can take any positive or negative integer value ($y_{ij} \in \mathbb{Z}$). In the case of the signed graphs, we will further denote $\mathcal{E}^{+}$ as the positive edge set meaning that  $y_{i,j} >0$ if the pair $(i,j) \in \mathcal{E}^{+}$, and accordingly $\mathcal{E}^{-}$ as the negative edge set with $y_{i,j} <0$ if the pair $(i,j) \in \mathcal{E}^{-}$.

Lastly, as a special case, we will focus on bipartite graphs. A graph $\mathcal{G}=(\mathcal{V},\mathcal{E})$ will be a bipartite graph when the vertex set $\mathcal{V}$ can be partitioned into two non-empty and disjoint subsets $\mathcal{V}_1$ and $\mathcal{V}_2$ in such a way that every edge in $\mathcal{E}$ connects a vertex from $\mathcal{V}_1$ to a vertex from $\mathcal{V}_2$. Formally now, for the vertex set $\mathcal{V}$, there exists a partition with $\mathcal{V} = \mathcal{V}_1 \cup \mathcal{V}_2$, and $\mathcal{V}_1 \cap \mathcal{V}_2 = \emptyset$ where $\mathcal{V}_1, \mathcal{V}_2$ being non-empty and disjoint sets. In addition, for any $(i, j)\in \mathcal{E}$, either $i\in \mathcal{V}_1$ and $j\in \mathcal{V}_2$, or $i\in \mathcal{V}_2$ and $j\in \mathcal{V}_1$.

\begin{figure*}[!t]
\centering

\subfloat[Link prediction]{{ \includegraphics[width=0.48\textwidth]{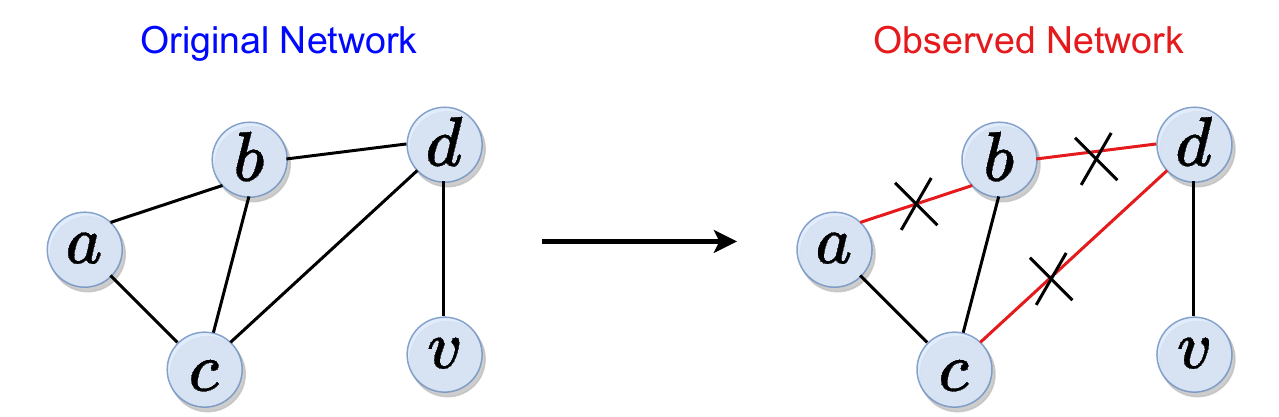} }}
\hfill
 \subfloat[Node classification]{{ \includegraphics[width=0.48\textwidth]{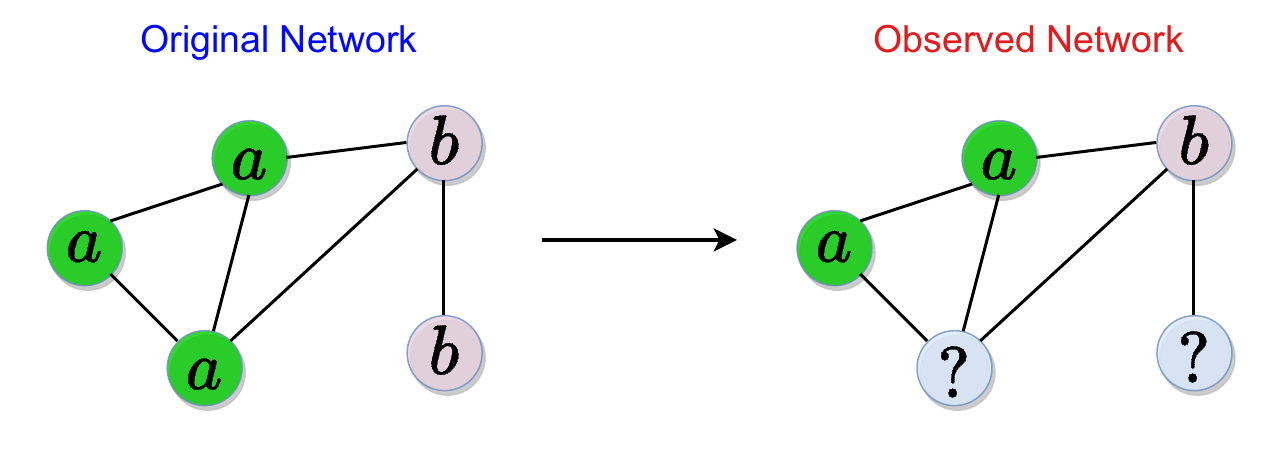} }}
\hfill
\subfloat[Community detection]{{ \includegraphics[width=0.42\textwidth]{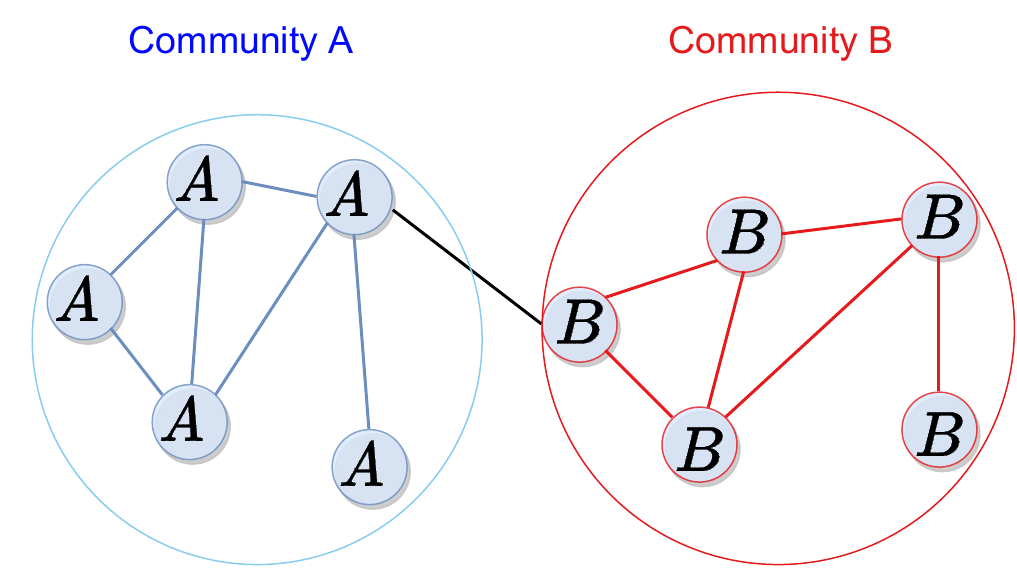} }}

\caption{Downstream tasks for Graph Representation Learning. (a)  \textcolor{blue}{Link prediction}: In this setting, the network is partially observed and the task is to predict the missing links and regain the original network structure. (b) \textcolor{blue}{Node classification}: In this setting, each network node has a label (in the example we have two labels $a$ and $b$), the task is to infer the node labels for nodes with missing/unknown labels. (c) \textcolor{blue}{Community detection}: In this setting, the whole network is observed and the task is to infer communities existing in the network (we show an example with two communities $A$ and $B$). }
\label{fig:downstream}
\end{figure*}

\subsection{Downstream tasks}
Machine learning on graphs focuses on characterizing emerging structures, discovering patterns, and extracting information that is present in graph-structured data \cite{chami2022machine}. The main goal of such an analysis/modeling of networks is the ability to perform and solve various important graph-related tasks. The most popular tasks, which will also be the main focus of this study, include relation/link prediction, node classification, and community detection.

Relation prediction or link prediction (we will use these terms interchangeably) refers to the task of predicting the existence or likelihood of certain relationships (edges) between node pairs in a graph. This task is particularly relevant in network analysis and has applications in various fields, including social networks \cite{libennowel-cikm03}, biological networks \cite{DS2}, and recommendation systems \cite{DS3}. Formally, relation prediction can be defined as follows: given a graph $\mathcal{G}=(\mathcal{V},\mathcal{E})$ the goal is to predict whether a particular edge $(i, j) \in \mathcal{E}$ exists in the graph or to score the likelihood of the existence of different types of relationships between nodes $i$ and $j$. The goal of relation prediction is to use the existing structure of the network and potentially additional features or attributes of nodes and edges to predict which links are most likely to form in the future or are currently missing from the network. This prediction task is usually formulated as either a binary classification problem, where the model predicts the presence or absence of a link between a pair of nodes, or as a ranking problem, where the model ranks the candidate links based on their likelihood score of existence. A toy example of such a task can be found in Figure \ref{fig:downstream} (a).

Node classification, also known as node attribute inference, is a fundamental problem in network analysis and machine learning on graphs \cite{survey_hamilton_leskovec,hamilton2017inductive,kipf2017semisupervised}. It involves predicting the class or label of a node in a network based on its structural properties, attributes, and the labels of its neighboring nodes. The goal of node classification is to learn a predictive model that can generalize from labeled nodes (nodes with known classes) to classify unlabeled nodes existing in the network. This is typically done using supervised learning techniques, where the model is trained on a subset of nodes, i.e. $\mathcal{V}_{train} \subset \mathcal{V} $ with known labels and then used to predict the labels of a test or validation set of nodes in the network. Examples of node classification tasks include research topic prediction in citation networks \cite{Sen2008CollectiveCI}, gene ontology type prediction in protein-protein interaction networks \cite{ontology}, and more \cite{node_class}. A simple example of a node classification task can be found in Figure \ref{fig:downstream} (b).

Community detection, also referred to as node clustering, in the context of graph analysis, is the task of identifying groups of nodes in a network that are densely connected among themselves while having sparser connections to nodes in other groups \cite{FORTUNATO201075,blondel2008fast,Newman_2004}. These groups are often referred to as "communities" or "clusters" and their detection is essential for understanding the underlying structure and organization of complex networks. Community detection is a fundamental problem in network science and has numerous applications in various domains, such as social networks \cite{CD_SN}, opinion dynamics \cite{CD_SN2}, protein-protein interactions \cite{CD_SN3}, disease dynamics \cite{CD_SN4}, and more \cite{CD_SN5}. By uncovering communities, researchers can gain insights into the modular organization and functional units within a network, which can lead to a better understanding of its behavior and facilitate targeted analyses and even interventions. Essentially, the task of community detection is to infer underlying latent structures or "communities" having only the network $\mathcal{G}=(\mathcal{V},\mathcal{E})$ as an input. A community detection example can be seen in Figure \ref{fig:downstream} (c).

In terms of the "classical" machine learning theory, node classification and relation prediction are often categorized as semi-supervised tasks since they work both with labeled and unlabeled data during inference. Specifically, for the test nodes/node pairs there exists information in terms of the nodes' neighborhood in the graph which differs from the traditional supervised setting where test data are completely unobserved. Community detection and clustering are considered the unsupervised extension of classical clustering tasks to network data. One major difference when working with graph data is that the assumption of independent and identically distributed data (i.i.d.) does not hold since nodes in any network are interconnected and thus dependent. It is worth mentioning the existence of inductive biases when modeling networks from different domains and disciplines, making the modeling of graphs more challenging than traditional machine learning problems.

So far, we have discussed node-level downstream tasks since they will be the main focus of this thesis. Nevertheless, there exist various important and interesting graph-level tasks. These include graph classification, regression, and clustering \cite{GRL_HAM}. In these cases, the inputs to the learning procedure can potentially be a set of graphs while the predictions, rather than focusing on a single graph, are generalized to multiple different graphs.

\section{Latent Space Models}
Latent Space Models (\textsc{LSM}s) for the representation of graphs have been well established over the past years \cite{past1,past2,past3,past4,past5,expl2,LSM_geo}. \textsc{LSM}s utilize the generalized linear model framework to obtain informative latent node embeddings while preserving network characteristics. The choice of latent effects in modeling the link probabilities between the nodes leads to different expressive capabilities for characterizing the network structure. Popular choices include the Latent Distance Model \cite{exp1} which defines a probability of an edge based on the Euclidean distance of the latent embeddings, and the Latent Eigen-Model \cite{hoff2007modeling} which generalizes stochastic blockmodels, as well as, distance models. Various non-Euclidean geometries of LSMs have also been studied in \cite{LSM_geo} with the hyperbolic case being of particular interest \cite{nickel2017poincare}.

\begin{figure}[!b]
    \centering
    \includegraphics[width=0.79\columnwidth]{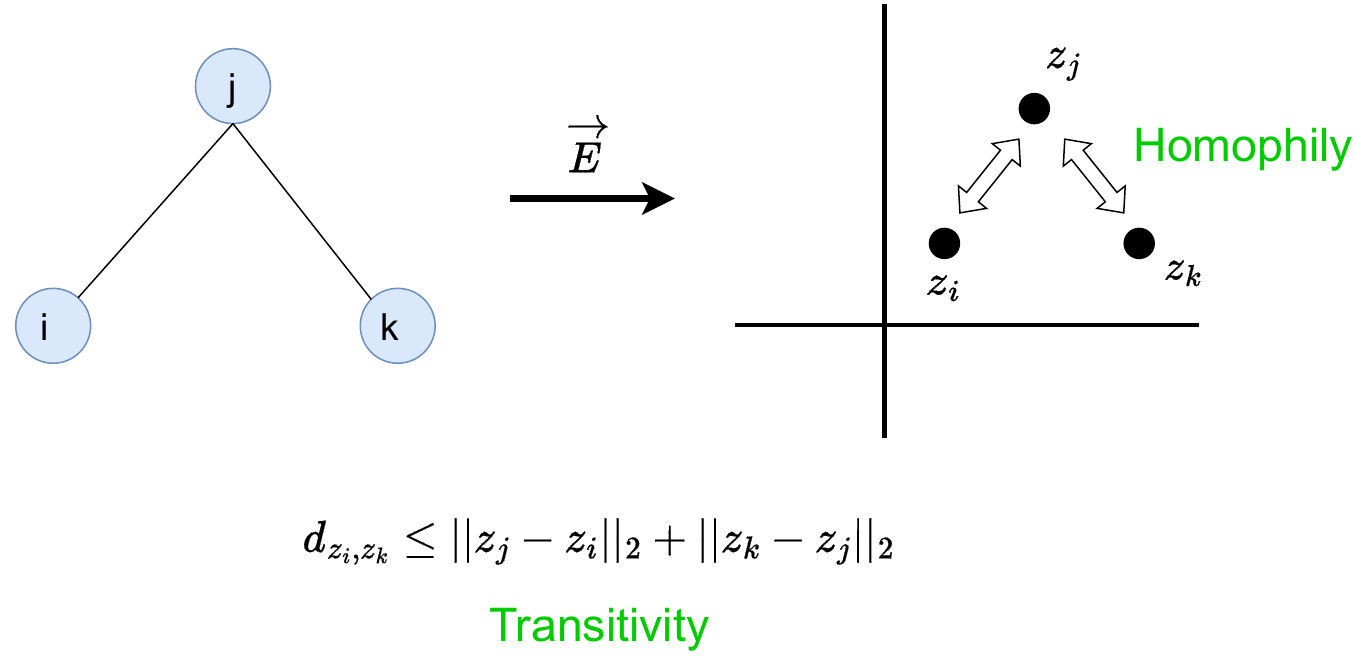}
    \caption{Expression of homophily and transitivity as imposed by the Latent Distance model. Black lines correspond to network edges. Connected nodes are positioned close to each other to define a high probability of an edge, e.g. pairs $\{i,j\}$ and $\{j,k\}$. Consequently, the distance of node pair $\{i,k\}$ is bounded by the triangle inequality, and thus node pair $\{i,k\}$ has to also be positioned in close proximity.}
    \label{fig:hom_tra}
\end{figure}

\subsection{Latent Distance Models}
Here, we will focus on the Latent Distance Model (\textsc{LDM}) \cite{exp1} where network nodes are positioned close in the latent space if they are connected or share additional similarities, such as high-order dependency or proximity. Most importantly, the \textsc{LDM} obeys the triangle inequality and thus naturally represents transitivity and homophily. The \textsc{LDM}extends the traditional homophily expression to the case where no node meta-data exist through the introduction of unobserved attributes as defined by the \textsc{LDM} \cite{KRIVITSKY2009204}. Specifically, we define a $D$-dimensional latent space $(\mathbb{R}^D)$ in which every node of the graph is characterized through the unobserved but informative node-specific variables $\{\mathbf{z}_i \in \mathbb{R}^D \}$. These variables are considered sufficient to describe and explain the underlying relationships between the nodes of the network. The probability of an edge occurring is considered conditionally independent given the unobserved latent positions and depends on the Euclidean distance. Consequently, the total probability distribution of the network can be written as:
\begin{equation}
    \label{eq:prob_adj}
    P(Y|\bm{Z},\bm{\theta})=\prod_{i< j}^Np(y_{i,j}|\mathbf{z}_i,\mathbf{z}_j,\bm{\theta}_{i,j}),
\end{equation}
 where $\bm{\theta}$ denotes any potential additional parameters, such as covariate regressors. A popular and convenient parameterization of  Equation \eqref{eq:prob_adj} for binary data is through the logistic regression model \cite{exp1,link2,KRIVITSKY2009204,doi:10.1198/016214504000001015}.
 
 For our study, we will focus on an \textsc{LDM} under the Poisson distribution \cite{doi:10.1198/016214504000001015}. The Poisson \textsc{LDM} generalizes the analysis to integer-weighted graphs while the exchange of the \textit{logit} to an \textit{exponential} link function when transitioning from a Bernoulli to a Poisson model defines nice decoupling properties over the predictor variables in the likelihood \cite{karrer2011stochastic,herlau2014infinite}. Importantly, we will also study binary networks where the use of a Poisson likelihood for modeling such relationships in a network does not decrease the predictive performance nor the ability of the model to detect the network structure \cite{6349745}. Formally now, we define the Poisson rate of the \textsc{LDM} for an occurring edge based on the Euclidean distance between the latent positions of the two nodes as:

 \begin{equation}
    \lambda_{ij}=\exp\big(\gamma_i+\gamma_j- d(\mathbf{z}_i,\mathbf{z}_j)\big).
    \label{eqn:random_effect}
\end{equation}
 
 In this formulation, we consider the \textsc{LDM} Poisson rate with node-specific biases or random effects \cite{doi:10.1198/016214504000001015,KRIVITSKY2009204}. In particular, $\gamma_i \in \mathbb{R}$ denotes the node-specific random effect and $d_{ij}(\mathbf{z}_i,\mathbf{z}_j)=||\mathbf{z}_i-\mathbf{z}_j||_2$ denotes the Euclidean distance (or potentially any distance metric obeying the triangle inequality) $\big\{ d_{ij}\leq d_{ik}+d_{kj},\:\forall(i,j,k) \in V^3 \big\}$. Considering variables $\{\mathbf{z}_i\}_{i\in V}$ as the latent characteristics, Equation \eqref{eqn:random_effect} shows that similar nodes will be placed closer in the latent space, yielding a high probability of an occurring edge and thus modeling homophily and satisfying network transitivity and reciprocity through the triangle inequality. Essentially, we extend the meaning of similarity to the unobserved (latent) covariates, i.e., latent embeddings matrix $\mathbf{Z}$. Connected or similar nodes define strong relationships that are to be translated by the \textsc{LDM} into the latent space, defining a high probability of observing connections. As a result, for two similar nodes $\{i,j\}$ the pairwise distance $||\mathbf{z}_i - \mathbf{z}_j||_2$ should be small which further implies that for a different node $\{k\}$ we obtain $||\mathbf{z}_i -\mathbf{z}_k||_2 \approx ||\mathbf{z}_j - \mathbf{z}_k||_2$. The latter concludes that nodes $\{i,j\}$ are similar since they share similar relationships with the rest of the nodes. For an illustration please visit Figure \ref{fig:hom_tra}. An immediate result of obeying the triangular inequality is that the \textsc{LDM} successfully models high-order interactions, as present in complex systems \cite{high_order1,high_order2}. The node-specific bias can account for degree heterogeneity, whereas the conventional \textsc{LDM} rate utilizing a global bias, $\gamma^{g}$, corresponds to the special case in which $\gamma_i=\gamma_j=0.5\gamma^{g}$.

\section{The Hierarchical Block Distance Model} \label{methods_hbdm}

The classical LDM, as introduced previously, naturally conveys the main motivation of Graph Representation Learning where similar nodes are positioned in close proximity in a constructed latent space. This comes as a direct consequence of the  Euclidean metric choice, representing homophily, transitivity, and high-order proximity. Unfortunately, two equally important properties, scalability (analysis of large-scale networks is infeasible) and structure characterization are not met by the LDM. Specifically, it scales quadratically in terms of the number of network nodes as $\mathcal{O}(N^2)$ while it is agnostic in terms of latent structures that potentially exist in different scales.

Our goal is to design a Hierarchical Block Model preserving homophily and transitivity properties with a total complexity allowing for the analysis of large-scale networks. Similar to a classical Poisson LDM, we define the rate of a link between each network dyad $(i,j)\in V\times V$ based on the Euclidean distance, as shown in Equation \eqref{eqn:random_effect}. Such a decision guarantees that our model will inherit the natural properties of the so-effective LDM, and satisfy homophily and transitivity.

Moving to the next two properties, we can define a block-alike hierarchical structure by a divisive clustering procedure over the latent variables in the Euclidean space. Incorporating a block structure into the model facilitates the retrieval of underlying structures, while the integration of a hierarchy accounts for the emergence of these structures across multiple scales. In addition, we will constrain the total optimization cost of such a model to a linearithmic upper bound complexity, making large-scale analysis feasible. We will start such a procedure by initially noticing that a shallow clustering of the latent space with a number of clusters, $K$, equal to the number of nodes, $N$, leads to the same log-likelihood as of the standard \textsc{LDM}, defining a sum over each ordered pair of the network, as:
\begin{align}
   \log P(\mathbf{Y}|\bm{\Lambda})=\!\!\sum_{\substack{i<j \\ y_{ij}=1}}\!\log(\lambda_{ij})\;-\;\sum_{\substack{i< j }}\Big(\lambda_{ij}+\log(y_{ij}!)\Big) \:.
    \label{eq:log_likel_lsm0}
\end{align}
where $\bm{\Lambda}=(\lambda_{ij})$ is the Poisson rate matrix which has absorbed the dependency over the model parameters while we presently ignore the linear scaling by dimensionality $D$ of the above log-likelihood function. Notably, the first term of Equation \eqref{eq:log_likel_lsm0}, which hereby we will refer to as link contribution/term $\sum_{y_{i,j}=1}\log(\lambda_{i,j})$, is responsible for positioning "similar" nodes closer in the latent space, expressing the desired homophily. This is straightforward by substituting Equation \eqref{eqn:random_effect} in the link contribution that is maximized when the distance is zero (for fixed random effects and fixed latent embeddings for all nodes except nodes $\{i,j\}$). The second term of Equation \eqref{eq:log_likel_lsm0} $\sum_{i< j}\lambda_{ij}$, from now on referred to as the non-link contribution/term, acts as the repelling force for dissimilar nodes, being responsible for positioning nodes far apart, and in the case of $y_{ij}=0$ is maximized when $d_{ij}\rightarrow +\infty$ (by fixing again the rest of parameters).

Focusing on the computational complexity of Equation \eqref{eq:log_likel_lsm0}, and given that large networks are highly sparse \cite{barabasi2016network} with the number of edges proportional to the number of nodes in the network, results in a low computation cost the link contribution. We empirically showed in Figure \ref{fig:nlogn} that it scales linearithmic or sub-linearithmic with $N$. Importantly, the link term removes rotational ambiguity between the different blocks/clusters of the hierarchy (as discussed later). For these reasons, no block structure is imposed on the calculation of the link contribution. 

\begin{figure}[!t]
  \centering
  \centerline{\includegraphics[scale=0.4]{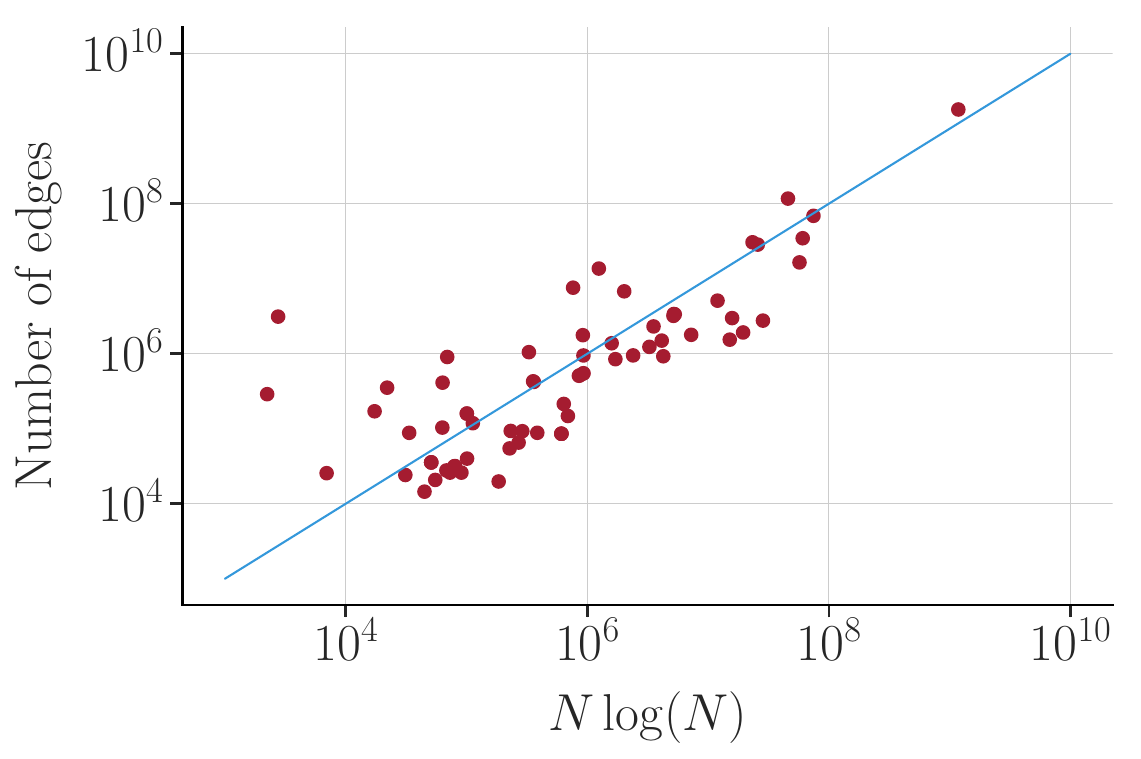}}
  \caption{Log-Log plot of the number of network edges versus $N\log N$ where $N$ the number of vertices, for $70$ datasets of the SNAP library \cite{snapnets}.}
      \label{fig:nlogn}
\end{figure}

Moving now to the non-link term, it requires the computation of all node pairs distance matrix and thus it scales as $\mathcal{O}(N^2)$ making the evaluation of the above likelihood infeasible for large networks and being the main overhead for both space and time complexities. In order to make such a calculation linearithmic, we aim to enforce a block structure, i.e., akin to stochastic block models \cite{white1976social, holland1983stochastic,doi:10.1198/016214501753208735}, when grouping the nodes into $K$ clusters we define the rate between block $k$ and $k'$ in terms of their distance between centroids. We initialize such a procedure, by a shallow block structure obtaining the following non-link expression:

\begin{align}
 \sum_{i< j}\lambda_{ij}\!\approx\! \sum_{k=1}^{K}\!\Bigg(\!\!\sum_{\substack{i<j \\ i,j \in C_{k} }}\!\!&\exp{\!\big(\gamma_i\!+\!\gamma_j \!-\! ||\mathbf{z}_i-\mathbf{z}_j||_2\big)}
\!+ \nonumber \\ & \! \sum_{k^{'} > k}^K\sum_{i\in C_{k}}\sum_{j\in C_{k^{'}}}\exp{\big(\gamma_i + \gamma_j - ||\bm{\mu}_{k}-\bm{\mu}_{k'}||_2\big)}\Bigg),
 \label{eq:log_likel_lsm_kmeans2}
\end{align}
where $\bm{\mu}_k$, has absorbed the dependency over the variables $\bm{Z}\in \mathbb{R}^{N\times D}$, and denotes the $k$'th cluster centroid over the set of $K$ total centroids $\bm{C}=\{C_1,\dots,C_K\}$. Cluster centroids $\bm{\mu}_k$ are implicit parameters defined as a function over the latent variables, as it will be clear later. In general, the clustering procedure is expected to naturally extend the concept of homophily to the level of clusters via the centroid expressions. This means that on a node level, closely related nodes will be grouped together in clusters while on a cluster level interconnected clusters will also be positioned closely in the latent space, creating an effective block structure representation. Overall, the clustering technique adheres to "cluster-homophily" and "cluster-transitivity" within the latent space.

\begin{figure}[!t]
\centering
 \subfloat[Hierarchical representation of a distance matrix]{{ \includegraphics[width=0.82\textwidth]{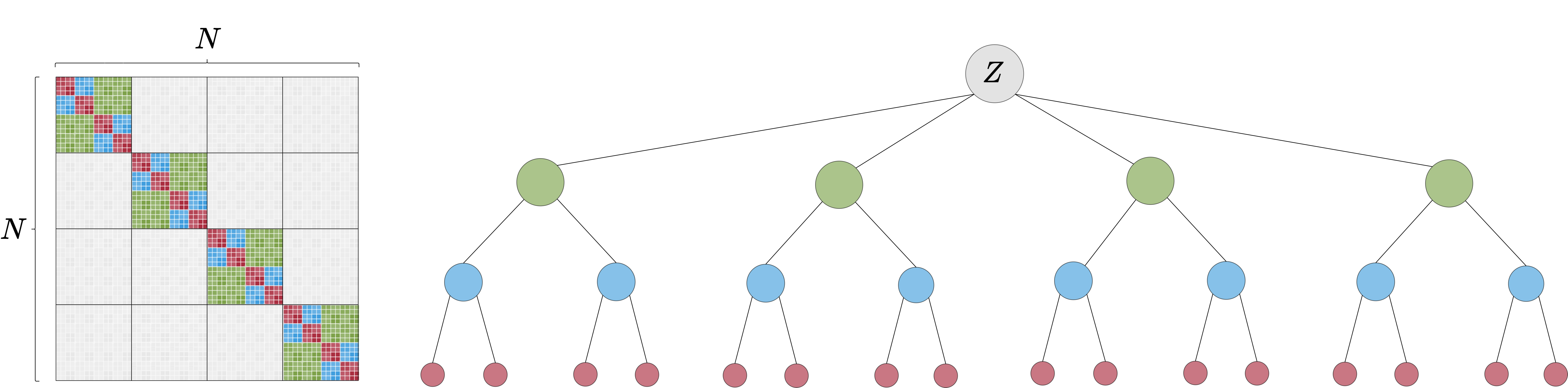} }}
\hfill
\subfloat[Pairwise distance approximation]{{ \includegraphics[width=0.82\textwidth]{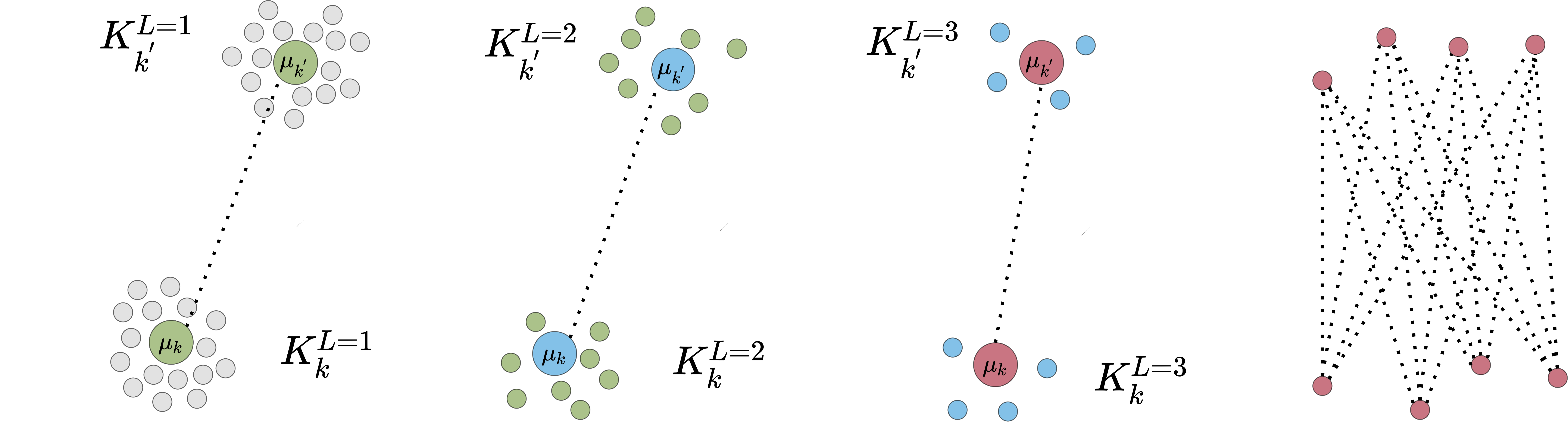} }}

\caption{Schematic representation of the distance matrix calculation for a hierarchical structure of the tree of height $L=3$ and for the number of observations $N=64$.  (a) Hierarchical representation of the all-pairs distance matrix. 
(b) Pairwise distance approximation based on cluster centroids across different levels of the hierarchy \cite{hbdm}.}
\label{fig:dist_mat_calc}
\end{figure}

\subsection{A Hierarchical Representation}
To attain the required hierarchical organization, we utilize hierarchical clustering via a divisive procedure. In more detail, we organize the embedded clusters into a hierarchical tree structure, forming a cluster dendrogram. The tree's root defines a single cluster containing all latent variable embeddings $\bm{Z}$. During the construction of the tree, clusters (tree-nodes) are divided at each level until every tree-node becomes a leaf. A cluster is considered a leaf node if it contains an equal or smaller number of network nodes than an established threshold, $N_{\text{leaf}}$. The threshold value is chosen in such a way that to maintain linearithmic efficiency in terms of complexity and is set to $N_{\text{leaf}}=\log N$, which leads to roughly $K=N/\log(N)$ total clusters. The tree-nodes belonging to a specific tree-level are considered the clusters for that specific tree height. Every new division of a non-leaf node is performed solely on the set of points assigned to the parent node in the tree (tree-node/cluster). At each level of the tree, the distance between corresponding cluster centroids is considered as the pairwise distances of datapoints that belong to different tree-nodes, as shown in Figure \ref{fig:dist_mat_calc} (ii). Utilizing these distances, we compute the likelihood contribution defined by the blocks and proceed with binary divisions, moving down the tree, for the nodes that are not leaf nodes.
When every tree-node is regarded as a leaf, we analytically determine the inner cluster pairwise distances for the corresponding likelihood contribution of the analytical blocks, as depicted in the final part of Figure \ref{fig:dist_mat_calc} (ii). This analytical calculation is carried out at a linearithmic cost of $\mathcal{O}(K N_{\text{leaf}}^2)=\mathcal{O}(N \log N)$, and it reinforces the homophily and transitivity characteristics of the model. Specifically, for network nodes that are most similar, the model calculates explicitly the latent distances in the same manner as the standard \textsc{LDM}.

We can thereby define a Hierarchical Block Distance Model with Random Effects (\textsc{HBDM-Re}) as:

\begin{align}
\log& P(Y|\mathbf{Z}, \bm{\gamma}) = \sum_{\substack{i < j \\ y_{i,j}=1}}\Bigg(\gamma_i+\gamma_j - ||\mathbf{z}_i-\mathbf{z}_j||_2\Bigg)\nonumber\\ & -\sum_{k=1}^{K_L}\Bigg(\sum_{\substack{i<j \\ i,j \in C^{(L)}_{k} }}\exp{(\gamma_i+\gamma_j - ||\mathbf{z}_i-\mathbf{z}_j||_2)}\Bigg)\nonumber\\ & -\sum_{l=1}^{L}\sum_{k=1}^{K_l}\sum_{k^{'}>k}^{K_l}\Bigg(\exp{(- ||\bm{\mu}^{(l)}_{k}-\bm{\mu}^{(l)}_{k'}||_2)}\nonumber\\ &\times\Big(\sum_{i\in C_{k}^{(l)}}\exp{\gamma_i}\Big)\Big(\sum_{j\in C_{k^{'}}^{(l)}}\exp{\gamma_j}\Big)\Bigg),
    \label{eq:log_likel_lsm}
\end{align}
where $l \in \{1,\ldots, L\}$ denotes the $l$'th dendrogram level, $k_l$ is the index representing the cluster id for the different tree levels, and $\bm{\mu}_{k}^{(l)}$ the corresponding centroid. We also consider a Hierarchical Block Distance Model (\textsc{HBDM}) without the random effects which is achieved by setting $\gamma_i=0.5\gamma^{g}$. For a multifurcating tree that splits into $K$ clusters and has $N/\log(N)$ terminal nodes or clusters, there are $\mathcal{O}(N/(K\log{N}))$ internal nodes. Each node requires the evaluation of $\mathcal{O}(K^2)$ pairs, leading to an overall complexity of $\mathcal{O}(NK/\log{N})$. Therefore, $K$ must be less than or equal to $\log{N}^2$ to achieve a scaling of $\mathcal{O}(N\log{N})$ \cite{trees}. It's noteworthy to observe that in Equation \eqref{eq:log_likel_lsm}, the random effects contributing to the non-link term are independent of the centroid distance calculations. As a result, the selection of the exponential link function allows an implicit calculation over the pairwise rates of the approximation term, facilitating efficient computations.

\subsection{Divisive partitioning using k-means with a Euclidean distance metric} 

The likelihood formula provided by Equation \eqref{eq:log_likel_lsm} can be minimized directly by allocating nodes to clusters given the tree structure. Unfortunately, performing such an evaluation for all $N$ nodes results in a scaling that becomes impractical, defining a $\mathcal{O}(N^2/\log{N})$ complexity. To make this more manageable, we employ a more efficient method of divisive partitioning, which minimizes the Euclidean norm $||\bm{\mu}_{k_l}-\bm{\mu}_{k_l'}||_2$. The divisive clustering procedure thus relies on the following Euclidean norm objective
\begin{equation}
    J(\mathbf{r},\bm{\mu})=\sum_{i=1}^N\sum_{k=1}^K r_{ik}||\mathbf{z}_i-\bm{\mu}_k||_2,
    \label{eqn:eucl_kmeans}
\end{equation}
where $k$ denotes the cluster id, $\mathbf{z}_i$ is the i'th data observation, $r_{ik}$ the cluster responsibility/assignment, and $\bm{\mu}_k$ the cluster centroid.

The given objective function is not supported by existing k-means clustering algorithms that depend only on the squared Euclidean norm. As a result, we now develop an optimization procedure specifically for k-means clustering under the Euclidean norm. This approach lies within the auxiliary function framework as developed in the context of compressed sensing in \cite{Tsutsu2012AnLP}. We establish an auxiliary function for \eqref{eqn:eucl_kmeans} as follows:

\begin{equation}
     J^+(\bm{\phi},\mathbf{r},\bm{\mu})=\sum_{i=1}^N\sum_{k=1}^K r_{ik}\Bigg(\frac{||\mathbf{z}_i-\bm{\mu}_k||_2^2}{2\phi_{ik}}+\frac{1}{2}\phi_{ik}\Bigg),
    \label{eqn:eucl_kmeans_aux}
\end{equation}
where $\bm{\phi}$ are the auxiliary variables. Thereby, minimizing Equation \eqref{eqn:eucl_kmeans_aux} with respect to $\bm{\phi}_{nk}$ yields $\bm{\phi}_{ik}^*=||\mathbf{z}_i-\bm{\mu}_k||_2$ and by plugging $\bm{\phi}_{ik}^*$ back to \eqref{eqn:eucl_kmeans} we obtain $J^+(\bm{\phi}^*,\mathbf{r},\bm{\mu})=J(\mathbf{r},\bm{\mu})$ 
verifying that \eqref{eqn:eucl_kmeans_aux} is indeed a valid auxiliary function for \eqref{eqn:eucl_kmeans}. The algorithm proceeds by optimizing cluster centroids as \begin{equation}\bm{\mu}_k=\Big(\sum_{i\in k}\frac{\mathbf{z_i}}{\phi_{ik}}/\sum_{i\in k}\frac{1}{\phi_{ik}}\Big),\end{equation} and assigning points to centroids as \begin{equation}\argmin_{\bm{C}}=\sum_{k=1}^{K}\sum_{\bm{z}\in C_k}\Big(\frac{||\bm{z}-\bm{\mu}_k||_2^2}{2\bm{\phi}_{k}}+\frac{1}{2}\bm{\phi}_{k}\Big),\end{equation} upon which $\bm{\phi}_k$ is updated. The overall complexity of this procedure is $\mathcal{O}(TKND)$ \cite{Hartigan1979KMeans} where $T$ is the number of iterations required to converge. As shown in \cite{Tsutsu2012AnLP}, Equation \eqref{eqn:eucl_kmeans_aux} is a special case of a general algorithm for an $l_p(0<p<2)$ norm minimization using an auxiliary function with the algorithm converging faster the smaller $p$ is. For a detailed study of the efficiency of the optimization procedure under such an auxiliary function, see \cite{Tsutsu2012AnLP}.

\textbf{Number of splits in each layer of the divisive procedure:} A straightforward approach to constructing the tree structure would be via an agglomerative procedure where essentially the nodes would be split into $K=N/\log(N)$ followed by binary merges until only one cluster survives. Despite this being possible under the above Euclidean k-means procedure, it scales prohibitive and thus does not respect the linearithmic complexity threshold. For that, we turn to a divisive clustering procedure for constructing the dendrogram. In such a procedure, lies a trade-off between the number of nodes belonging to each cluster and the distance approximation quality. It is evident that an initial binary split would be a very crude distance approximation and as a result we choose in the initial split to create the maximum allowed number of clusters respecting the linearithmic complexity threshold  $\mathcal{O}(N\log{N})$, that is equal to $K=\log{N}$. Continuing to divide into $\log{N}$ clusters might seem like an appealing approach, but for a balanced multifurcating tree that has $N/\log{N}$ leaf clusters, this strategy would lead to a height scaling of $\mathcal{O}(\log{N}/\log{\log{N}})$. Consequently, the overall complexity of this method would be $\mathcal{O}(N\log^{2}(N)/\log{\log{N}})$ \cite{trees}. A balanced binary tree at all levels beneath the root leads to a height scaling of $\mathcal{O}(\log{N})$, with each level of the tree accounting for $\mathcal{O}(DN)$ operations. When including the linear scaling factor due to dimensionality $D$, this results in an overall complexity of $\mathcal{O}(DN\log{N})$. Figure \ref{fig:dist_mat_calc} (i) depicts the resulting tree for a small problem involving $N=64$ nodes. In this example, the nodes are first divided into 4 clusters (approximately equal to $\log(64)$), and then binary splits are performed until each leaf cluster contains 4 nodes (also roughly equivalent to $\log(64)$)\footnote{ For visualization purposes only, we show equally sized clusters.}.

\begin{figure}[!t]
\begin{center}
\begin{minipage}[]{.38\textwidth}
\begin{center}
\subfloat[ Non-optimal rotation over leaf clusters.]{{ \includegraphics[width=1.05\textwidth]{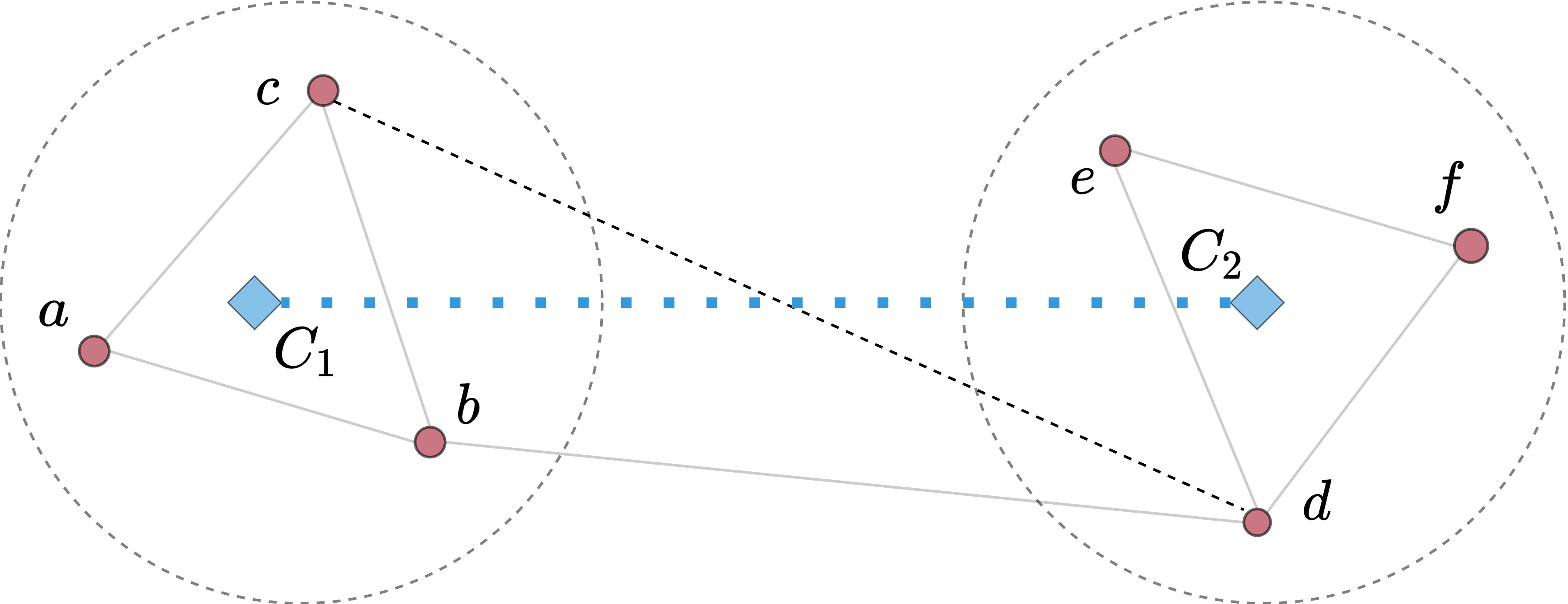} }}%
\vfill
\subfloat[ Optimal rotation over leaf clusters.]{{ \includegraphics[width=1.05\textwidth]{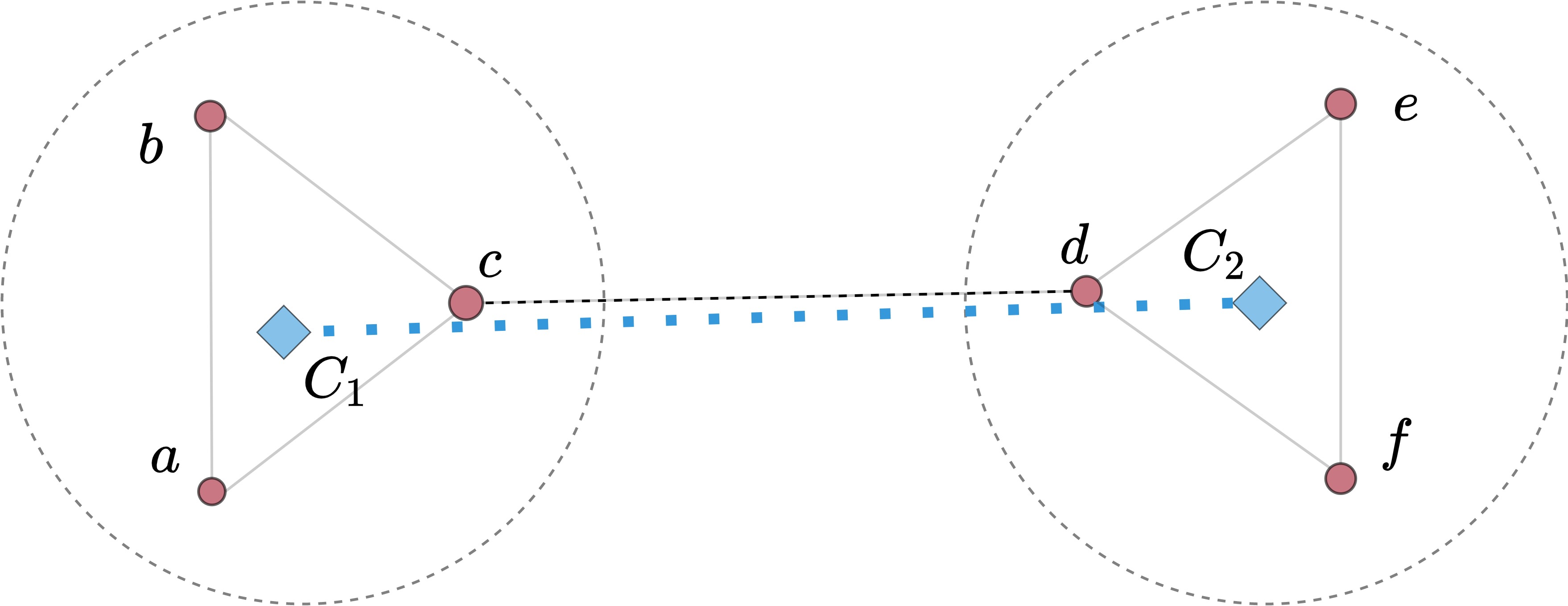} }}%
\end{center}
\end{minipage}
\hfill
\begin{minipage}[]{.58\textwidth}
\begin{center}
\subfloat[ Three latent block structures.]{{ \includegraphics[width=.95\textwidth]{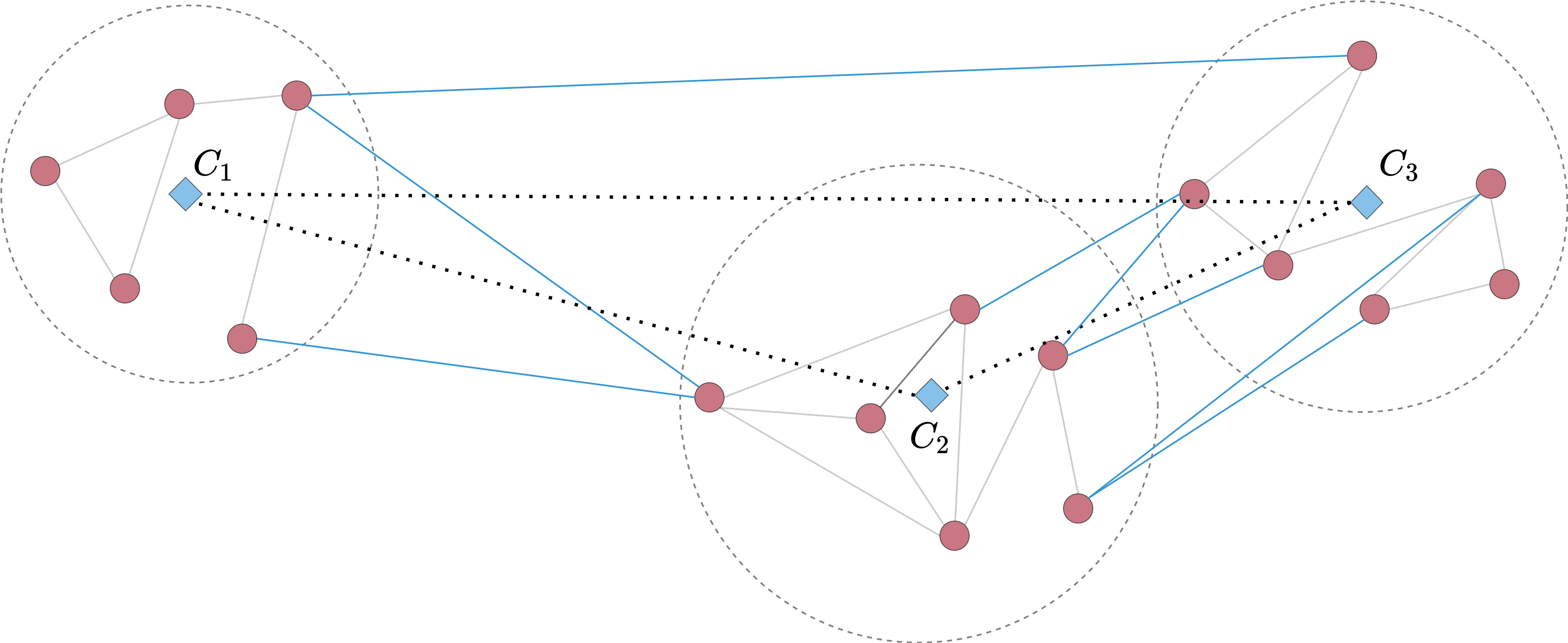} }}%
\end{center}
\end{minipage}
\end{center}
\caption{The clusters within the dashed circles denote the leaf block structures. The red circles and blue rhombuses indicate the node embeddings and the centroids, respectively. Gray lines represent the links and the dashed lines the distance between the cluster centers \cite{hbdm}.}\label{fig:cluster_homophily}
\end{figure}

\subsection{Hierarchical Block Representations Expressing Homophily and Transitivity}
A crucial aspect in maintaining the homophily and transitivity properties of \textsc{HBDM-Re} and \textsc{HBDM} is to avoid approximating the link terms at the block level, as is done in (hierarchical) SBMs. Instead, the link contribution to the log-likelihood across the entire hierarchy should be calculated analytically, going beyond the leaf/analytical blocks. Figure \ref{fig:cluster_homophily} (a) and (b) depict two leaf clusters connected by a link. Suppose that the distances within the blocks are computed analytically and that both the link and non-link contributions of pairs across different clusters are estimated based on the distance between their centroids. Such an approach would essentially permit any rotation of each cluster throughout the hierarchy, as neither the inner-block distances (analytical) nor the centroid distances would be altered by these rotations, resulting in an identical likelihood, showcasing a block-level rotational invariance. In this scenario, homophily expression would be compromised. For instance, the distance between nodes $c$ and $d$ might not necessarily be shorter than the distance between other disconnected inter-cluster pairs (e.g., as shown in Figure \ref{fig:cluster_homophily} (a)). This illustrates how the rotation of the blocks can have a significant effect on the homophily characteristics of \textsc{HBDM-Re} and \textsc{HBDM}.

Computing the link contributions between different clusters analytically resolves this ambiguity, as the likelihood is penalized more when nodes $c$ and $d$ are positioned in a way that is not aware of rotations. The computational cost of taking into account all the link terms analytically means that the model's complexity is tied to the number of network edges contrary to a total block structure where the complexity is strictly linearithmic. However, we demonstrate empirically in Figure \ref{fig:nlogn} that the number of network edges $E$ scales linearly with $N\log N$, so this analytical term complies with our complexity boundary. Figure \ref{fig:cluster_homophily} (c) illustrates examples of clusters that define cases of block interconnections between both sparsely connected blocks (\{$C_1,C_3$\}, \{$C_1,C_2$\}) and densely connected blocks (\{$C_2,C_3$\}). The analytical links between clusters (depicted as blue lines) are instrumental in determining the proper orientation of the blocks. Furthermore, these inter-cluster links guide the proximities of the centroids at the cluster level, thus playing a vital role in upholding the properties of cluster homophily and transitivity.

In the HBDM, pairwise distances remain unaffected by rotation, reflection, and translation operations of the latent space due to its inheritance from the \textsc{LDM} \cite{exp1} (even though the dyad rates $\lambda_{ij}$ are uniquely defined). These isometries can be addressed through a Singular-Value-Decomposition procedure. Analytically, let $\bm{Z}$ denote the embedding of our proposed \textsc{HBDM(-Re)} such that the $i$'th row $(\bm{Z})_i=\mathbf{z}_i$. Then, visualizations of the inferred latent space can be uniquely determined by imposing a centering step $\bm{\widehat{Z}}=\bm{Z}-\bm{\Bar{Z}}$ followed by a singular value decomposition of the latent positions $\bm{U}\bm{\Sigma}\bm{V}^T=\text{SVD}(\bm{\widehat{Z}})$ to remove rotation ambiguity. Thereby, we can introduce $(\bm{Z^*})_i=(\bm{U}\bm{\Sigma})_i$ which determines uniquely identifiable latent positions as long as the singular values are distinct.

While the analytical calculation of link terms introduces rotational awareness to the \textsc{HBDM} clusters, we further explore the conditions under which a continuous operation that defines infinitesimal rotations (relative to the cluster centroid) is permissible. This exploration seeks to understand the situations in which the loss function of Equation \eqref{eq:log_likel_lsm} remains invariant to continuous rotations. In Lemma \ref{lemma} (proof follows), we start our investigation of this problem by showing that blocks with a unique inter-cluster link connection reduce the clusters' degree of rotational freedom by one.

\begin{lemma}\label{lemma}
 Let $\mathcal{G}=(\mathcal{V}, \mathcal{E})$ be a graph and let $\mathcal{C}$ be a cluster with its centroid located at $\boldsymbol{\mu}\in \mathbb{R}^D$ having an edge $(i,j)\in \mathcal{E}$ for some $i\in \mathcal{C}$ and $j\in \mathcal{V}\backslash\mathcal{C}$ such that $\mathbf{z}_i \neq \boldsymbol{\mu}$. If $\mathbf{\tilde{z}}_i=\boldsymbol{\mu}+\mathbf{R}(\boldsymbol{\theta})(\mathbf{z}_i-\boldsymbol{\mu})$ such that $\mathbf{R}(\boldsymbol{\theta})$ is a rotation matrix acting on the embeddings of nodes in cluster $\mathcal{C}$, then the maximum degree of freedom of any infinitesimal $\lambda_{ij}$-invariant rotation is defined by $\boldsymbol{\theta}\in\mathbb{R}^{D-2}$.
 \end{lemma}

 \begin{proof}
 A general rotation matrix, $\mathbf{R}(\boldsymbol{\theta})$, for a D-dimensional space is given by a rotation angle vector ${\boldsymbol{\theta}}\in\mathbb{R}^{D-1}$. Define $\tilde{\lambda}(\boldsymbol{\theta})_{ij}=\exp\left(\gamma_i+\gamma_j-\|\tilde{\mathbf{z}}_i-\mathbf{z}_j\|\right)$ such that $\tilde{\lambda}(\mathbf{0})_{ij}=\lambda_{ij}$. Then, an infinitesimal rotation leaving $\lambda_{ij}$ unchanged must be along the direction of a non-zero vector $\boldsymbol{v}\in\mathbb{R}^{D-1}$ requiring $\langle \frac{\partial \tilde{\lambda}(\boldsymbol{\theta})_{ij}}{\partial \boldsymbol{\theta}}\Bigr|_{\boldsymbol{\theta} =\boldsymbol{0}}, \boldsymbol{v}\rangle=0$. This equation is satisfied either if (i) $\frac{\partial \tilde{\lambda}(\boldsymbol{\theta})_{ij}}{\partial \boldsymbol{\theta}}\Bigr|_{\boldsymbol{\theta}=\boldsymbol{0}} = \boldsymbol{0}$,  which would require either $|| \mathbf{z}_i - \mathbf{z}_j ||$ is maximum or minimum on the sphere defined by the rotation such that any infinitesimal rotation would respectively decrease or increase $\tilde{\lambda}_{ij}$;  (ii) $\boldsymbol{v}$ is orthogonal to the gradient $\frac{\partial \tilde{\lambda}(\boldsymbol{\theta})_{ij}}{\partial \boldsymbol{\theta}}\Bigr|_{\boldsymbol{\theta} =\boldsymbol{0}}$; consequently, this removes a degree of rotational freedom such that $\boldsymbol{\theta}\in\mathbb{R}^{D-2}$.
\end{proof}

An immediate implication of Lemma \ref{lemma} is that in a two-dimensional embedding, continuous rotation of a cluster with only one external edge is not possible. For connected graphs, there is always a path from one node to all others, and thus every cluster must possess at least one external link. When considering the more general scenario of blocks with multiple inter-cluster edges, rotations that preserve the aggregate sum of pairwise distances among node embeddings become highly improbable, as elaborated in the next paragraph. As a result, for connected networks, we can generally anticipate the uniqueness of (local) minimum solutions, with no continuous rotations allowed that would leave the \textsc{HBDM} loss function of Equation \eqref{eq:log_likel_lsm} invariant.

As previously mentioned, local operations defined on the clusters with respect to their centroids can potentially leave the loss function value invariant since \textsc{HBDM(-Re)} calculates the non-link contributions between clusters in the objective function
based on their centroids distances. It can be said that there are almost surely no infinitesimal local cluster rotations of local minima solutions of Equation \eqref{eq:log_likel_lsm} for $2$-dimensional embeddings. We consider infinitesimal rotations on the cluster embeddings since our main motivation relies on the uniqueness of embeddings around the local minima so we also discard the local reflections of the clusters since they do not provide continuous transformation operations. Specifically, let $\mathcal{C}$ be a cluster with multiple external edges and let $\mathbf{\tilde{z}}_i=\boldsymbol{\mu}+\mathbf{R}(\boldsymbol{\theta})(\mathbf{z}_i-\boldsymbol{\mu})$ such that $\mathbf{R}(\boldsymbol{\theta})$ is a rotation matrix acting on the embeddings of nodes in cluster $\mathcal{C}$. We first note that any rotation of the cluster $\mathcal{C}$ by  $\mathbf{R}(\boldsymbol{\theta})$ will leave all terms invariant in Equation \eqref{eq:log_likel_lsm} except for the sum over external edges 
$S=\sum_{(i,j)\in E_{\mathcal{C},V \backslash \mathcal{C}}}\log{\lambda_{ij}}$.
For a local minima, no infinitesimal rotation exists that will reduce the overall sum of distances between node pairs defining the external edges as such rotation would improve the solution violating that it is a (local) minima. We can therefore assume that any rotation will result in either no change of the sum of distances or that the overall sum will increase. For embeddings in two-dimensional space, the rotations can be parameterized by the single parameter $\boldsymbol{\theta}$ that for a local minimum has the property $\frac{\partial S}{\partial \boldsymbol{\theta}}=\sum_{(i,j)\in E_r}\frac{\partial\log{\tilde{\lambda}_{ij}}}{\partial \boldsymbol{\theta}} -\sum_{(i,j)\in E_i}\frac{\partial\log{\tilde{\lambda}_{ij}}}{\partial\boldsymbol{\theta}}=0$. As a result, edges reducing their distances ($E_r$) will have a positive gradient of $\boldsymbol{\theta}$ whereas edges increasing ($E_i$) will have a negative gradient, and these parts perfectly cancel out for a local minimum. However, as the overall sum of distances for the local minima cannot decrease, the overall sum must remain the same. Therefore, for all node pairs for which $(\mathbf{z}_i-\boldsymbol{\mu})^\top(\mathbf{z}_j-\boldsymbol{\mu})>0$, we have that the distance increases for increasing edges more than the reduction for decreasing edges. Furthermore for node pairs for which $(\mathbf{z}_i-\boldsymbol{\mu})^\top(\mathbf{z}_j-\boldsymbol{\mu})<0$, we, in general, expect further distances between edge pairs, thus less impact on the rotation. As a result, it is highly unlikely that clusters with more than one external edge can be rotated in two-dimensional space. As the likelihood in Equation \eqref{eq:log_likel_lsm} is defined on a connected network every cluster will have at least one external edge. In combination with Lemma \ref{lemma}, a local minima can therefore not be infinitesimally rotated.

\subsection{A Hierarchical Block Distance Model for Bipartite Networks}
Our proposed frameworks, \textsc{HBDM} and \textsc{HBDM-Re} have straightforward generalizations to both directed and bipartite graphs. In the following, we provide the mathematical extension for the bipartite case (the directed network formulation of our proposed model can be considered a special case of the bipartite framework in which self-links are removed and thus omitted from the below log-likelihood).

For a bipartite network with adjacency matrix $Y^{N_1 \times N_2}$ we can formulate the log-likelihood as:

\begin{align}
\log& P(Y|\bm{\Lambda})=\sum_{\substack{i, j \\ y_{i,j}=1}}\Bigg(\psi_i+\omega_j- ||\mathbf{w}_i-\mathbf{v}_j||_2\Bigg) \nonumber\\ &-\sum_{k_L=1}^{K_L}\Bigg(\sum_{i,j \in C_{k_L}}\exp{(\psi_i+\omega_j - ||\mathbf{w}_i-\mathbf{v}_j||_2)}\Bigg)\nonumber
\\ 
& -\sum_{l=1}^{L}\sum_{k=1}^{K_l}\sum_{k^{'} > k}^{K_l}\Bigg(\exp{(- ||\bm{\mu}_{k}^{(l)}-\bm{\mu}_{k'}^{(l)}||_2)}\nonumber\\ &\times\Big(\sum_{i\in C_{k}^{(l)}}\exp{\psi_i}\Big)\Big(\sum_{j\in C_{k^{'}}^{(l)}}\exp{\omega_j}\Big)\Bigg),\label{eq:log_likel_lsm_bip}
\end{align}
where $\{\bm{\mu}_k^{(l)}\}_{k=1}^{K_L}$ are the latent centroids that have absorbed the dependency of both sets of latent variables $\{\mathbf{w}_i,\mathbf{v}_j\}$ while we define the Poisson rate as:
\begin{equation}
    \lambda_{ij}=\exp\big(\psi_i+\omega_j- d(\mathbf{w}_i,\mathbf{v}_j)\big),
    \label{eqn:random_effect_bip}
\end{equation}
where $\psi_i$ and $\omega_j$ are the corresponding random effects and $\{\mathbf{w}_i$, $\mathbf{v}_j\}$ are the latent variables of the two disjoint sets of the vertex set of sizes $N_1$ and $N_2$, respectively. In this setting, we use our divisive Euclidean distance hierarchical clustering procedure over the concatenation $\mathbf{Z}=[\mathbf{W};\mathbf{V}]$ of the two sets of latent variables. Therefore, we define an accurate hierarchical block structure for bipartite networks, with each block including nodes from both of the two disjoint modes. Here, a centroid is considered a leaf if the corresponding tree-cluster contains less than $\log (N_1)$ of the latent variables $\{\mathbf{w}_i\}_{i=1}^{N_1}$ or less than $\log(N_2)$ of $\{\mathbf{v}\}_{j=1}^{N_2}$.

\subsection{Complexity Comparison}

TABLE \ref{tab:complexities} offers a comparison of the time complexities for various notable \textsc{GRL} methods, expressed in Big $\mathcal{O}$ notation, akin to \cite{verse}. From this comparison, it becomes evident that our proposed \textsc{HBDM} model ranks among the most competitive frameworks. Regarding space complexity, our model exhibits linearithmic complexity, setting it apart from the majority of the considered baseline methods, which typically display quadratic space complexity, as shown in \cite{verse}.

\begin{table}[!tb]
\caption{Complexity analysis of methods. $N := \left| V \right|$ is the vertex set, $\left| E\right|$: edge set, $\mathcal{W}$: number of walks, $\mathcal{L}$: walk length, $H$: height of the hierarchical tree, $D$: node representation size, $k$: number of negative instances, $q$: order value, $c$: Chebyshev expansion order, $\gamma$: window size, $\alpha_1$ and $\alpha_2$ constants such as $\alpha_1,\alpha_2\ll N$.} 
\label{tab:complexities}
\begin{center}
\resizebox{0.5\textwidth}{!}{%
\begin{tabular}{rc}\toprule
Method & Complexity\\\cmidrule(rl){1-1}\cmidrule(rl){2-2}
\textsc{DeepWalk} \cite{deepwalk-perozzi14} & $\mathcal{O}\left(\gamma N \log{(N)} \mathcal{W} \mathcal{L} \mathcal{D}  \right)$ \\
\textsc{Node2Vec} \cite{node2vec-kdd16} & $\mathcal{O}\left(\gamma N \mathcal{W} \mathcal{L} \mathcal{D}k \right)$  \\
\textsc{LINE} \cite{line} & $\mathcal{O}\left( |E| D k \right)$ \\
\textsc{NetMF} \cite{netmf-wsdm18} & $\mathcal{O}\left(N^2 D\right)$  \\
\textsc{NetSMF} \cite{netsmf-www2019} & $\mathcal{O}\left( |E|(\gamma+D)+ND^2 + D^3 \right)$ \\
\textsc{RandNE} \cite{randne-icdm18} & $\mathcal{O}\left(ND^2 + |E|Dq\right)$ \\
\textsc{LouvainNE} \cite{louvainNE-wsdm20} & $\mathcal{O}\left( |E| \mathcal{H} + N D \right)$ \\
\textsc{ProNE} \cite{prone-ijai19}& $\mathcal{O}\left(ND^2 + |E|c\right)$ \\
\textsc{Verse} \cite{verse} & $\mathcal{O}\left(N(\mathcal{W} + kD)\right)$ \\
\textsc{HBDM(-Re)} & $\mathcal{O}\left(\alpha_2N\log{(N)} D\right)$ \\\bottomrule
 & 
\end{tabular}%
}
\end{center}

\end{table}

\section{Hybrid memberships, Matrix Factorization, and Latent Distance Models}

Revisiting our main goal which is to learn a representation in a lower dimensional space, expressing the property that similar nodes in the network are positioned closer in such a space. We here, also aim to define community-aware representations, meaning that each embedding should convey information about the community structure. Overall, we would like to define a Graph Representation Learning method expressing the desired properties of homophily and transitivity coupled with latent structure characterization and under a unique optimization procedure (i.e. no post-processing steps such as clustering procedures). Under such a direction, we will focus on finding mappings of the nodes into the unit $D$-simplex set, $\Delta^{D} \subset \mathbb{R}_{+}^{D+1}$ formally defined as
\begin{align*}
   \Delta^{D}= \left\{ (x_0,\ldots, x_{D})\in\mathbb{R}^{D+1}\Bigg| \sum_{d=0}^{D}x_d \!=\! 1, \ x_d \geq 0, \ \forall \; d \in \{0,\ldots, D\} \right\}.
\end{align*}

\begin{figure}[!t]
    \centering
    \includegraphics[width=0.4\columnwidth]{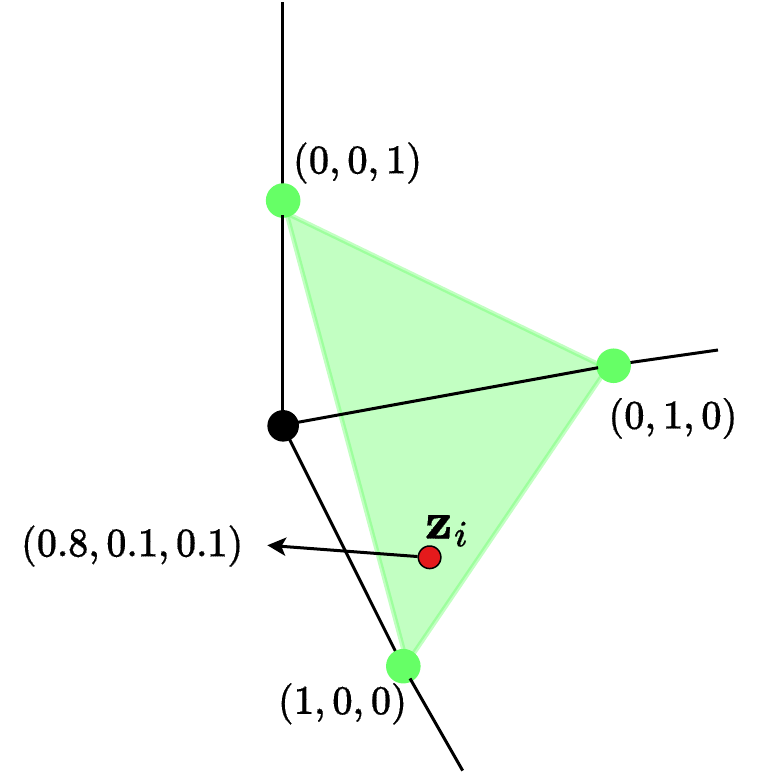}
    \caption{The standard $2$-simplex in $\mathbb{R}^3$ which is a triangle. Any point $\mathbf{z}_i$ of the simplex lies on the affine hyperplane and is denoted with the green-colored area, and can be expressed as a convex combination of the three corresponding vertices (corners).}
    \label{fig:simplex}
\end{figure}

In addition, we provide the standard $2$-simplex in Figure \ref{fig:simplex} with an example of an embedding $\mathbf{z}_i$. A direct consequence of constraining node representation on the simplex is that the extracted node embeddings can convey information about latent community memberships. Numerous \textsc{GRL} methods lack guarantees for identifiable or unique solutions, making their interpretation heavily reliant on the initial setting of hyper-parameters. In this chapter, we are also focusing on the issue of identifiability. We aim to find identifiable solutions, though these can only be realized to the extent of permutation invariance, as described in Def. \ref{def:identifiabilty}.
 
 \begin{figure}[!t]
  \centering
    \subfloat[Translation invariances.]{{
  \includegraphics[width=0.21\columnwidth]{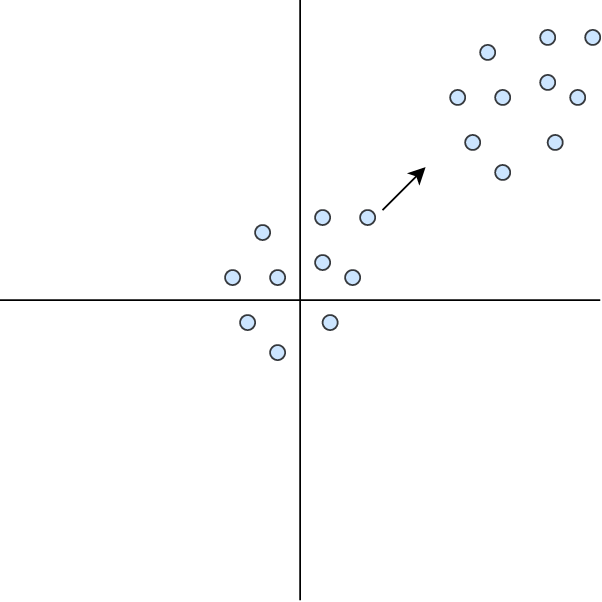} }}%
  \hfill
   \subfloat[Rotation invariances.]{{
  \includegraphics[width=0.25\columnwidth]{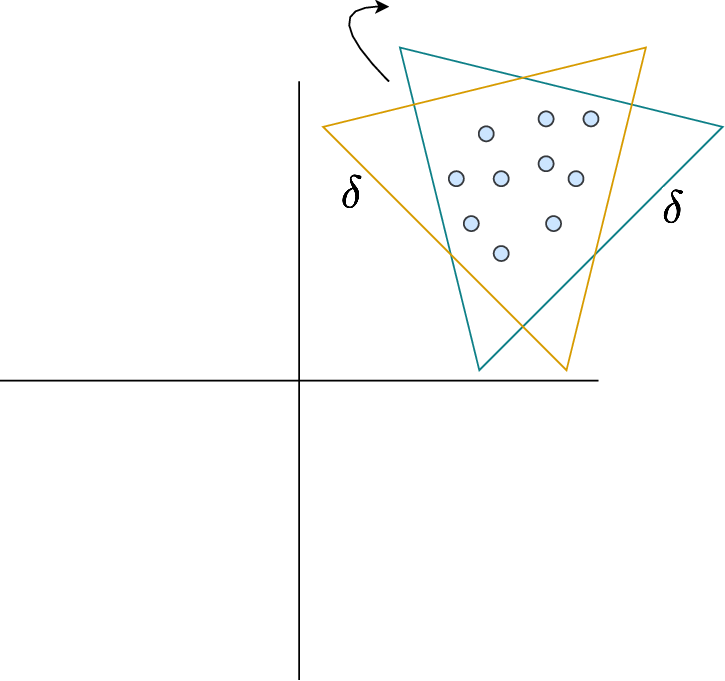} }}%
  \hfill
  \subfloat[Decreased simplex volume ensuring identifiability.]{{\includegraphics[width=0.34 \columnwidth]{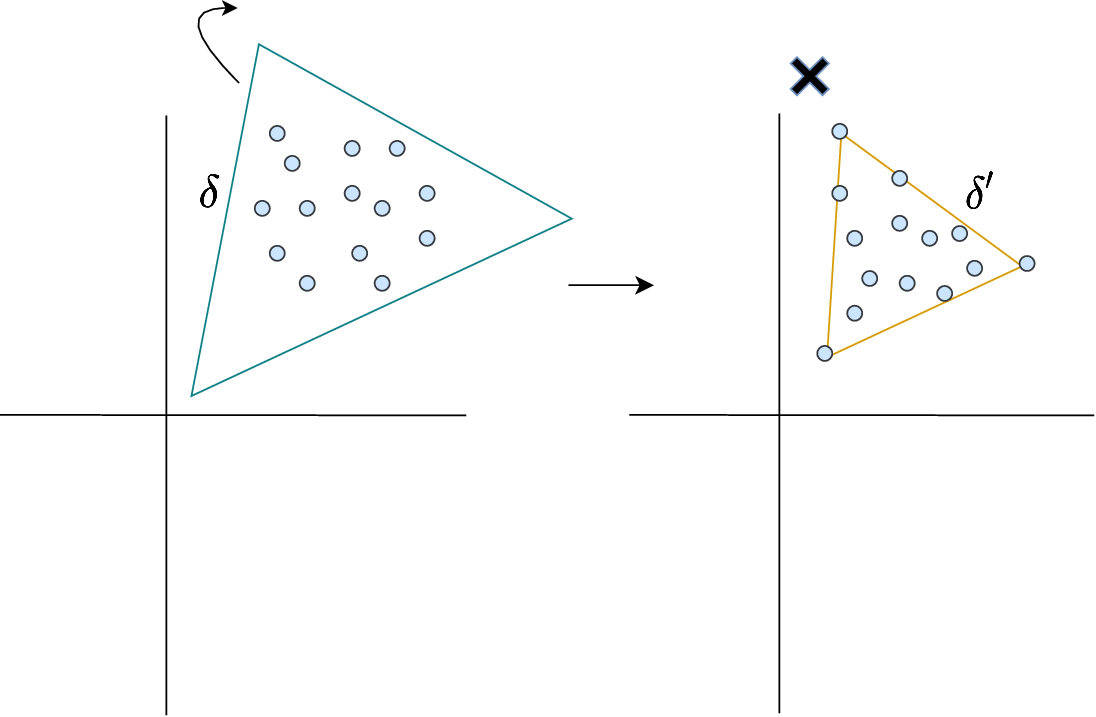} }}%
  \caption{A $2$-dimensional latent space with the $2$-simplex given as the green and yellow triangles, the blue points denote embedding positions of the $\textsc{LDM}$ and $\delta$ is the simplex size \cite{hmldm}.}\label{fig:invariances}
\end{figure}

\begin{definition}[\textbf{Identifiabilty}]\label{def:identifiabilty}
An embedding matrix $\mathbf{Z}$ whose rows indicate the corresponding node representations is called an \textit{identifiable solution up to a permutation} if it holds $\widetilde{\mathbf{Z}}=\mathbf{Z}\mathbf{P}$ for a permutation $\mathbf{P}$ and a solution $\widetilde{\mathbf{Z}} \not= \mathbf{Z}$.
\end{definition}
\subsection{Hybrid memberships under a latent distance model}
We will consider a Poisson LDM, defining a log-likelihood over the adjacency matrix $\mathbf{Y}$ of the network as introduced by Equation \eqref{eq:log_likel_lsm0}. To combine powerful and community-aware representations, we propose the Hybrid-Membership Latent Distance Model \ (\textsc{HM-LDM}) with a log-rate based on the $\ell^2$-norm as:
  \begin{equation}
     \label{eq:nmf_rate}
     \log \lambda_{ij}=\Big(\gamma_i+\gamma_j-\delta^p\cdot||\mathbf{z}_i -\mathbf{z}_j||_2^p\Big),
 \end{equation}
 where $\mathbf{z_i} \in [0,1]^{D+1}$ are the latent embeddings constrained to the $D-$simplex, i.e. $\sum_{d=1}^{D+1} w_{id}=1$, $\delta \in \mathbb{R}_+$ is the non-negative value controlling the simplex volume, and $\gamma_i \in \mathbb{R}$ a bias term of node $i\in\mathcal{V}$ accounting for node-specific effects such as degree heterogeneity. Lastly, $p$ is the power of the $\ell_2$ norm with $p \in\{1,2\}$ which governs the model specification. Specifically, power $p$ modifies the effect of the embedding distances within the rate functions. In other words, in Equation \ref{eq:nmf_rate} we constrain the latent space to the $D-$simplex, and the simplex's edge lengths ($1$-faces) are scaled by the non-negative constant $\delta$, controlling the length of the sides of the simplex, and consequently, the volume of the simplex itself.

A notable characteristic of Equation \eqref{eq:nmf_rate}, is that it resembles a positive Eigenmodel with random effects: $ \tilde{\gamma}_i+\tilde{\gamma}_j+(\mathbf{\tilde{w}}_i\bm{\Lambda}\mathbf{\tilde{w}}_j^{\top})$ where $\bm{\Lambda}$ is a diagonal matrix having non-negative elements, i.e. $\tilde{\gamma}_i=\gamma_i-\delta^2\cdot||\mathbf{z}_i||^2_2$, $\tilde{\gamma}_j=\gamma_j-\delta^2\cdot||\mathbf{z}_j||^2_2$ and $\tilde{\mathbf{z}}_i\bm{\Lambda}\tilde{\mathbf{z}}_j^\top=2\delta^2\cdot \mathbf{z}_i\mathbf{z}_j^\top$. Therefore, the squared Euclidean distance acts as a bridge between the traditional \textsc{LDM} and the non-negativity-constrained Eigenmodel. While not entirely conforming to the definition of a metric, the squared Euclidean distance still conveys homophily, resulting in an interpretable latent space. Although it doesn't precisely satisfy the triangle inequality, it maintains the order of pairwise Euclidean distances and is often favored in applications due to its nature as a strictly convex smooth function. By the well-known cosine formula, we have
\begin{align}
||\mathbf{z}_i-\mathbf{z}_j||_2^2 &= ||\mathbf{z}_i-\mathbf{z}_k||_2^2+|| \mathbf{z}_k-\mathbf{z}_j ||_2^2-2||\mathbf{z}_i-\mathbf{z}_k||_2||\mathbf{z}_k-\mathbf{z}_j||_2\cos(\theta),\nonumber
\end{align}
\noindent where $\theta \in (-\pi/2, \pi/2)$ represents the angle between $\mathbf{z}_i-\mathbf{z}_k$ and $\mathbf{z}_k-\mathbf{z}_j$. Note that the third term also approaches to $0$ for $\theta \rightarrow \pi/2$. For the case where $\theta \in [\pi/ 2, 3\pi/ 2]$, it satisfies the triangle inequality: $||\mathbf{z}_i - \mathbf{z}_j ||_2^2 \leq || \mathbf{z}_i - \mathbf{z}_k ||_2^2  + || \mathbf{z}_k - \mathbf{z}_j ||_2^2$. 

The embedding vectors, $\{\mathbf{z}_i\}_{i=1}^{N}$ in Equation \eqref{eq:nmf_rate}, are constrained to non-negative values and to sum to one. As a result, they are positioned on a simplex that shows the participation of node $i\in\mathcal{V}$ across $D+1$ latent communities. Any \textsc{LDM} can be constrained to the non-negative orthant without diminishing its performance or expressive power. Non-negative embeddings do not change the distance metric, as it remains constant under translation, as illustrated in Figure \ref{fig:invariances} (a). Furthermore, the $D$-dimensional non-negative orthant can be reconstructed by a large enough $D$-simplex. From these considerations, it can be effortlessly shown that for high values of the $\delta$ parameter in Equation \eqref{eq:nmf_rate}, the sum-to-one constraint on the embeddings \(\mathbf{Z}\) results in an unconstrained \textsc{LDM}, since the distances are unbounded when $\delta\rightarrow+\infty$. In this scenario, the memberships defined by the rows of matrix $\mathbf{Z}$ cannot be uniquely identified due to the distance invariance of rotation, as depicted in Figure \ref{fig:invariances} (b).

Nonetheless, by reducing the volume of the simplex (which is the same as lowering $\delta$), the  $D$-dimensional space of \textsc{LDM} will eventually cease to fit within the $D$-simplex, forcing the nodes to begin occupying the corners of this reduced simplex. A node is referred to as a \textit{champion} if its latent representation corresponds to a standard binary unit vector.

\begin{definition}[\textbf{Community champion}]\label{def:champion}
A node for a latent community is called \textit{champion} if it belongs to the community (simplex corner) while forming a binary unit vector.
\end{definition}

Champion nodes hold considerable importance for the model's identifiability. If each corner of the simplex has at least one node (champion), then the model's solution is identifiable (subject to a permutation matrix) (as per Def. \ref{def:identifiabilty}). This occurs because any random rotation no longer maintains the solution's invariance, as illustrated by Figure \ref{fig:invariances} (c). It is evident that the scalar, $\delta$, controls important properties, such as identifiability and the type of community memberships, while also the expressive capability of the model. Specifically, an \textsc{HM-LDM} with a large value of $\delta$ is equivalent to an unconstrained \textsc{LDM} that includes high expressive capability but also a rotation invariant space. In contrast, small values of $\delta$ result in identifiable solutions and can ultimately drive hard cluster assignments. Therefore, with very low values of $\delta$, nodes are exclusively positioned at the corners of the simplex. Lastly, we can also find regimes of values for $\delta$ that offer identifiable solutions, and mixed-memberships but also performance similar to \textsc{LDM}, defining a silver lining.

\begin{figure*}[!t]
\centering
 \subfloat[Non-Negative Matrix Factorization]{{ \includegraphics[width=0.45\textwidth]{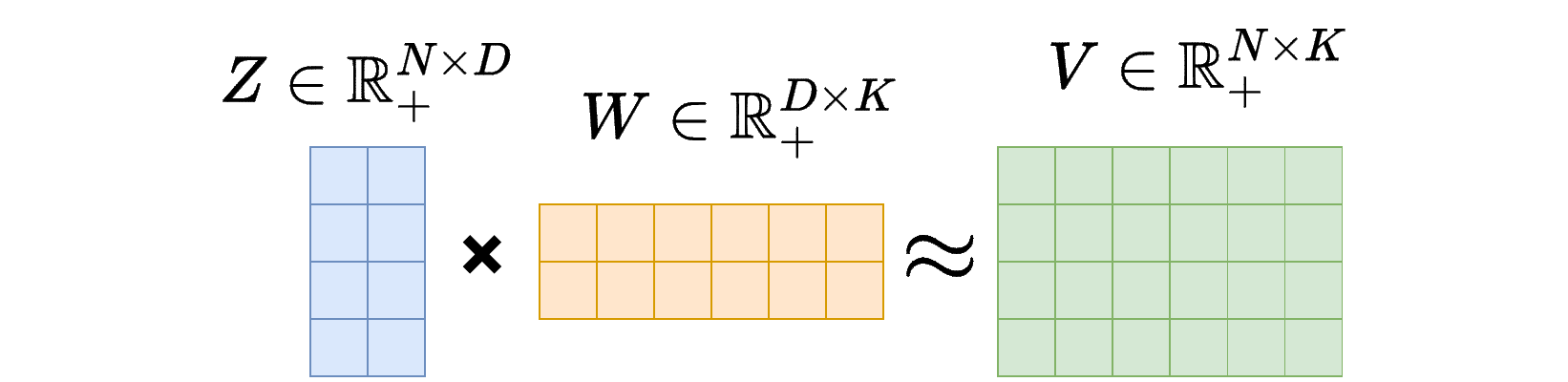} }}
\hfill
\subfloat[Symmetric Non-Negative Matrix Factorization]{{ \includegraphics[width=0.45\textwidth]{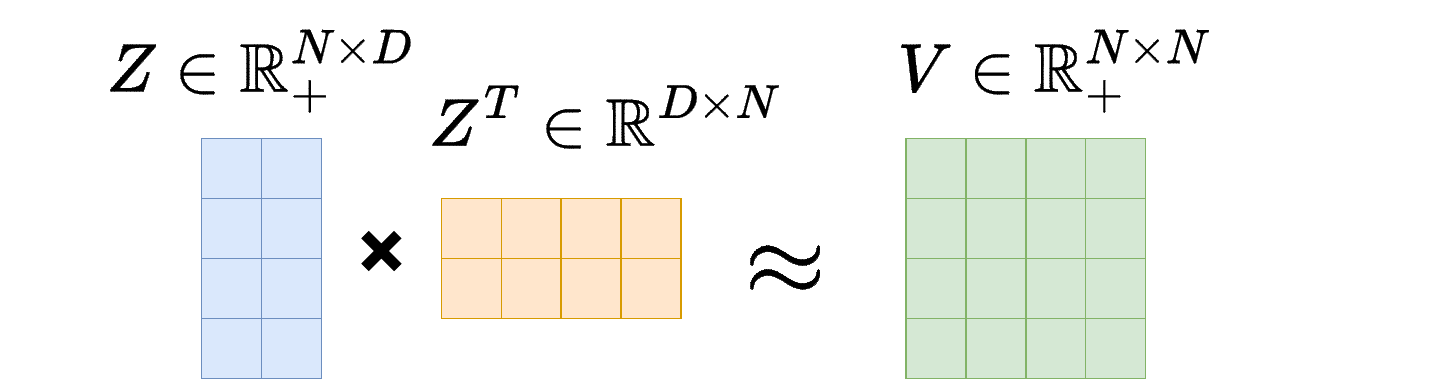} }}

\caption{Factorization of a non-negative matrix $V$ into two non-negative matrices $Z$ and $W$. If the matrix is symmetric the factorization defines only the non-negative matrix $Z$.}
\label{fig:NMF}
\end{figure*}

A different take on the identifiability of the model for $p=2$, can also be given under the Non-negative Matrix Factorization (NMF) theory. Figures \ref{fig:NMF} (a) and (b), show both a non-symmetric and a symmetric NMF factorization. Specifically, a non-negative matrix $V$ is factorized into two matrices $Z$ and $W$, also non-negative. If the matrix is symmetric the factorization defines only the non-negative matrix $Z$. We will focus on undirected networks and make use of the symmetric NMF while we will not present extensions to bipartite and directed networks since they are trivial to obtain by switching to a non-symmetric NMF operation.

We can now easily show a re-parameterization of Equation \eqref{eq:nmf_rate} by $\tilde{\gamma}_i+\tilde{\gamma}_j+2\delta^2\cdot(\mathbf{z}_i\mathbf{z}_j^{\top})$ as described in Equation \eqref{eq:nmf_rate}. In such a formulation, the product $\mathbf{Z}\mathbf{Z}^{\top}$ defines a symmetric NMF problem which is an identifiable and unique factorization (up to permutation invariance) when $\mathbf{Z}$ is full-rank and at least one node resides solely in each simplex corner, ensuring separability \cite{nmf4,nmf5}. 

Under this NMF formulation, the product $\mathbf{z}_i\mathbf{z}_j^{\top} \in [0,1]$ achieves its maximum value only when both nodes $i$ and $j$ reside in the same corner of the simplex. The parameter, $\delta$, acts as a simple multiplicative factor in the first term of the objective function of \textsc{HM-LDM}, given in Equation \eqref{eq:prob_adj}, while in the second term acts as a power of the exponential function. For small values of $\delta$, the model is biased towards hard latent community assignments of nodes since similar nodes achieve high rates only when they belong to the same latent community (simplex corner). On the other hand, nodes heading towards the simplex corners for large values of $\delta$ lead to an exponential change in the second term of the log-likelihood function given in Equation \eqref{eq:prob_adj}. Thus, a possible hard allocation of dissimilar nodes to the same community penalizes the likelihood severely. For this reason, high values of $\delta$ benefit mixed-membership allocations.

\section{Signed integer weighted graphs}
We continue now with the analysis of signed integer weighted graphs. Our aim is to learn representations for signed networks while expressing the properties of homophily, structure retrieval, and importantly heterophily/animosity as expressed by negative relationships. Animosity or heterophily refers to the tendency for nodes to interact negatively when they express opposing or dissimilar views or opinions. In this setting, we would like to generalize transitivity properties to the expression of balance theory, a socio-psychological theory admitting four rules: “The
friend of my friend is my friend", “The enemy of my friend is my enemy",  “The friend of my enemy is my enemy", and “The enemy of my enemy is my friend", also presented in Figure \ref{fig:blt}. We can observe that transitivity is contained in balance theory and corresponds to the first case of Figure \ref{fig:blt}. We will move to the design of a \textsc{GRL} model able to characterize such properties and extend LDMs to the analysis of signed networks. In particular, we will utilize Archetypal Analysis \cite{cutler1994a,5589222} allowing for model specifications allowing for archetype retrieval of relational data able to characterize network polarization.

\begin{figure}[!t]
    \centering
    \includegraphics[width=0.79\columnwidth]{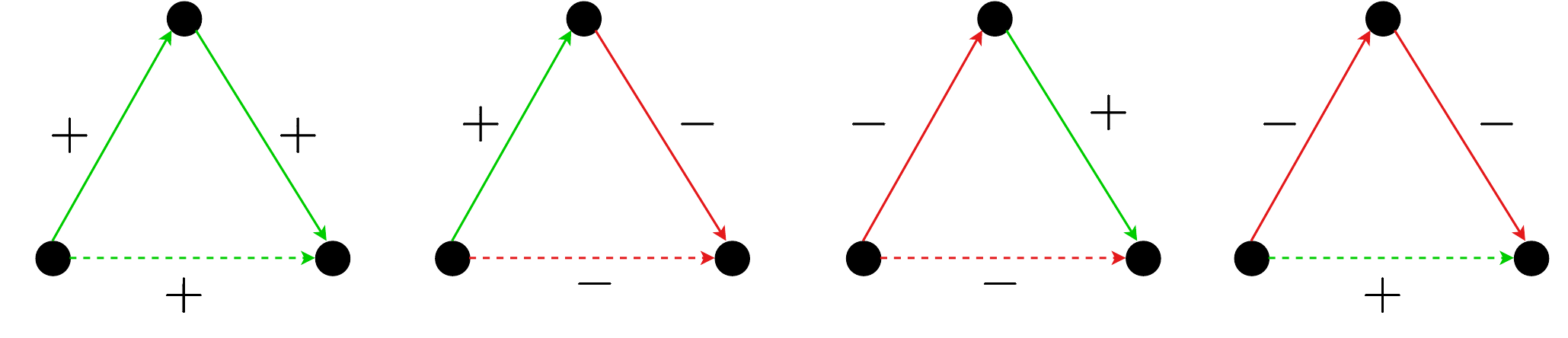}
    \caption{Graphical representation of the four balance theory properties, where black dots correspond to network nodes, green arrows as positive directed links, and red arrows as negative directed links. Dashed lines with arrows denote the inferred relationship under the balance theory. Analytically the panels show from left to the right, \textcolor{blue}{Case 1}: The friend of my friend is my friend, \textcolor{blue}{Case 2}: The enemy of my friend is my enemy, \textcolor{blue}{Case 3}: The friend of my enemy is my enemy, \textcolor{blue}{Case 4}: The enemy of my enemy is my friend. }
    \label{fig:blt}
\end{figure}

\section{The Skellam Latent Distance Model (SLDM)}\label{lab:random_Effects}
We now generalize our main purpose which is to learn latent node representations $\{\mathbf{z}_i\}_{i\in\mathcal{V}}\in\mathbb{R}^{D}$ in a low dimensional space to signed networks $\mathcal{G}=(\mathcal{V}, \mathcal{Y})$ ($D \ll |\mathcal{V}|$). Therefore, the edge weights can take any integer value to represent the positive or negative tendencies between the corresponding nodes. We model these signed interactions among the nodes using  
 the Skellam distribution \cite{jg1946frequency}, which can be formulated as the difference of two independent Poisson-distributed random variables ($y=N_1 - N_2\in\mathbb{Z}$) with respect to the rates $\lambda^{+}$ and $\lambda^{-}$: 
\begin{align*}
P(y|\lambda^{+},\lambda^{-}) = e^{-(\lambda^{+}+\lambda^{-})}\left(\frac{\lambda^{+}}{\lambda^{-}}\right)^{y/2}\mathcal{I}_{|y|}\left(2\sqrt{\lambda^{+}\lambda^{-}}\right),
\end{align*}
where $N_1 \sim Pois(\lambda^{+})$ and $N_2 \sim Pois(\lambda^{-})$, and $\mathcal{I}_{|y|}$ is the modified Bessel function of the first kind and order $|y|$. As far as we are aware, the Skellam distribution has not previously been used to model the likelihood of a network. We are introducing a novel latent space model that employs the Skellam distribution by adapting the latent distance model, originally devised for undirected and unsigned binary networks as a logistic regression model \cite{exp1}. This was subsequently expanded to include various generalized linear models \cite{doi:10.1198/016214504000001015}, such as the Poisson regression model tailored for integer-weighted networks. The negative log-likelihood of a latent distance model under the Skellam distribution can be formulated as follows: 
\begin{align*}
\mathcal{L}(\mathcal{Y}) :=\log p(y_{ij}|\lambda^{+}_{ij},\lambda^{-}_{ij}) = \sum_{i<j}{(\lambda^{+}_{ij}+\lambda^{-}_{ij})} - \frac{y_{ij}}{2}\log\left(\frac{\lambda^{+}_{ij}}{\lambda^{-}_{ij}}\right)-\log(I_{ij}^{*}),
\end{align*}
where $I_{ij}^{*} := \mathcal{I}_{|y_{ij}|}\left(2\sqrt{\lambda^{+}_{ij}\lambda^{-}_{ij}}\right)$. As it can be noticed, the Skellam distribution has two rate parameters, and we consider them to learn latent node representations $\{\mathbf{z}_i\}_{i\in\mathcal{V}}$ by defining them as follows:
\begin{align}
\lambda_{ij}^{+} &= \exp\big(\gamma_{i} + \gamma_{j} - ||\mathbf{z}_i-\mathbf{z}_j||_2\big)\label{eq:rate1},
\\
\lambda_{ij}^{-} &= \exp\big(\delta_{i} + \delta_{j} + ||\mathbf{z}_i-\mathbf{z}_j||_2\big),
\label{eq:rate2}
\end{align}
where the set $\{\gamma_i,\delta_i\}_{i\in\mathcal{V}}$ denote the node-specific random effect terms, and $||\cdot||_2$ is the Euclidean distance function. More specifically, $\gamma_i,\gamma_j$ represent the "social" effects/reach of a node and the tendency to form (as a receiver and as a sender, respectively) positive interactions, expressing positive degree heterogeneity (indicated by $+$ as a superscript of $\lambda$). In contrast, $\delta_i,\delta_j$ provides the "anti-social" effect/reach of a node to form negative connections and thus models negative degree heterogeneity (indicated by $-$ as a superscript of $\lambda$). The rate formulation for the positive interactions $\lambda_{ij}^{+}$ naturally conveys the homophily property (a high positive rate is achieved when the distance is small) while negative interaction rate expression $\lambda_{ij}^{-}$ models heterophily (a high negative rate is achieved when the distance is large). In addition, the corresponding rates in Equation \eqref{eq:rate1} and  Equation \eqref{eq:rate2} satisfy balance theory, as it is a direct consequence of the high-order effects caused by the expression of homophily and heterophily, as seen by Figure \ref{fig:hom_anim}.

\begin{figure}[!t]
    \centering
    \includegraphics[width=0.99\columnwidth]{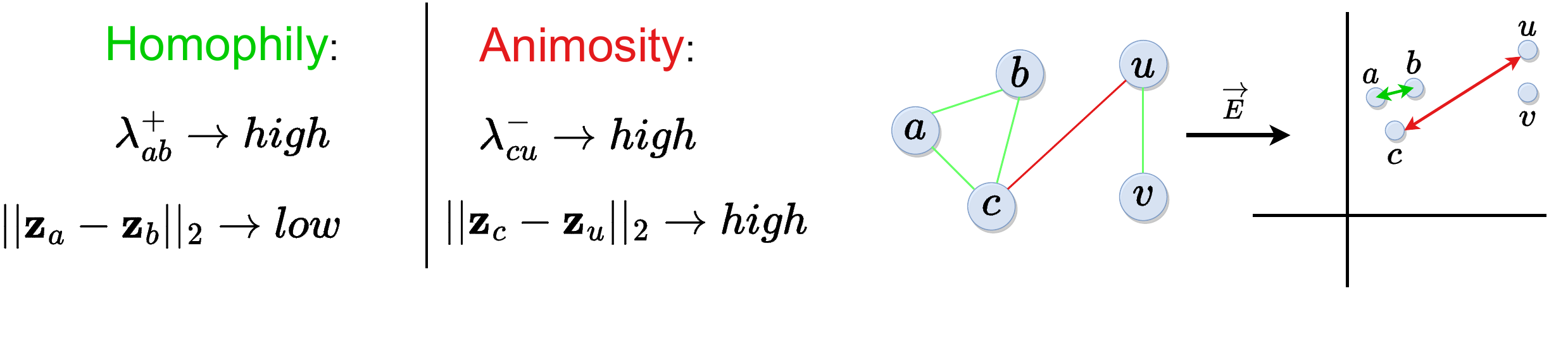}
    \caption{Expression of homophily and animosity as imposed by the Skellam Latent Distance model. Green lines correspond to positive interactions and red lines to negative interactions. Positively related nodes (i.e. pair $\{a,b\}$) are positioned close in space in order for the $\lambda^+$ rate to be high while negatively interacting nodes (i.e. pair $\{c,u\}$) are positioned far apart in space for the $\lambda^-$ rate to be high.}
    \label{fig:hom_anim}
\end{figure}

By imposing standard normally distributed priors elementwise on all model parameters $\bm{\theta}=\{\boldsymbol{\gamma},\bm{\delta}, \mathbf{Z}\}$, i.e., $\theta_i\sim \mathcal{N}(0,1)$, We define a maximum a posteriori (MAP) estimation over the model parameters, via the loss function to be minimized (ignoring constant terms):
\begin{align}\label{eq:loss_sk}
\begin{split}
    Loss = \sum_{i<j}\Bigg( \lambda_{ij}^{+}+\lambda_{ij}^{-} -\frac{y_{ij}}{2}\log\left( \frac{\lambda_{ij}^{+} }{\lambda_{ij}^{-}}\right)\Bigg)& 
    - \sum_{i<j}\log I_{|y_{ij}|}\Big (2\sqrt{\lambda_{ij}^{+}\lambda_{ij}^{-}}\Big )
    \\
    +& \frac{\rho}{2}\Big(||\mathbf{Z}||_F^2+||\bm{\gamma}||_F^2+||\bm{\delta}||_F^2\Big),
    \end{split}
\end{align}
where $||\cdot||_F$ denotes the Frobenius norm. In addition, $\rho$ is the regularization strength with $\rho=1$ yielding the adopted normal prior with zero mean and unit variance. Importantly, by setting $\lambda_{ij}^{+}$ and $\lambda_{ij}^{-}$ based on Equation \eqref{eq:rate1}
and \eqref{eq:rate2}, the model effectively makes positive (weighted) links attract and negative (weighted links) deter nodes from being in proximity of each other.

\subsection{Archetypal Analysis}
Archetypal Analysis (AA) \cite{cutler1994a,5589222} is a technique used in data clustering that identifies "archetypes" from a given observational dataset. An archetype is a pure or idealized form that represents an essential aspect or fundamental pattern within the data. In other words, archetypes are extreme points or corners of the convex hull of the data, and they can be thought of as the most representative or "extreme" examples of different behaviors or characteristics found in the data. In essence, AA is a powerful tool for clustering that provides a nuanced view of the data by identifying and utilizing extreme or "archetypal" patterns within the dataset. By expressing data in terms of these fundamental elements, AA offers an insightful perspective on the underlying structure and relationships within the data, aiding in cluster identification and interpretation. The definition of the embedded data points is given as follows:
\begin{eqnarray}
    \mathbf{X}\approx \mathbf{XCZ}\quad
    \text{s.t. }\boldsymbol{c}_d\in \Delta^{N} \text{ and } \mathbf{z}_j\in \Delta^{D}.
\end{eqnarray}

The archetypes, represented by the columns of $\mathbf{A}=\mathbf{X}\mathbf{C}$, define the corners of the extracted polytope, serving as convex combinations of the observations. Meanwhile, $\mathbf{Z}$ outlines how each observation is reassembled as convex combinations of these extracted archetypes. 

While Archetypal Analysis confines the representation within the convex hull of the data, alternative methods for modeling pure or ideal forms have included Minimal Volume (MV) approaches. One advantage of these approaches is that, unlike AA, they don't necessitate the existence of pure observations in the data. However, they come with the disadvantage of needing careful regularization tuning to determine an appropriate volume \cite{zhuang2019regularization}. Additionally, the precise calculation of the volume for general polytopes demands the computation of determinants of the sum of all simplices that define the polytope \cite{bueler2000exact} which comes with a high computational cost.

Archetypal Analysis and Minimal Volume extraction techniques have been recognized for their ability to uncover latent polytopes that define trade-offs. Within these polytopes, the vertices symbolize the maximally enriched and distinct aspects, or archetypes, which allow for the identification of specific tasks or prominent roles that the vertices represent \cite{shoval2012evolutionary,hart2015inferring}. Owing to the computational challenges associated with regularizing high-dimensional volumes and the intricate fine-tuning required for those regularization parameters, our current focus is centered on polytope extraction as framed by the AA formulation, rather than adopting an MV approach.

\subsection{A Generative Model of Polarization}\label{generative}
Combining AA with the Skellam Latent Distance Model, i.e. constraining the latent space into a polytope allows for the modeling of polarization, as present in many signed networks. Specifically, we extend the Skellam \textsc{LDM} and express polarization based on defining node positions as convex combinations of the polytope - what we denote as a sociotope. The corners of the sociotope are considered the different "poles" that drive polarization and are sufficient to express the social dynamics of the network. Essentially, these uncovered "poles" are the extracted archetypes/extreme profiles as proposed by AA while every other node representation is a convex combination of these extremes.

In our generative model of polarization,
we further suppose that the bias terms introduced in the definitions of the Poisson rates, $(\lambda_{ij}^{+},\lambda_{ij}^{-})$, are normally distributed. Since latent representations $\{\mathbf{z}_i\}_{i\in\mathcal{V}}$ according to AA and MV lie in the standard simplex set $\Delta^D$, we further assume that they follow a Dirichlet distribution. Formally, we can summarize the generative model as follows:
\begin{align*}
\gamma_i &\sim \mathcal{N}(\mu_\gamma,\sigma_\gamma^2) && \forall i\in\mathcal{V}, \\
\delta_i &\sim \mathcal{N}(\mu_\delta,\sigma_\delta^2) && \forall i\in\mathcal{V}, \\
\mathbf{a}_{d} &\sim \mathcal{N}(\boldsymbol{\mu}_A, \sigma_A^2\mathbf{I}) && \forall d\in\{1,\ldots,D\},\\
\mathbf{z}_i &\sim Dir(\bm{\alpha}) && \forall i\in\mathcal{V}, \\
\lambda_{ij}^{+} &= \exp\big(\gamma_{i} + \gamma_{j} - \|\mathbf{A}(\mathbf{z}_i-\mathbf{z}_j)\|_2\big),\\
\lambda_{ij}^{-} &= \exp\big(\delta_{i} + \delta_{j} + \|\mathbf{A}(\mathbf{z}_i-\mathbf{z}_j)\|_2\big),\\
y_{ij} &\sim Skellam(\lambda^{+}_{ij},\lambda_{ij}^{-}) && \forall (i,j)\in\mathcal{V}^2.
\end{align*}
According to the above generative process, positive ($\boldsymbol{\gamma}$) and negative ($\boldsymbol{\delta}$) random effects for the nodes are first drawn, upon which the location of extreme positions $\mathbf{A}$ (i.e., corners of the polytope denoted archetypes) are generated. In addition, as the dimensionality of the latent space increases linearly with the number of archetypes, i.e. $\boldsymbol{A}$ is a square matrix, with probability zero archetypes will be placed in the interior of the convex hull of the other archetypes. Subsequently, the node-specific convex combinations $\mathbf{Z}$ of the generated archetypes are drawn, and finally, the weighted signed link is generated according to the node-specific biases and distances between dyads within the polytope utilizing the Skellam distribution. The polarization level of the generative process can easily be controlled by the concentration parameter $\alpha$ of the Dirichlet distribution, defining the reconstruction matrix $\bm{Z}$.

\subsection{The Signed Relational Latent Distance Model}

For inference, we exploit how polytopes can be efficiently extracted using archetypal analysis. We, therefore, define the \textsc{S}igned \textsc{L}atent relational d\textsc{I}stance \textsc{M}odel (\textsc{SLIM}) by defining a relational archetypal analysis approach endowing the generative model a parameterization akin to archetypal analysis in order to efficiently extract polytopes from relational data defined by signed weighted networks. Specifically, we formulate the relational AA in the context of the family of LDMs, as:
\begin{align}
    \lambda_{ij}^{+} &=\exp \big( \gamma_{i} + \gamma_{j} - \|\mathbf{A} (\mathbf{z}_{i}-\mathbf{z}_{j})\|_{2}\big)
    \\ 
    &=\exp \big( \gamma_{i} + \gamma_{j} -\|\mathbf{RZC}(\mathbf{z}_{i}-\mathbf{z}_{j})\|_2\big).\label{LRPM_inensity_function_1}
    \\
\lambda_{ij}^{-} &=\exp \big( \delta_{i} + \delta_{j} + \|\mathbf{A} (\mathbf{z}_{i}-\mathbf{z}_{j})\|_{2}\big)
\\ 
&=\exp \big( \delta_{i} + \delta_{j} +\|\mathbf{RZC}(\mathbf{z}_{i}-\mathbf{z}_{j})\|_2\big).\label{LRPM_inensity_function_2}
\end{align}

Notably, in the AA formulation $\mathbf{X}=\mathbf{RZ}$ corresponds to observations formed by convex combinations $\mathbf{Z}$ of positions given by the columns of $\boldsymbol{R}^{D\times D}$. Furthermore, in order to ensure what is used to define archetypes $\mathbf{A}=\mathbf{XC}=\mathbf{RZC}$ corresponds to observations using these archetypes in their reconstruction $\mathbf{Z}$,
we define $\boldsymbol{C}\in \boldsymbol{R}^{N\times D}$ as a gated version of $\mathbf{Z}$ normalized to the simplex such that $\boldsymbol{c}_d\in\Delta^{N}$ by defining 
\begin{equation}
    c_{nd}=\frac{(\mathbf{Z}^\top\circ [\sigma(\mathbf{G})]^\top)_{nd}}{\sum_{n^\prime}(\mathbf{Z}^\top\circ [\sigma(\mathbf{G})]^\top)_{n^\prime d}}
\end{equation}
in which $\circ$ denotes the elementwise (Hadamard) product and $\sigma(\mathbf{G})$ defines the logistic sigmoid elementwise applied to the matrix $\boldsymbol{G}$. As a result, the extracted archetypes are ensured to correspond to the nodes assigned the archetype, whereas the location of the archetypes can be flexibly placed in space as defined by $\mathbf{R}$. By defining $\mathbf{z}_i=\operatorname{softmax}(\tilde{\mathbf{z}}_i)$ we further ensure $\mathbf{z}_i\in \Delta^D$. Examples of two latent spaces where archetypes correspond and do not correspond to observations using these archetypes are displayed in Figure \ref{fig:POLytopes} with the gate function securing informative polytopes. Such a formulation is necessary since there is no guarantee that making the polytope matrix $\bm{A}$ a free parameter will lead to an informative latent space. This comes as a consequence of the fact that a large enough volume of the polytope $\bm{A}$ matrix can enclose an unconstrained \textsc{LDM} and avoid representing trade-offs that would force the nodes to use the corners.

 \begin{figure}[!t]
  \centering
    \subfloat[Non-informative polytope ]{{
  \includegraphics[width=0.4\columnwidth]{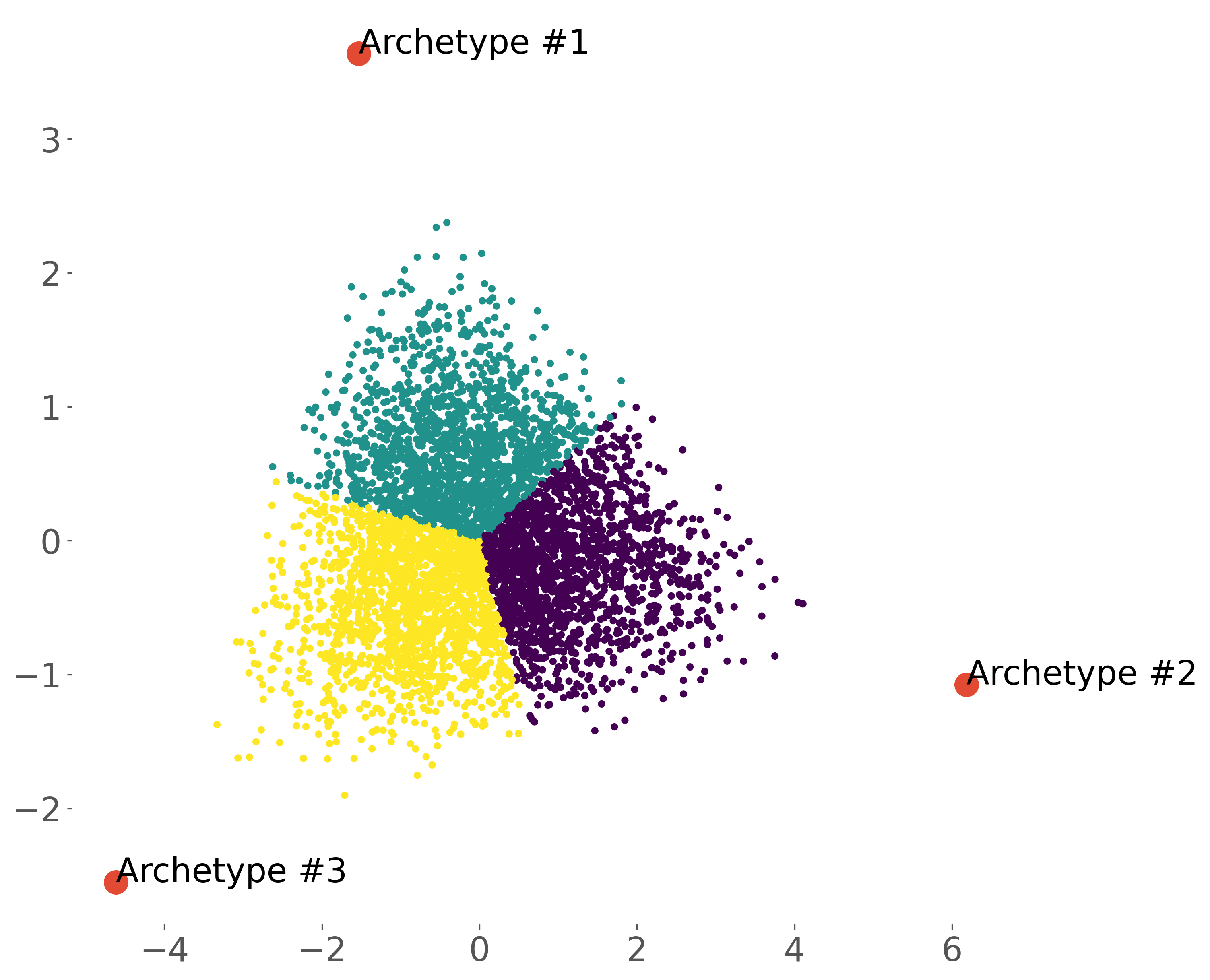} }}%
  \hfill
   \subfloat[Informative polytope]{{
  \includegraphics[width=0.4\columnwidth]{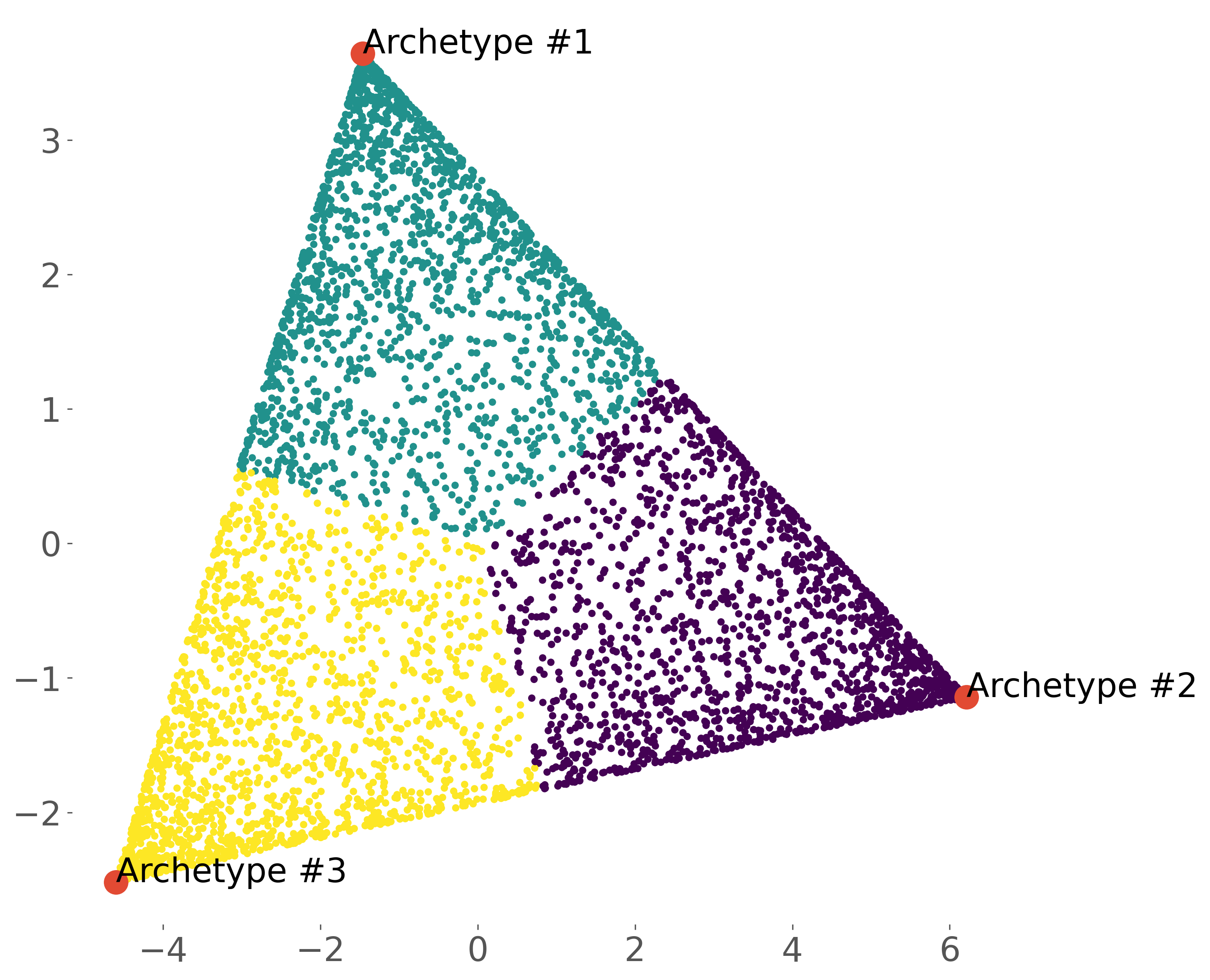} }}%
  \caption{Example of two $2$-dimensional polytopes projected into the first two principal components, defining a loose versus a tight latent space. Left panel: An example of a non-informative polytope where archetypes are not defined as observations belonging to the data, Right panel: Informative polytope where archetypes are defined as observations belonging to the data. Points are colored based on the archetype they express the maximum membership to.}\label{fig:POLytopes}
\end{figure}

Importantly, the loss function of Equation \eqref{eq:loss_sk} is adopted for the relational AA formulation forming the \textsc{SLIM}, with the prior regularization applied to the corners of the extracted polytope $\bm{A}=\mathbf{RZC}$ instead of the latent embeddings $\bm{Z}$ imposing a standard elementwise normal distribution as prior $a_{k,k^\prime}\sim \mathcal{N}(0,1)$. Furthermore, we impose a uniform Dirichlet prior on the columns of $\bm{Z}$, i.e. $(\bm{z}_i\sim Dir(\mathbf{1}_D)$, this only contributes constant terms to the joint distribution, and therefore the maximum a posteriori (MAP) optimization only constant terms. As a result, the loss function optimized is given by Equation \eqref{eq:loss_sk} replacing $\|\mathbf{Z}\|_F^2$ with $\|\boldsymbol{A}\|_F^2$.

\section{Directed Case Model Formulations}
As we have discussed, one of the most important properties of signed network models is the expression of balance theory which naturally describes directed relationships. In this section, we describe how our proposed frameworks can be extended to the study of directed networks (which at least for the \textsc{SLIM}  formulations is not trivial). We further explore additional model formulations allowing for more capacity and expressive power.

\subsection{The Skellam Latent Distance Model for the Directed Case (LDM)}
Our main purpose here is to learn two latent node representations $\{\mathbf{z}_i\}_{i\in\mathcal{V}}\in\mathbb{R}^{D}$ and $\{\mathbf{w}_i\}_{i\in\mathcal{V}}\in\mathbb{R}^{D}$ in a low dimensional space for a given directed signed network $\mathcal{G}=(\mathcal{V}, \mathcal{Y})$ ($D \ll |\mathcal{V}|$). The two sets of the latent embeddings correspond to modeling directed relationships $i\rightarrow j$ of nodes, with $\mathbf{z}_i$ the source node and $\mathbf{w}_j$ the target node, and vice-versa for an oppositely directed relationship $i\leftarrow j$. Similar to the main paper, we can formulate the negative log-likelihood of a latent distance model under the Skellam distribution as:
\begin{align*}
\mathcal{L}(\mathcal{Y}) &:=\log p(y_{ij}|\lambda^{+}_{ij},\lambda^{-}_{ij}) \\
&= \sum_{i,j}{(\lambda^{+}_{ij}+\lambda^{-}_{ij})} - \frac{y_{ij}}{2}\log\left(\frac{\lambda^{+}_{ij}}{\lambda^{-}_{ij}}\right)-\log\left(\mathcal{I}_{|y_{ij}|}\left(2\sqrt{\lambda^{+}_{ij}\lambda^{-}_{ij}}\right)\right),
\end{align*}
For the directed case, the Skellam distribution has two rate parameters as well, and we consider them to learn latent node representations $\{\mathbf{z}_i\}_{i\in\mathcal{V}}$ and $\{\mathbf{w}_j\}_{j\in\mathcal{V}}\in\mathbb{R}^{D}$ by defining them as follows:
\begin{align}
\lambda_{ij}^{+} &= \exp\big(\beta_{i} + \gamma_{j} - ||\mathbf{z}_i-\mathbf{w}_j||_2\big)\label{eq:rate1_dir},
\\
\lambda_{ij}^{-} &= \exp\big(\delta_{i} + \epsilon_{j} + ||\mathbf{z}_i-\mathbf{w}_j||_2\big),
\label{eq:rate2_dir}
\end{align}
where the set $\{\beta_{i},\gamma_i,\delta_i,\epsilon_i\}_{i\in\mathcal{V}}$ denote the node-specific random effect terms. More specifically, the sender $\beta_i$ and the receiver $\gamma_j$ random effects represent the "social" reach of a node and the tendency to form positive interactions, expressing positive degree heterogeneity (indicated by $+$ as a superscript of $\lambda$). In contrast, $\delta_i$ and $\epsilon_j$ provide the "anti-social" sender and receiver effect of a node to form negative connections, and thus model negative degree heterogeneity (indicated by $-$ as a superscript of $\lambda$). 

By imposing (as in the undirected case) standard normally distributed priors elementwise on all model parameters $\bm{\theta}=\{\bm{\beta},\boldsymbol{\gamma},\bm{\delta},\bm{\epsilon}, \mathbf{Z},\boldsymbol{W}\}$, i.e., $\theta_i\sim \mathcal{N}(0,1)$, We define a maximum a posteriori (MAP) estimation over the model parameters, via the loss function to be minimized (ignoring constant terms):
\begin{align}\label{eq:loss_sk1}
\begin{split}
    Loss =& \sum_{i,j}\Bigg( \lambda_{ij}^{+}+\lambda_{ij}^{-} -\frac{y_{ij}}{2}\log\left( \frac{\lambda_{ij}^{+} }{\lambda_{ij}^{-}}\right)\Bigg) - \sum_{i,j}\log I_{|y_{ij}|}\Big (2\sqrt{\lambda_{ij}^{+}\lambda_{ij}^{-}}\Big )
    \\
    +& \frac{\rho}{2}\Big(||\mathbf{Z}||_F^2+||\mathbf{W}||_F^2+||\bm{\gamma}||_F^2+||\bm{\beta}||_F^2+||\bm{\delta}||_F^2+||\bm{\epsilon}||_F^2\Big),
    \end{split}
\end{align}
where $||\cdot||_F$ denotes the Frobenius norm. In addition, $\rho$ is the regularization strength with $\rho=1$ yielding the adopted normal prior with zero mean and unit variance. 

\subsection{The Signed Relational Latent Distance Model for Directed Networks}
We formulate the relational AA in the context of the family of LDMs and for directed networks, as:
\begin{align}
    \lambda_{ij}^{+} &=\exp \big( \beta_{i} + \gamma_{j} - \|\mathbf{A} (\mathbf{z}_{i}-\mathbf{w}_{j})\|_{2}\big)
    \\ 
    &=\exp \big( \beta_{i} + \gamma_{j} -\|\mathbf{R}[\mathbf{Z};\mathbf{W}]\mathbf{C}(\mathbf{z}_{i}-\mathbf{w}_{j})\|_2\big).\label{LRPM_inensity_function_1}
    \\
\lambda_{ij}^{-} &=\exp \big( \delta_{i} + \epsilon_{j} + \|\mathbf{A} (\mathbf{z}_{i}-\mathbf{w}_{j})\|_{2}\big)
\\ 
&=\exp \big( \delta_{i} + \epsilon_{j} +\|\mathbf{R}[\mathbf{Z};\mathbf{W}]\mathbf{C}(\mathbf{z}_{i}-\mathbf{w}_{j})\|_2\big).\label{LRPM_inensity_function_2}
\end{align}

Notably, in the AA formulation for directed networks $\mathbf{X}=\mathbf{R}[\mathbf{Z};\mathbf{W}]$ corresponds to observations formed by the concatenations of the convex combinations $\mathbf{Z}$ and $\boldsymbol{W}$ of positions given by the columns of $\boldsymbol{R}^{D\times D}$. Similar to the undirected case, in order to ensure what is used to define archetypes $\mathbf{A}=\mathbf{XC}=\mathbf{R}[\mathbf{Z};\mathbf{W}]\mathbf{C}$ corresponds to observations using these archetypes in their reconstruction $[\mathbf{Z};\mathbf{W}]$,
we define $\boldsymbol{C}\in \boldsymbol{R}^{2N\times D}$ as a gated version of $[\mathbf{Z};\mathbf{W}]$ normalized to the simplex such that $\boldsymbol{c}_d\in\Delta^{2N}$ by defining 
\begin{equation}
    c_{nd}=\frac{([\mathbf{Z};\mathbf{W}]^\top\circ [\sigma(\mathbf{G})]^\top)_{nd}}{\sum_{n^\prime}([\mathbf{Z};\mathbf{W}]^\top\circ [\sigma(\mathbf{G})]^\top)_{n^\prime d}}.
\end{equation}
 As a result, the extracted archetypes are ensured to correspond to the nodes assigned the archetype, whereas the location of the archetypes can be flexibly placed in space as defined by $\mathbf{R}$. By defining $\mathbf{z}_i=\operatorname{softmax}(\tilde{\mathbf{z}}_i)$ and $\boldsymbol{w}_i=\operatorname{softmax}(\tilde{\boldsymbol{w}}_i)$ we further ensure $\mathbf{z}_i,\boldsymbol{w}_i\in \Delta^K$.

As in the undirected case, the loss function of Equation \eqref{eq:loss_sk1} is adopted for the relational AA formulation forming the \textsc{SLIM}, with the prior regularization applied to the corners of the extracted polytope $\bm{A}=\mathbf{R}[\mathbf{Z};\mathbf{W}]\mathbf{C}$ instead of the latent embeddings $\bm{Z},\bm{W}$ imposing a standard elementwise normal distribution as prior $a_{k,k^\prime}\sim \mathcal{N}(0,1)$. Furthermore, we impose a uniform Dirichlet prior on the columns of $\bm{Z},\bm{W}$, i.e. $(\bm{z}_i,\bm{w}_i\sim Dir(\mathbf{1}_K)$, this only contributes constant terms to the joint distribution. As a result, the loss function is given by Equation \eqref{eq:loss_sk1} replacing $\|\mathbf{Z}\|_F^2$ and $\|\boldsymbol{W}\|_F^2$ with $\|\boldsymbol{A}\|_F^2$ for the maximum a posteriori (MAP) optimization.

\subsection{Model Extensions for Additional Capacity}\label{model_extension}
Directed relationships usually require additional expressive capability than in the case of modeling unindicted relationships. For that, we will briefly discuss alternative model formulations, yielding different distances for the positive and negative rates to define additional expressive capability (as opposed to the standard model version where latent distances were shared across rates). We consider a formulation such as setting the Skellam rates as, $\lambda_{ij}^+=\exp(\beta_i+\gamma_j-||\mathbf{z}_i-\mathbf{w}_j||_2)$
and $\lambda_{ij}^-=\exp(\delta_i+\epsilon_j-||\boldsymbol{u}_i-\boldsymbol{w}_j||_2)$. Under this assumption, a positive directed relationship $(i \rightarrow j)$ shows that node $i$ "likes" node $j$ and "dislikes" node $j$ if it is negative. The latent embedding  $\boldsymbol{w}_j$ is then the receiver position for the "likes" and "dislikes" with embeddings $\mathbf{z}_i$ and $\boldsymbol{u}_i$ being the sender positions for positive and negative relationships, respectively. In this case, we introduce three latent embeddings instead of the conventional two for the undirected case. The disparity of location $\mathbf{z}_i$ and $\boldsymbol{u}_i$ here can point out how polarity is formed between the two regions of the latent space. This model specification introduces an additional regularization for the third embedding matrix $\mathbf{U}$ in the loss function of Equation \eqref{eq:loss_sk1}. For the RAA case, we thereby define $\mathbf{X}=\mathbf{R}[\mathbf{Z};\mathbf{U};\mathbf{W}]$, i.e., as the concatenation of all three latent positions and with $\boldsymbol{C}\in \boldsymbol{R}^{3N\times D}$.

\section{The Signed Hybrid-Membership Latent Distance Model}
In the previous chapter, we extended \textsc{LDM}s to the study of signed networks while characterizing network polarization via the use of Archetypal Analysis and the Skellam distribution. 

Whereas in \textsc{SLIM}  the network representations were constrained to the convex hull as defined by the inferred representations, we briefly discussed additional modeling direction for the discovery of pure/ideal forms based on Minimal Volume (MV) approaches. More formally, such approaches can be defined as
\begin{eqnarray}
\mathbf{X}\approx \mathbf{AZ}\quad \text{s.t. } vol(\mathbf{A})= v \text{ and } \boldsymbol{z}_j\in \Delta^{D},
\end{eqnarray}
where $\mathbf{A} \in \mathbb{R}^{(D+1)\times (D+1)}$ is the matrix describing the archetypes (extreme points of the convex hull) of the latent space, and $vol(\mathbf{A})$ is the volume of matrix $\mathbf{A}$ which can be expressed through the determinant as $|det(\mathbf{A})|$ when $\mathbf{A}$ is a square matrix \cite{hart2015inferring,zhuang2019regularization}. A main advantage is that the extraction of distinct aspects/profiles through MV does not require the existence of ``pure'' observations defining the convex-hull or else the extracted polytope/simplex. As the volume decreases, observations are naturally ``forced'' to populate the corners of the polytope, yielding archetypal characterization when the reconstruction of data is defined through convex combinations of these corners.

The principal disadvantage of MV procedures lies in the meticulous requirement for regularization tuning to delineate volumes that both guarantee identifiability and retain sufficient capacity to represent the data with minimal reconstruction error \cite{zhuang2019regularization}. Furthermore, analytical and tractable computation of the volume of polytopes requires calculating the sum of determinants for all simplexes used to construct the inferred polytope \cite{bueler2000exact}. This is computationally expensive (especially in high dimensions) and sometimes unstable when $\mathbf{A}$ comes close to singular.

 In this chapter, we constrain the columns of matrix $\mathbf{A}$ to the $D$-simplex with length $\delta$. Thus, by controlling the volume of $\mathbf{A}$, we essentially define a constrained-to-simplexes MV approach. Calculating the volume for the $D-$simplex with length $\delta$ is straightforward and computationally efficient. Rather than including regularization over the volume of $\mathbf{A}$ in the loss function during inference, we deterministically control the simplex length $\delta$ which is given as an input to the model and is gradually decreased until uniqueness guarantees are obtained. Volume minimization can be obtained trivially by decreasing $\delta$. Such a procedure gives us explicit control over the model capacity by fixing the volume which is harder to obtain with classical MV approaches where the volume expression is inserted in the loss function.
 
 Essentially, by defining $\mathbf{A}$ as $\mathbf{A}=\delta\cdot \mathbf{I}$, with $\mathbf{I}$ being the $(D+1)\times (D+1)$ identity matrix, we obtain as a special case of archetypal analysis under a constrained MV formulation. In addition, if every corner of the introduced simplex is populated by at least one node champion we obtain unique representations defining hybrid memberships.

 We now introduce the \textsc{s}igned \textsc{H}ybrid-\textsc{M}embership \textsc{L}atent \textsc{D}istance \textsc{M}odel (\textsc{sHM-LDM}).
 The \textsc{sHM-LDM} is able to analyse signed networks, and similar to \cite{slim} it introduces two Skellam rate parameters as:
\begin{align}
\lambda_{ij}^{+} &= \exp\big(\beta_{i} + \beta_{j} - \delta^p||\mathbf{z}_i-\mathbf{z}_j||_2^p\big)\label{eq:rate1_sh},
\\
\lambda_{ij}^{-} &= \exp\big(\psi_{i} + \psi_{j} + \delta^p||\mathbf{z}_i-\mathbf{z}_j||_2^p\big),
\label{eq:rate2_sh}
\end{align}
where again $\mathbf{z_i} \in [0,1]^{D+1}$ and $\sum_{d=1}^{D+1} z_{id}=1$, $\delta \in \mathbb{R}_+$ and $\beta_i,\psi_j \in \mathbb{R}$ denote the node-specific random-effects. As explained in Section \ref{lab:random_Effects}, $\beta_i,\beta_j$ express positive degree heterogeneity while $\psi_i,\psi_j$ models negative degree heterogeneity. The norm degree $p \in \{1,2\}$ controls the power of the $\ell^2$-norm, and thus the model specification, as in the unsigned case. 

As in \cite{slim}, we define a maximum-a-posteriori (MAP) estimation, utilizing the Skellam likelihood over the adjacency matrix $\mathbf{Y}$ of the network $\mathcal{G}=(\mathcal{V}, \mathcal{E})$. We conditionally assume an independent likelihood given the unobserved latent positions and random effects. The corresponding loss function excluding constant terms is:
\begin{equation}\label{eq:loss_sk2}
    L = \sum_{i<j}\Bigg( \lambda_{ij}^{+}+\lambda_{ij}^{-} -\frac{y_{ij}}{2}\log\left( \frac{\lambda_{ij}^{+} }{\lambda_{ij}^{-}}\right)\Bigg)  - \sum_{i<j}\log I_{|y_{ij}|}\Big (2\sqrt{\lambda_{ij}^{+}\lambda_{ij}^{-}}\Big )+ \frac{\rho}{2}\Big(||\bm{\beta}||_F^2+||\bm{\psi}||_F^2\Big),
\end{equation}
where $\mathcal{I}_{|y|}$ is the modified Bessel function of the first kind and order $|y|$, $||\cdot||_F$ denotes the Frobenius norm. In addition, $\rho$ is the regularization strength where $\rho=1$ is assumed throughout this paper yielding a normal prior with zero mean and unit variance for the random effects. For the latent positions, we assume a uniform Dirichlet distribution as a prior which only adds a constant term in Equation \ref{eq:loss_sk2} and thus is excluded.

Choosing the case where $p=2$, meaning that the \textsc{sHM-LDM} utilizes the squared Euclidean norm, we are able once more to relate the model to an Eigenmodel by creating the following reparameterizations. For the rate responsible for positive interactions $\{\lambda_{ij}^+\}$ as: $ \tilde{\beta}_i+\tilde{\beta}_j+(\mathbf{\tilde{w}}_i\bm{\Lambda}\mathbf{\tilde{w}}_j^{\top})$ where $\bm{\Lambda}$ is a diagonal matrix having non-negative elements, i.e. $\tilde{\beta}_i=\beta_i-\delta^2\cdot||\mathbf{w}_i||^2_2$, $\tilde{\beta}_j=\beta_j-\delta^2\cdot||\mathbf{w}_j||^2_2$ and $\tilde{\mathbf{w}}_i\bm{\Lambda}\tilde{\mathbf{w}}_j^\top=2\delta^2\cdot \mathbf{w}_i\mathbf{w}_j^\top$. Similarly, for the rate responsible for expressing animosity $\{\lambda_{ij}^-\}$ as: $ \tilde{\psi}_i+\tilde{\psi}_j+(\mathbf{\tilde{w}}_i\bm{\Lambda}\mathbf{\tilde{w}}_j^{\top})$ where $\bm{\Lambda}$ is a diagonal matrix having non-positive elements, i.e. $\tilde{\psi}_i=\psi_i-\delta^2\cdot||\mathbf{w}_i||^2_2$, $\tilde{\psi}_j=\psi_j-\delta^2\cdot||\mathbf{w}_j||^2_2$ and $\tilde{\mathbf{w}}_i\bm{\Lambda}\tilde{\mathbf{w}}_j^\top=2\delta^2\cdot \mathbf{w}_i\mathbf{w}_j^\top$. We witness that homophily in the case of \textsc{sHM-LDM} is expressed through a non-negative Eigenmodel (as in the unsigned case) while animosity/heterophily is expressed through a non-positive Eigenmodel able to express stochastic equivalence \cite{hoff2007modeling}. These two formulations admit the same embedding matrix $\mathbf{W}$ which balances the expression of ``opposing'' forces (homophily and animosity) in the latent space. Lastly, for $p=2$ both expressions admit to an NMF operation, obtaining an identifiable and unique factorization (up to permutation invariance) when $\mathbf{W}$ is full-rank and at least one node resides solely in each simplex corner \cite{nmf4,nmf5} as in the case of \textsc{HM-LDM} for unsigned networks.

\begin{figure}[!t]
\centering
\includegraphics[scale=.2]{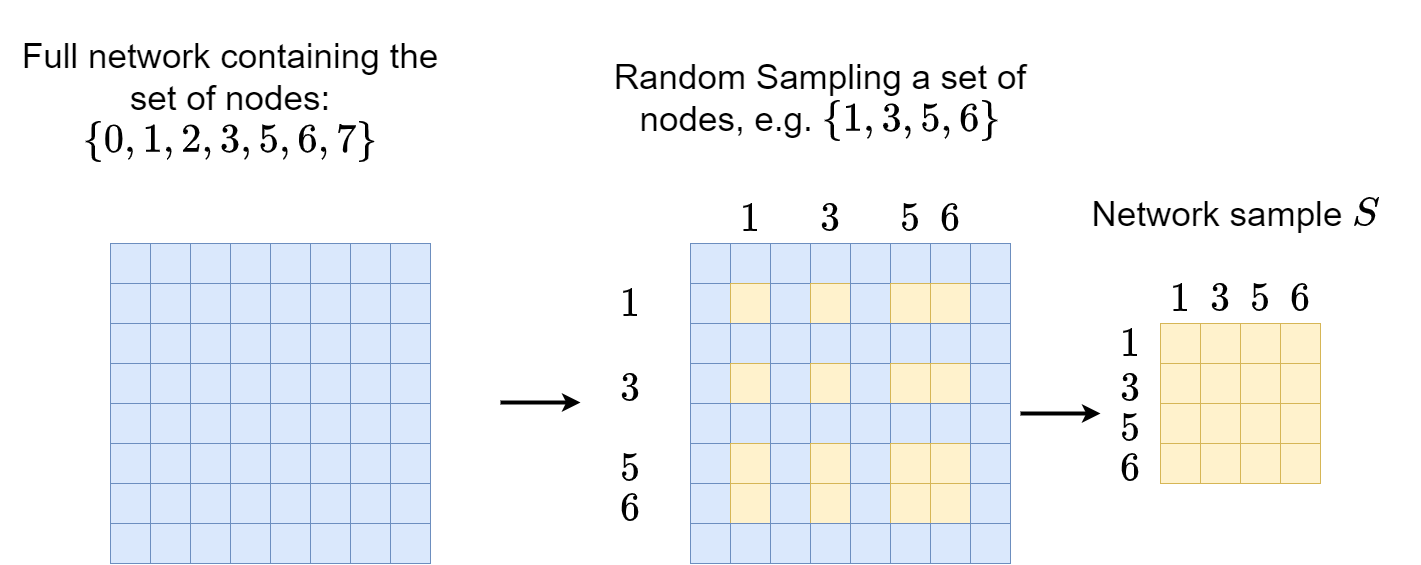}
\caption{A random sampling procedure over an undirected network with eight total nodes and a sample size of four nodes. The node sample set defines a network block sample defining an $\mathcal{O}(S^2)$ space and time complexity.}
\label{fig:rs}
\end{figure}

\section{Complexity analysis.} 

Modern graphs can potentially contain millions of nodes, with even billion-scale networks becoming already common in the real world. As a direct consequence, the computational scaling of \textsc{GRL} models is of vital importance. All of the proposed methods of this thesis, at their core, are distance models and thus they scale prohibitively as $\mathcal{O}(N^2)$ since the node pairwise distance matrix needs to be computed. This does not allow the analysis of large-scale networks. In Section \ref{methods_hbdm}, we showed how we can successfully scale LDMs while characterizing for structure at multiple scales, defining a linearithmic $\mathcal{O}(N\log N)$ space and time complexity. The \textsc{HBDM} methodology naturally extends to all of the proposed models in this chapter. Nevertheless, we chose to scale the rest of the models by adopting an unbiased estimation of the log-likelihood through random sampling \cite{RS__} (it is an unbiased estimator since every node is sampled with an equal probability). More specifically, gradient steps are based on the log-likelihood of the block formed by a sampled (per iteration and with replacement) set $S$ of network nodes, as: 

 $$ \log P(Y|\bm{\lambda}))=\underbrace{\sum_{i<j:y_{ij}=1}\log(\lambda_{ij})}_{\substack{\text{Link Term $\mathcal{O}(S)$} }}\;-\underbrace{\sum_{i< j}\lambda_{ij}}_{\substack{\text{Non-Link Term $\mathcal{O}(S^2)$}}} \text{ with } \: i,j\in S.$$
 
This makes inference scalable defining an $\mathcal{O}(S^2)$ space and time complexity allowing for the analysis of large-scale networks. A toy example of a random sampling procedure is provided in Figure \ref{fig:rs}. More options for scalable inference of distance models have also been proposed in \cite{case_control}.

\section{The Single-Event Poisson Process}\label{sec:SE-PP}

In many studies, various real networks are represented with static structures, not taking advantage of the rich temporal information they may offer. In such a direction, researchers have considered the analysis of temporal networks both in discrete \cite{discrete_1,discrete_3,discrete_4,discrete_5,discrete_6}, as well as, continuous \cite{past5, hawkes_1,hawkes_2,hawkes_3,fan2021continuous} time settings. Furthermore, important network types, with the prominent example of citation networks, are characterized by a temporal structure with links between a pair of nodes occurring maximum once throughout the time horizon. Such networks have traditionally been studied as static \cite{deepwalk-perozzi14,line,node2vec-kdd16,hbdm,hmldm}. Contrary to such practices, we here introduce a framework utilizing a new likelihood formation under a Single-Event Poisson Process, capable of analyzing single-event networks, capitalizing on the rich temporal information that static models are blind to. In this regard, we assume that the studied temporal networks are composed maximally of a single event between a node pair (dyad), and once an event between two nodes has occurred no more events are admissible between these two nodes, see also Figure \ref{fig:sen} (right panel).

Before presenting our modeling strategy for the links of networks, we will first establish the notations used throughout the sections referring to single-event networks. We utilize the conventional symbol, $\mathcal{G}=(\mathcal{V},\mathcal{E})$, to denote a directed Singe-Event-Network over the timeline $[0,T]$ where $\mathcal{V}=\{1,\ldots,N\}$ is the vertex and $\mathcal{E} \subseteq \mathcal{V}^2\times[0,T]$ is the edge set such that each node pair has at most one link. Hence, a tuple, $(i,j,t_{ij}) \in\mathcal{E}$, shows a directed event (i.e., instantaneous link) from source node $j$ to target $i$ at time $t_{ij}\in[T]$, and there can be at most one $(i,j,t_{ij})$ element for each $(i,j)\in\mathcal{V^2}$ and some $t_{ij}\in[0,T]$. 

We always assume that the timeline starts at $0$ and the last time point is $T$, and we represent the interval by symbol, $[T]$. We employ $t_1\leq t_2\leq\cdots \leq t_N$  to indicate the appearance times of the corresponding nodes $1,2,\ldots,N\in\mathcal{V}$, and we suppose that node labels are sorted with respect to their incoming edge times. In other words, if $i < j$, then we know that there is a node $k\in\mathcal{V}$ such that $t_{ik} \leq t_{jl}$ for all $l \in \mathcal{V}$.

\subsection{Inhomogenous Poisson Point Process}
The inhomogeneous Poisson Point Process (\textsc{IPP}) has been a prominent method for modeling the number of events occurring between nodes at different times throughout the study period of the temporal network \cite{GONZALEZ2016505}. Such a process defines an event intensity yielding the Poisson process rate function which represents the average event density. The probability of sampling $m$ event points in a time interval $[T]$ is given by as,

\begin{align}\label{eq:ipp_sampling_prob}
p_{N}(M(T)=m) := \frac{ \left[\Lambda(0,T)\right]^m }{m!}\exp(-\Lambda(0,T)),
\end{align}

where $M(T)$ is the random variable showing the number of events occurring over the interval $[T]$, and $\Lambda(T) := \int_{0}^{T}\lambda(t^{\prime})\mathrm{d}t^{\prime}$ for the intensity function $\lambda:[T]\rightarrow\mathbb{R}^+$ (Please visit \cite{streit2010poisson} for an overview). We point here once again, to earlier studies \cite{6349745, hbdm} which have demonstrated that adopting the Poisson likelihood for modeling binary relationships does not degrade the methods' predictive performance.

We now focus on SENs and more specifically on citation networks, for which we employ an \textsc{IPP} for characterizing the occurrence time of a link (i.e., a single event point indicating the publication date and thus the citation time). This is unlike conventional practice in \textsc{IPP} literature which concentrates on modeling the occurrence of an arbitrary number of events between a pair of nodes. Consequently, we assume that a pair can have at most one edge (i.e., link), and we discretize the probability of sampling $m$ events given in Equation \eqref{eq:ipp_sampling_prob} as having either one event or no event cases. More formally, by applying Bayes' rule, we can write it as a conditional distribution of $M(t)$ being equal to $m \in \{0,1\}$ as follows:
\begin{align}
p_{M|M\leq 1}\left( M(T) = m \right) &= \frac{ p_{M,M\leq 1}\left( M(T) = m, M(T) \leq 1 \right) }{ p_{M\leq 1}\left( M(T) \leq 1 \right) } \nonumber\\ &= \frac{ p_M\left( M(T) = m \right) }{ p_M\left( M(T) = 0 \right) + p_{M}\left( M(T) = 1 \right) }\nonumber
\\
&= \frac{ \exp\left(-\Lambda(T)\right)\left[\Lambda(T)\right]^m }{ \exp(-\Lambda(T)) + \exp(-\Lambda(T))\Lambda(T) }
\end{align}
The conditional probability of a single-event occurrence under the proposed \textit{Single-Event Poisson Process} is given by:
\begin{align}\label{eq:sepp_probability}
p_{M|M\leq 1}\left( M(T) = 1 \right) = \frac{\Lambda(T)}{1 + \Lambda(T)}.
\end{align}

Let now $(Y,\Theta)$ be a random variable where $Y$ shows whether a link occurred and $\Theta$ indicates the time of the link occurrence. Then together with Eq \eqref{eq:sepp_probability}, we can write the likelihood of $(Y,\Theta)$ evaluated at $(1, t^{*})$ as follows:
\begin{align}\label{eq:prob_single_event}
{p_{Y,\Theta}\left(1,t^{*}\right)} &= p_Y\left\{ {Y = 1} \right\} p_{\Theta|Y}\left\{ \Theta=t^{*} | Y=1  \right\} \nonumber\\ &= \left(\frac{\Lambda(T)}{1+\lambda(T)}\right) \left(\frac{ \Lambda(t^{*}) }{ \Lambda(T) }\right) = \frac{ \lambda(t^{*}) }{1+\Lambda(T)}
\end{align}

Consequently, the log-likelihood of the whole network, assuming that each dyad follows the Single-Event Poisson Process, can be written as:
\begin{align}\label{eq:ipp_likelihood}
\mathcal{L}_{SE-PP}(\Omega) := \log p( \mathcal{G} | \Omega ) = \sum_{1\leq i,j\leq N}\Big( y_{ij}\log\lambda(t_{ij}) - \log\big( 1 + \Lambda_{ij}(t_i,T) \big) \Big)
\end{align}
where $\Omega$ is the model hyper-parameters and $\Lambda_{ij}(t_i,T) := \int_{t_i}^T\lambda_{ij}(t^{\prime})\mathrm{d}t^{\prime}$. Note that for a homogenous Poisson process with constant intensity $\lambda_{ij}$ for each node pair $i$ and $j$, the probability of having an event throughout the timeline is equal to $\Lambda_{ij}(T) / (1 + \Lambda_{ij}(T)) = T\lambda_{ij} / (1 + T\lambda_{ij})$ by Equation \eqref{eq:sepp_probability}. In this regard, the objective function stated in Equation \eqref{eq:ipp_likelihood} is equivalent to a static Bernoulli model \cite{exp1}:
\begin{align}\label{eq:bern_likelihood}
\mathcal{L}_{Bern}(\Omega) := \log p(\mathcal{G}|\Omega) &= \sum_{\substack{i,j\in\mathcal{V}}}\Big(y_{ij}\log{(\Tilde{\lambda}_{ij})}- \log\left(1+\!\Tilde{\lambda}_{ij}\right) \Big),
\end{align}

where we have used the re-parameterization $T\lambda_{ij}=\Tilde{\lambda}_{ij}$.

\section{Dynamic Impact Characterization}
In the realm of impact analysis and risk assessment, characterizing dynamic events is pivotal in understanding and managing potential consequences. We know that papers generally undergo the process of aging over time since novel works introduce more original concepts. In this regard, we model the distribution of the impact of a paper $\{i\}$ by the \textsc{Truncated} normal distribution:

\begin{align}
f_i(t) = \frac{1}{\sigma}\frac{\phi(\frac{t-\mu}{\sigma})}{\Phi(\frac{\kappa-\mu}{\sigma})-\Phi(\frac{\rho-\mu}{\sigma})}
\end{align}

where $\mu$ and $\sigma$ are the parameters of the distribution which lie in $(\rho,\kappa) \in \mathbb{R}$, $\phi(x)=\frac{1}{\sqrt{2\pi}}\exp{(-\frac{1}{2}x^2)}$, and $\Phi(\cdot)$ is the cumulative distribution function $\Phi(x)=\frac{1}{2}\Big(1+ \text{erf}(\frac{x}{\sqrt{2}})\Big)$.
  In addition, as an alternative impact function, and similar to \cite{ts5}, we consider the \textsc{Log Normal} distribution:
 \begin{align}
f_i(t) = \frac{1}{t\sigma\sqrt{2\pi}}\exp\left( - \frac{ \ln{(t-\mu)^2} }{ 2\sigma^2 } \right)
\end{align}
where $\mu$ and $\sigma$ are the parameters of the distribution. Such distributions are particularly valuable for capturing the inherent variability and asymmetry in the lifecycle of a paper.

\section{Single-Event Network Embedding by the Latent Distance Model}
Our main purpose is to represent every node of a given single-event network in a low $D$-dimensional latent space ($D \ll N$) in which the pairwise distances in the embedding space should reflect various structural properties of the network, like homophily and transitivity \cite{hbdm}. For instance, in the \textit{Latent Distance Model} \cite{exp1}, one of the pioneering works, the probability of a link between a pair of nodes depended on the log-odds expression, $\gamma_{ij}$, as $\alpha - \| \mathbf{z}_i -\mathbf{z}_j \|_2$ where $\{\mathbf{z}_i\}_{i\in\mathcal{V}}$ are the node embeddings, and $\alpha\in\mathbb{R}$ is the global bias term responsible for capturing the global information in the network. It has been proposed for undirected graphs but can be extended for directed networks as well by simply introducing another node representation vector $\{\mathbf{w}_i\}_{i\in\mathcal{V}}$ in order to differentiate the roles of the node as source (i.e., sender) and target (i.e., receiver). By the further inclusion of two sets of random effects $\{\alpha_i,\beta_j\}$ describing the in and out degree heterogeneity, respectively, we can define the log-odds expression as:
\begin{align}\label{eq:natural_param_defn}
\gamma_{ij}=\alpha_i+\beta_j - \| \mathbf{z}_i - \mathbf{w}_j \|_2
\end{align}

We can now combine a dynamic impact characterization function with the \textit{Latent Distance Model}, to obtain an expression for the intensity function of the proposed \textit{Single-Event Poisson Process}, as:
\begin{equation}
\label{eq_ldm_in}
    \lambda_{ij}(t_{ij})=\frac{f_i(t_{ij})\exp{\alpha_i}\exp{\beta_j}}{\exp{\| \mathbf{z}_i - \mathbf{w}_j \|_2}}.
\end{equation}

Combining the intensity function of Equation \eqref{eq_ldm_in} with the log-likelihood expression of Equation \eqref{eq:ipp_likelihood} yields the Dynamic Impact Single-Event Embedding Model (\textsc{DISEE}) model. Under such a formulation, we exploit the time information data indicating when links occur through time, so we can grasp a more detailed understanding of the evolution of networks, generate enriched node representations, and quantify a node's temporal impact on the network.

\subsection{Case-Control Inference}
With \textsc{DISEE} being a distance model, it scales prohibitively as $\mathcal{O}(N^2)$ since the all-pairs distance matrix needs to be calculated. In order to scale the analysis to large-scale networks we adopt an unbiased estimation of the log-likelihood similar to a case-control approach \cite{case_control}. In our formulation, we calculate the log-likelihood as:
 \begin{align}\label{eq:single_event_poisson_likelihood_case_control}
\log p_{ij}(\mathcal{G}|\Omega) \!&=\!\!\!\! \sum_{\substack{j: y_{ij}=1}}\!\left(y_{ij}\log\left(\lambda_{ij}(t_{ij}^*) \right)- \log\left(1\!+\!\!\int_{t_i}^{T}\!\!\!\!\lambda_{ij}(t^{\prime}) dt^{\prime}\right)\right)\nonumber\\ &+\sum_{j:y_{ij}=0}-\log\left(1\!+\!\!\int_{t_i}^{T}\!\!\!\!\lambda_{ij}(t^{\prime}) dt^{\prime}\right)\nonumber\\ &=l_{1}+l_{0}
\end{align}

As already mentioned, large networks are usually sparse so the link (case) likelihood contribution term $l_{1}$ can be calculated analytically, even for massive networks. The non-link (control) likelihood contribution term $l_{0}$ has a quadratic complexity $\mathcal{O}(N^2)$ in terms of the size of the network, and thus its computation is infeasible. For that, we introduce an unbiased estimator for $l_{i,0}$ which is regarded as a population total statistic \cite{case_control}. We estimate the non-link contribution of a node $\{i\}$ via:
\begin{equation}
    l_{i,0}=\frac{N_{i,0}}{n_{i,0}}\!\!\sum_{k=1}^{n_{i,0}}-\log\left(1\!+\!\!\int_{t_i}^{T}\!\!\!\!\lambda_{ik}(t^{\prime}) dt^{\prime}\right),
\end{equation}
where $N_{i,0}$ is the number of total non-links (controls) for node $\{i\}$, and $n_{i,0}$ is the number of samples to be used for the estimation. We set the number of samples based on the node degrees as $n_{i,0}=5*\text{degree}_i$. This makes inference scalable defining an $\mathcal{O}(cE)$ space and time complexity.
\part{Graph Representation Learning of positive integer weighted
networks}
\chapter{A Hierarchical Block Distance Model for Ultra Low-Dimensional
Graph Representations}
Two of the most important properties that are found in graphs and especially in social networks are homophily and transitivity, as introduced in Chapter \ref{trans_hom} of this thesis. The ultimate goal of Graph Representation Learning is to find mappings that define latent spaces where a graph is to be projected on. In such a space, closely related or connected nodes of a graph should be positioned in close proximity, in terms of their distance in the latent space. Additional properties of a successful and powerful method for \textsc{GRL} are firstly the scalability of the method, meaning that a model should offer competitive space and time computational complexities, and secondly, the ability to characterize the structure of networks that usually emerge at multiple scales. We here motivate our work, arguing that a model supporting such properties can potentially provide expressive and multi-purpose embeddings that can help us investigate the latent structures and perform downstream tasks on massive graphs.

For such a direction we turn to Latent Space Models, and more specifically to Latent Distance Models where the use of the Euclidean distance for the construction of the latent space of the network naturally conveys the main motivation of Graph Representation Learning. This comes as a direct consequence of the  Euclidean metric choice, naturally representing homophily, transitivity, and high-order nodal proximity. Unfortunately, the latter two properties, of scalability and structure characterization, are not expressed by a classical Latent distance model, as it scales prohibitively as $\mathcal{O}(N^2)$ both in space and time complexity since it requires the computation of the all-pairs Euclidean distance matrix. Infusing the Latent distance models with the ability to account for hierarchical representations, as well as, define complexities allowing for the analysis of modern and large-scale graphs will be the goal of this chapter. Importantly, we will exploit hierarchical representations in order to define a linearithmic (in terms of network nodes) space and time complexity, forming the Hierarchical Block Distance Model (\textsc{HBDM}) \cite{hbdm}.

 \begin{figure}[!b]
    \centering
    \includegraphics[width=0.9\columnwidth]{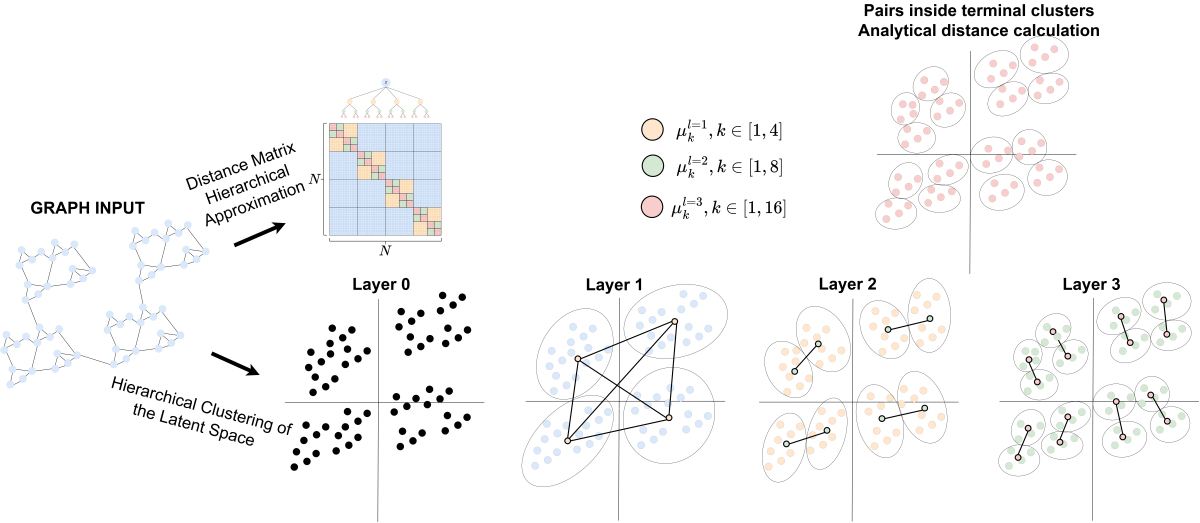}
    \caption{Hierarchical Block Distance Model procedure overview for a small network containing $N=64$ number of nodes. Given a graph as input the model defines a divisive clustering of the latent space appropriate for a hierarchical approximation of the all-pairs distance matrix. Layer 1 defines the first divisive step, splitting the network embeddings into $K=\log N=\log 64 \approx 4$ clusters and the defined centroids $\bm{\mu}_k^{l=1}$ are used to approximate the node pairs belonging to different clusters (pairs inside the blue blocks of the displayed distance matrix). Then, binary splits are defined until each cluster contains a maximum of $\log N=\log 64 \approx 4$ points. Centroids of Layer 2 and 3 are used to approximate pair distances belonging only to the opposing cluster (the cluster that has the same parent cluster) as denoted with the yellow and green blocks of the displayed distance matrix. Distances for pairs inside the clusters of Layer 3 are calculated analytically and the clustering procedure terminates.}
    \label{fig:hbdm_}
\end{figure}

\section{Contributions}
We reconcile hierarchical block structures emerging at multiple scales, scalability, and network properties such as homophily and transitivity through a new Graph Representation Learning approach, namely the Hierarchical Block Distance Model. This comes through the hierarchical approximation of the all-pairs Euclidean distance matrix that the \textsc{LDM} defines via a novel divisive Euclidean k-means algorithm. The procedure overview is provided in Figure \ref{fig:hbdm_}. Analytically our contributions are outlined as:

\begin{itemize}
    \item We combine embedding and hierarchical characterizations for Graph Representation Learning, imposing a hierarchical block structure akin to stochastic block modeling (SBM) but explicitly accounting for homophily and transitivity properties throughout the inferred hierarchy.

    \item We design a hierarchical approximation of the the all-pairs Euclidean distance matrix admitting a linearithmic total time and space complexity, in terms of the number of nodes in the network (i.e., $\mathcal{O}(N \log N)$). Moreover, our proposed procedure is importance-aware meaning that the distance approximation becomes more accurate the more similar two nodes are.

     \item We define a novel objective function for an Euclidean k-means clustering algorithm, utilizing an auxiliary function framework to form an alternate least-squares convex optimization problem.

    \item We generalize the method for bipartite networks where multi-scale geometric representations, joint hierarchical structures, and community discovery are arduous tasks.

    \item We extensively and for the first time benchmark Latent Distance models against state-of-the-art \textsc{GRL} baselines and large-scale networks. 

\end{itemize}

\section{Experimental design, results, and key findings}
We adopt an extensive experimental evaluation framework that includes eleven prominent \textsc{GRL} and thirteen moderate-sized and large-scale networks, containing networks with more than one million nodes. We then establish the performance of our model in terms of multiple downstream tasks that include link prediction, node classification, network completion, and visualizations of both unipartite and bipartite networks. Furthermore, we make the downstream tasks setting even more challenging by constraining the embeddings to ultra-dimensions of a maximum of eight dimensions. Finally, we train the model by minimizing the negative log-likelihood via the Adam \cite{kingma2017adam} optimizer.

The obtained results for the tasks of link prediction, network completion, and node classification showcase the favorable performance of \textsc{HBDM} against all baselines where in most cases the model significantly outperforms most baselines or defines on-par performance against the most competitive ones. Surprisingly, such a performance is achieved while using ultra-low-dimensional embeddings while we observe performance saturation when we reach $D=8$ dimensions. Our results further highlight that the inferred hierarchical organization can facilitate accurate visualization of network structure even when using only $D=2$ dimensional representations. Additionally, we show how our proposed framework extends the hierarchical multi-resolution structure to bipartite networks and provides the characterization of communities at multiple scales, with superior performance in the task of link prediction. For two visualization examples, please visit Figure \ref{fig:amazon_dend} where a product co-purchase unipartite Amazon network is provided, and Figure \ref{fig:github_dend} where a Github user-product bipartite network is shown. An extensive study on the sensitivity of the three total hyperparameters of \textsc{HBDM} (dimension size, learning rate, and number of iterations) showed robust performance of the proposed frameworks on the downstream tasks. Importantly, the scalability of the model was studied both theoretically (in terms of the Big $\mathcal{O}$ notation) and experimentally, verifying the desired linearithmic space and time complexity. (For more details and the full experiment results please visit the full paper \cite{hbdm}.)

\section{Conclusion}
Overall, the use of an Euclidean distance metric for projecting complex networks into a latent space leads to high expressive capabilities even when using ultra-low dimensions. This allows for high network compression without a significant loss when performing multiple downstream tasks. The hierarchical approximation and extension of the \textsc{LDM} respected the so-desired properties of homophily and transitivity which allowed for high performance in downstream tasks. This came as an additional benefit to the scaling of \textsc{LDM} where we successfully exchanged a quadratic space and time complexity for a linearithmic one, allowing for scaling the analysis to large networks. The importance of accounting for multi-scale structures in complex networks was evident throughout the experiments where different resolution levels of the hierarchy led to different network communities and characterizations. All of these properties are generalized to bipartite networks successfully introducing multi-scale geometric representations, community discovery, and high downstream task performance even with ultra-low dimensions. Lastly, we considered multiple downstream tasks in each of which various baselines were found to be competitive against our \textsc{HBDM} frameworks. In general, the \textsc{HBDM} is characterized by the most consistent performance across tasks, making it state-of-the-art.

\begin{figure*}[t]
  \centering
  \subfloat[Dendrogram]{{ \includegraphics[scale=0.05]{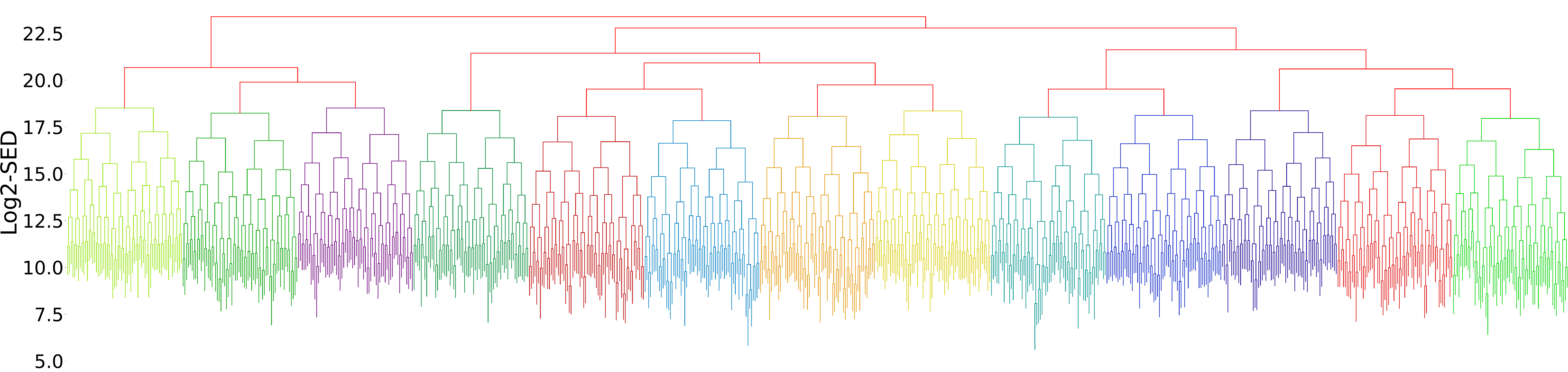} }}%
  \hfill
   \subfloat[Embedding Space]{{ \includegraphics[scale=0.15]{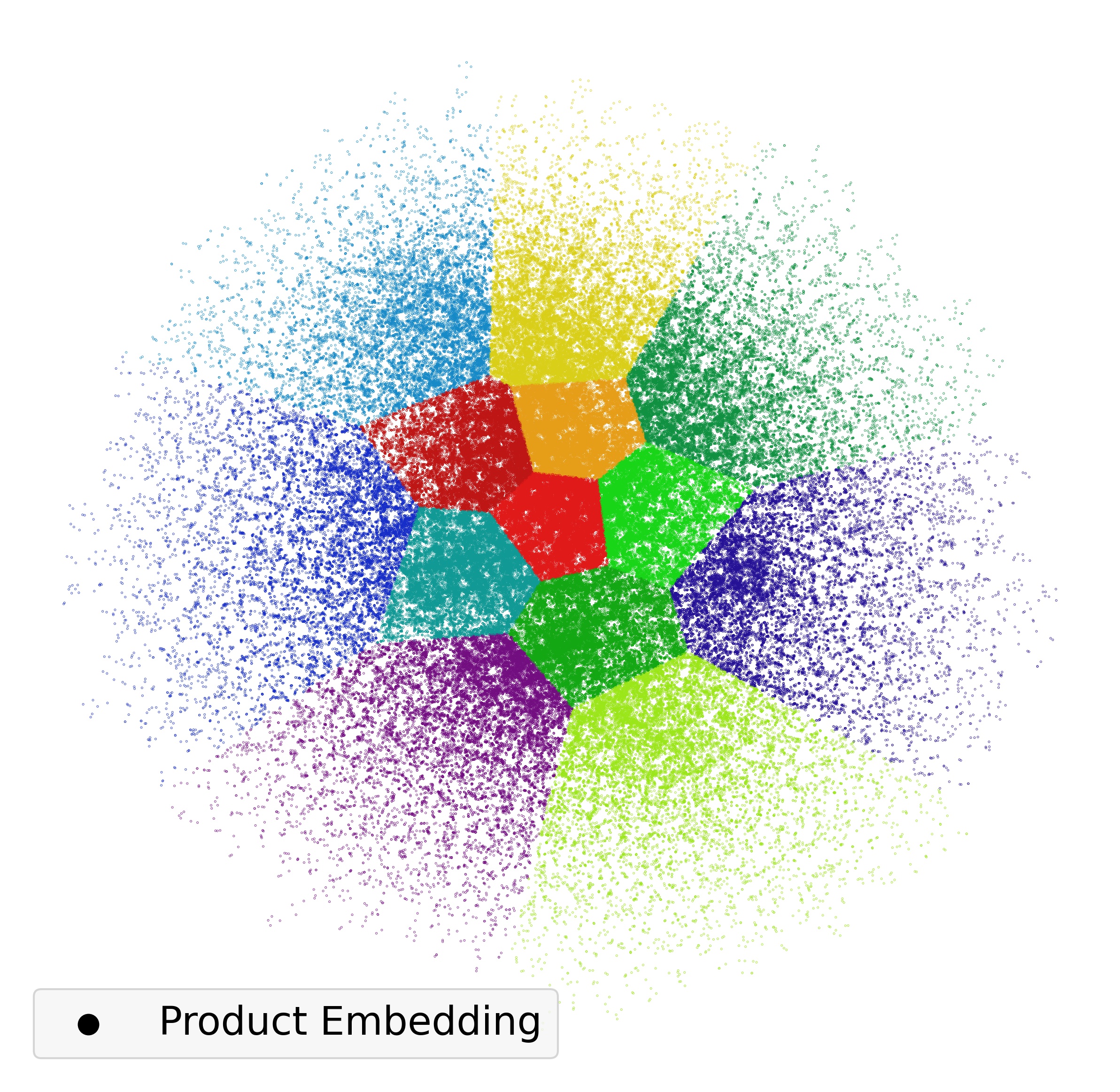} }}%
  \hfill
  \subfloat[L=1]{{ \includegraphics[scale=0.14]{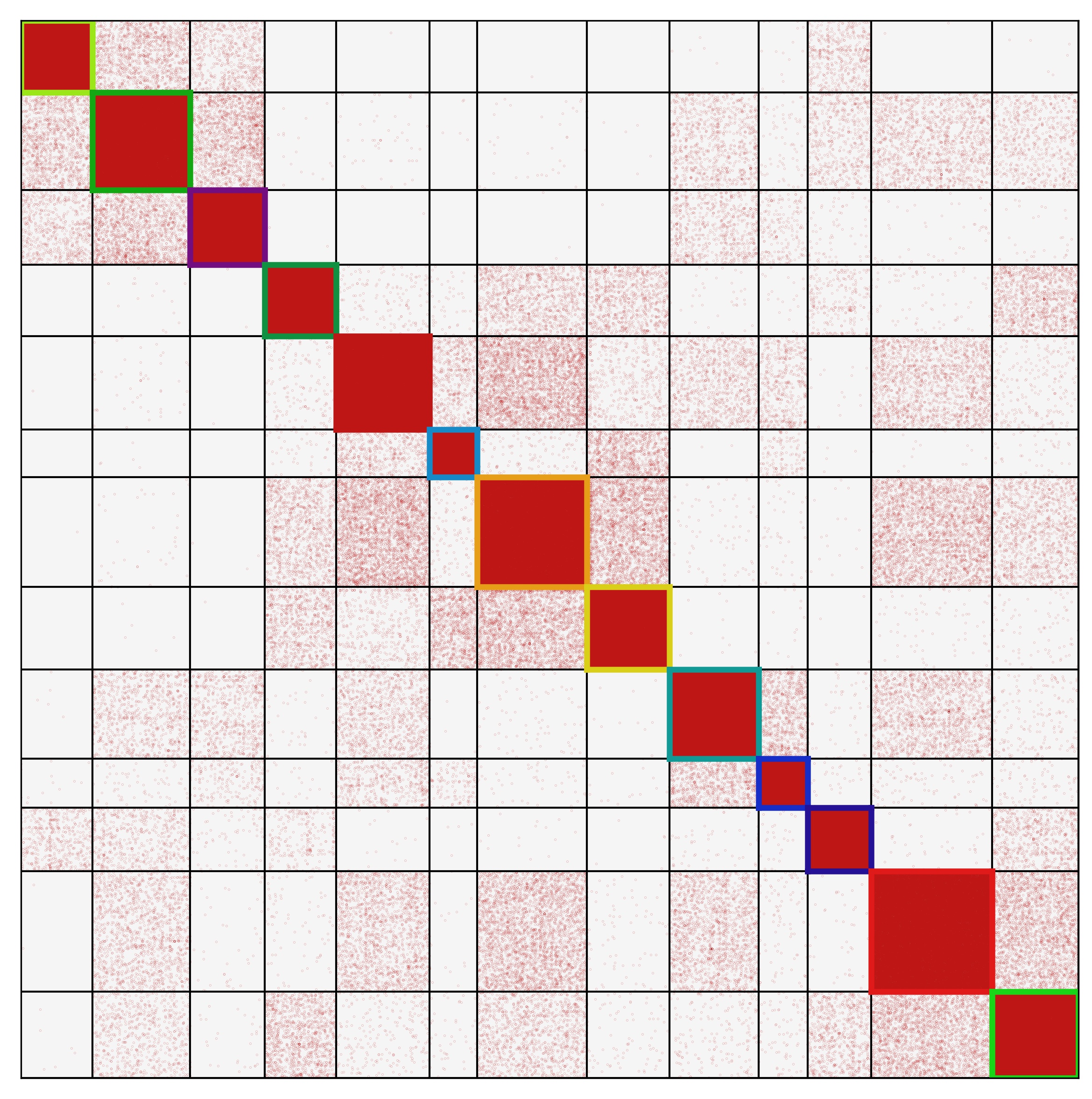} }}%
  \hfill
  \subfloat[L=3]{{ \includegraphics[scale=0.14]{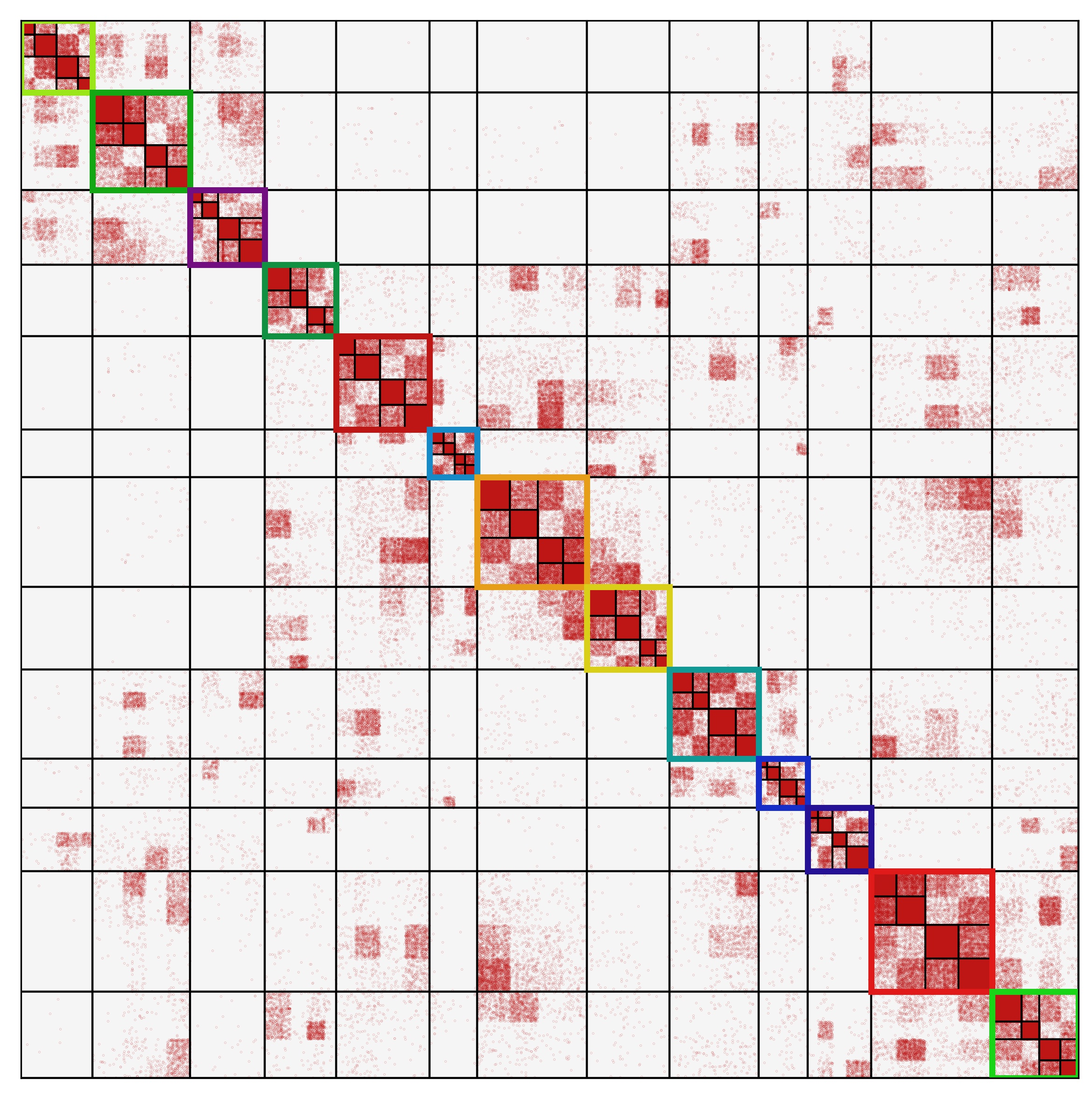} }}%
  \hfill
    \subfloat[L=5]{{ \includegraphics[scale=0.14]{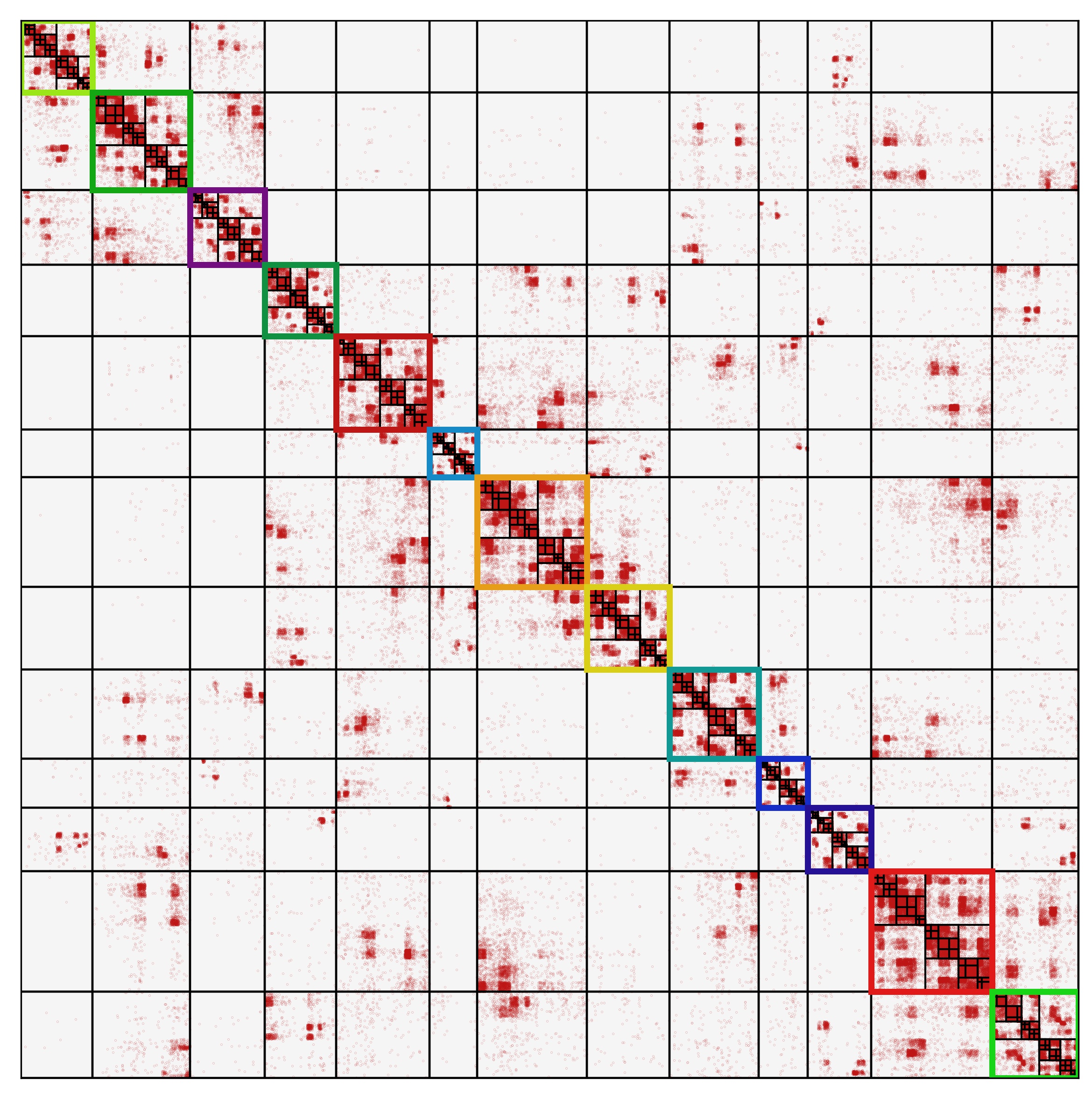} }}%
  \hfill
  \subfloat[L=8]{{ \includegraphics[scale=0.14]{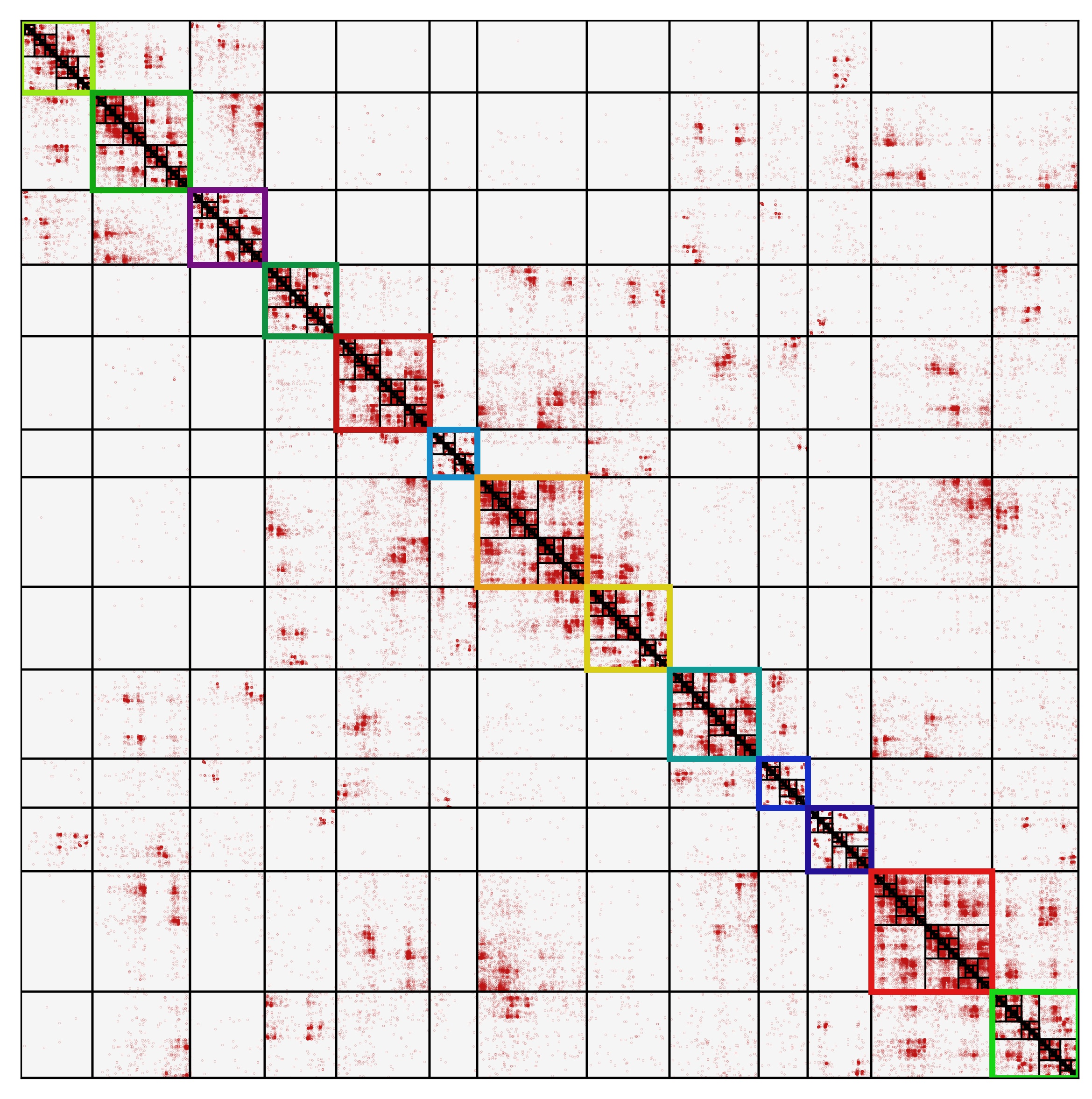} }}%
   \caption{\textsl{Amazon} network \cite{amazon_youtube} dendrogram, embedding space and ordered adjacency matrices for the learned $D=2$ embeddings of \textsc{HBDM-Re} and various levels $(L)$ of the hierarchy \cite{hbdm}. }\label{fig:amazon_dend}
\end{figure*}

\begin{figure*}[]
  \centering
  \subfloat[Dendrogram]{{ \includegraphics[scale=0.05]{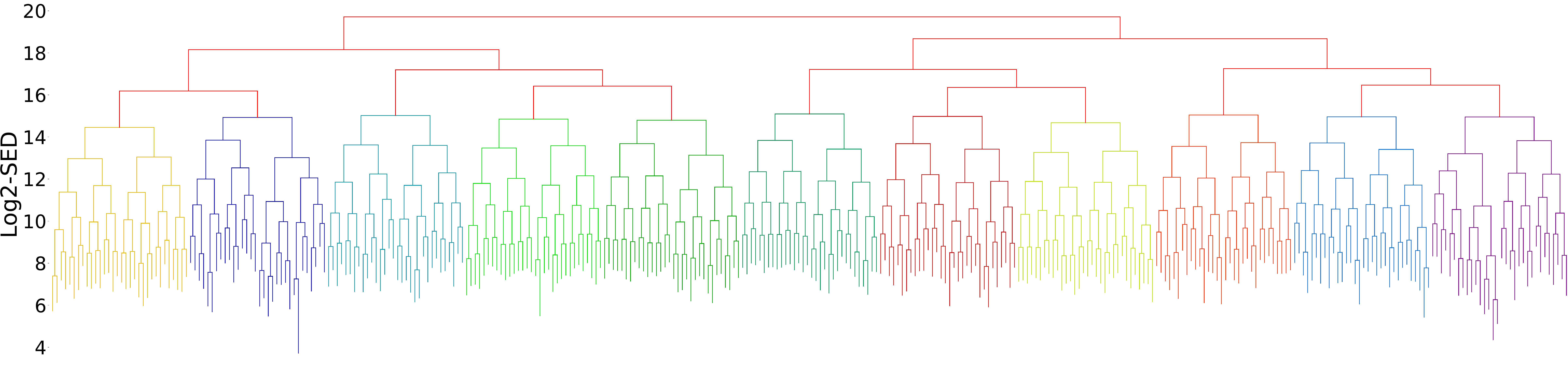} }}%
  \hfill
  \subfloat[Embedding Space]{{ \includegraphics[scale=0.15]{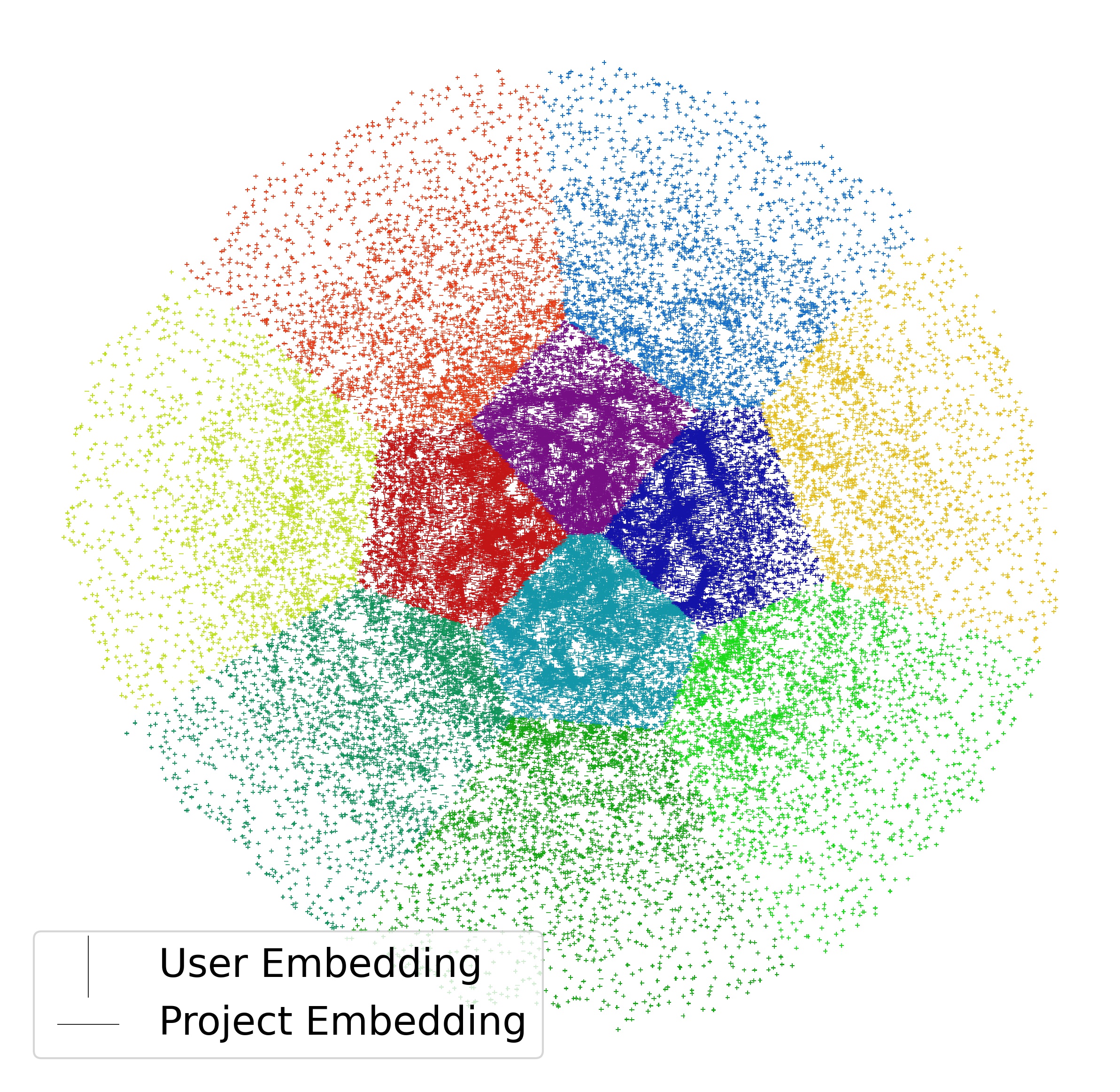} }}%
  \hfill
  \subfloat[L=1]{{ \includegraphics[scale=0.18]{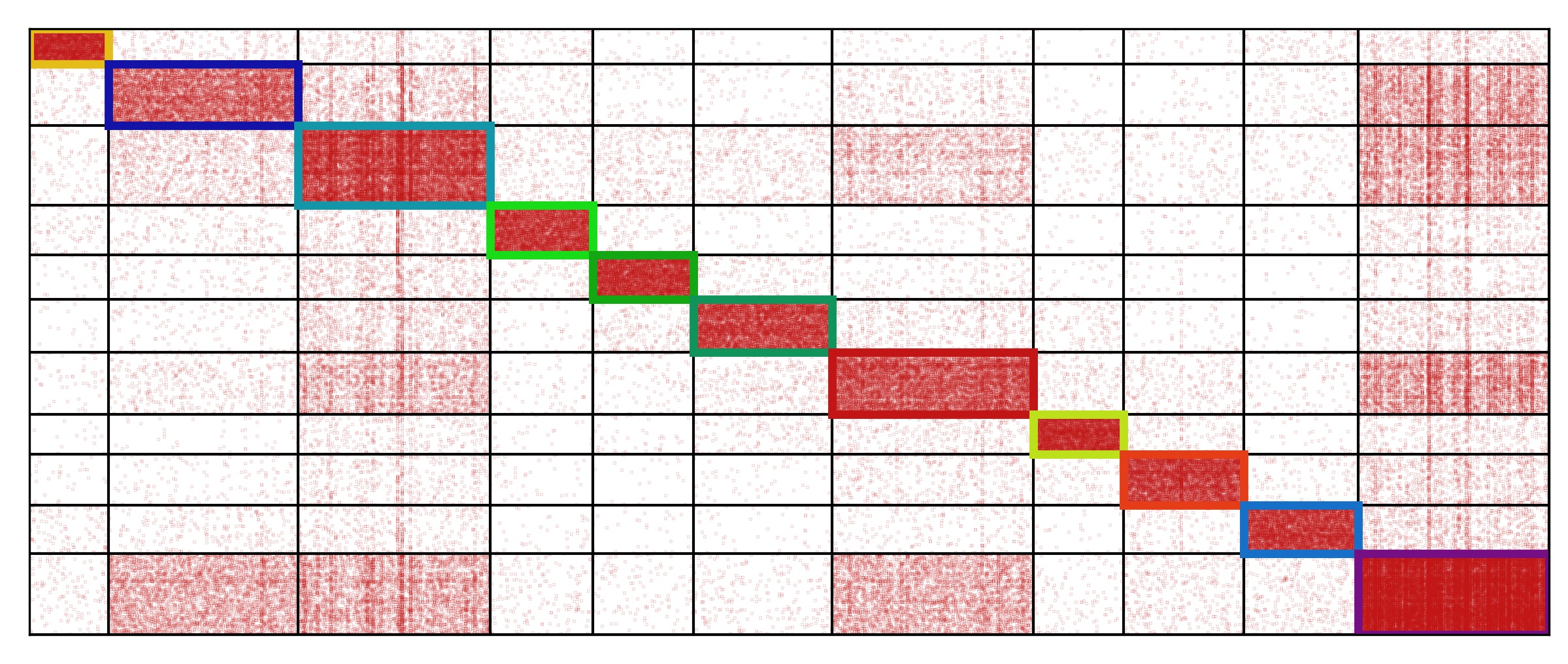} }}%
  \hfill
  \subfloat[ L=3]{{ \includegraphics[scale=0.18]{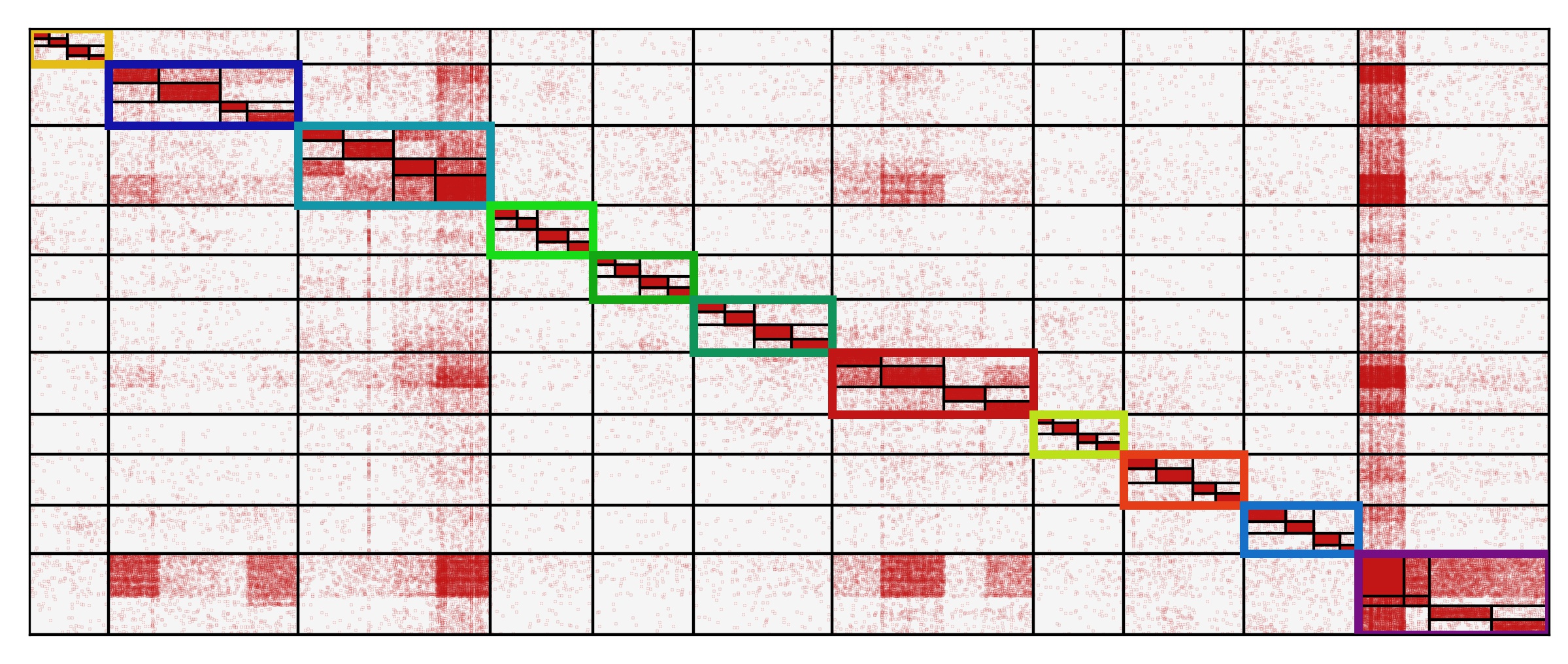} }}%
  \hfill
  \subfloat[L=6]{{ \includegraphics[scale=0.18]{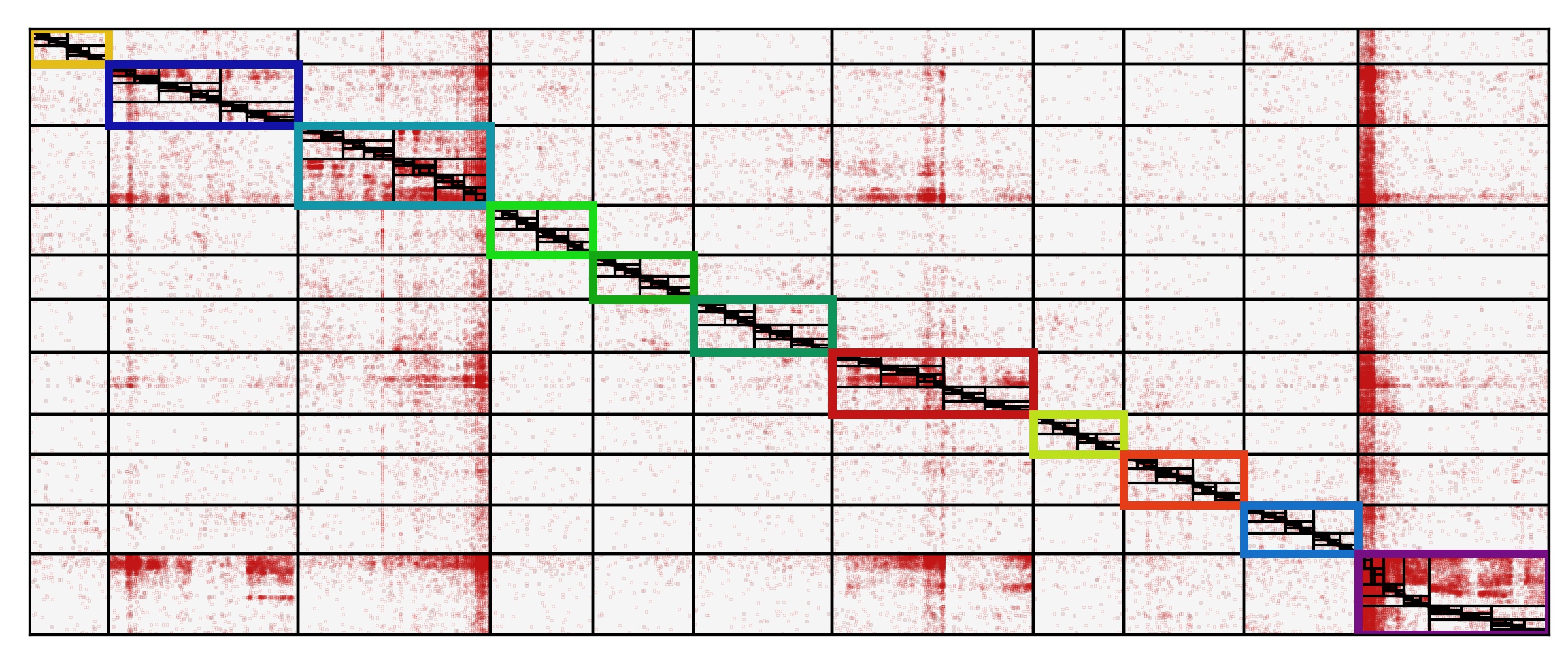} }}%
  \caption{\textsl{GitHub} network \cite{github} dendrogram, embedding space and ordered adjacency matrices for the learned $D=2$ embeddings of \textsc{HBDM-Re} and various levels $(L)$ of the hierarchy \cite{hbdm}. }\label{fig:github_dend}
\end{figure*}

\chapter{HM-LDM: A Hybrid-Membership Latent Distance Model}
Community detection, alongside link prediction, and node classification, is one of the most notable downstream tasks in network science, and Graph Representation Learning. Often, graph embedding models are blind to community structures, or require additional post-processing steps (e.g. clustering procedures) to be able to account for community characterization. Furthermore, community detection can require soft, as well as, hard membership assignments to extract overlapping or non-overlapping communities, respectively. In the \textsc{GRL} literature, most algorithms impose hard community membership constraints with overlapping community detection (when possible) requiring careful designing and tuning of these models. Importantly, \textsc{GRL} models imposing overlapping community structures are usually able to be equally competitive to additional downstream tasks, such as link prediction, and node classification.  Finally, many \textsc{GRL} approaches also do not provide identifiable or unique solution guarantees, so their interpretation highly depends on the initialization of the hyper-parameters, leading to the non-unique characterization of latent structures.

To provide a solution to such problems we here focus on combining a Non-Negative Matix Factorization with the Latent Distance Model. More specifically, we turn to the NMF theory and its uniqueness guarantees, under the scope of the \textsc{LDM}, where we can achieve unique soft and hard community memberships. Importantly, distance models offer high performance in additional important tasks, such as link prediction, and node classification, which is significantly superior to competing baselines in ultra-low dimensions \cite{hbdm}. As such, we aim to create an embedding model capable of characterizing community and latent structure without imposing any constraints on the type of memberships, providing unique representations but still explicitly accounting for homophily and transitivity, leading to superior performance on the main downstream tasks for ultra-low dimensions.

 \begin{figure}[!b]
    \centering
    \includegraphics[width=0.99\columnwidth]{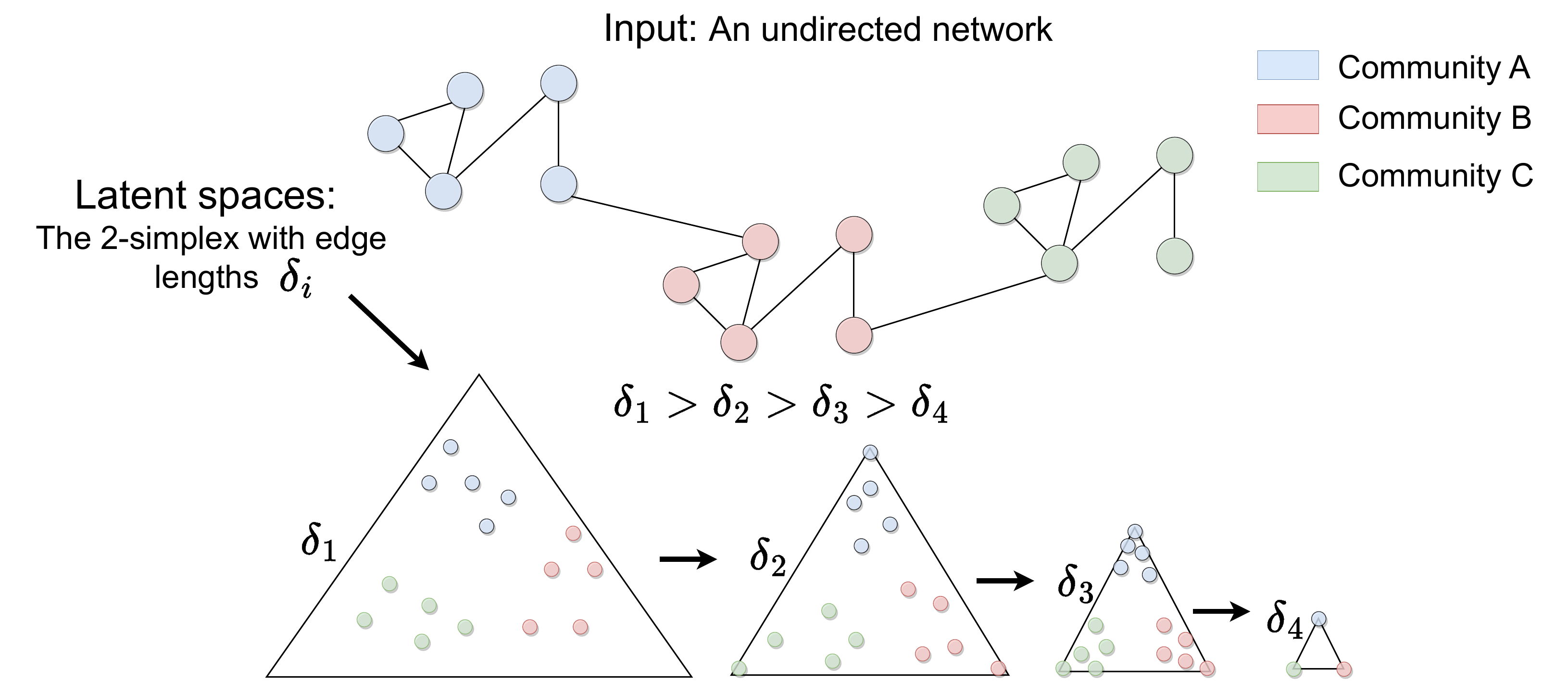}
    \caption{Hybrid Membership-Latent Distance Model procedure overview, considering a network with three communities and the $2$-simplex. Given as an input an undirected network with a (latent) community structure decreasing the volume of the latent space starts characterizing the structure, defining initially mixed memberships while for a sufficiently shrunk volume, it defines hard assignments. Large simplex edge lengths (i.e. $\delta_1$) define a large enough space that can enclose the whole representation without any decrease in the expressive capacity of the model. As the simplex edge lengths start being decreased more and more node representations move toward the corners (i.e. $\delta_1,\delta_2$), where eventually all node embeddings lie on a simplex corner (i.e. $\delta_4$).}
    \label{fig:hmldm_}
\end{figure}

\section{Contributions}
Following the primary objective of modeling complex networks, we effectively learn graph representations in order to detect structures and predict link and node properties. In such a direction, we presently reconcile \textsc{LSM}s with latent community detection by constraining the \textsc{LDM} representation to the $D$-simplex forming the Hybrid-Membership Latent Distance Model (\textsc{HM-LDM}) \cite{hmldm}. Specifically, the \textsc{HM-LDM} offers part-based representations of networks relating to a Non-negative matrix factorization, while the \textsc{LDM} constructs low-dimensional latent spaces satisfying similarity properties such as homophily and transitivity. Additionally, we define a method that permits us to capture the latent community structure of the networks using a simple continuous optimization procedure over the log-likelihood of the network. Notably, unlike most existing approaches imposing hard community membership constraints, the assignment of community memberships in our proposed hybrid model can be controlled and altered through the simplex volume formed by the latent node representations. Specifically, we show that by systematically reducing the volume of the simplex, the model becomes unique and ultimately leads to hard assignments of nodes to simplex corners. We validate the effectiveness of the \textsc{HM-LDM} through extensive experiments, demonstrating accurate node representations and valid community extraction in regimes that ensure identifiability. Importantly, we provide a systematic investigation of trade-offs between hard and mixed membership latent embeddings in terms of the model's ability to execute downstream tasks. We extensively evaluate the performance of the proposed method in link prediction, as well as, community discovery over various networks of different types, demonstrating that our model outperforms recent methods. The procedure overview is provided in Figure \ref{fig:hmldm_}. Analytically our contributions are outlined as:

\begin{itemize}
    \item We define community-aware latent representations by simply constraining the \textsc{LDM} to the $D$-simplex, forming the Hybrid-Membership Latent Distance Model which is explicitly accounting for homophily and transitivity properties, as well as, community and latent structure characterization.

    \item We design and empirically evaluate a continuous optimization procedure over the log-likelihood of the network by altering the latent space/simplex volume, allowing for control over soft and hard unique assignments to communities, and defining hybrid memberships.

     \item We provide uniqueness guarantees for the embeddings as obtained by the \textsc{HM-LDM} which is achieved up to permutation invariances.

     \item We show mathematically how a squared Euclidean \textsc{LDM} constrained to the $D$-simplex relates to the Non-Negative Matrix Factorization, defining a non-negative Latent Eigenmodel, and when such a factorization is unique.

      \item We systematically analyze the trade-offs that soft and hard community memberships define under the scope of link prediction and community detection tasks.

    \item We generalize the method for bipartite networks where structure-aware geometric representations, joint embedding spaces, and community discovery are arduous tasks.

    \item We extensively benchmark our proposed model against state-of-the-art \textsc{GRL} baselines, including models for both overlapping and non-overlapping community extraction under various and well-established network data. 

\end{itemize}

\section{Experimental design, results, and key findings}
We adopt an extensive experimental evaluation framework that includes twelve prominent \textsc{GRL} baselines, including methods that use an NMF to learn their representation. In addition, we make use of four moderate-sized networks without known community labels to evaluate the model on link prediction and its ability to detect latent structures; four networks with known community labels to evaluate the model on its ability to perform community detection; and two bipartite networks to showcase the generalization of the model. We consider performance comparisons in downstream tasks of multiple \textsc{HM-LDM} versions defining big, moderate, and small latent space volumes, to understand the trade-offs that the resulting soft and hard memberships have on downstream tasks.

The obtained results for the link prediction task, showed that \textsc{HM-LDM} outperformed the considered baselines significantly, especially when compared to models that define mixed memberships under an NMF operation. For the community detection task, once more the \textsc{HM-LDM} outperformed or provided on par results with the most competitive baselines. For the study of trade-offs between soft and hard memberships or equivalently small and large volumes, the \textsc{HM-LDM} results showed high community detection results when the volume was particularly small and defined hard assignments to communities. Small volumes hampered the link prediction performance, as expected since such a small space decreases the expressive capability of the model. For large volumes, we saw a link prediction performance equivalent to the classical \textsc{LDM} since a large enough volume can absorb the whole non-negative orthant, making essentially the simplex constraint "powerless" as the latent distances can take very large values. These results are highlighted in Figure \ref{fig:roc_d}. The community detection results in high simplex volumes show a significant decrease in the performance as the model now suffers from identifiability issues. Importantly, the experiments for moderate-sized simplex volumes led to the existence of a silver lining where the model is identifiable and performed well on community detection while having almost an insignificant decrease in link prediction results when compared to the classical LDM. Identifiability results based on the type of community memberships for different simplex volumes are given in Figure \ref{fig:phase_transitions} while latent community extraction examples on two real-world networks are provided in Figure \ref{fig:adj}. Finally, \textsc{HM-LDM} experiments on the two bipartite networks empirically showed successful latent structure extraction and identification. (For more details and the full experiment results please visit the full paper \cite{hmldm}.)

\section{Conclusion}
In this paper, we propose the \textsc{HM-LDM}, a model that reconciles network embedding and latent community detection. The approach utilizes normal and squared Euclidean specifications of the latent distance model. A squared Euclidean specification integrates the non-negativity-constrained Eigenmodel with the Latent Distance Model. We extensively showed that the model could be constrained to the simplex without losing expressive power. The reduced simplex provides unique representations, ultimately resulting in the hard clustering of nodes to communities when the simplex is sufficiently shrunk. Notably, the proposed \textsc{HM-LDM} combines network homophily and transitivity properties with latent community detection enabling explicit control of soft and hard assignment through the volume of the induced simplex. We observed favorable link prediction performance in regimes in which the \textsc{HM-LDM} provides unique representations while enabling the ordering of the adjacency matrix in terms of prominent latent communities. Finally, we showed the ability of the model to extract valid community structures across multiple networks and showcased how the analysis extends to bipartite networks.
 \begin{figure}[!t]
  \centering
  \subfloat[\textsl{AstroPh} \cite{astroph_grqc_hepth}]{{\includegraphics[width=0.23\columnwidth]{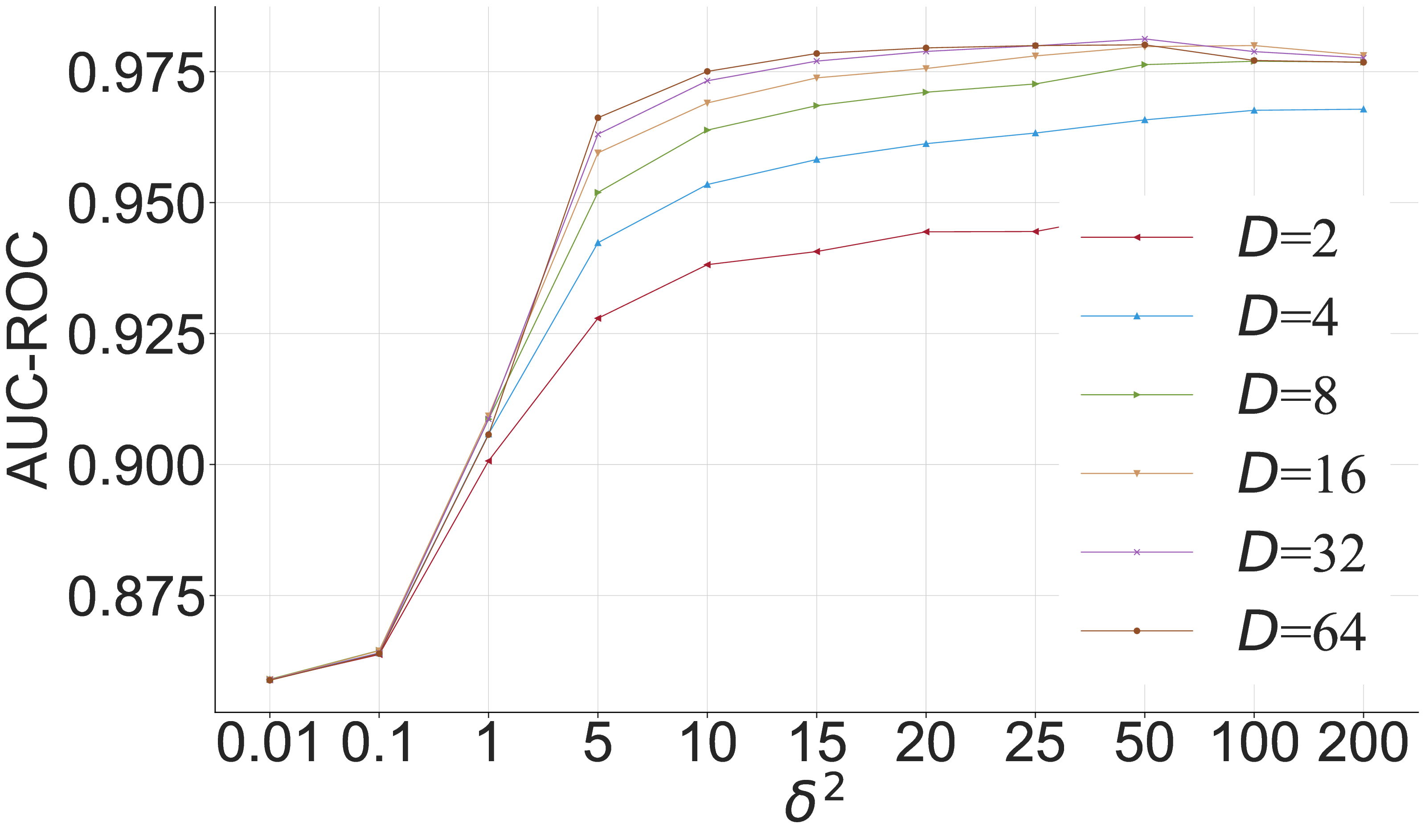} }}%
\hfill 
  \subfloat[\textsl{Facebook}\cite{facebook}]{{ \includegraphics[width=0.23\columnwidth]{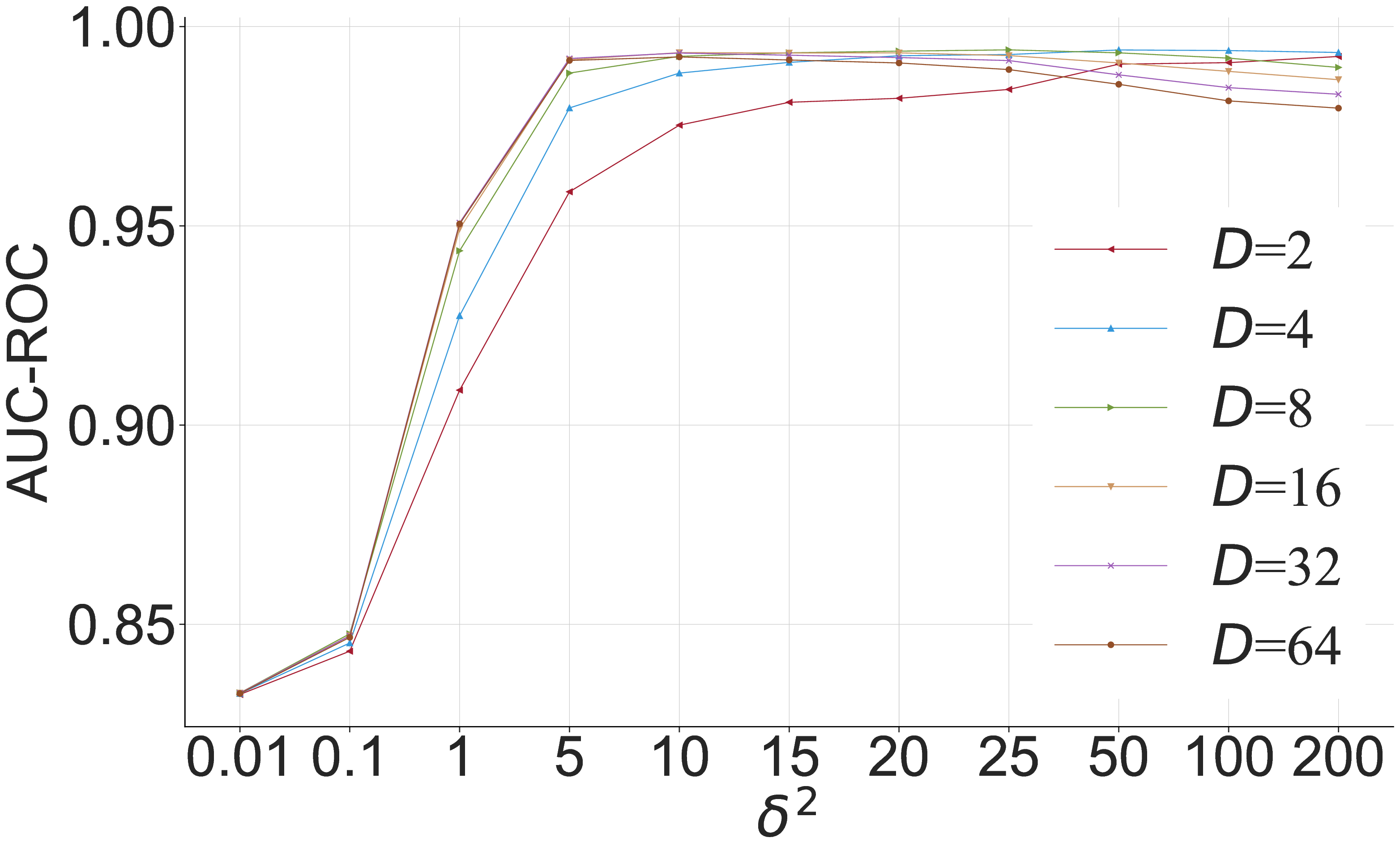} }}%
\hfill 
  \subfloat[\textsl{GrQc}\cite{astroph_grqc_hepth}]{{ \includegraphics[width=0.23\columnwidth]{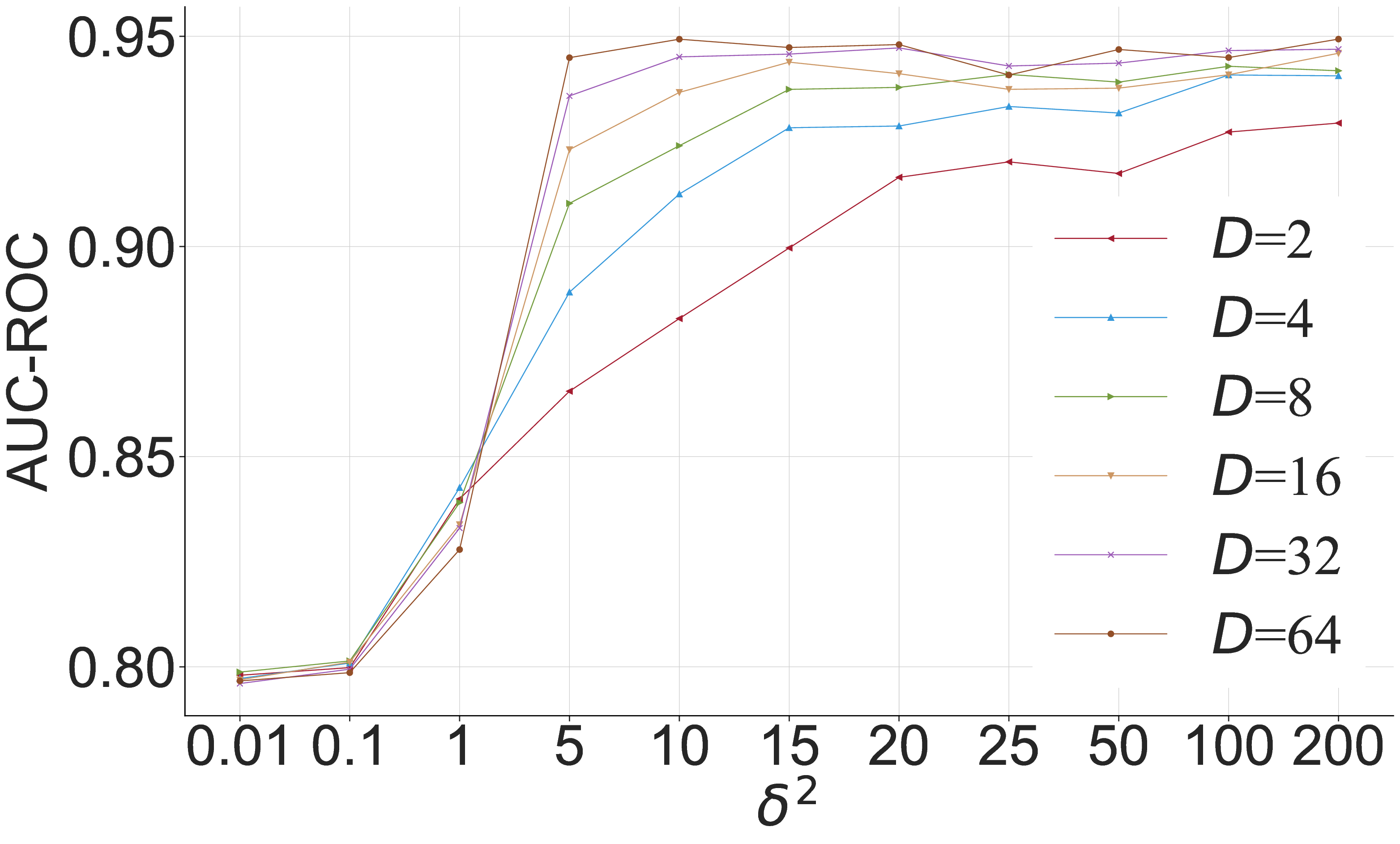} }}%
\hfill
  \subfloat[\textsl{HepTh}\cite{astroph_grqc_hepth}]{{ \includegraphics[width=0.23\columnwidth]{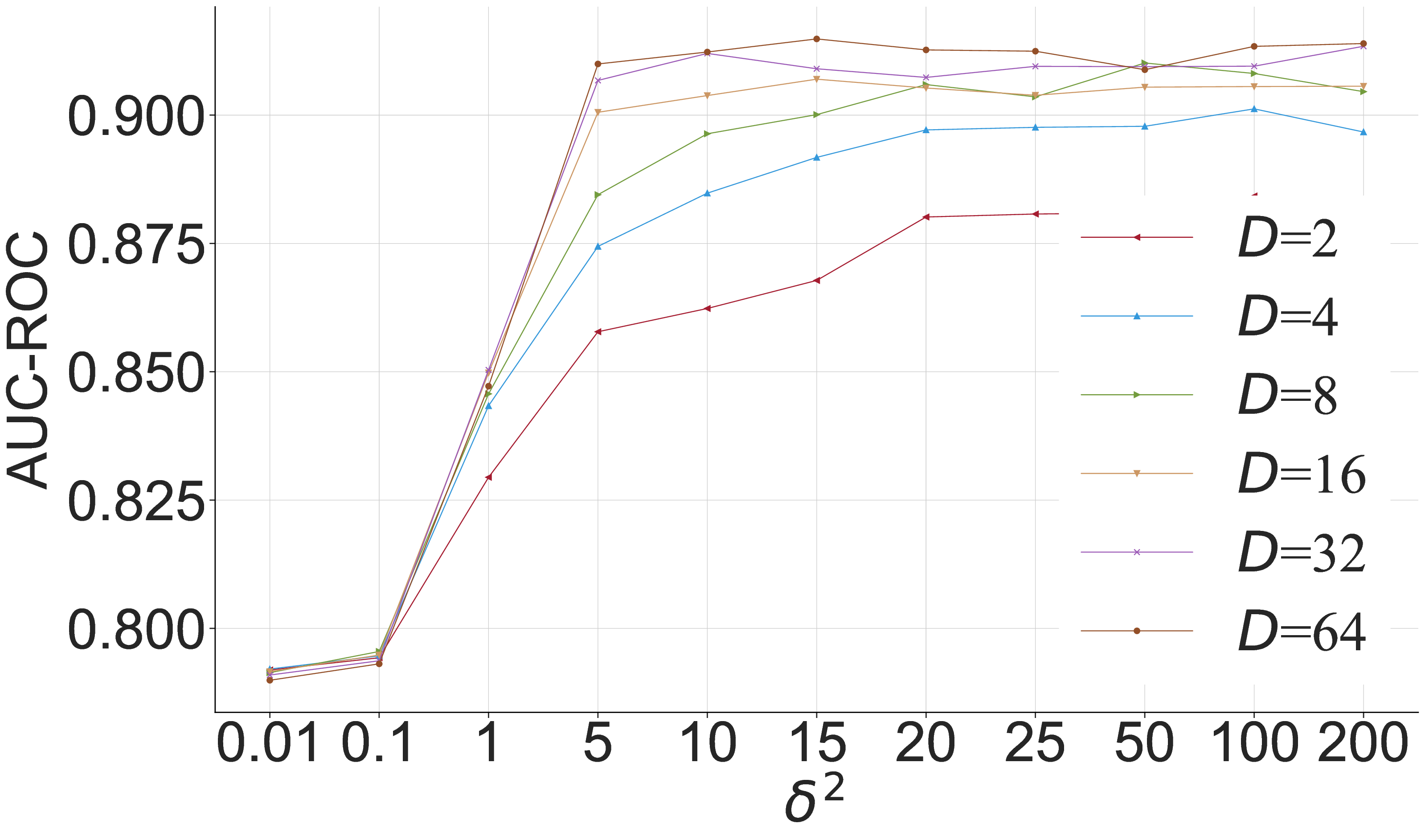} }}%
  \hfill
    \subfloat[\textsl{AstroPh}]{{ \includegraphics[width=0.23\columnwidth]{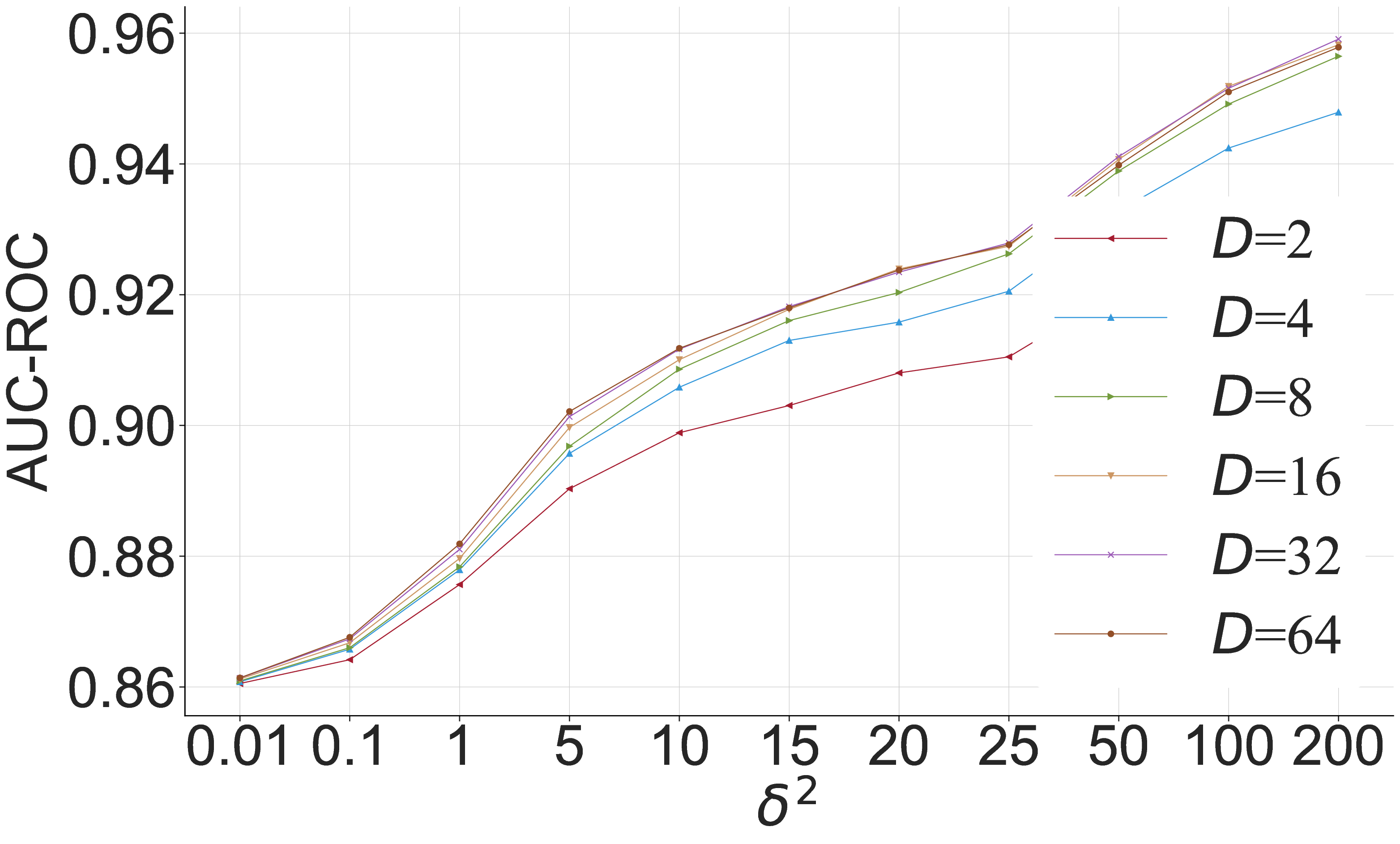} }}%
\hfill 
  \subfloat[\textsl{Facebook}]{{ \includegraphics[width=0.23\columnwidth]{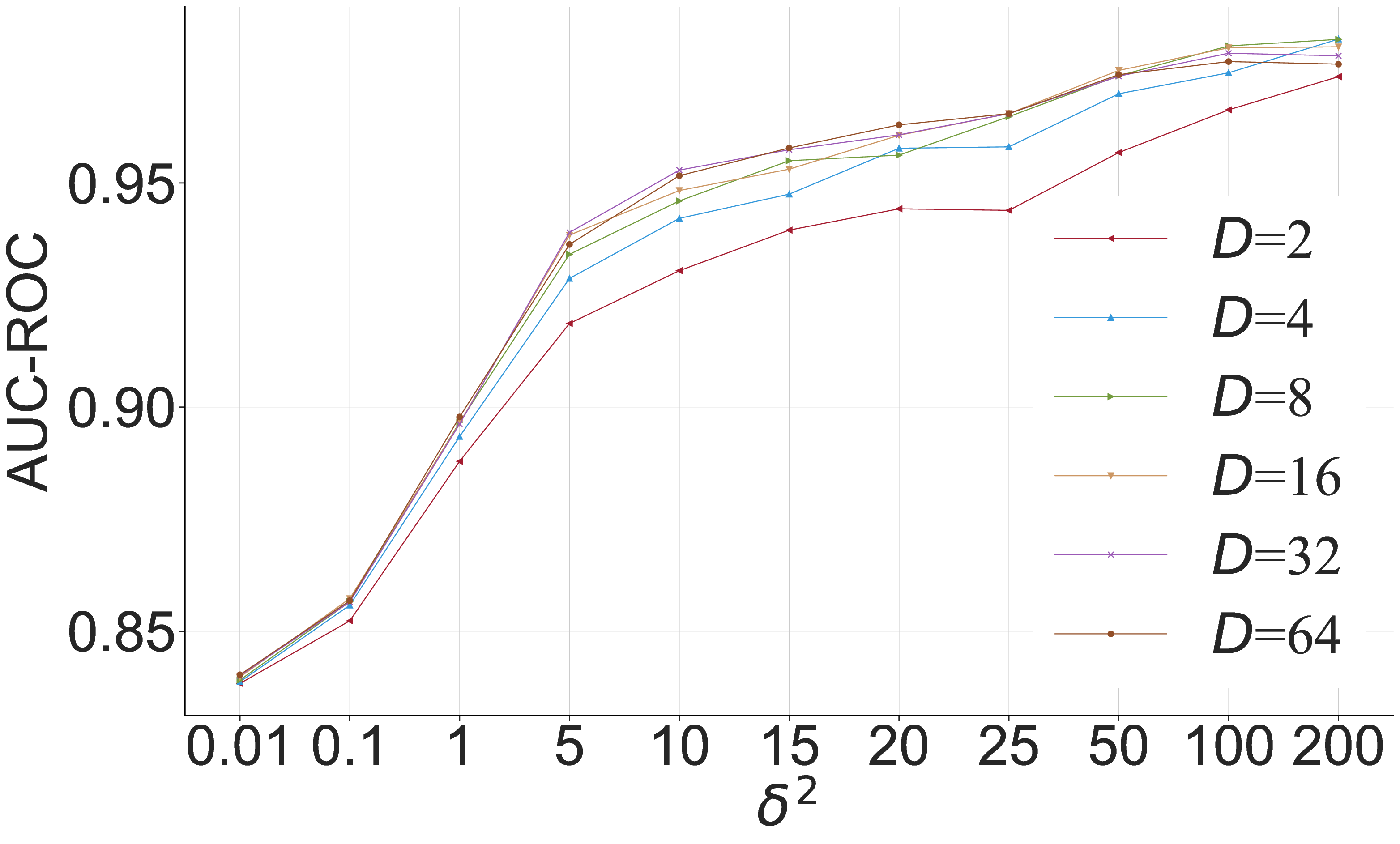} }}%
\hfill 
  \subfloat[\textsl{GrQc}]{{ \includegraphics[width=0.23\columnwidth]{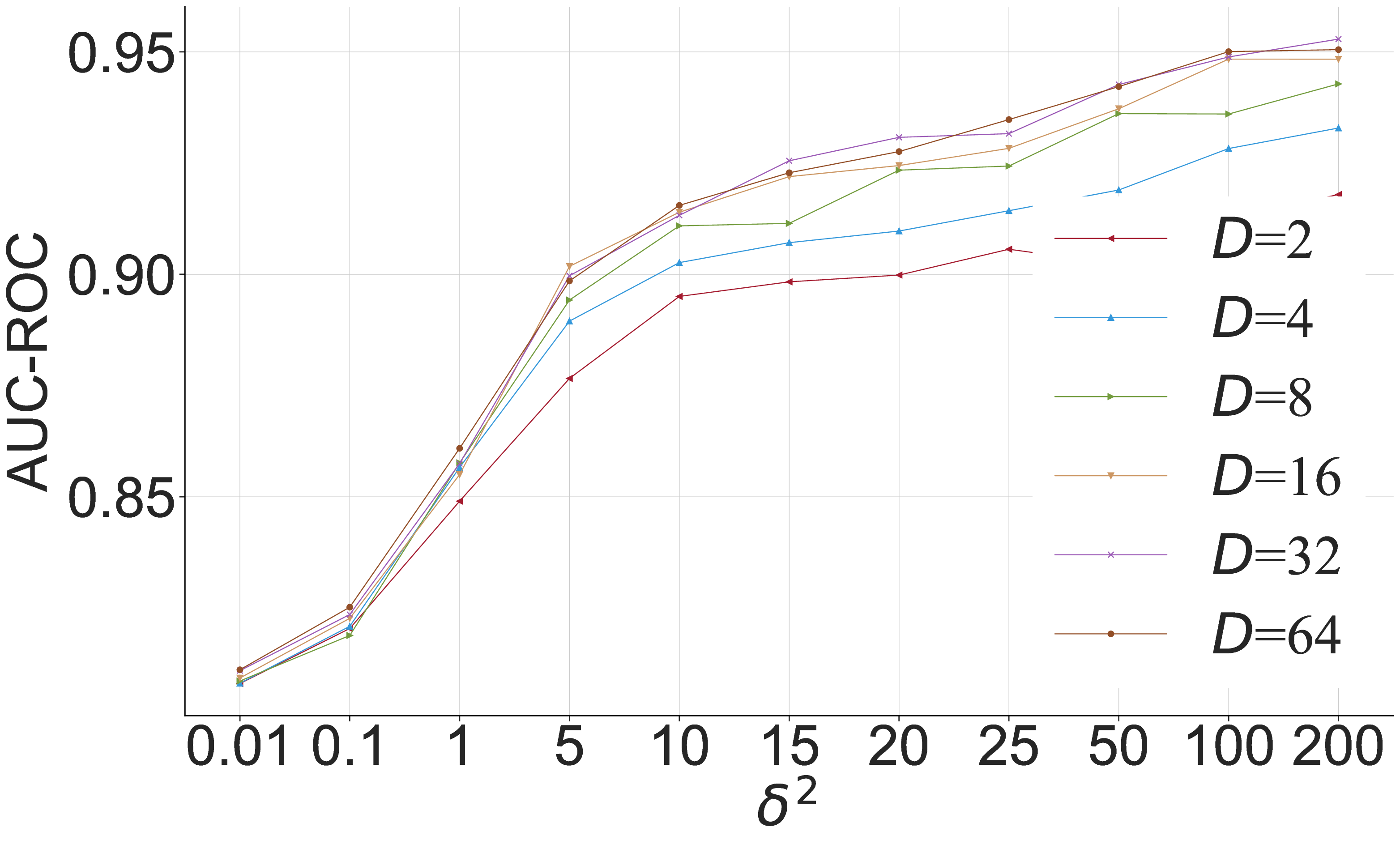} }}%
\hfill
  \subfloat[\textsl{HepTh}]{{ \includegraphics[width=0.23\columnwidth]{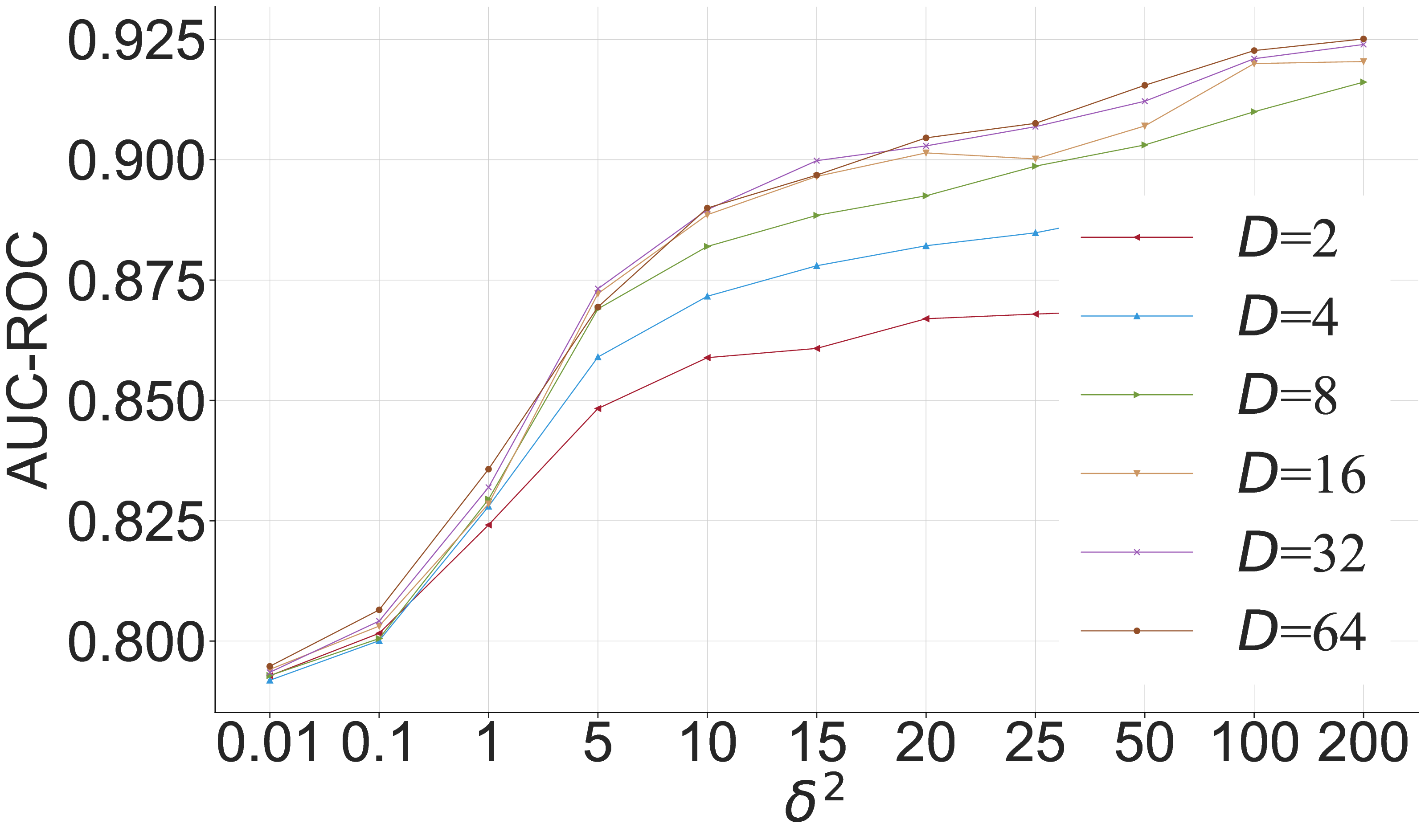} }}%
  \caption{AUC-ROC scores as a function of $\delta^2$ across dimensions for \textsc{HM-LDM}. Top row: $p=2$. Bottom row $p=1$ \cite{hmldm}.}\label{fig:roc_d}
\end{figure}

\begin{figure} [!t]
  \centering
  \subfloat[\textsl{AstroPh}]{{ \includegraphics[width=0.23\columnwidth]{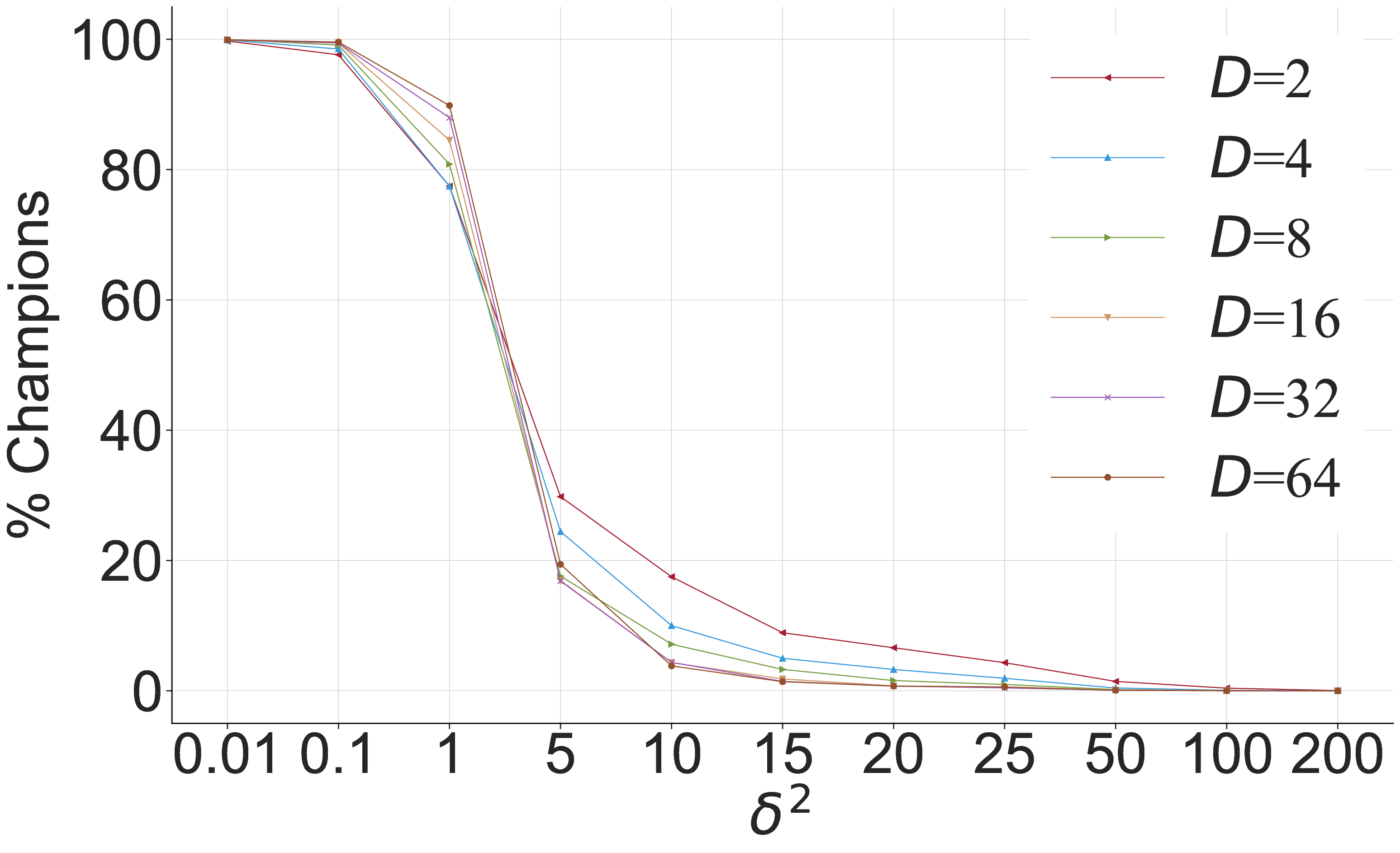} }}%
  \hfill
  \subfloat[\textsl{Facebook}]{{ \includegraphics[width=0.23\columnwidth]{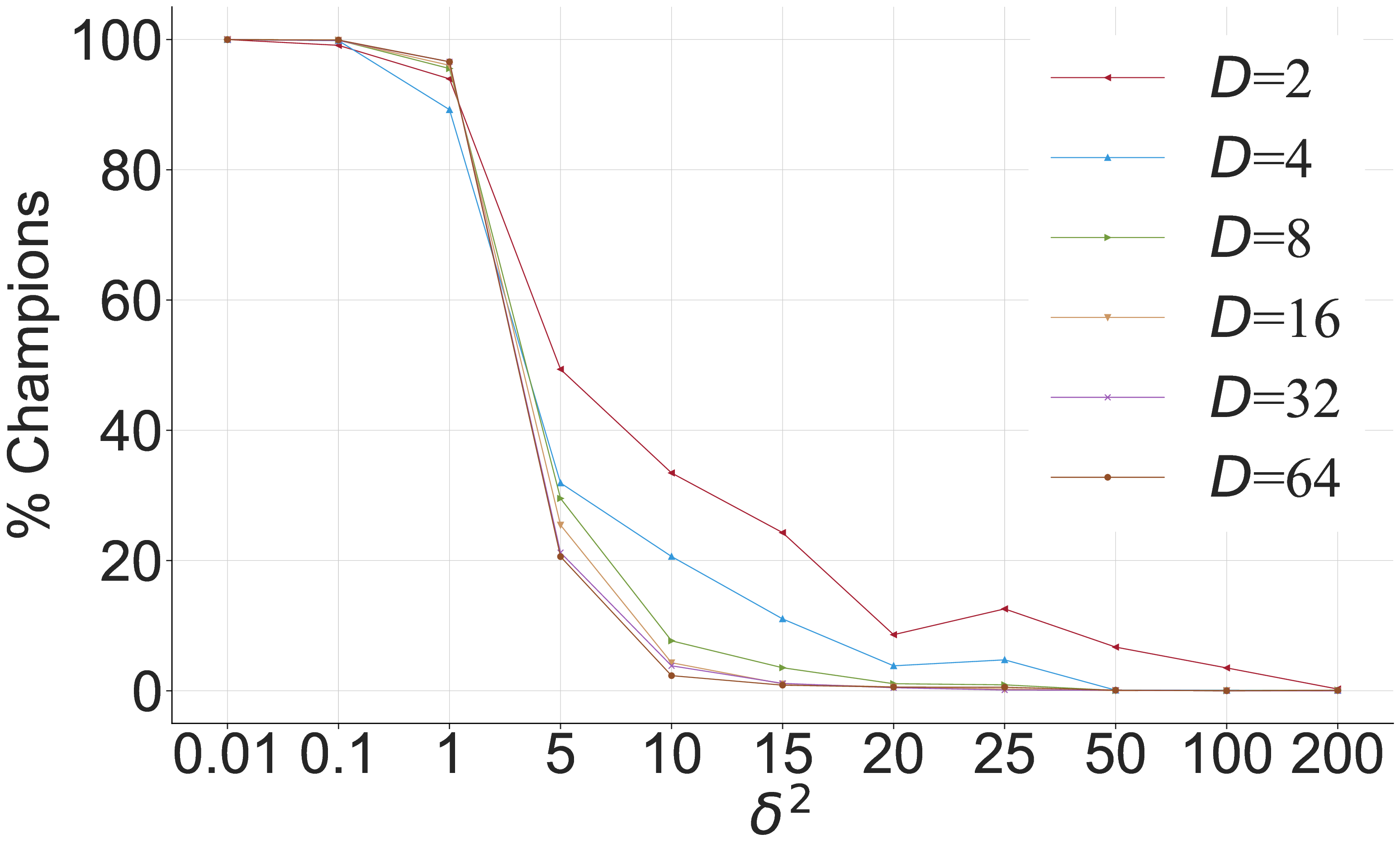} }}%
  \hfill
  \subfloat[\textsl{GrQc}]{{ \includegraphics[width=0.23\columnwidth]{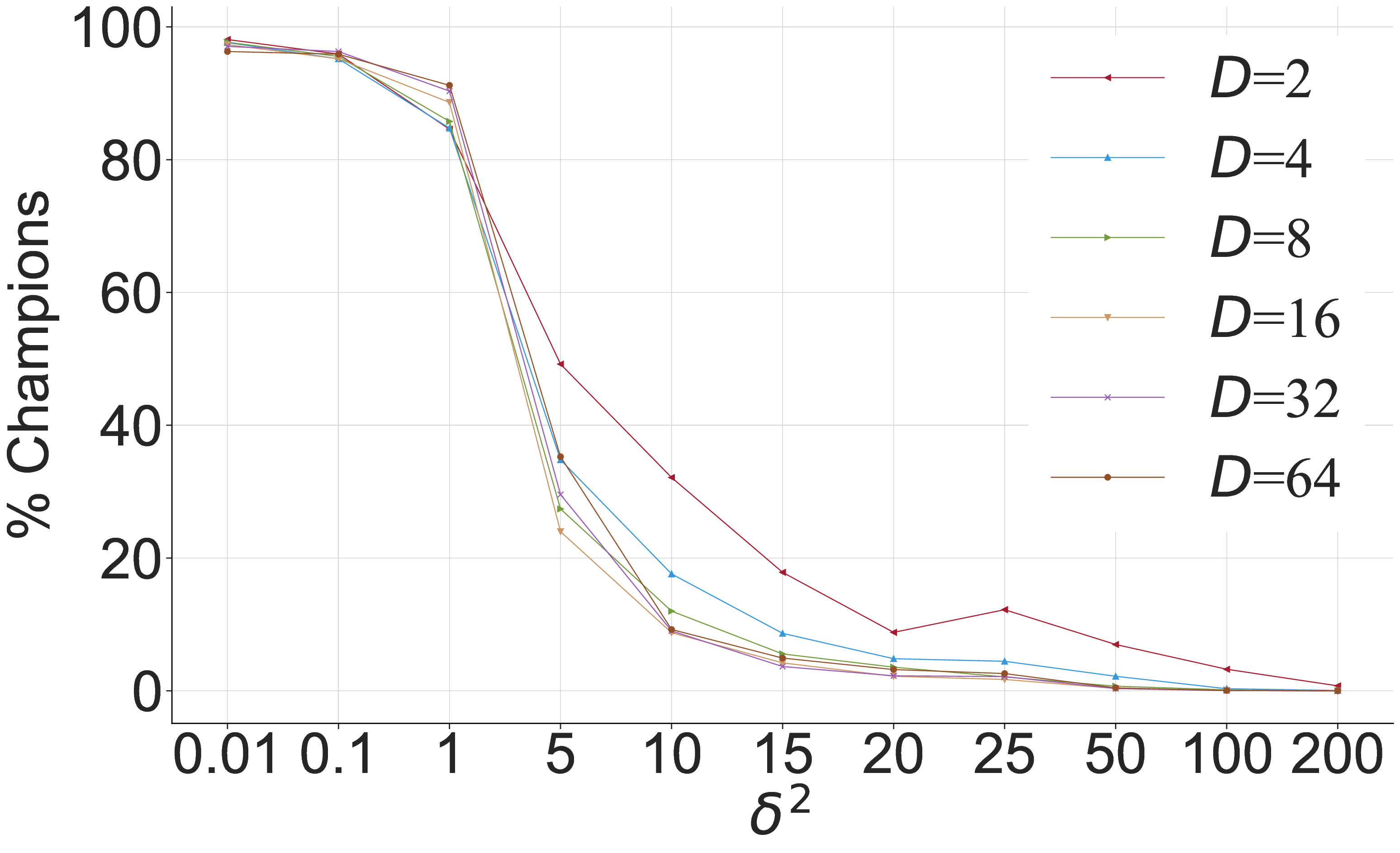} }}%
  \hfill
  \subfloat[\textsl{HepTh}]{{ \includegraphics[width=0.23\columnwidth]{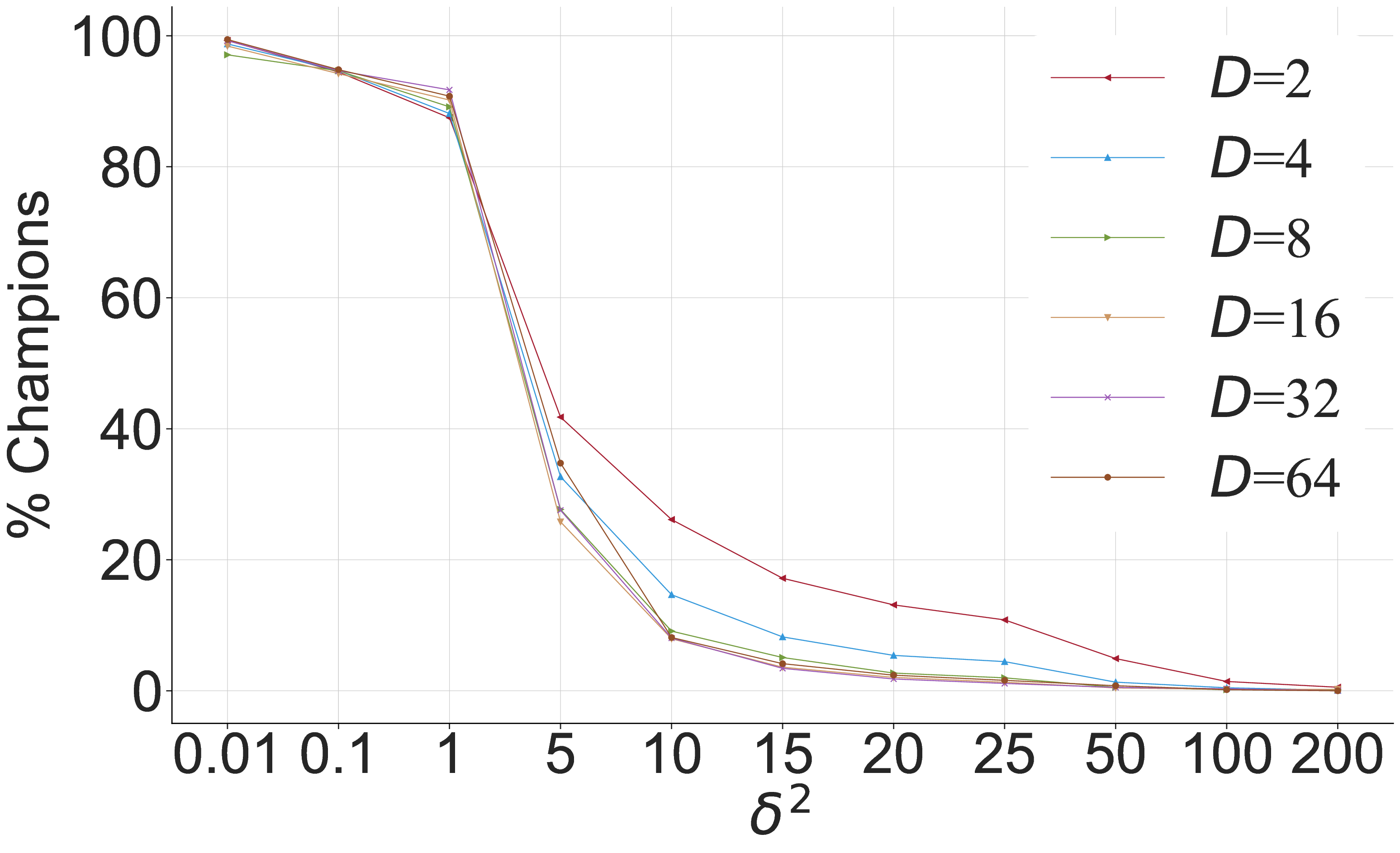} }}%
  \hfill
    \subfloat[\textsl{AstroPh}]{{ \includegraphics[width=0.23\columnwidth]{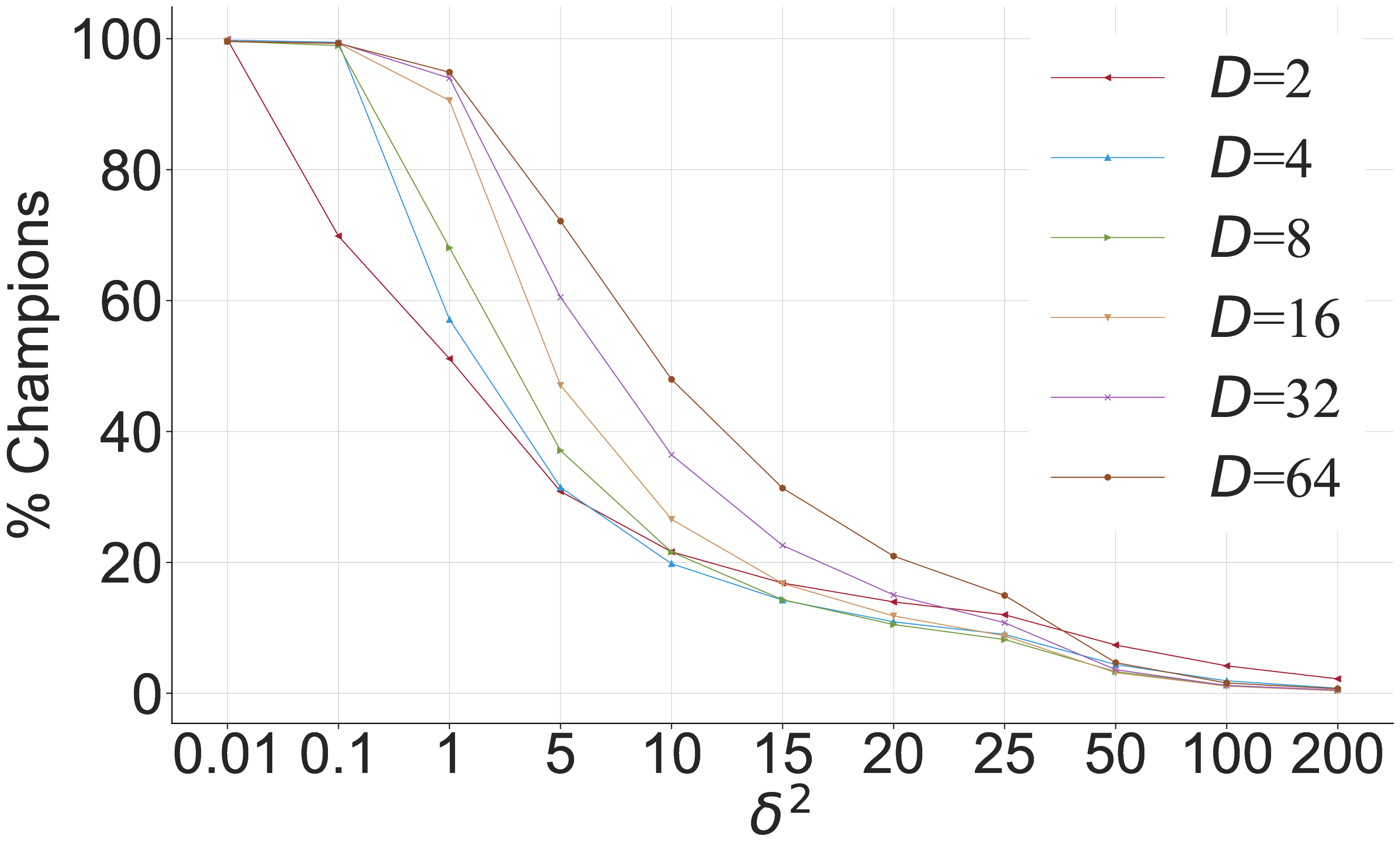} }}%
  \hfill
  \subfloat[\textsl{Facebook}]{{ \includegraphics[width=0.23\columnwidth]{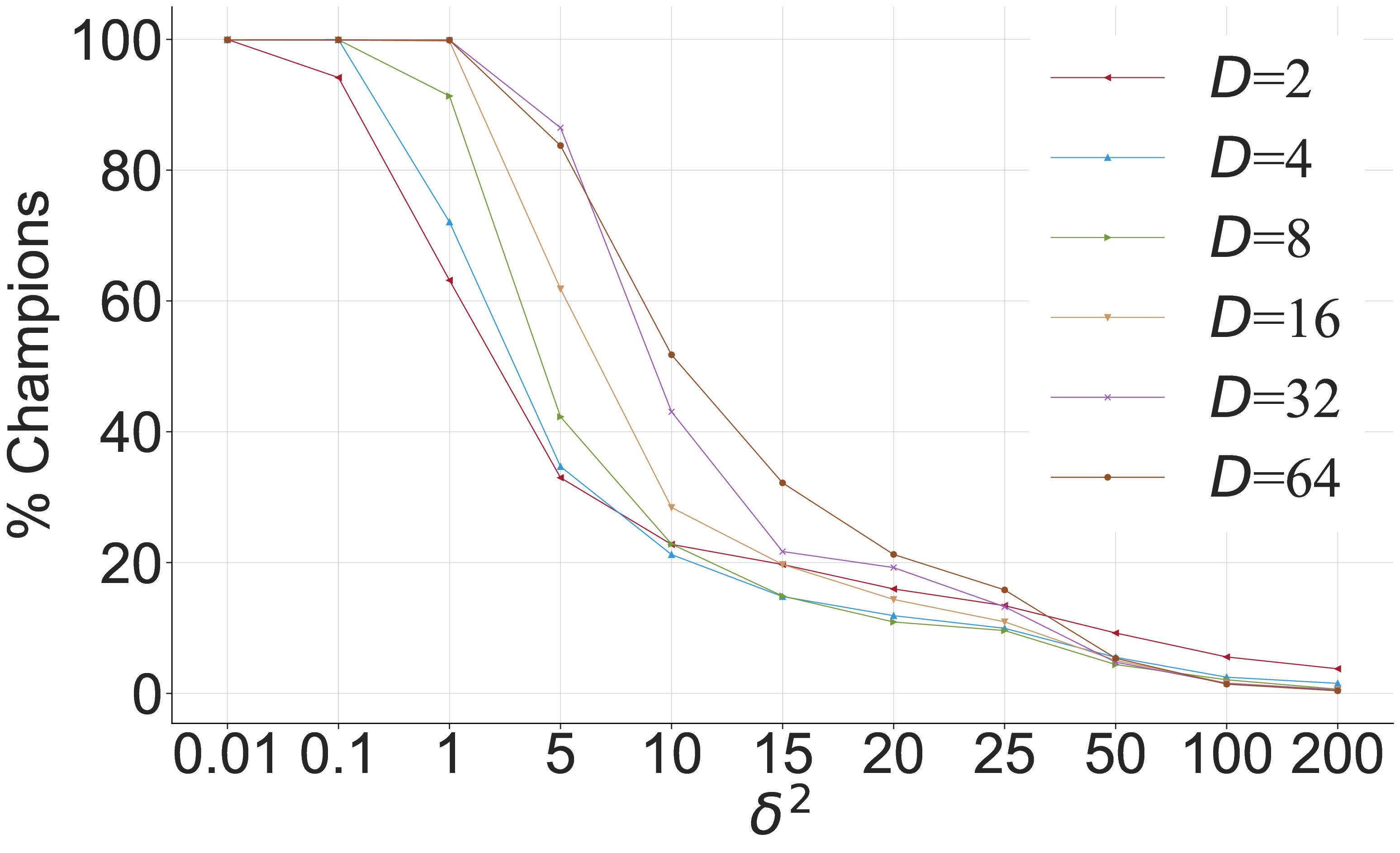} }}%
  \hfill
  \subfloat[\textsl{GrQc}]{{ \includegraphics[width=0.23\columnwidth]{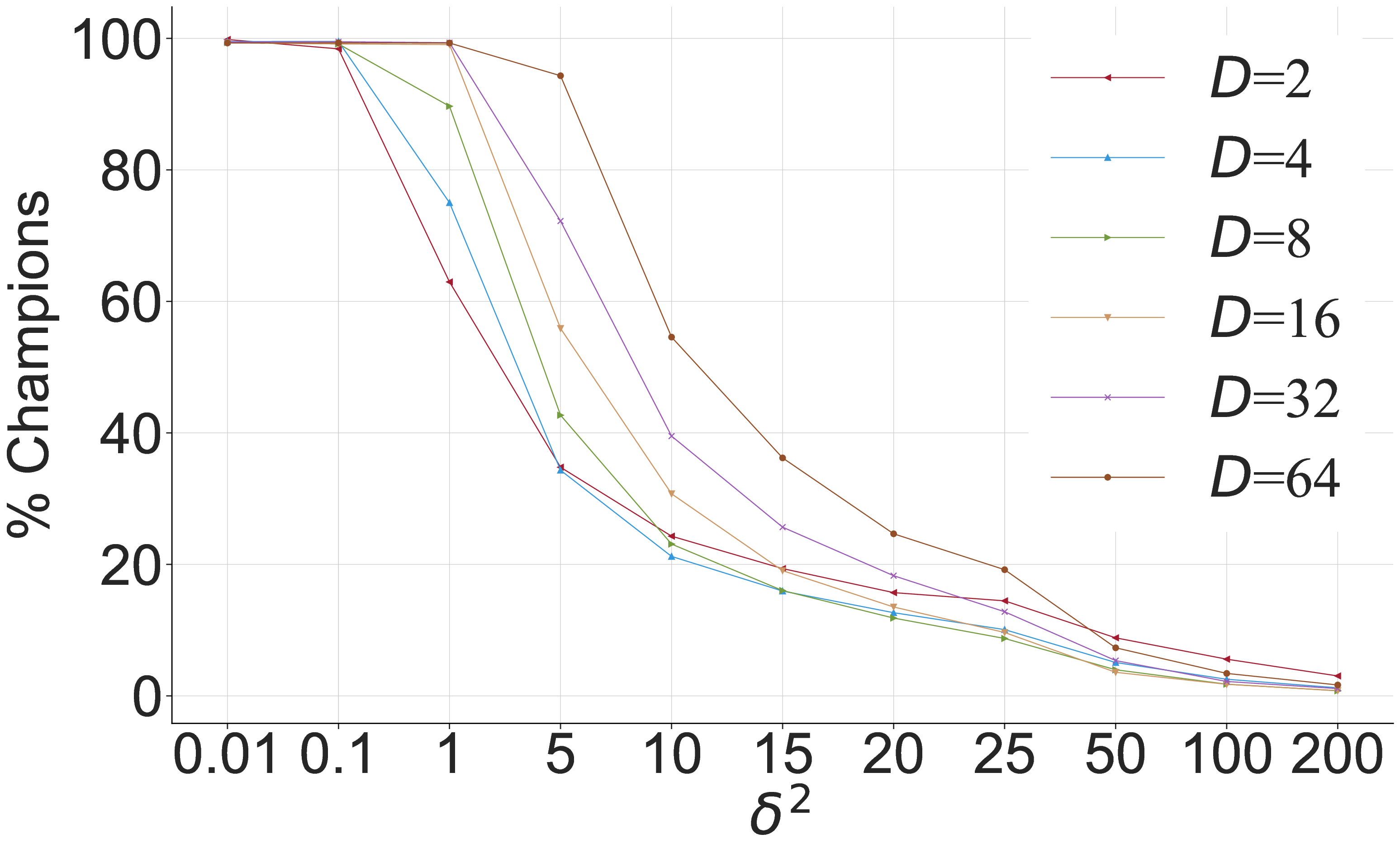} }}%
  \hfill
  \subfloat[\textsl{HepTh}]{{ \includegraphics[width=0.23\columnwidth]{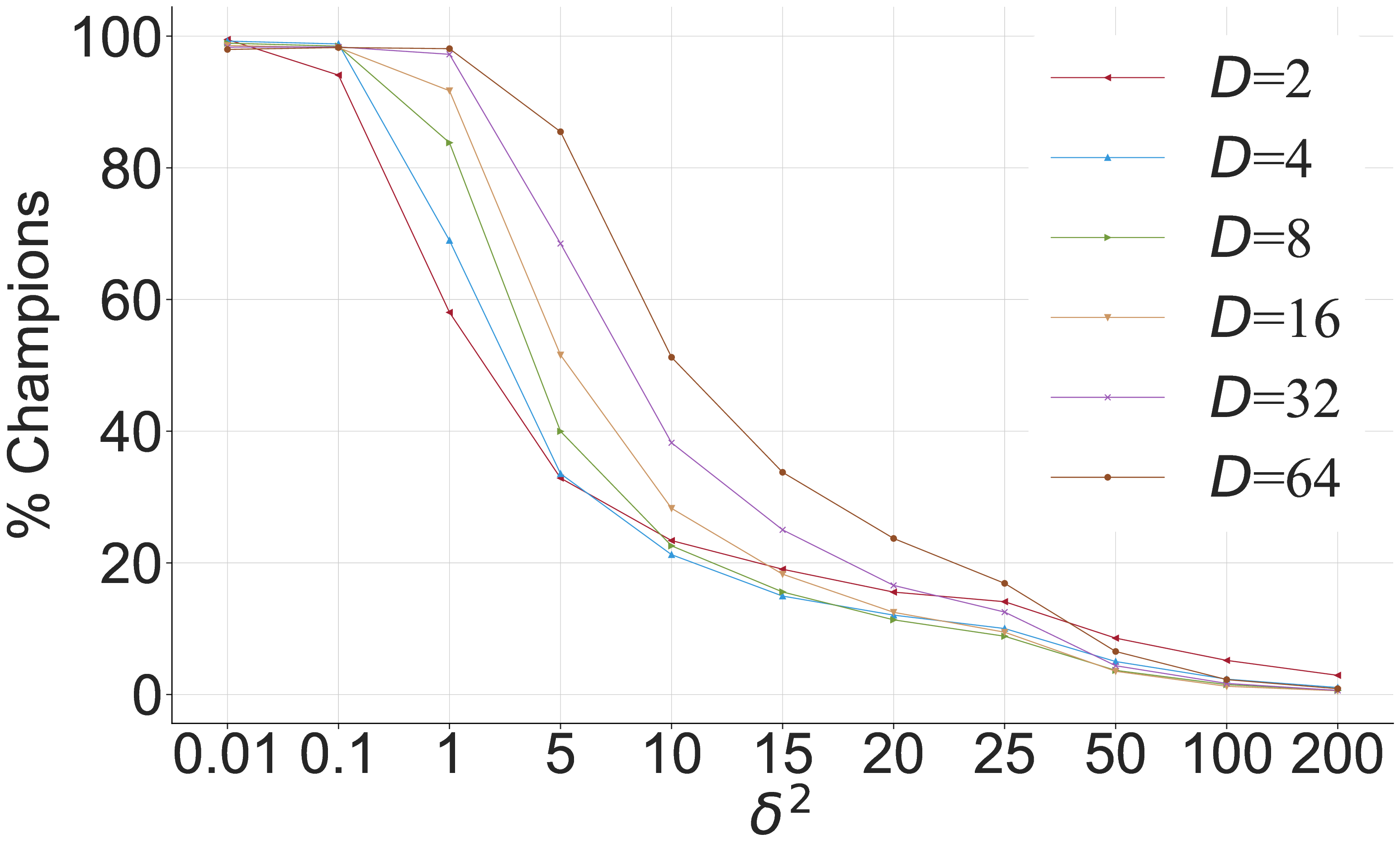} }}%
  \caption{Total community champions (\%) in terms of $\delta^2$ across dimensions for \textsc{HM-LDM}. Top row: $p=2$. Bottom row $p=1$ \cite{hmldm}.}  \label{fig:phase_transitions}
\end{figure}

  \begin{figure}[!t]
  \centering
  \subfloat[\textsl{GrQc} $(p=2)$]{{ \includegraphics[width=0.23\columnwidth]{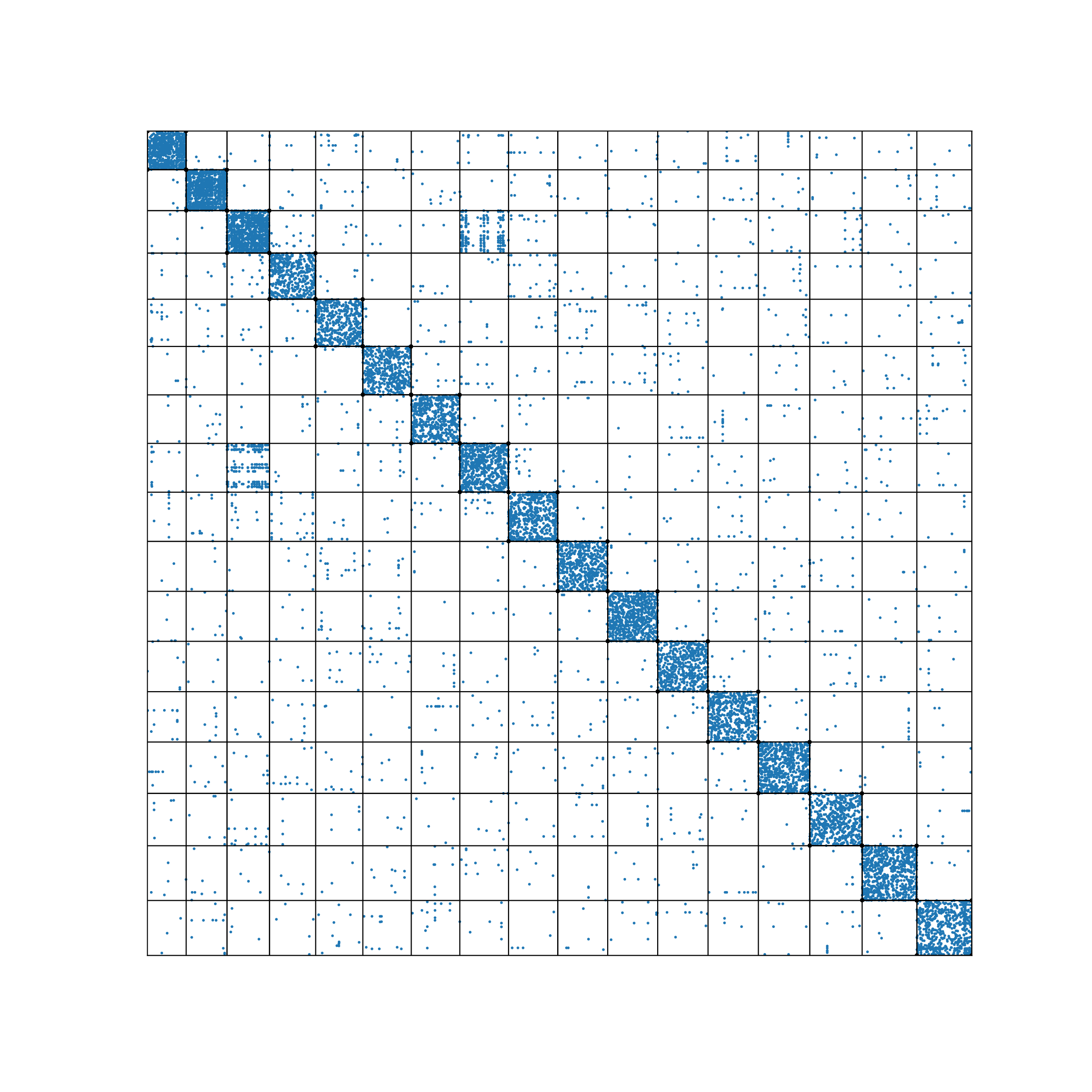} }}%
\hfill
  \subfloat[\textsl{HepTh} $(p=2)$]{{ \includegraphics[width=0.23\columnwidth]{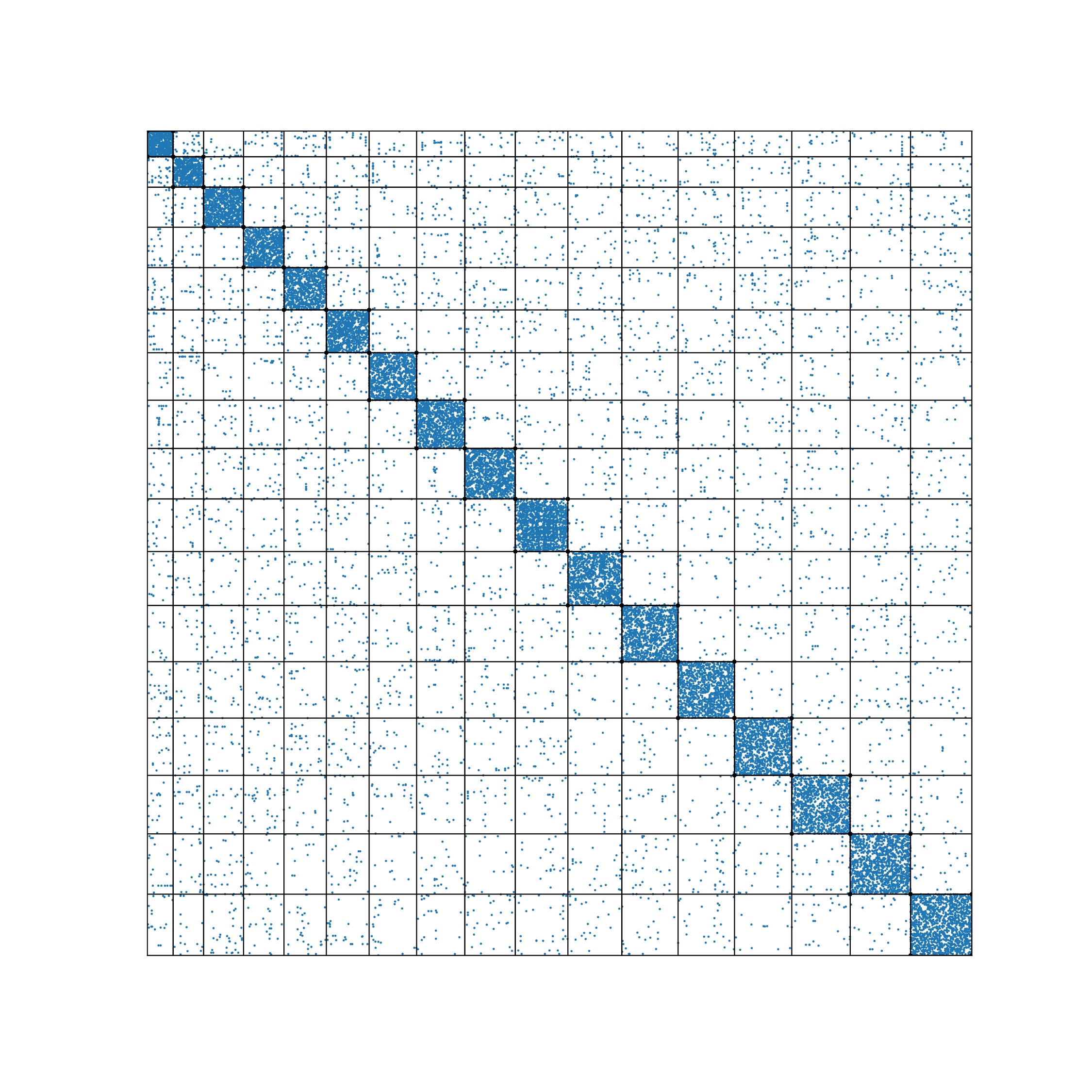} }}%
\hfill
  \subfloat[\textsl{GrQc} $(p=1)$]{{ \includegraphics[width=0.23\columnwidth]{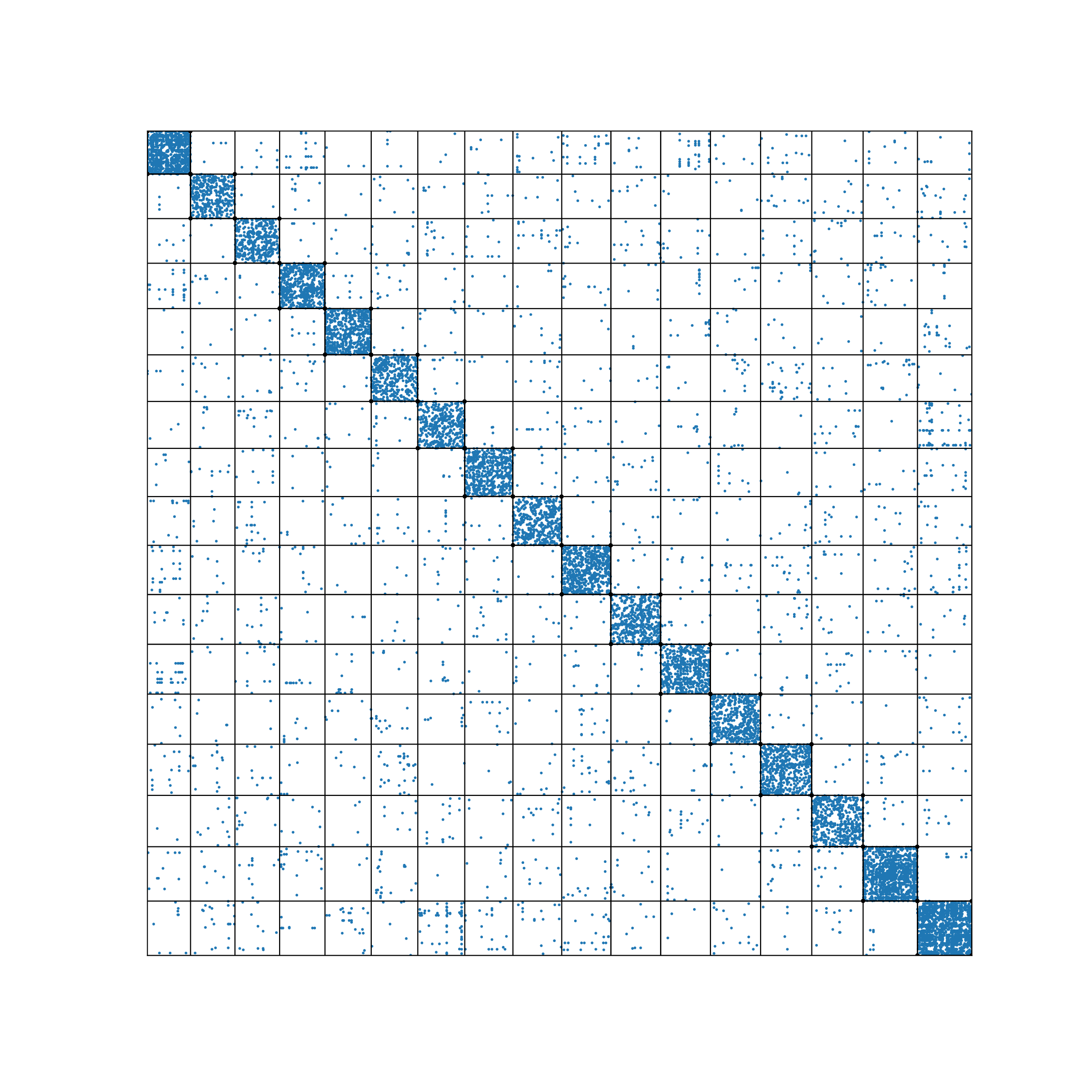} }}%
\hfill
  \subfloat[\textsl{HepTh} $(p=1)$]{{ \includegraphics[width=0.23\columnwidth]{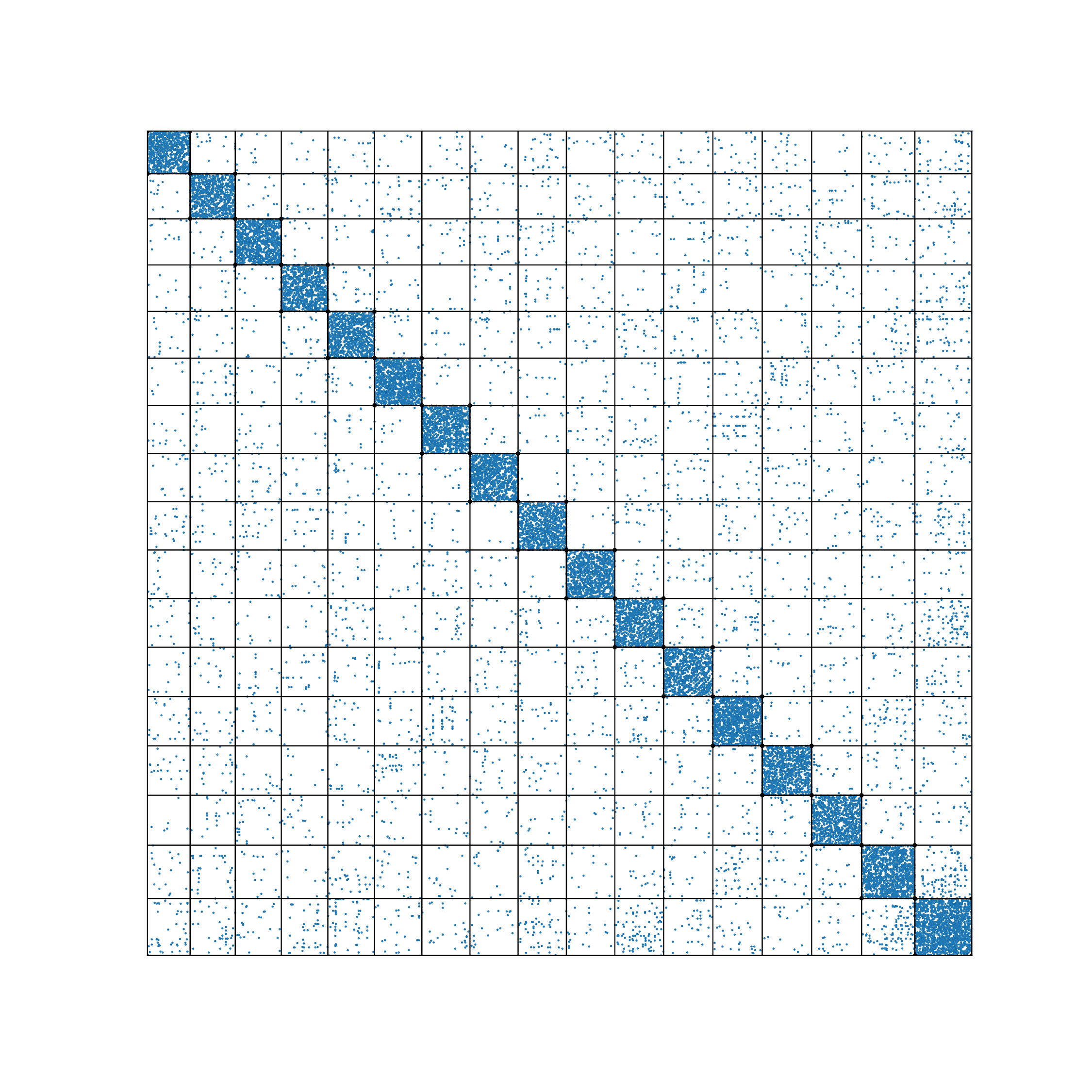} }}%
  \caption{Ordered adjacency matrices based on the memberships of a $D=16$ dimensional \textsc{HM-LDM} with $\delta$ values ensuring identifiability \cite{hmldm}.}\label{fig:adj}
\end{figure}

\part{Graph Representation Learning of signed integer weighted
networks}
\chapter{Characterizing Polarization in Social Networks using the
Signed Relational Latent Distance Model}

Unlike traditional networks modeling only positive links or their absence between entities, signed networks can capture more complex relations, such as cooperative and antagonistic ties. They are instrumental in modeling more realistic and richer representations of real social structures. Hence, the analysis of signed networks can reveal significant insights into understanding how the network structure is actually formed. The proverb \textit{``The enemy of my enemy is my friend''} is a very known example demonstrating that driving forces leading individuals to form connections are not merely positive inclinations. The \textit{balance theory} explains these motives by proposing that individuals have an inner desire to provide balance and consistency in their relationships. Specifically, it is a socio-psychological theory admitting four rules: “The
friend of my friend is my friend", “The enemy of my friend is my enemy",  “The friend of my enemy is my enemy", and “The enemy of my enemy is my friend", also presented in Figure \ref{fig:blt}. In addition, signed networks can help us understand better ideological, as well as, affective polarization phenomena as present in social networks, as signed networks capture positive, negative, and neutral relationships between nodes and can characterize opposing views more accurately than unsigned networks. Ideological polarization refers to the substantial differences in how certain policies are viewed by elite or specialized groups, such as politicians, academics, or thought leaders. These groups might have widely divergent opinions on issues such as economic policies, social justice, or foreign relations. Essentially, ideological polarization is about the "what" of political disagreements. Contrary to ideological polarization which focuses on differing opinions on policies, affective polarization refers to how ordinary voters feel about those differences, including strong emotions, such as anger or fear, that voters may feel towards the policy positions of parties or individuals they oppose. Affective polarization is more about "how" people feel about political disagreements rather than the content of the disagreements themselves. In addition, the media often presents the differences in policy positions in an extreme light, portraying them as existential or life-and-death threats. The culmination of these factors can lead to a divisive mentality ("us-versus-them") where different sides see each other not just as opponents with differing views but as existential threats. This binary view can stifle productive dialogue and compromise, leading to further polarization and possibly even hostility between different factions within society.

A first necessity to address and understand polarization phenomena is to devise a powerful framework for the analysis of signed networks. For that, we turn to the family of Latent Distance Models. Contrary to the case of unsigned networks, we now require on top of the homophily properties of a model to also be able to express animosity/heterophily in the latent space. This is necessary in order to extend the so-important transitivity properties present in the unsigned case with the more general balance theory. In addition, such a model should provide a valid likelihood function, describing both positive and negative interactions, as well as, defining a generative process over signed networks. In such a direction we turn to the Skellam distribution, a discrete probability distribution of the difference between two independent Poisson random variables, and extend Latent Distance Models forming the Skellam Latent Distance Model. In order to address and capture polarization phenomena we turn to Archetypal Analysis (AA) as introduced for observational data, and extend it to the analysis of relational data. Specifically, we focus on extreme positions and argue that the "us-versus-them" multipolarity, reinforced by homophily, animosity, and balance theory, can be represented by a latent position model. This model is applied to networks that are confined to a social space formed by a polytope akin to Archetypal Analysis, which we refer to as a "sociotope." The corners of the sociotope represent distinct aspects/poles formed by polarized network tendencies, where positive ties reinforce homophily among similar individuals, and negative ties repel dissimilar individuals to opposing poles. These multiple poles are important for defining the corners of the sociotope and revealing the different aspects of the social network. Thus for the modeling of signed networks and for the characterization of polarization, we present the Signed Relational Latent Distance Model, a combination of the Skellam Latent Distance Model and Archetypal Analysis.

\begin{figure}[!t]
    \centering
    \includegraphics[width=0.99\columnwidth]{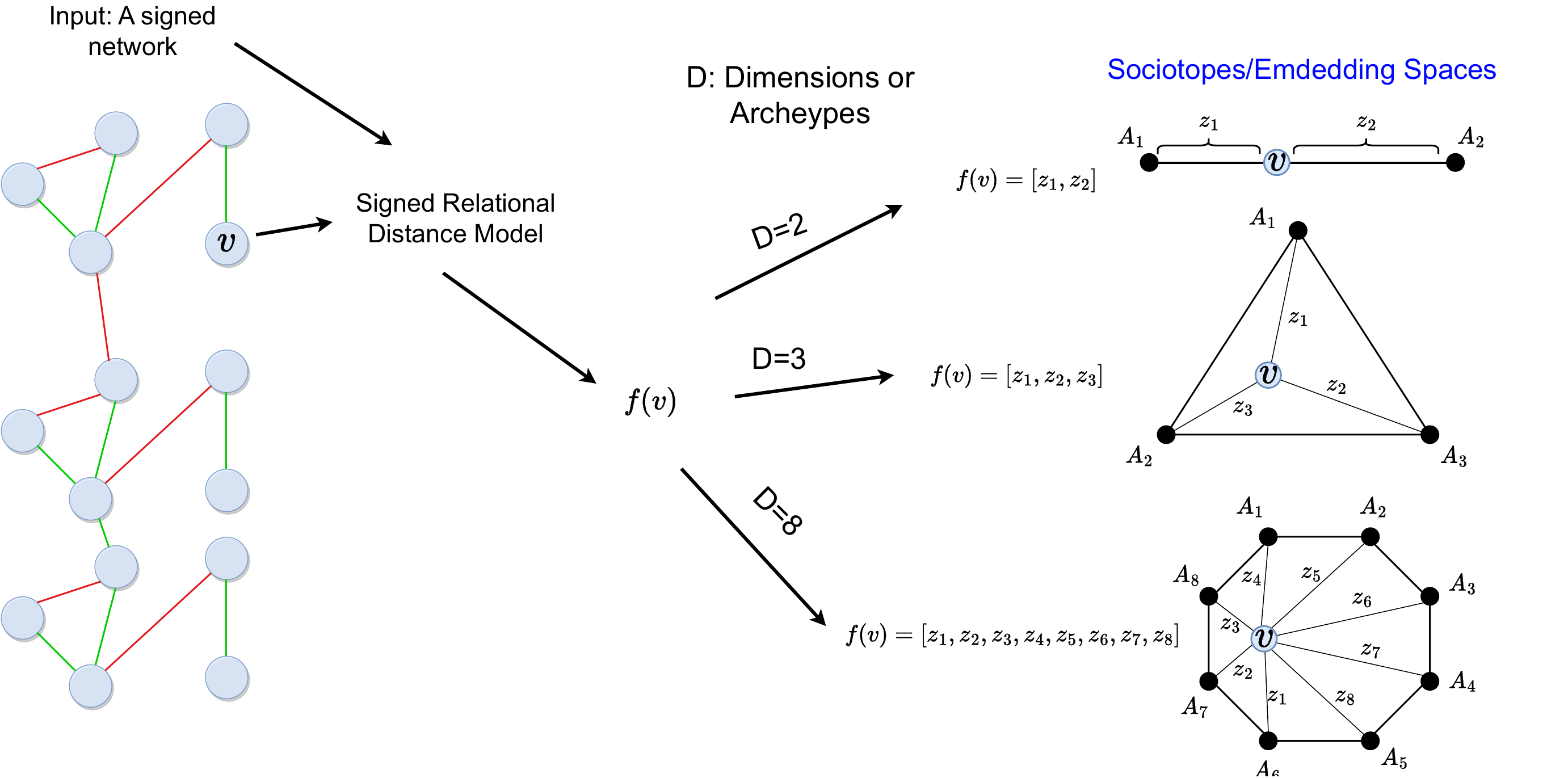}
    \caption{Signed Relational Latent Distance Model procedure overview. Analytically, for a given dimensionality and a signed network as inputs, the goal is to find a mapping function $f(\cdot)$ that projects a network node (e.g. $\{u\}$) into a latent space that is constrained to a polytope/sociotope, with every corner defining an archetype/extreme profile. Any node representation is characterized as a convex combination of the archetypes as these are the corner points of the convex hull defined by matrix $\bm{A}$. Sociotopes having dimensionality $D=3,8$ are denoted in a two-dimensional space for visualization purposes only. }
    \label{fig:slim_}
\end{figure}

\section{Contributions}
We extend Latent Distance Models to the analysis of signed networks, utilizing for the first time the Skellam distribution as network likelihood forming the Skellam Latent Distance Model \cite{slim}. We show how such a model naturally conveys balance theory which comes as a direct consequence of the expression of homophily and heterophily properties our modeling design offers. We then reconcile Archetypal Analysis with the Skellam Latent Distance Model forming the Signed Relational Latent Distance Model and allowing for the characterization of polarization in terms of participation to extreme views or profiles as uncovered by the model. We extensively evaluate the performance of our frameworks and on four real social signed networks of polarization, we demonstrate that the models extract low-dimensional characterizations that well predict friendship and animosity while the Signed Relational Latent Distance Model provides interpretable visualizations defined by extreme positions when restricting the embedding space to polytopes akin to Archetypal Analysis. Furthermore, we successfully showcase a generative process allowing for the creation of networks with a controlled level of polarization while we further show how our frameworks generate accurate network representation when learning from real networks. The procedure overview is provided in Figure \ref{fig:slim_}. Analytically our contributions are outlined as:

\begin{itemize}
    \item We, for the first time, utilize the Skellam distribution as a network likelihood forming the Skellam Latent Distance model which satisfies the balance theory. We further, under the Skellam distribution, provide rate specifications allowing for different levels of model capacity, performance, latent space interpretation, and the modeling of directed and undirected relationships.

    \item We present the Signed Relational Latent Distance Model, a novel method that extends Archetypal Analysis to relational data. We discuss how such a model successfully characterizes network polarization based on the discovery of distinct and extreme profiles being present in signed networks.

     \item We, contrary to the state-of-art, define generative models capable of generating signed networks of different polarization levels. Furthermore, we showcase the generative capabilities of our model on both real and artificial data, experimentally verifying that our model formulation can distinguish the different levels of network polarization.

    \item We extensively benchmark our proposed model against state-of-the-art \textsc{GRL} baselines designed for the analysis of signed networks. In multiple task settings, including sign link prediction, as well as, the more challenging task of signed link prediction, our model returns superior performance in most cases. 

    \item We showcase how sociotope visualizations facilitate the characterization of network polarization, and importantly the successful discovery of influential nodes behaving as the driving forces of polarization for both directed and undirected settings.

\end{itemize}

\section{Experimental design, results, and key findings} We employ four networks, describing electoral voting records and opinions. We benchmark the performance of our proposed frameworks against five prominent signed Graph Representation Learning methods, including random-walk-based methods and graph neural networks. We create a test set by removing $20\%$ of the total network links while preserving connectivity on the residual network. We define two prediction tasks, \textit{Link sign prediction ($p@n$):} In this setting, we utilize the link test set containing the negative/positive cases of removed connections. We then ask the models to predict the sign of the removed links. \textit{Signed link prediction:} A more challenging task is to predict removed links against disconnected pairs of the network, as well as, infer the sign of each link correctly. For that, the test set is split into two subsets positive/disconnected and negative/disconnected. We then evaluate the performance of each model on those subsets. The tasks of signed link prediction between positive and zero samples are denoted as $p@z$ while the negative against zero is $n@z$. Furthermore, as the Signed Relational Latent Distance Model formulation facilitates the inference of a polytope describing the distinct aspects of networks, we visualize the latent space across $D=8$ dimensions for all of the corresponding networks. To facilitate visualizations we use Principal Component Analysis (PCA), and project the space based on the first two principal components of the final embedding matrix. In addition, we provide circular plots where each archetype of the polytope is mapped to a circle every $\text{rad}_d=\frac{2\pi}{D}$ radians, with $D$ being the number of archetypes. Such polytope visualizations can be found in Figure \ref{fig:soc_viz}.

During the evaluation, we focused on certain scoring metrics that are suitable for highly imbalanced data sets, specifically the AUC-ROC score and the AUC-PR score. We applied these scores to assess two particular tasks: link sign prediction and signed link prediction. In these tasks, we found that both the Skellam Latent Distance Model and the Signed Relational Latent Distance Model performed competitively when measured against all baseline models. Specifically, in most cases, our frameworks outperformed significantly most baselines or defined on-par performance against the most competitive ones. What further adds to the appeal of our models is that they are also designed with generative processes (for an example please visit Figure \ref{fig:real_gen_noreg}), making them particularly well-suited for the analysis of signed networks. By visualizing the sociotopes, we illustrate how the polytope method can successfully identify extreme positional nodes within the network. To put it more clearly, in all networks, there is at least one archetypal node that functions as a "dislike" hub, and at least one that operates as a "like" hub. These archetypes are characterized by having high values of either negative or positive interactions, respectively. In some networks, we also notice archetypes with a very low degree of connection. This phenomenon can be explained by the fact that some nodes, which are only associated with negative interactions, are pushed away from the main cluster of nodes. These isolated nodes can be considered "outliers" within the sociotope. However, such outliers are not merely anomalies but are discovered since they provide high expressive power for the model (allowing for a large volume of the polytope). (For more details and the full experiment results please visit the full paper \cite{slim}.) 

\section{Conclusion} 
The Skellam Latent Distance Model and Signed Latent Relational Distance Model that we have proposed allow for an easily interpretable visualization of signed networks, performing well in link-related prediction tasks when focusing on weighted signed networks. The Skellam Latent Distance Model extends the representation power of classical LDMs generalizing homophily and transitivity properties to the expression of balance theory. In addition, the Signed Latent Relational Distance Model defines a space that is constrained to polytopes, a feature that enables us to identify unique characteristics in social networks. This allows for the detection of extreme positions within the network, a process similar to traditional archetypal analysis but adapted for graph-structured data. The Skellam distribution is particularly useful for the modeling of signed networks, adding depth to our understanding of these structures; while the relational extension of Archetypal Analysis can be used for different likelihood specifications, under general Latent Distance Models. In summary, this study lays the groundwork for utilizing new likelihood formulations that are suitable for analyzing weighted signed networks while it extends concepts similar to Archetypal Analysis to a broader context, offering a new way to analyze networks.

\begin{figure*}
\centering
\begin{subfigure}{0.28\textwidth}
    \includegraphics[width=\textwidth]{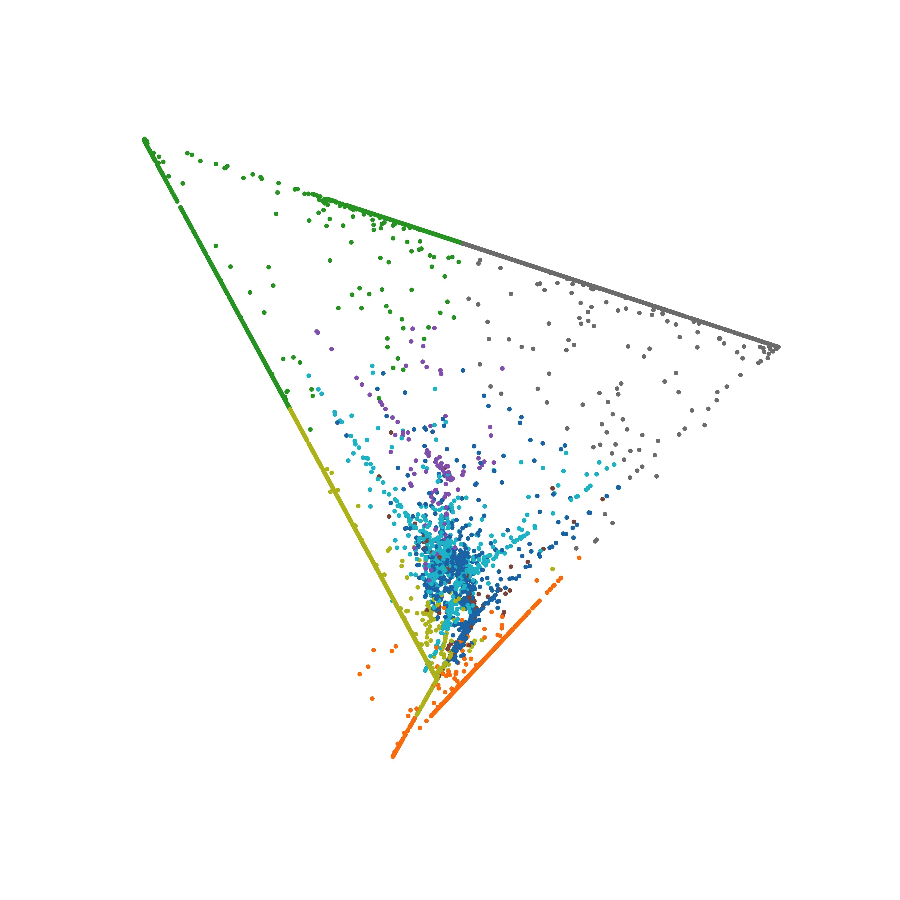}
    \caption{\textsl{WikiElec}}
    \label{fig:first}
\end{subfigure}
\hfill
\begin{subfigure}{0.28\textwidth}
    \includegraphics[width=\textwidth]{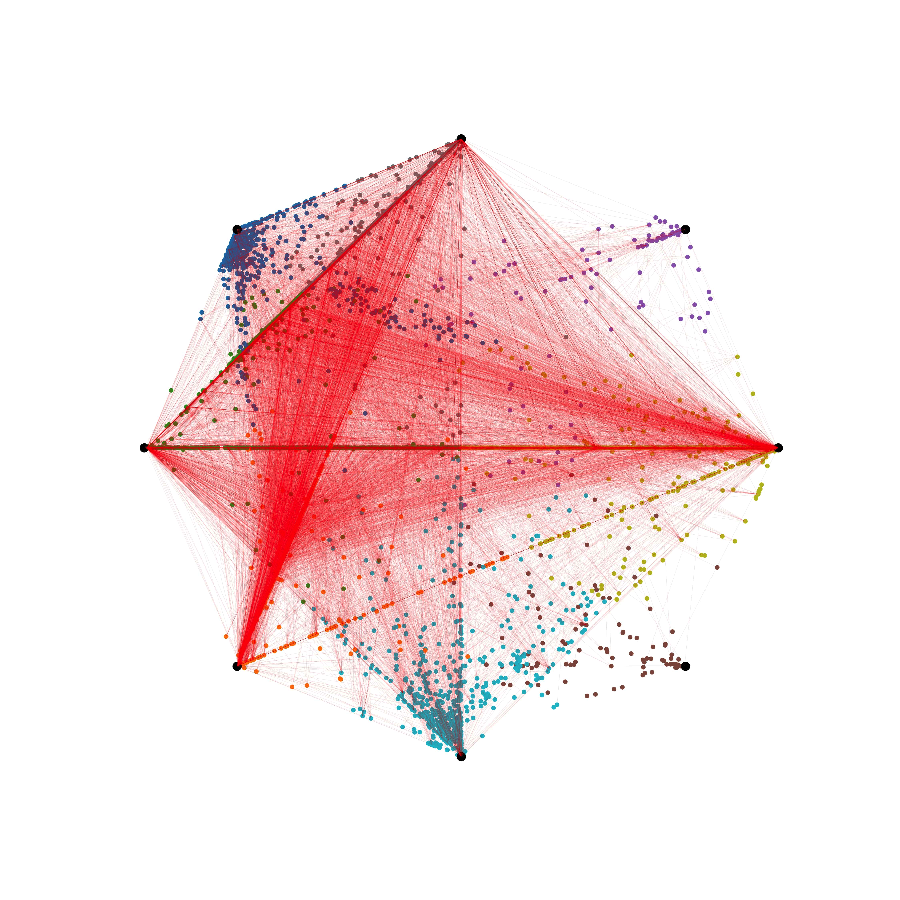}
    \caption{\textsl{WikiElec} \cite{dataset_wikielec}}
    \label{fig:second}
\end{subfigure}
\hfill
\begin{subfigure}{0.28\textwidth}
    \includegraphics[width=\textwidth]{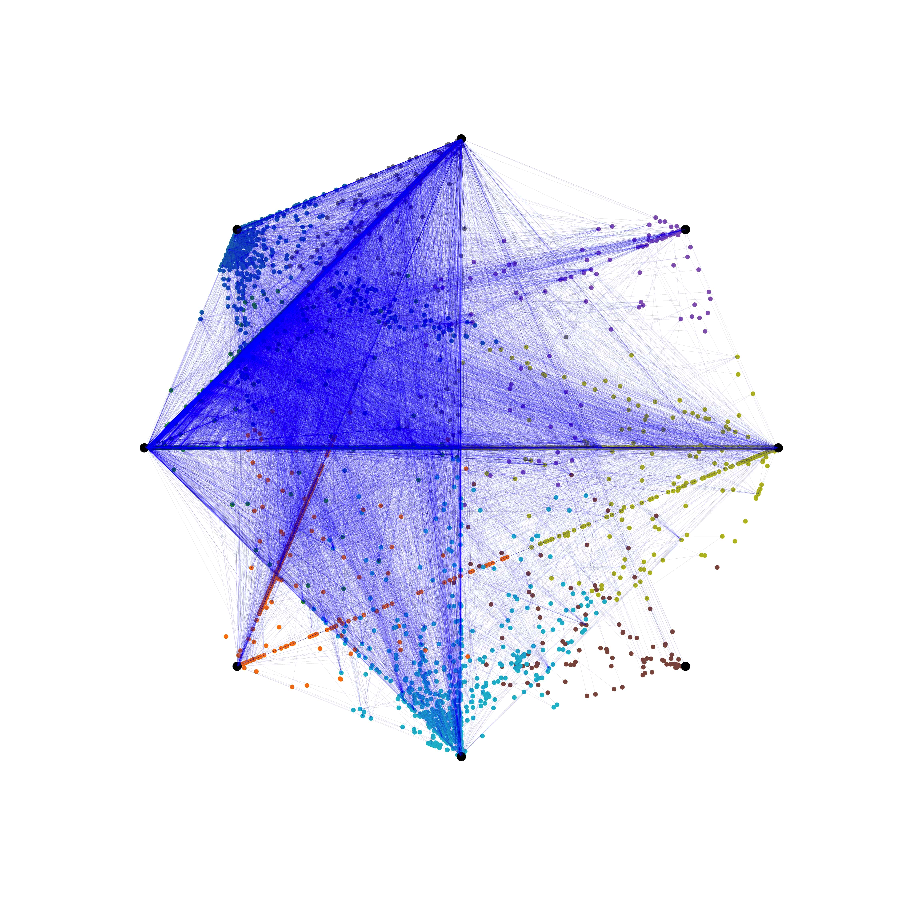}
    \caption{\textsl{WikiElec}}
    \label{fig:third}
\end{subfigure}
\begin{subfigure}{0.28\textwidth}
    \includegraphics[width=\textwidth]{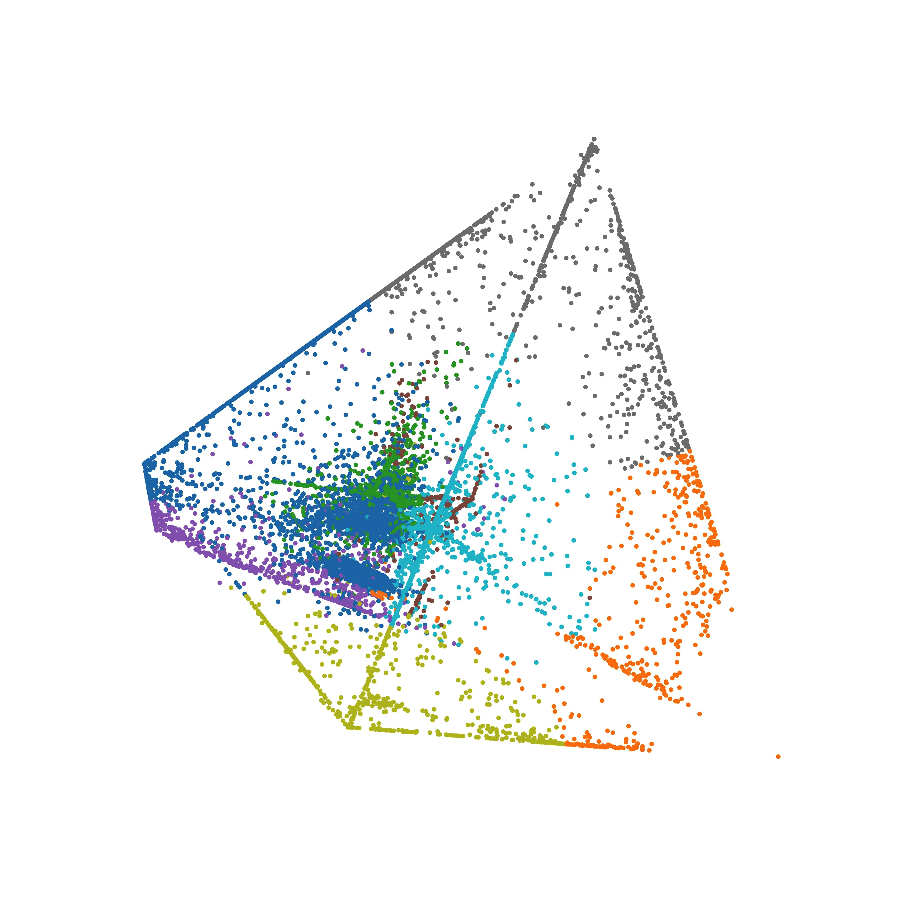}
    \caption{\textsl{WikiRfa}\cite{dataset_wikirfa}}
    \label{fig:first-2}
\end{subfigure}
\hfill
\begin{subfigure}{0.28\textwidth}
    \includegraphics[width=\textwidth]{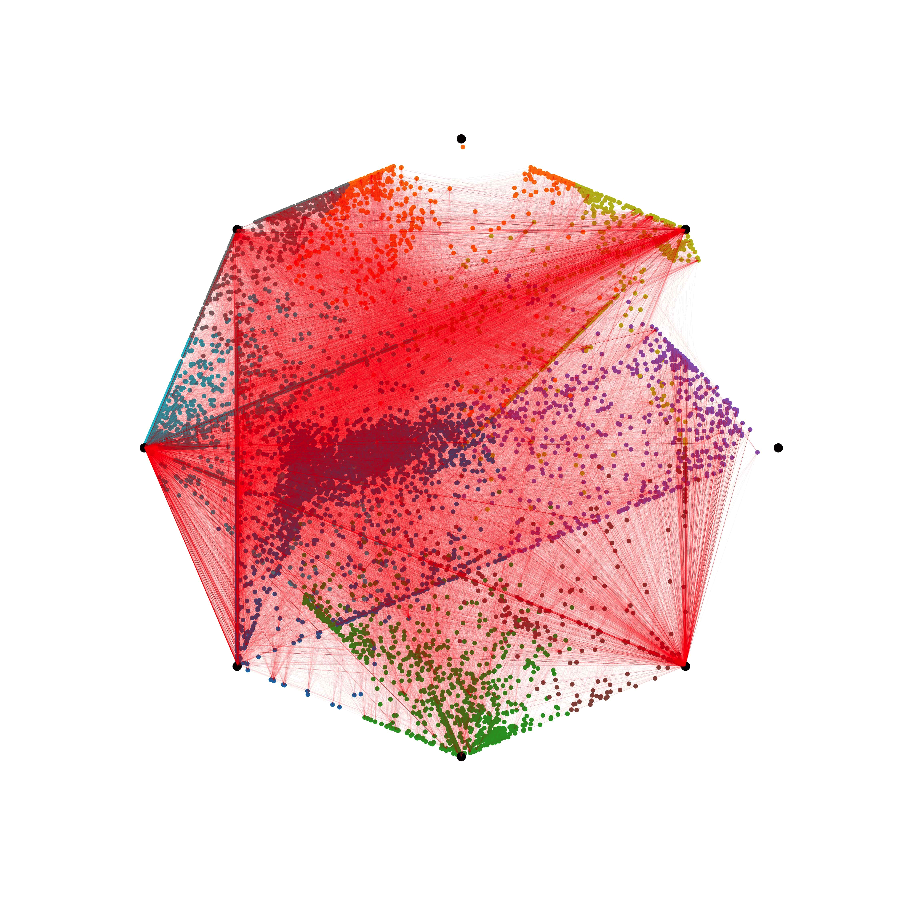}
    \caption{\textsl{WikiRfa}}
    \label{fig:second-2}
\end{subfigure}
\hfill
\begin{subfigure}{0.28\textwidth}
    \includegraphics[width=\textwidth]{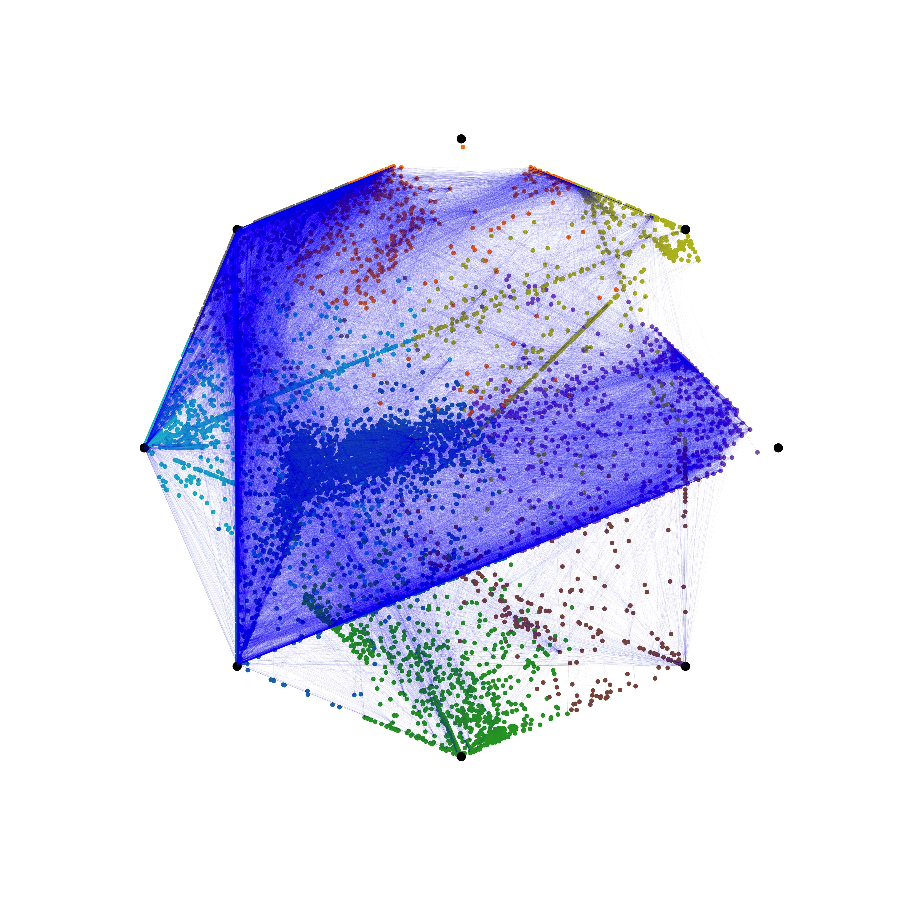}
    \caption{\textsl{WikiRfa}}
    \label{fig:third-2}
\end{subfigure}
\begin{subfigure}{0.28\textwidth}
    \includegraphics[width=\textwidth]{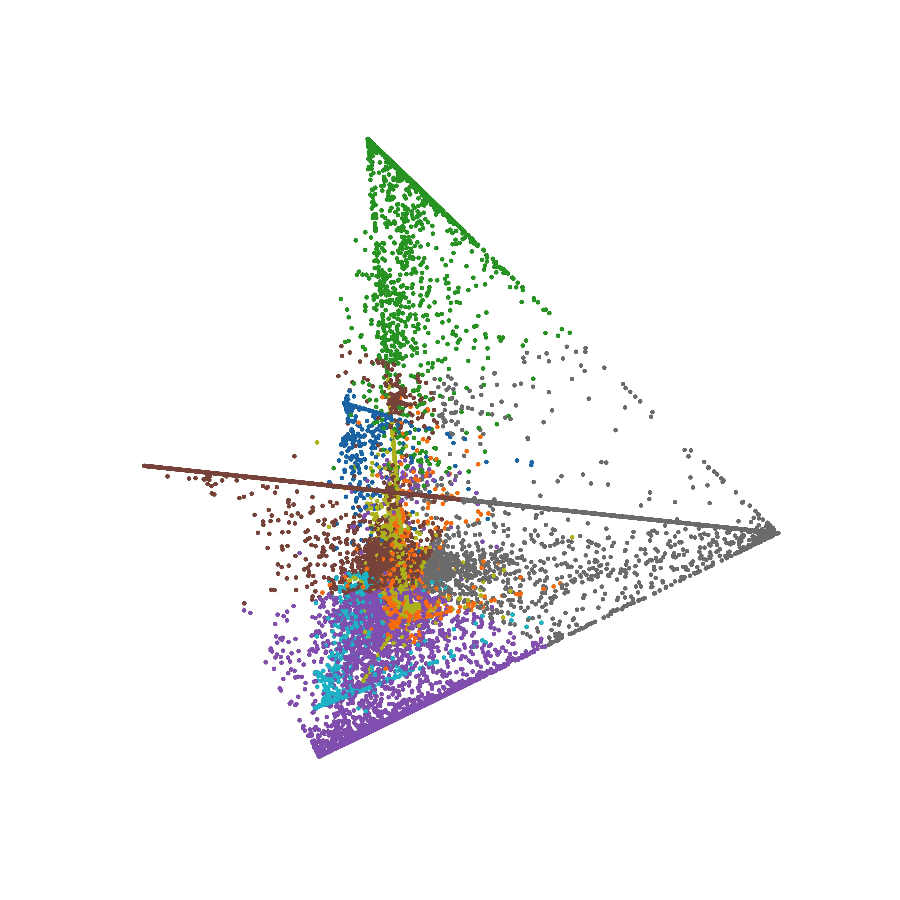}
    \caption{\textsl{Reddit}\cite{dataset_reddit}}
    \label{fig:first-3}
\end{subfigure}
\hfill
\begin{subfigure}{0.28\textwidth}
    \includegraphics[width=\textwidth]{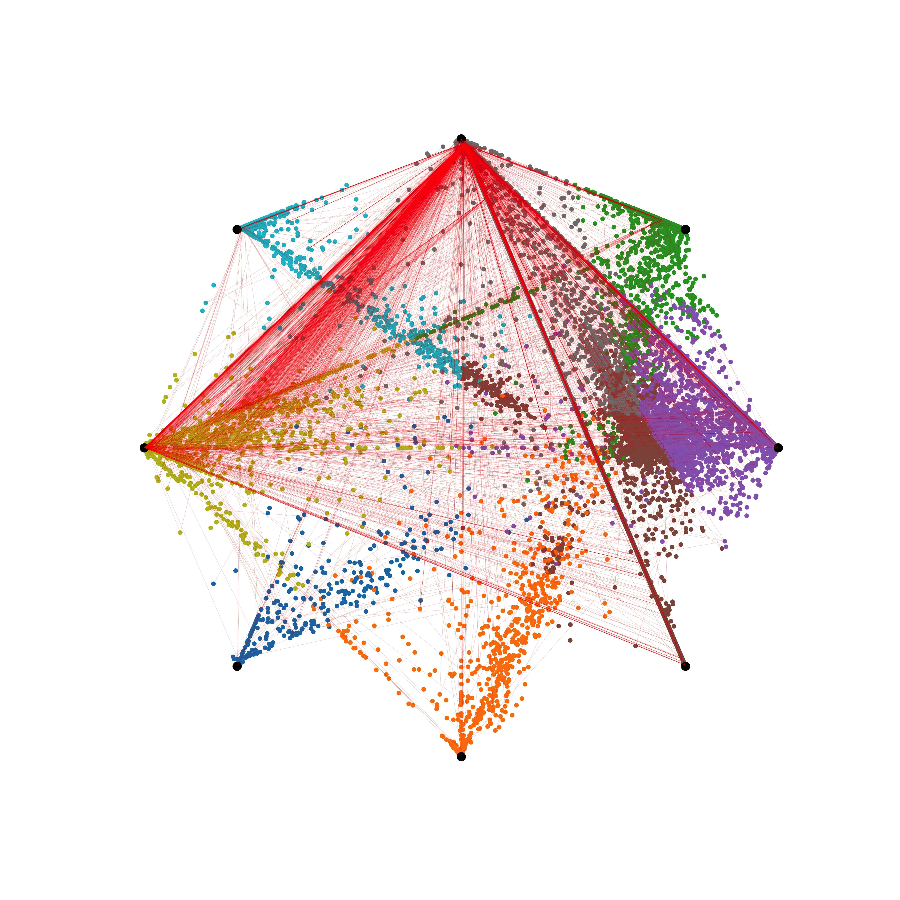}
    \caption{\textsl{Reddit}}
    \label{fig:second-3}
\end{subfigure}
\hfill
\begin{subfigure}{0.28\textwidth}
    \includegraphics[width=\textwidth]{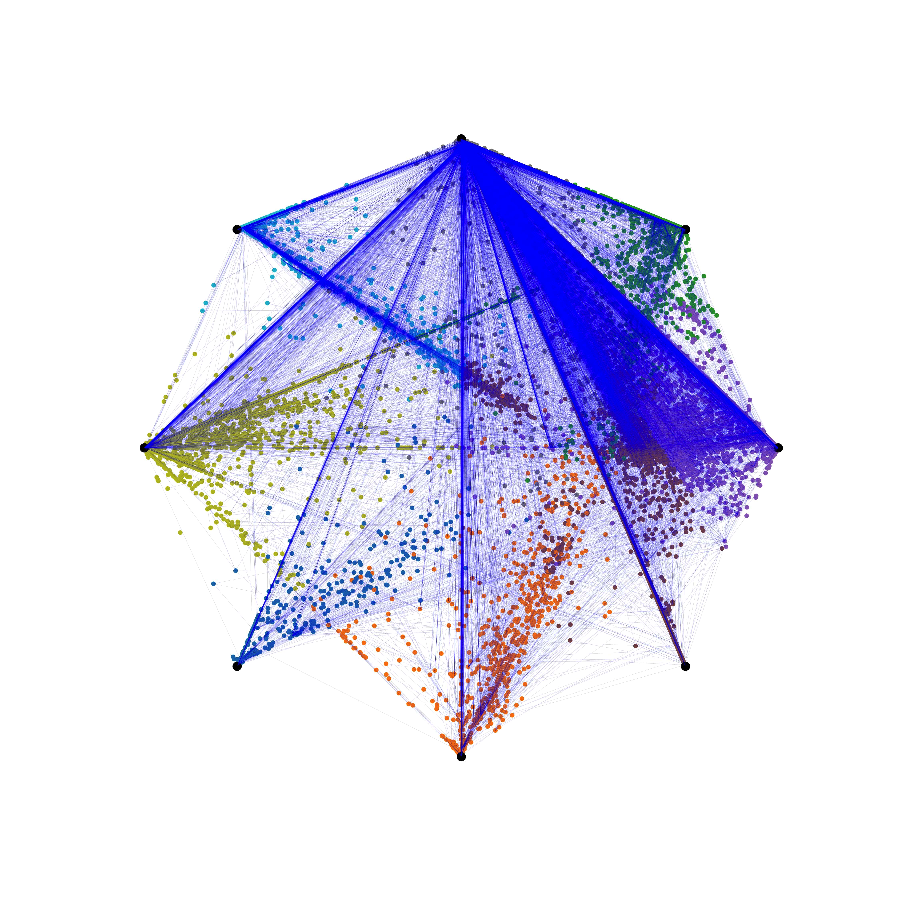}
    \caption{\textsl{Reddit}}
    \label{fig:third-3}
\end{subfigure}
\begin{subfigure}{0.28\textwidth}
    \includegraphics[width=\textwidth]{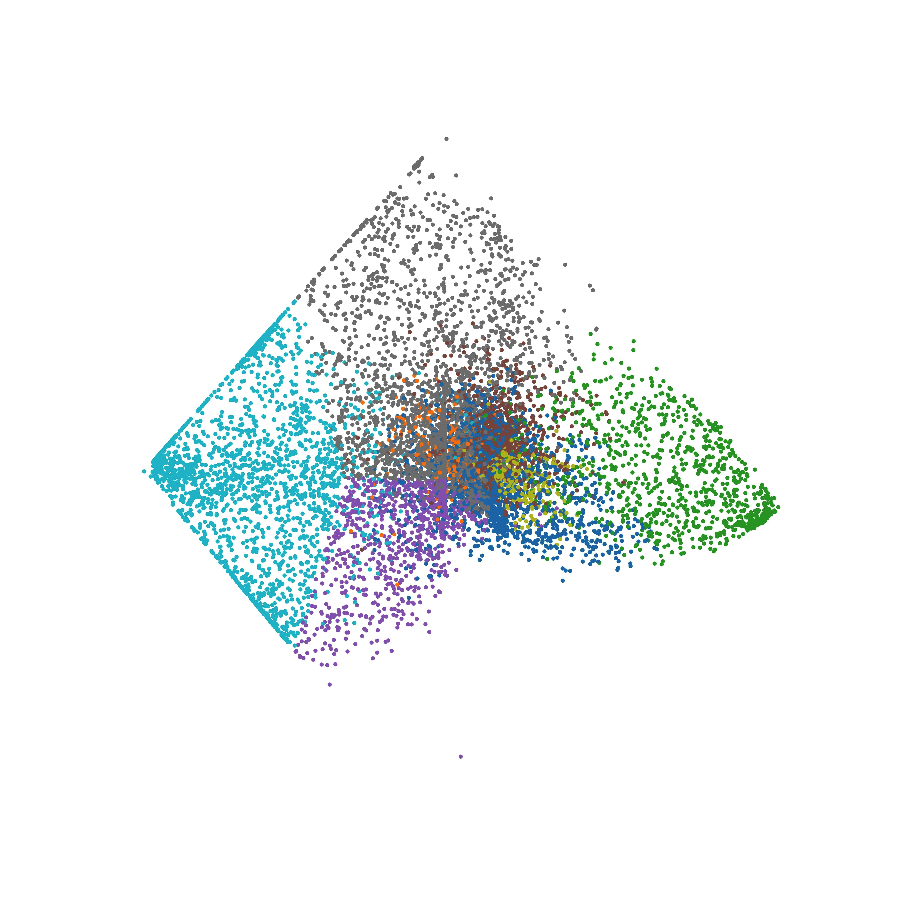}
    \caption{\textsl{Twitter} \cite{dataset_twitter}}
    \label{fig:first-4}
\end{subfigure}
\hfill
\begin{subfigure}{0.28\textwidth}
    \includegraphics[width=\textwidth]{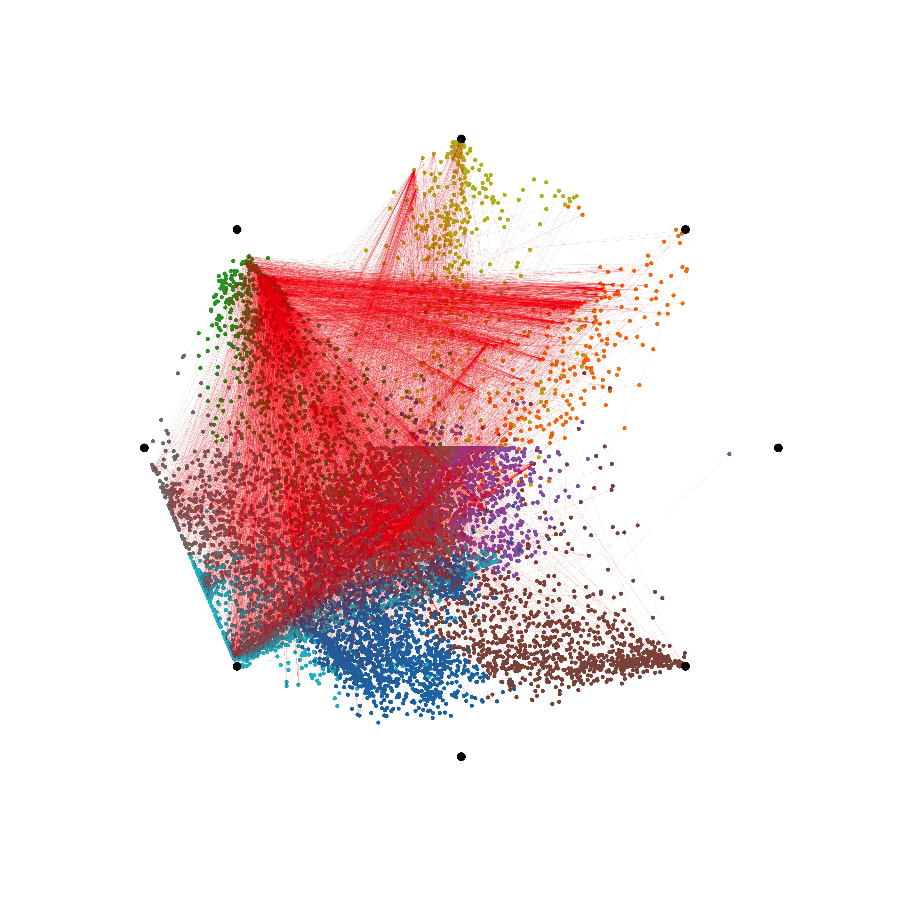}
    \caption{\textsl{Twitter}}
    \label{fig:second-5}
\end{subfigure}
\hfill
\begin{subfigure}{0.28\textwidth}
    \includegraphics[width=\textwidth]{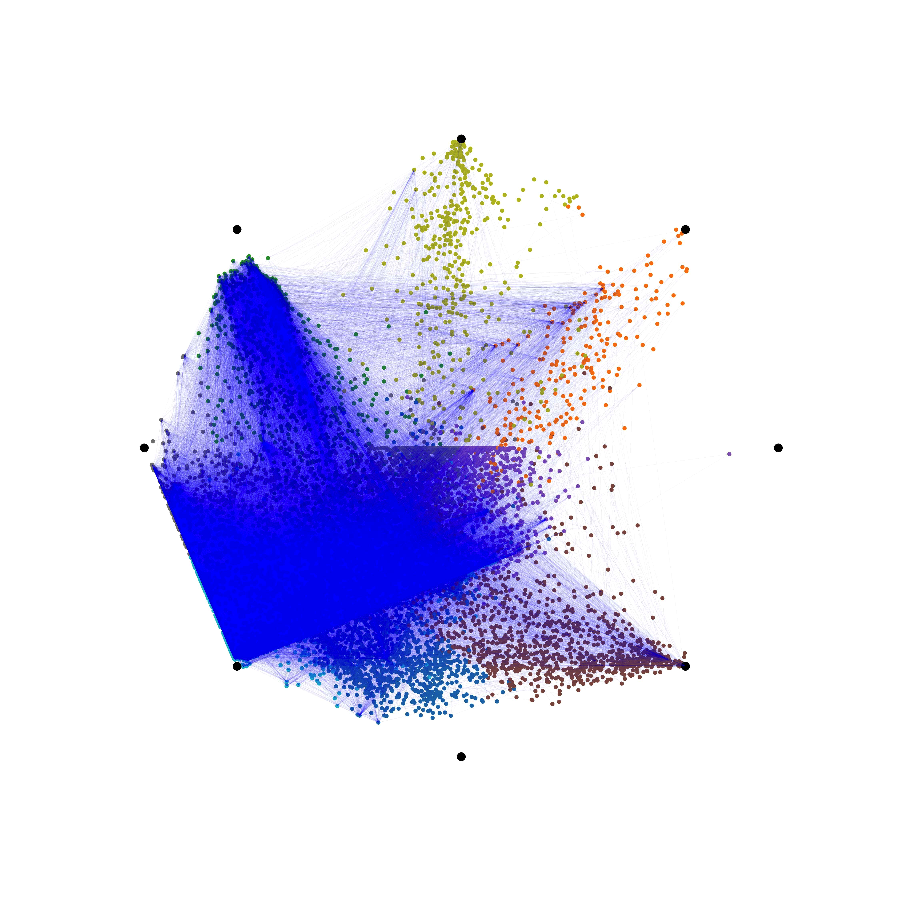}
    \caption{\textsl{Twitter}}
    \label{fig:third-6}
\end{subfigure}
\caption{Inferred polytope visualizations for various networks. The first column showcases the $K=8$ dimensional sociotope projected on the first two principal components (PCA) --- second and third columns provide circular plots of the sociotope enriched with the negative (red) and positive (blue) links, respectively \cite{slim}.}
\label{fig:soc_viz}
\end{figure*}

\begin{figure*}[!h]
\centering
\begin{subfigure}{.48\textwidth}
  \centering
  \includegraphics[width=\textwidth]{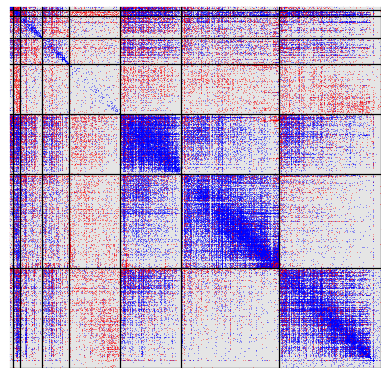}
  \caption{Ground Truth: $(.003,78\%,22\%)$}
\end{subfigure}
\begin{subfigure}{.48\textwidth}
  \centering
  \includegraphics[width=\textwidth]{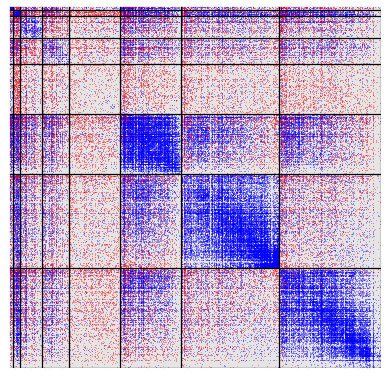}
  \caption{Generated: $(.003,76\%,24\%)$}
\end{subfigure}
\caption{\textsl{wikiElec} ground truth (left) adjacency matrix and generated (right) adjacency matrix based on inferred parameters with a \textsc{SLIM} \textbf{without regularization priors} over the parameters. The parenthesis shows the network statistics as: (density,\% of positive (blue) links,\% of negative (red) links). All network adjacency matrices are ordered based on $\mathbf{z}_i$, in terms of maximum archetype membership and internally according to the magnitude of the corresponding archetype most used for their reconstruction \cite{slim}.}
\label{fig:real_gen_noreg}
\end{figure*}
\chapter{A Hybrid Membership Latent Distance Model for Signed
Integer Weighted Networks}

Signed networks, unlike traditional networks that model only positive and neutral connections, capture complex relations like cooperation and antagonism, providing a more realistic view of social structures. The balance theory with its four rules exemplifies the driving forces behind these connections, encompassing positive, negative, and neutral relationships. These concepts allow for a better understanding of phenomena such as ideological and affective polarization, which involve significant differences in how policies are viewed by various groups and the intense emotions voters may feel towards differing positions. Usually, signed networks contain nodes concentrating a high degree of both positive and negative ties. Such high-degree nodes act as driving forces of polarization, forming an archetypal ideology. This can be realized when considering the properties of balance theory, e.g. "The enemy of my friend is my enemy". In such cases, extreme profiles in networks can be easily uncovered by a model constraining the network projection into polytopes. Unfortunately, there is no guarantee that such "pure" nodes will be always present in polarized networks. For that, additional approaches have been developed such as Minimum Volume, where data representations are constrained to a polytope under a minimum volume constraint. As the volume decreases nodes are pushed to the corners of the defined space providing extreme profile characterizations akin to archetypal analysis. Such procedures are traditionally expressed by a high computational complexity since the volume calculation of a polytope requires calculating the sum of determinants for all simplexes used to construct the polytope which is particularly expensive, especially in high dimensions. Finally, many \textsc{GRL} approaches also do not provide identifiable or unique solution guarantees, so their interpretation highly depends on the initialization of the hyper-parameters, leading to the non-unique characterization of latent structures

To derive an efficient Minumum Volume approach we turn to Latent Distance models and the Skellam distribution to form the Signed Hybrid-Membership Latent Distance Model. This new model, inspired by recent advances in Graph Representation Learning \cite{slim}, is designed to highlight and uncover the unique characteristics of signed networks. Specifically, we constrain the latent space to the $D$-simplex. We show that the Signed Hybrid-Membership Latent Distance Model relates to archetypal analysis for relational data as a minimal volume approach and as a special case when polytopes are constrained to simplexes. Extraction of distinct aspects/profiles through MV does not require the presence of ``pure'' observations defining the convex-hull or else the extracted polytope/simplex. As the volume decreases, observations are ``forced'' to populate the corners of the polytope, yielding archetypal characterization when the reconstruction of data is defined through convex combinations of these corners. Based on the volume size we are able to control the type of memberships in these convex combinations. Specifically, we show that large volumes allow nodes to be expressed through many archetypes but as the volume decreases trade-offs are emerging, forcing nodes to collapse onto a unique archetype. Furthermore, constraining the polytope to the $D$-simplex allows for a trivial volume calculation which we can control simply by the edge length ($1$-faces) value of the simplex. We denote the edge length of the $D$-simplex as $\delta$ which is provided to the model as a continuously decreasing hyperparameter and as a consequence the model defines a continuously decreasing simplex volume, yielding archetypal characterization. Under such a formulation, we provide uniqueness guarantees by extending the Non-Negative Matrix Factorization theory to the study of signed networks which are achieved up to a permutation matrix.

 \begin{figure}[!t]
    \centering
    \includegraphics[width=0.99\columnwidth]{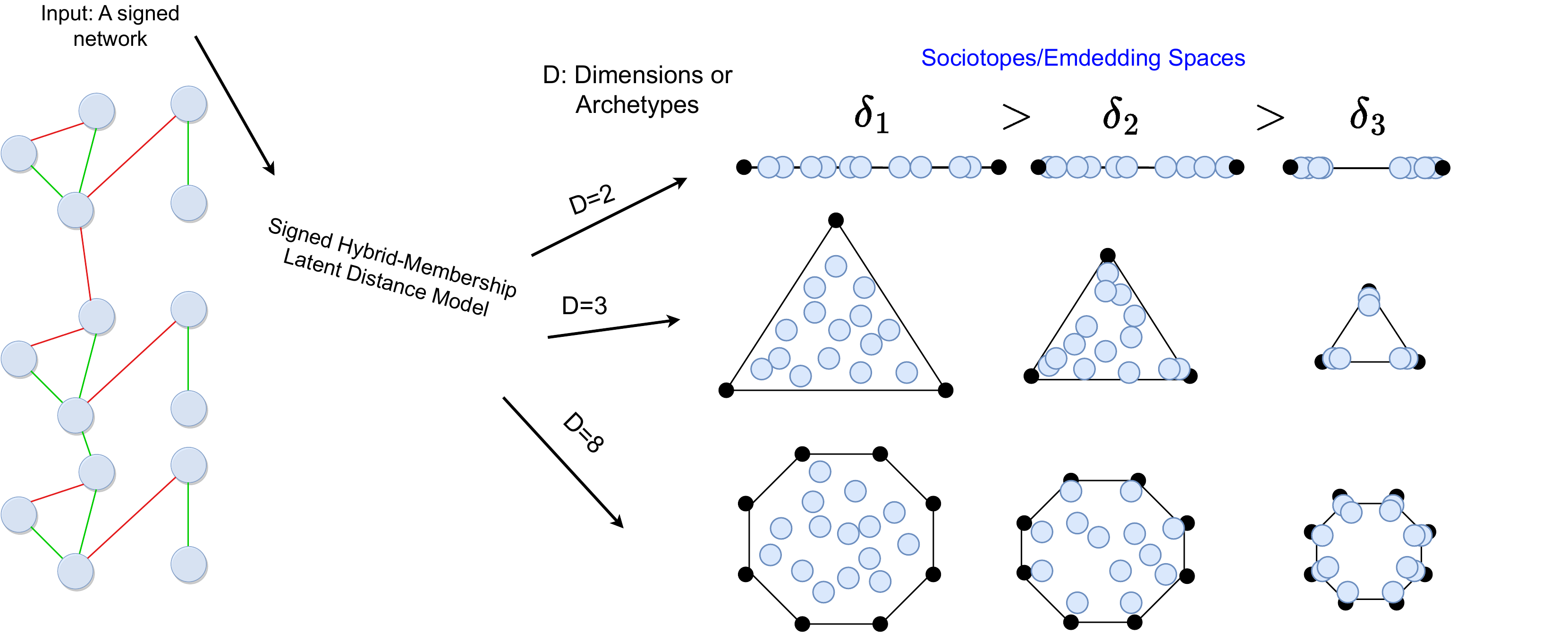}
    \caption{Signed Hybrid-Membership Latent Distance Model procedure overview. Analytically, for a given dimensionality $D$ and a signed network as inputs, the model defines the $(D-1)$-simplex with edge length $\delta_i$. As the length decreases the nodes start to populate the corners uncovering extreme profiles present in the graph data. This corresponds to the Archetypal Analysis of relational data when the polytope is constrained to the $(D-1)$-simplex and naturally extends hybrid memberships coupled with latent distance models to the analysis of signed networks.}
    \label{fig:shmldm_proc}
\end{figure}

\section{Contributions}
We presently derive a Minimum Volume approach for the archetypal characterization and analysis of signed networks forming the Signed Hybrid-Membership Latent Distance Model \cite{shmldm}. Specifically, we show that constraining the polytope to the $D$-simplex comes with no loss of expressive power when the volume of the simplex is not significantly decreased. The model is characterized by high predictive performance in simplex volumes providing identifiable solutions and uncovering distinct aspects of signed networks. Importantly, by controlling the simplex volume we are able to control the type of memberships or participation across the different archetypes/pure forms. Decreasing sufficiently the volume of the latent space forces nodes to converge to their core ideologies/archetypes allowing for the expression of trade-offs. Furthermore, we consider different model specifications utilizing both the traditional Euclidean distance, as well as, the squared Euclidean distance. For the latter, we show that the Signed Hybrid-Membership Latent Distance Model is the combination of a non-negative Eigenmodel expressing homophily and a non-positive Eigenmodel yielding animosity/heterophily properties able to express stochastic equivalence. We benchmark the performance of our model against prominent signed network representation learning approaches and across four real signed networks, while we extend the analysis to two real bipartite networks. The procedure overview is provided in Figure \ref{fig:shmldm_proc}. Analytically our contributions are outlined as:

\begin{itemize}
    \item We, successfully derive a Minimum Volume approach for the analysis of signed networks, offering archetypal characterization. Constraining the polytope to the $D$-simplex, alleviates any computational burdens and restrictions that characterize the volume calculation of high-dimensional general polytopes.

    \item We design and empirically evaluate a continuous optimization procedure over the log-likelihood of the network by altering the latent space/simplex volume, allowing for control
over the memberships across the different archetypes/pure forms.

     \item We provide uniqueness guarantees for the embedding solution as obtained by the Signed Hybrid-Membership Latent Distance Model which is achieved up to permutation invariances.

    \item We show mathematically how a squared Euclidean Skellam Latent Distance Model constrained to the $D$-simplex relates to the Non-Negative Matrix Factorization, defining a non-negative and a non-positive Latent Eigenmodels, and when such these factorizations are unique.

      \item We systematically analyze the trade-offs that soft and hard archetypal characterizations define under the scope of signed link prediction and sign link prediction tasks.

    \item We generalize the method for bipartite networks where archetypal-aware geometric representations, joint embedding spaces, and extreme node discovery are arduous tasks.

    \item We extensively benchmark our proposed model against state-of-the-art \textsc{GRL} baselines, including models utilizing graph neural networks and random walk approaches under various and well-established network data

\end{itemize}

\section{Experimental design, results, and key findings} We employ four unipartite and two bipartite networks, describing electoral voting records and opinions. We benchmark the performance of our proposed frameworks against five prominent signed Graph Representation Learning methods, including random-walk-based methods and graph neural networks. We create a test set by removing $20\%$ of the total network links while preserving connectivity on the residual network. We define two prediction tasks, \textit{Link sign prediction ($p@n$):} In this setting, we utilize the link test set containing the negative/positive cases of removed connections. We then ask the models to predict the sign of the removed links. \textit{Signed link prediction:} A more challenging task is to predict removed links against disconnected pairs of the network, as well as, infer the sign of each link correctly. For that, the test set is split into two subsets positive/disconnected and negative/disconnected. We then evaluate the performance of each model on those subsets. The tasks of signed link prediction between positive and zero samples are denoted as $p@z$ while the negative against zero is $n@z$. Furthermore, as the Signed Relational Latent Distance Model formulation facilitates the inference of a polytope describing the distinct aspects of networks, we visualize the latent space across various dimensions for all of the corresponding networks. To facilitate visualizations we use Principal Component Analysis (PCA), and project the space based on the first two principal components of the final embedding matrix. In addition, we provide circular plots where each archetype of the polytope is mapped to a circle every $\text{rad}_d=\frac{2\pi}{D}$ radians, with $D$ being the number of archetypes. Polytope visualizations for multiple latent dimensions can be found in Figure \ref{fig:bip_signed_p2}.

During the evaluation, we focused on certain scoring metrics that are suitable for highly imbalanced data sets, specifically the AUC-ROC score and the AUC-PR score. We applied these scores to assess two particular tasks: link sign prediction and signed link prediction. In these tasks, we found that the Signed Hybrid-Membership Latent Distance Model performed competitively when measured against all baseline models. Specifically, in most cases, our framework outperformed significantly most baselines or defined on-par performance against the most competitive ones. Surprisingly, when compared to the Skellam Latent Distance Model and the Signed Relational Latent Distance Model which define a higher model capacity our framework defined on-par performance. Controlling the volume of the simplex in the Signed Hybrid-Membership Latent Distance Model showed a small decrease in the predictive performance when the volume was decreased significantly. The type of memberships and participation across different archetypes became hard assignments when the volume again decreased significantly, showcasing a unique archetypal selection for the node reconstruction. Experiments on the two bipartite networks show that the model uncovered successful patterns modeling polarization of voting records both in the U.S. Congress and Senate, as seen by Figure \ref{fig:bip_signed_p2}. (For more details and the full experiment results please visit the full paper \cite{shmldm}.) 

\section{Conclusion} 
The Signed Hybrid-Membership Latent Distance Model allows for the archetypal characterization of signed networks even in the case where pure nodes are not present. Easily interpretable visualizations of signed networks are achieved by drawing the inferred latent space which in addition can provide more specialized interpretations as smaller volumes lead to node reconstructions from a unique archetype. Importantly, uniqueness guarantees allow for the robust interpretation of the inferred solution. Constraining the polytope to the $D$-simplex did not hamper the predictive performance but allowed for control over the type of memberships to the different archetypes.

\begin{figure*}
\centering
\begin{subfigure}{0.24\textwidth}
    \includegraphics[width=0.8\textwidth]{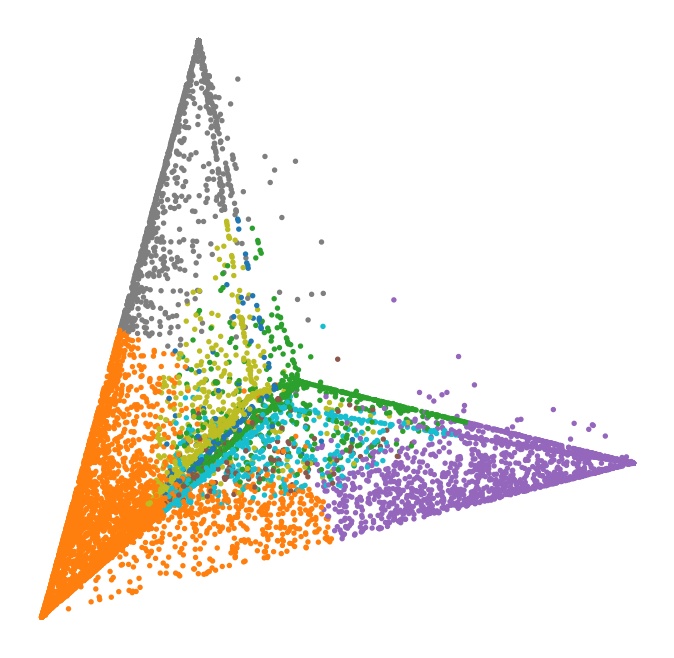}
    \caption{PCA $(D=8)$}
    
\end{subfigure}
\hfill
\begin{subfigure}{0.24\textwidth}
    \includegraphics[width=0.8\textwidth]{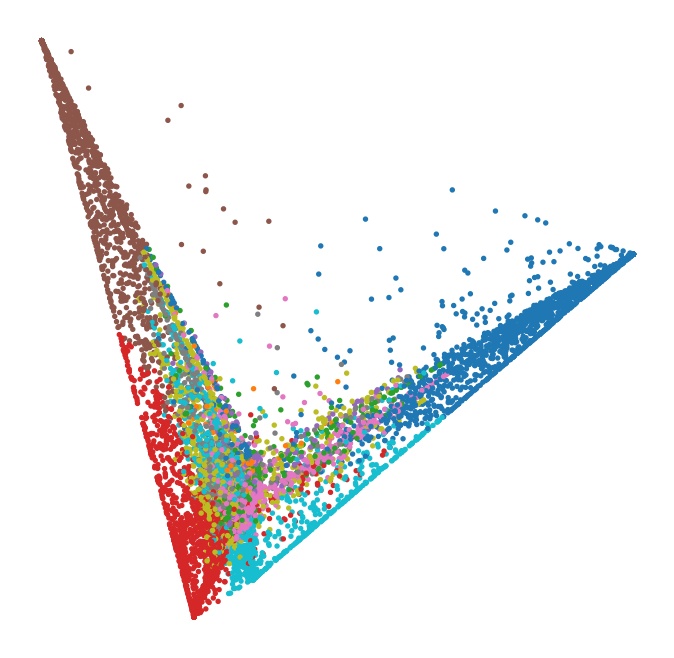}
    \caption{PCA $(D=16)$}
    
\end{subfigure}
\hfill
\begin{subfigure}{0.24\textwidth}
    \includegraphics[width=0.8\textwidth]{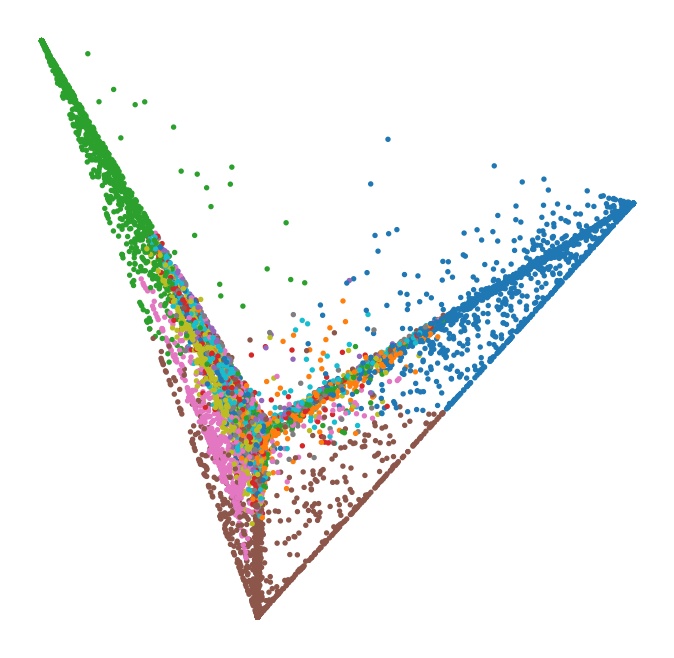}
    \caption{PCA $(D=32)$}
    
\end{subfigure}
\begin{subfigure}{0.24\textwidth}
    \includegraphics[width=0.8\textwidth]{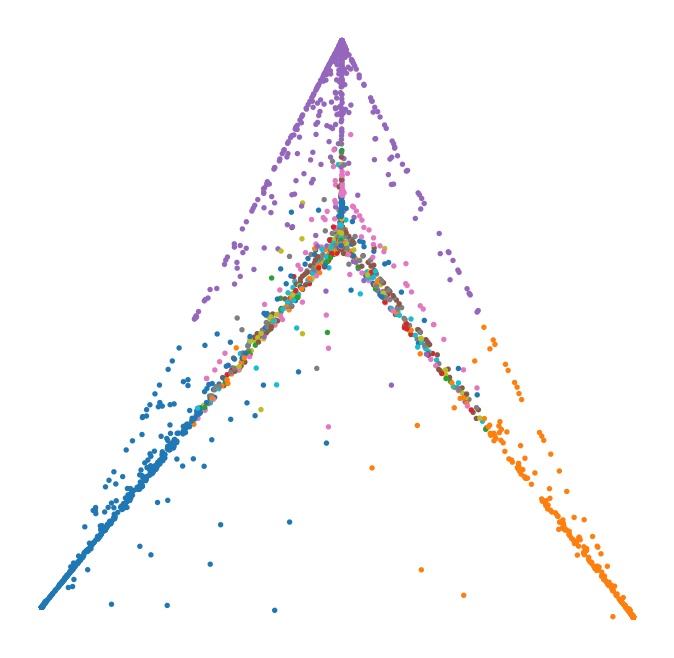}
    \caption{PCA $(D=64)$}
    
\end{subfigure}
\hfill
\begin{subfigure}{0.24\textwidth}
    \includegraphics[width=\textwidth]{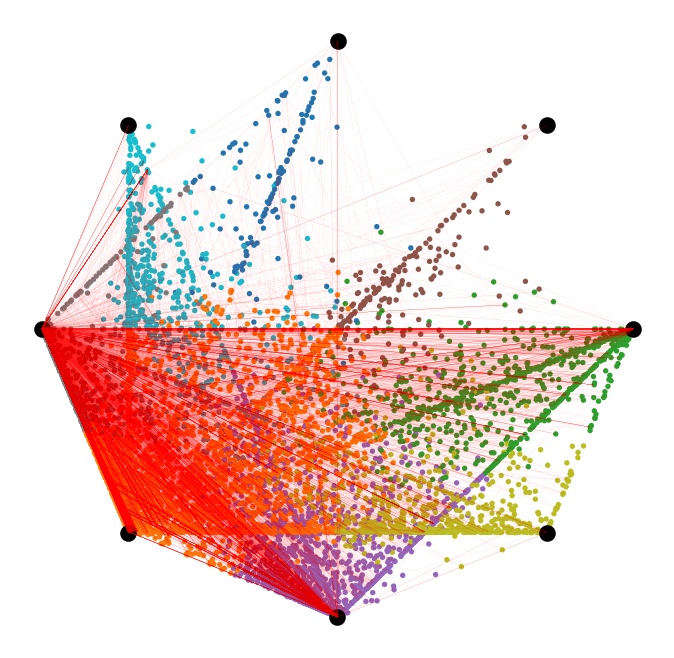}
    \caption{\textsc{NCP} $(D=8)$}
    
\end{subfigure}
\hfill
\begin{subfigure}{0.24\textwidth}
    \includegraphics[width=\textwidth]{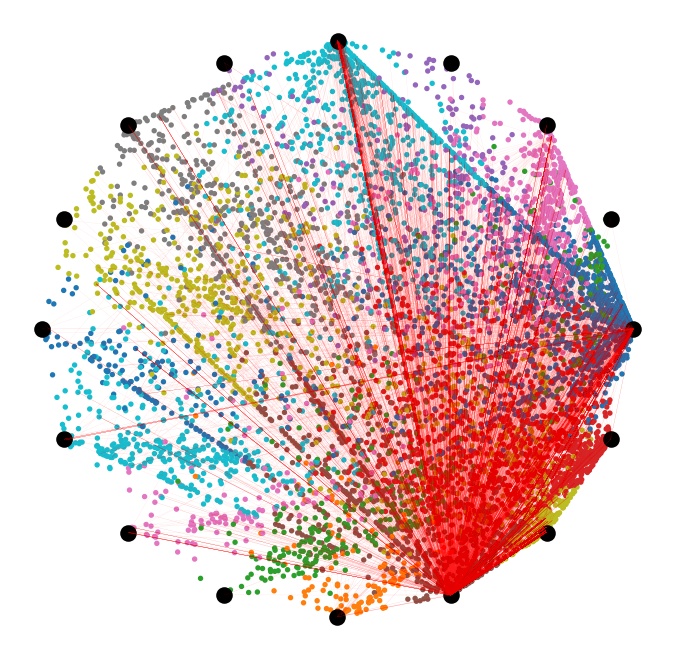}
    \caption{\textsc{NCP} $(D=16)$}
    
\end{subfigure}
\begin{subfigure}{0.24\textwidth}
    \includegraphics[width=\textwidth]{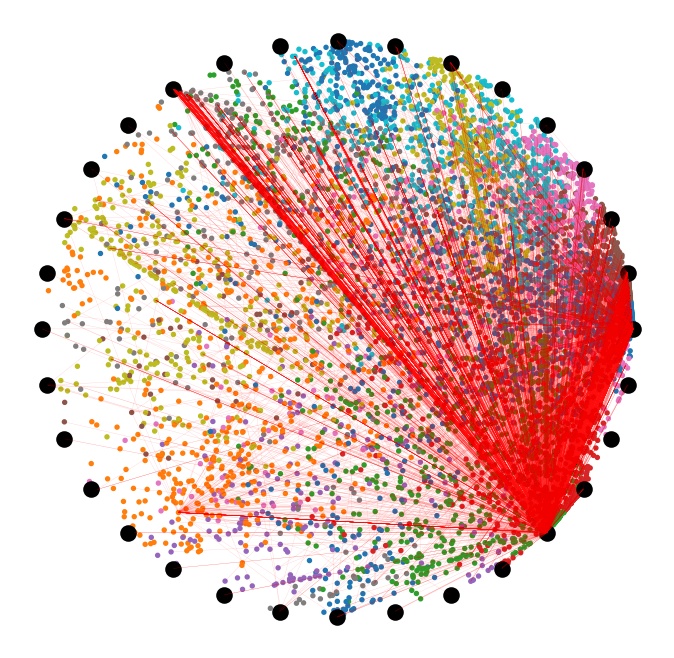}
    \caption{\textsc{NCP} $(D=32)$}
    
\end{subfigure}
\hfill
\begin{subfigure}{0.24\textwidth}
    \includegraphics[width=\textwidth]{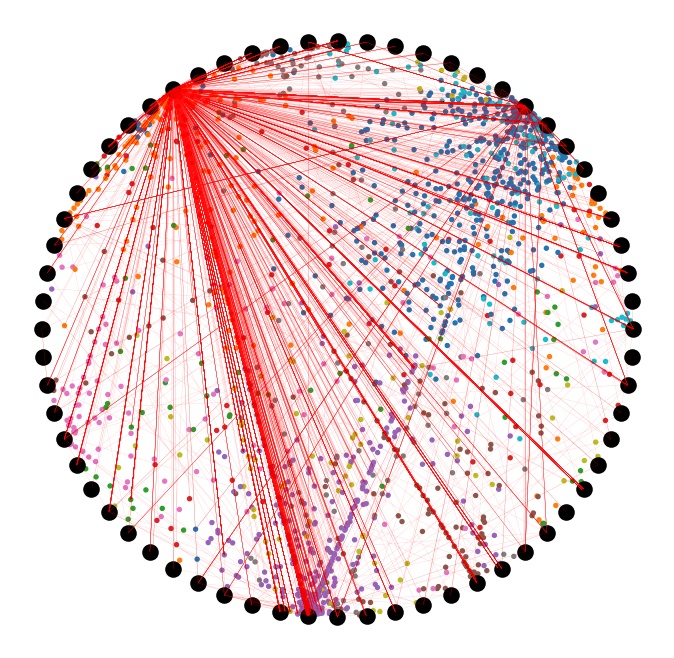}
    \caption{\textsc{NCP} $(D=64)$}
    
\end{subfigure}
\hfill
\begin{subfigure}{0.24\textwidth}
    \includegraphics[width=\textwidth]{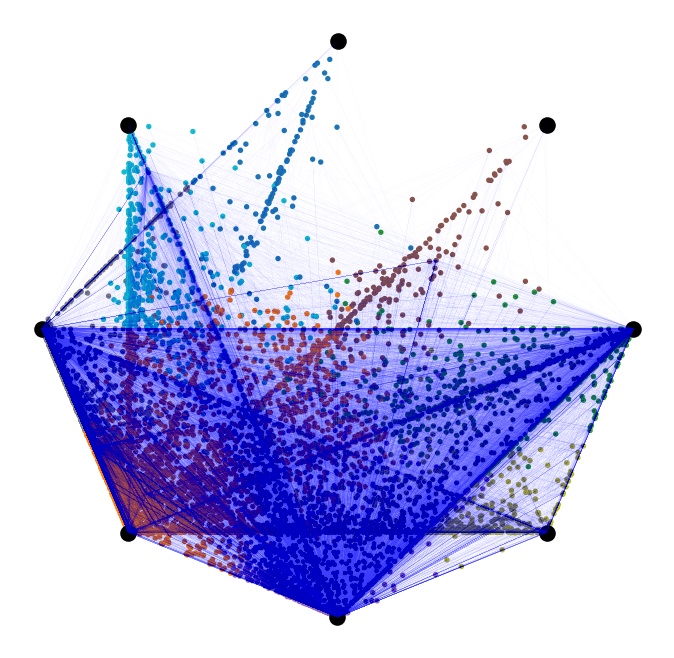}
    \caption{\textsc{PCP} $(D=8)$}
    
\end{subfigure}
\begin{subfigure}{0.24\textwidth}
    \includegraphics[width=\textwidth]{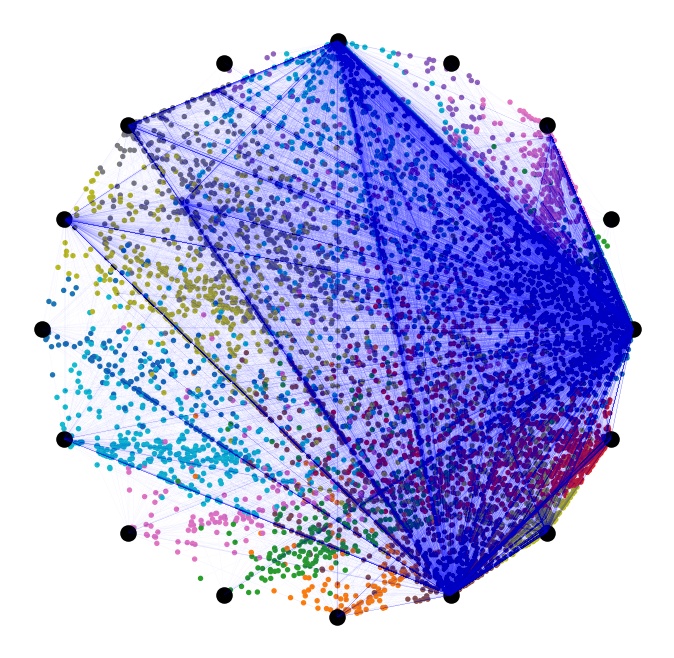}
    \caption{\textsc{PCP} $(D=16)$}
    
\end{subfigure}
\hfill
\begin{subfigure}{0.24\textwidth}
    \includegraphics[width=\textwidth]{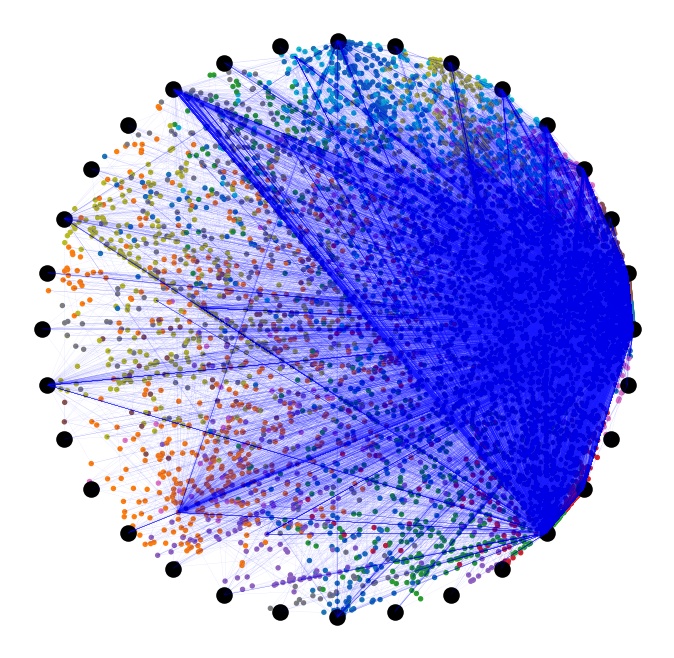}
    \caption{\textsc{PCP} $(D=32)$}
    
\end{subfigure}
\hfill
\begin{subfigure}{0.24\textwidth}
    \includegraphics[width=\textwidth]{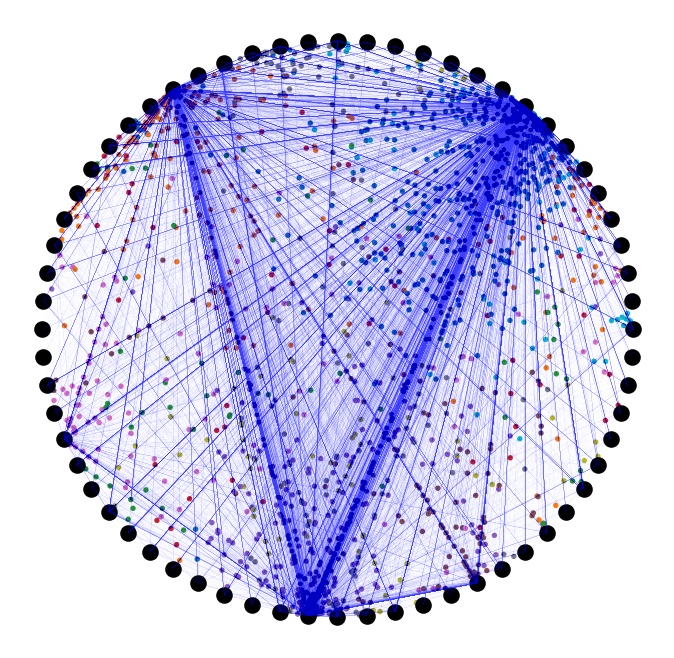}
    \caption{\textsc{PCP} $(D=64)$}
    
\end{subfigure}
\hfill
\begin{subfigure}{0.24\textwidth}
    \includegraphics[width=0.9\textwidth]{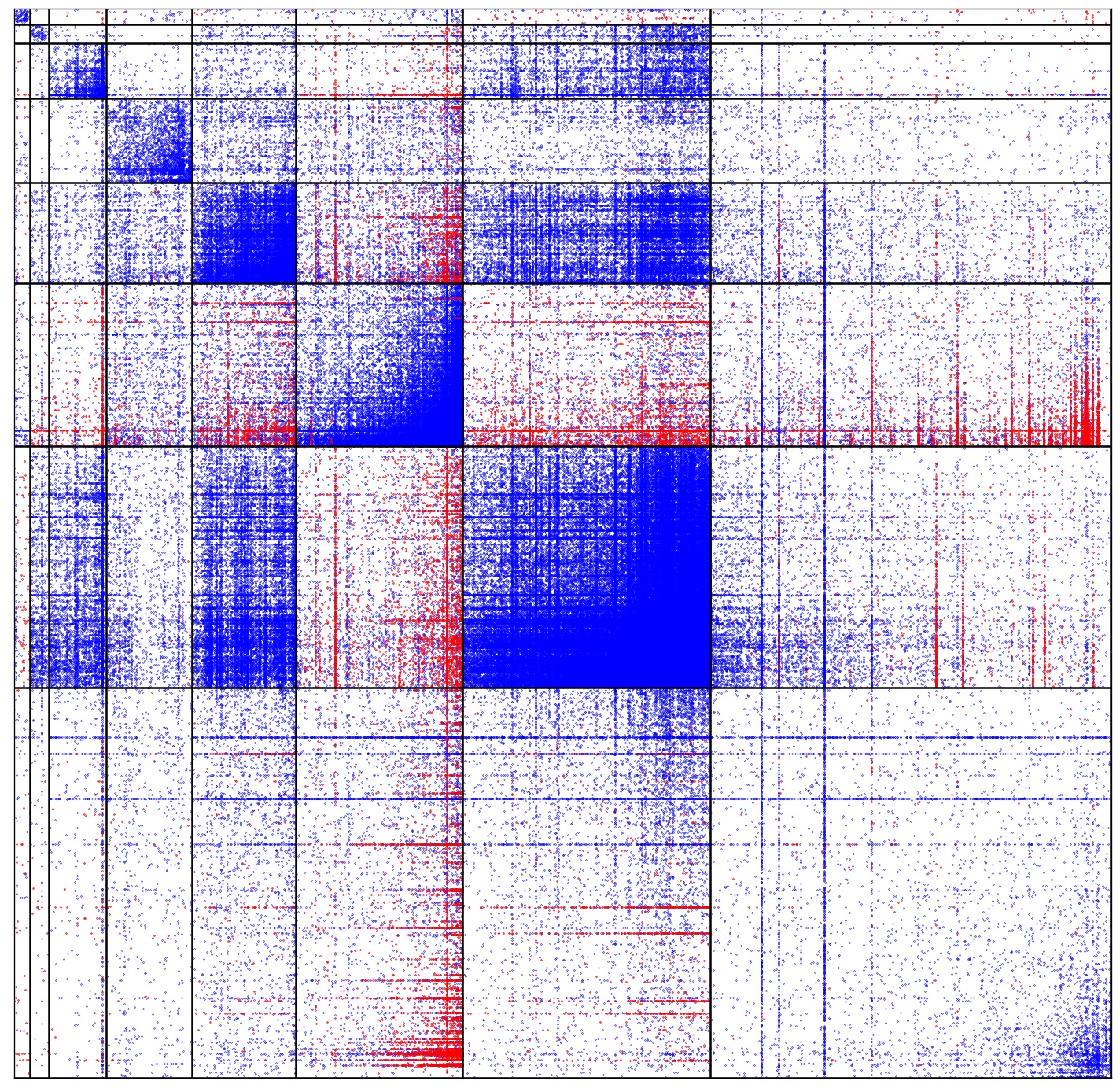}
    \caption{\textsc{OrA} $(D=8)$}
    
\end{subfigure}
\hfill
\begin{subfigure}{0.24\textwidth}
    \includegraphics[width=0.9\textwidth]{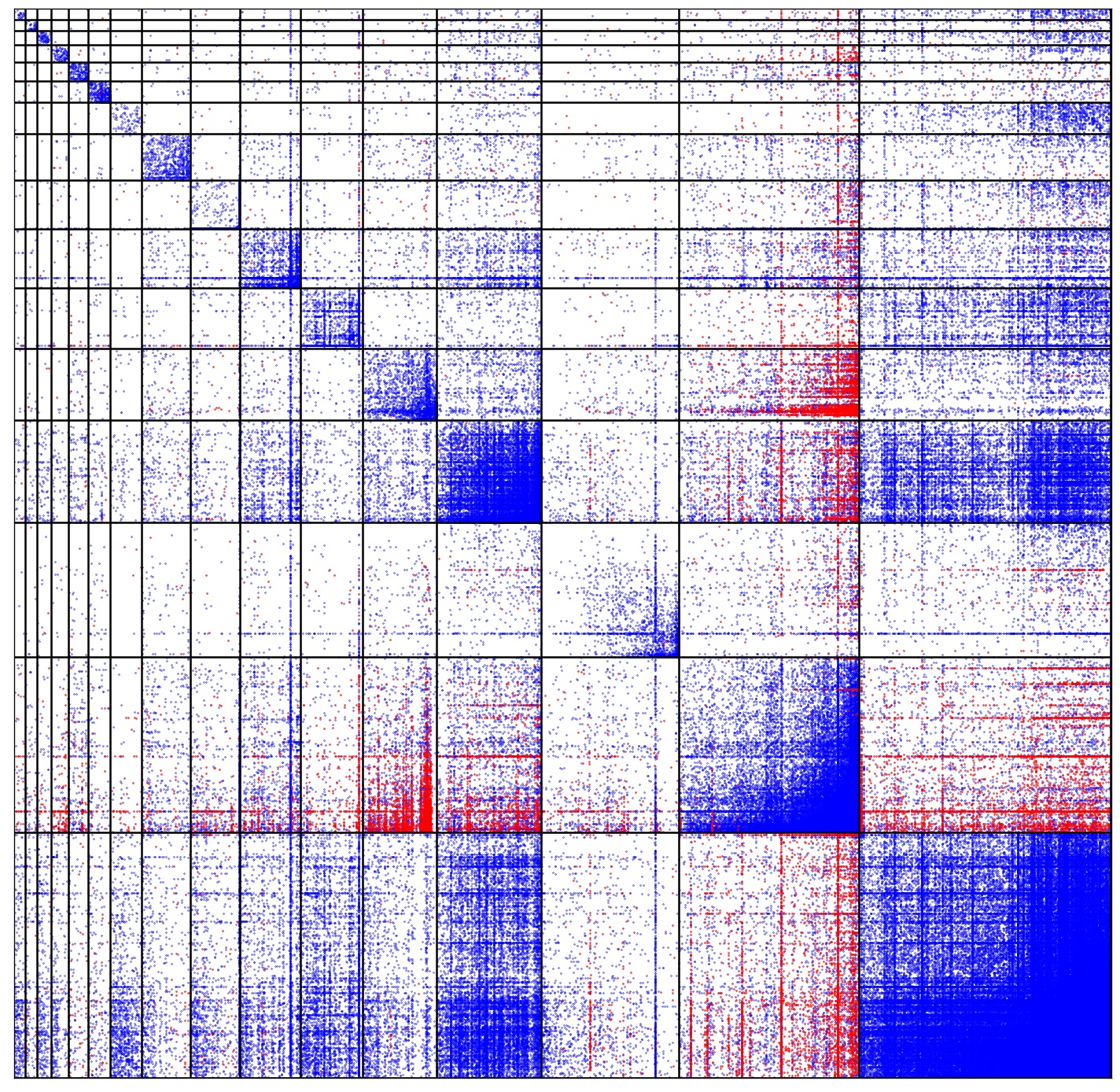}
    \caption{\textsc{OrA} $(D=16)$}
    
\end{subfigure}
\begin{subfigure}{0.24\textwidth}
    \includegraphics[width=0.9\textwidth]{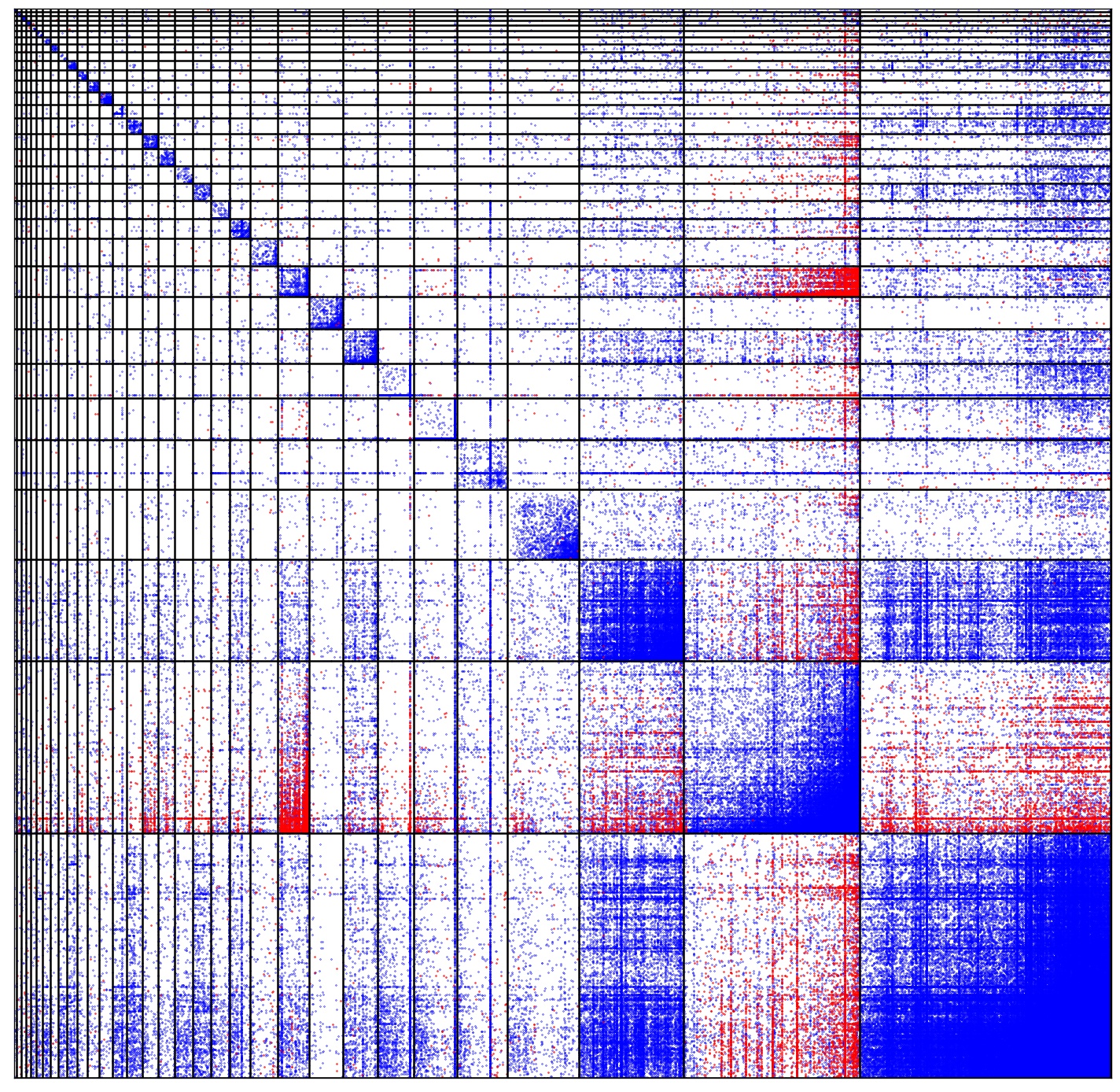}
    \caption{\textsc{OrA} $(D=32)$}
    
\end{subfigure}
\hfill
\begin{subfigure}{0.24\textwidth}
    \includegraphics[width=0.9\textwidth]{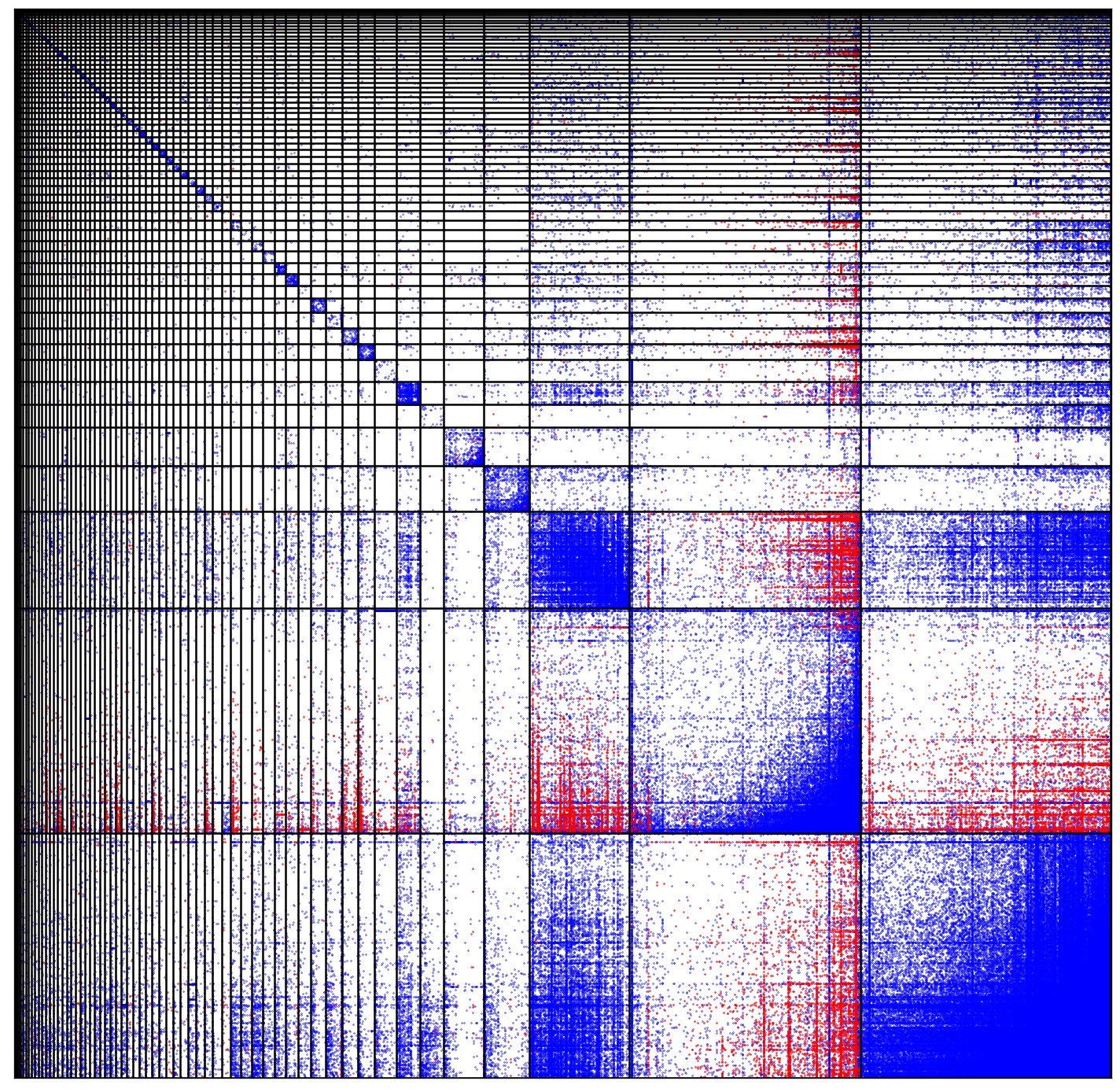}
    \caption{\textsc{OrA} $(D=64)$}
    
\end{subfigure}
\caption{\textbf{\textsc{sHM-LDM}(p=2)}: \textsl{Twitter} Network \cite{dataset_twitter}---Inferred simplex visualizations and ordered adjacency matrices for various dimensions $D$ and with simplex side lengths $\delta$ ensuring identifiability. The first row shows the latent space projection to the first two Principal Components---The second row provides a Negative Circular Plot (\textsc{NCP}) with red lines showcasing negative links between nodes---The third row shows a Positive Circular Plot (\textsc{PCP}) with the blue lines denoting positive links between node pairs---The fourth and final row shows the Ordered Adjacency (\textsc{OrA}) matrices sorted based on the memberships $\mathbf{w}_i$, in terms of maximum simplex corner responsibility, and internally according to the magnitude of the corresponding corner assignment for their reconstruction \cite{shmldm}.}
\label{fig:soc_viz_p2}
\end{figure*}

\begin{figure*}[!b]
\centering
\begin{subfigure}{0.32\textwidth}
    \includegraphics[width=0.7\textwidth]{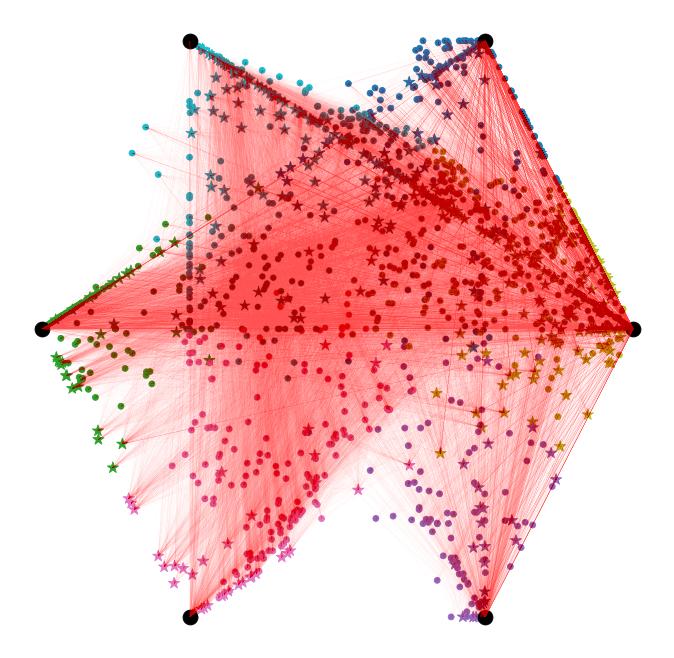}
    \caption{\textsc{NCP} \textsc{U.S.-Senate}}
    
\end{subfigure}
\hfill
\begin{subfigure}{0.32\textwidth}
    \includegraphics[width=0.7\textwidth]{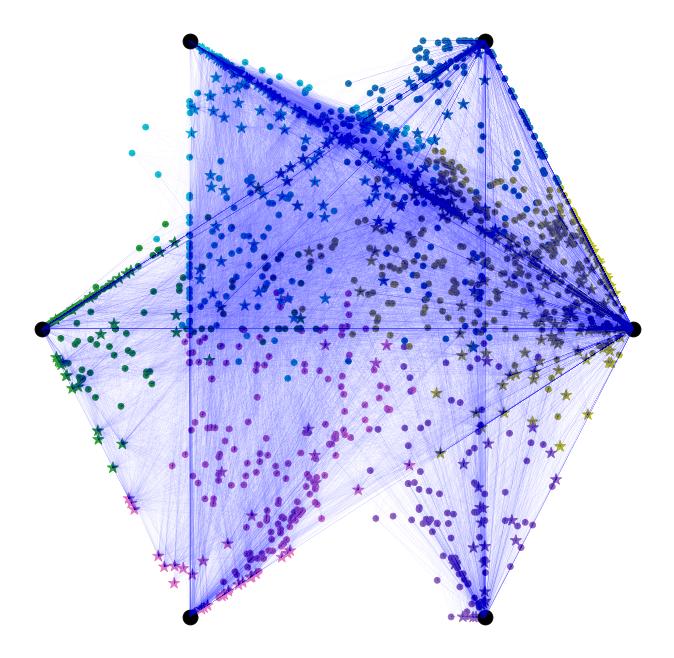}
    \caption{\textsc{PCP} \textsc{U.S.-Senate}}
    
\end{subfigure}
\hfill
\begin{subfigure}{0.32\textwidth}
    \includegraphics[width=0.7\textwidth]{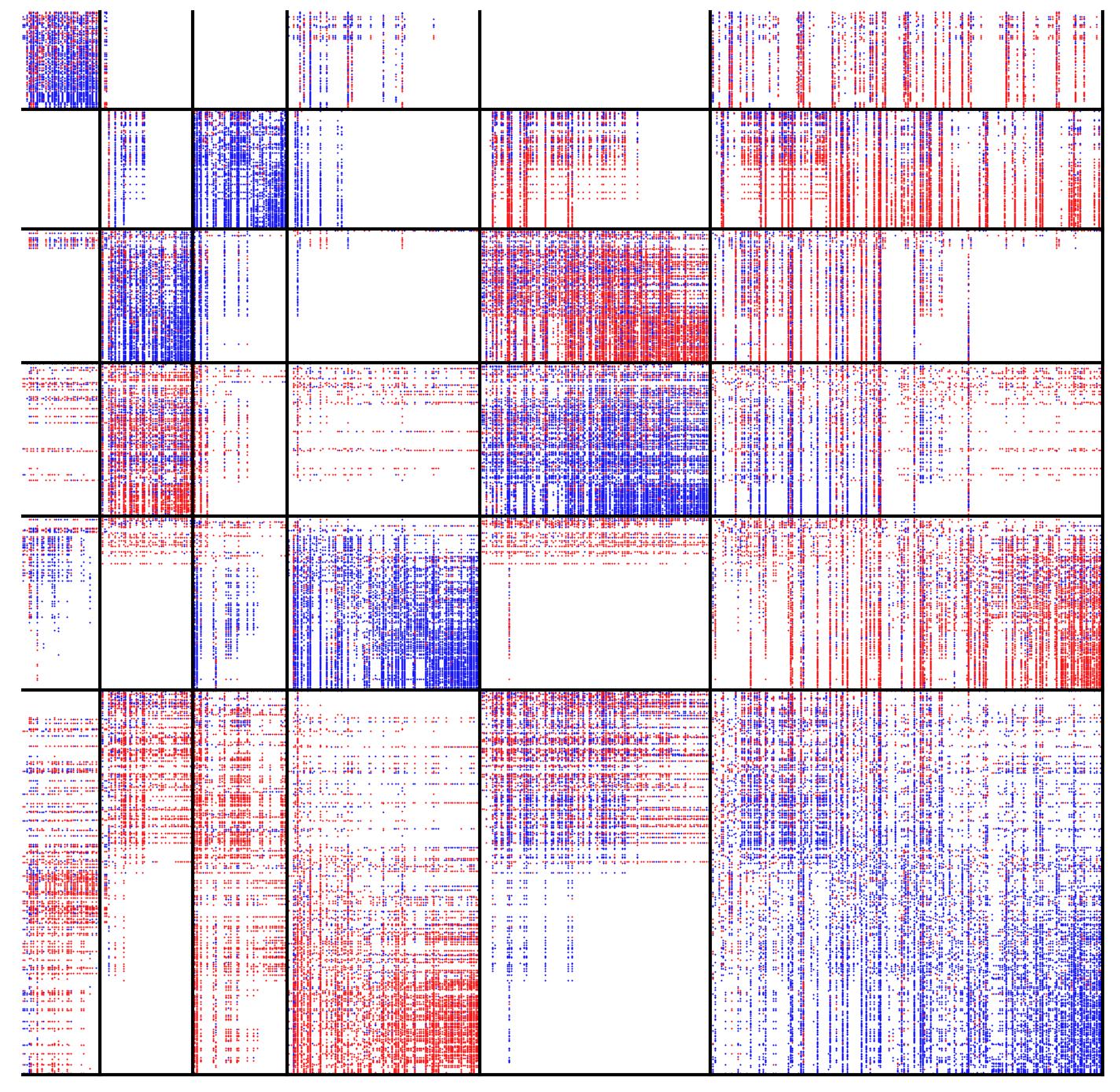}
    \caption{\textsc{OrA} \textsc{U.S.-Senate}}
    
\end{subfigure}
\begin{subfigure}{0.32\textwidth}
    \includegraphics[width=0.7\textwidth]{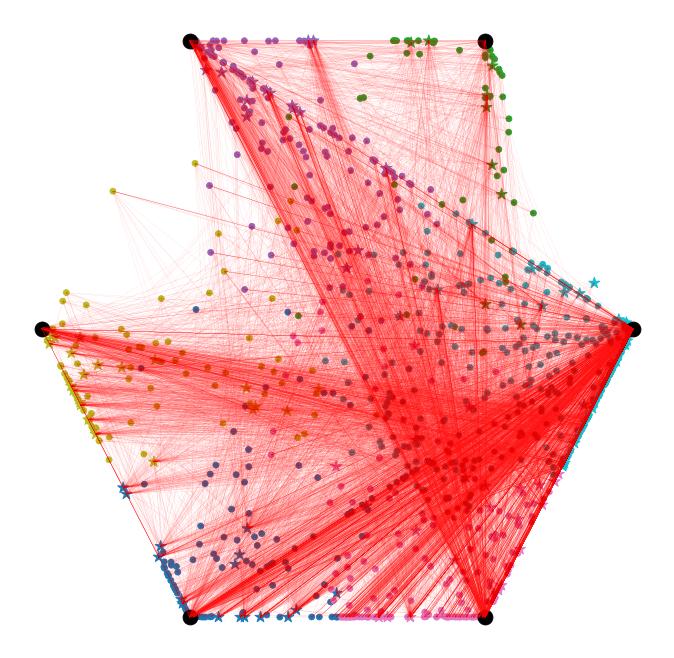}
    \caption{\textsc{PCP} \textsc{U.S.-House}}
    
\end{subfigure}
\hfill
\begin{subfigure}{0.32\textwidth}
    \includegraphics[width=0.7\textwidth]{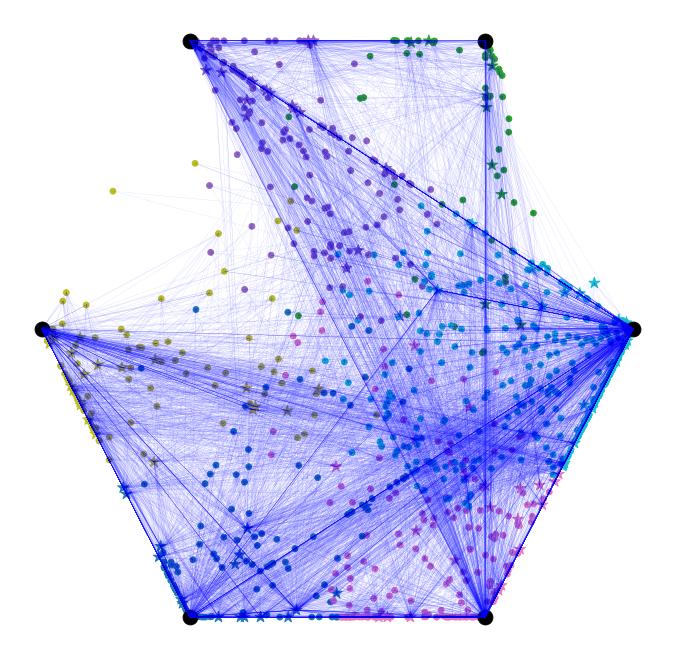}
    \caption{\textsc{PCP} \textsc{U.S.-House}}
    
\end{subfigure}
\hfill
\begin{subfigure}{0.32\textwidth}
    \includegraphics[width=0.7\textwidth]{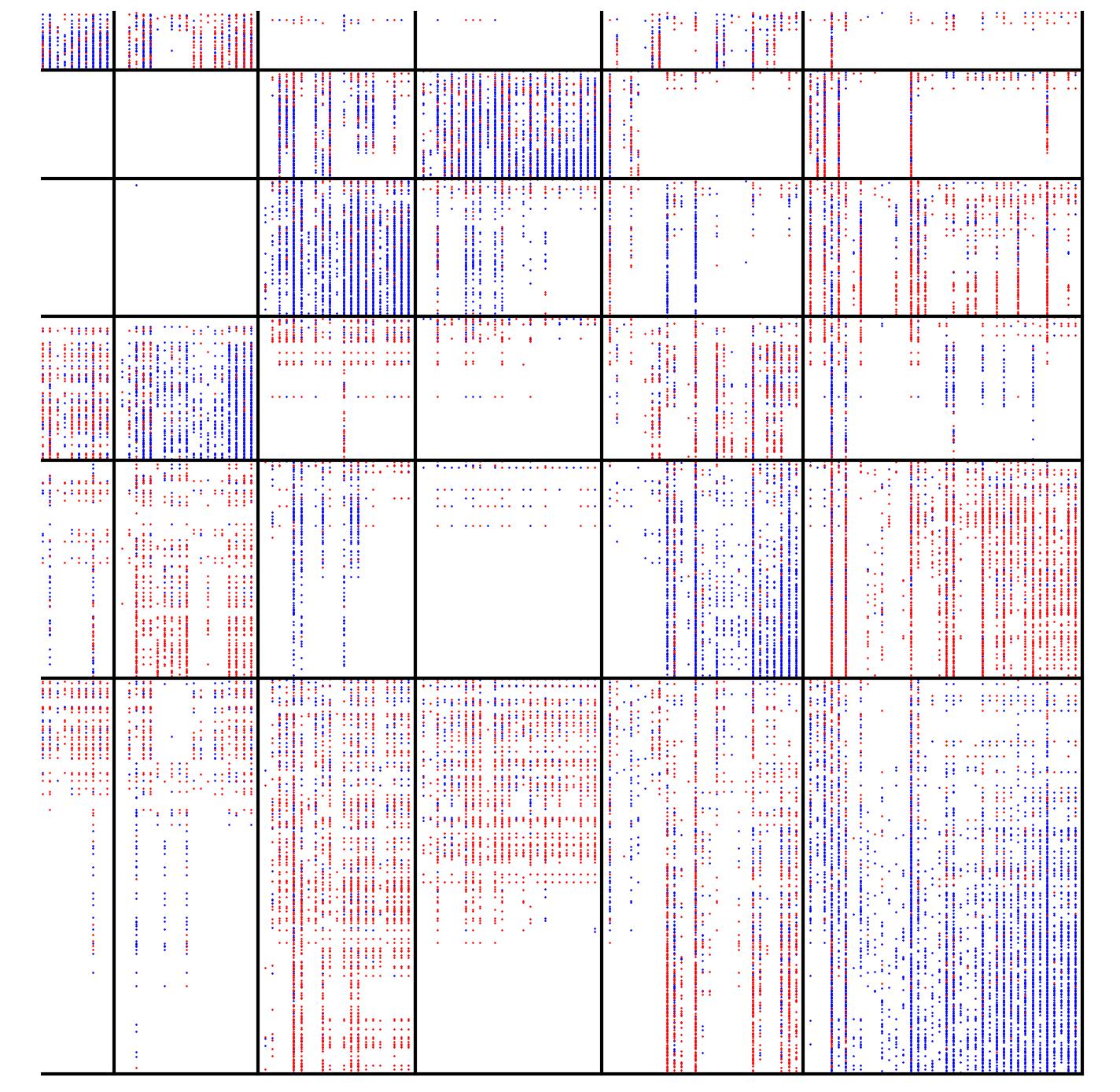}
    \caption{\textsc{OrA} \textsc{U.S.-House}}
    
\end{subfigure}
\caption{\textbf{\textsc{sHM-LDM}}(p=2): Inferred simplex visualizations and ordered adjacency matrices for a $D=6$ dimensional simplex with side lengths $\delta$ ensuring identifiability. The first column provides a Negative Circular Plot (\textsc{NCP}) with red lines showcasing negative links between nodes---The second column shows a Positive Circular Plot (\textsc{PCP}) with the blue lines denoting positive links between node pairs---The third and final column shows the Ordered Adjacency (\textsc{OrA}) matrices ordered based on the memberships, in terms of maximum simplex corner responsibility, and internally according to the magnitude of the corresponding corner assignment for their reconstruction. Top row: \textsc{U.S.-House} \cite{house_senate}. Bottom row \textsc{U.S.-Senate} \cite{house_senate} \cite{shmldm}.}
\label{fig:bip_signed_p2}
\end{figure*}

\part{Graph Representation Learning of Single-Event Temporal
Networks}
\chapter{Time to Cite: Modeling Citation Networks using the Dynamic Impact Single-Event Embedding Model}

A major focus has been given to the understanding of SciSci through the lens of complex network analysis, studying the structural properties and dynamics, of naturally occurring graph data describing SciSci. These include collaboration networks describing how scholars cooperate to advance various scientific fields. In particular, pioneering works \cite{sci_col1,sci_col2,sci_col3} have analyzed multiple network statistics such as degree distribution, clustering coefficient, and average shortest paths. Furthermore, citation networks define an additional prominent case where graph structure data describe SciSci. Citation networks, essentially describe the directed relationships of papers (nodes) with an edge occurring between a dyad if paper $A$ cites paper $B$, e.g. A$\rightarrow$B. Multiple efforts towards Graph Representation Learning of citation networks have been made, although treating such networks as static in time. Notably, citation networks are dynamic. Whereas dynamic modeling approaches can uncover structures obscured when aggregating networks across time to form static networks, the dynamic modeling approaches are in general based on the assumption that multiple links occur between the dyads in time. Therefore no optimal likelihood formulation has been explored for such networks defined as single-event networks (SENs). Furthermore, lots of attention has been given in SciSci to the temporal impact characterization of papers in terms of their citation dynamics. Importantly, most of these studies relied on carefully designed heuristics that utilized classical machine learning methods based on various scholarly features, as well as, paper textual information. Such features are used to quantify and predict a paper's impact included linear/logistic regression, k-nearest neighbors, support vector machines, random forests, and many more \cite{feat1,feat2,feat3,feat4,feat5,feat6}. These studies focused primarily on carefully designing and including proper features to be used for the impact prediction task. Unfortunately, no method has successfully combined a Graph Representation Learning approach under an appropriate \textsc{SEN} likelihood while also accounting for impact characterization.

Consequently, we here focus on citation networks to alleviate such limitations. It is worth mentioning, that despite focusing only on citation graphs, our approach is eligible for the analysis of every network that falls under the \textsc{SEN} umbrella. We here turn to the Inhomogenous Poisson Point Process for which we constrain to the modeling of the maximum one event that may appear per dyad, yielding the Single-Event Poisson Process define for the first time a principled likelihood expression for single events networks. In order to define powerful ultra-low dimensional network embeddings we turn to the representation power of the directed network version \textsc{LDM}. Specifically, for every paper we define static embeddings distinguishing between source and target roles, i.e. we introduce a different position in the latent space for the roles of papers when citing or being cited. In addition, we define paper random effects that can be reparametrized to represent paper masses, again distinguishing between "being cited" and "citing" masses. For the "being cited" mass we introduce a temporal impact function that characterizes the incoming citation dynamics. eligible for impact quantification. The impact function is parameterized through appropriate probability density functions, including the log-normal, as well as, the truncated normal distributions.

\begin{figure*}[!t]
\centering
\includegraphics[width=\textwidth]{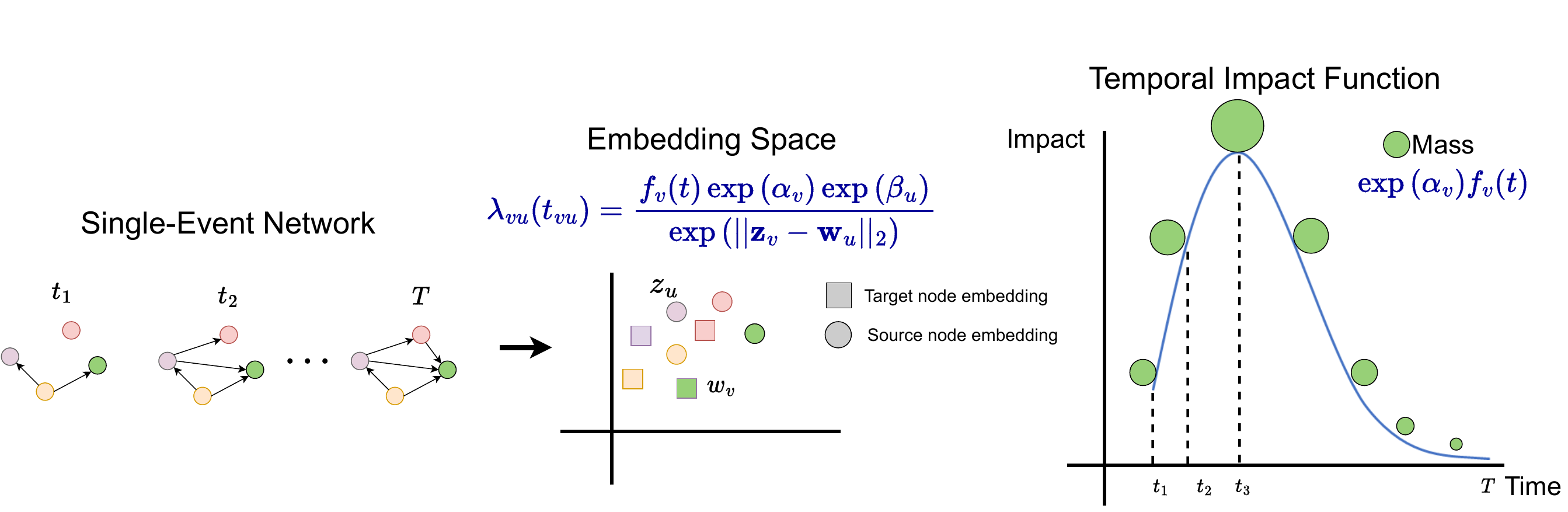}
\caption{\textsc{DISEE} procedure overview. Given a Single-Event Network (\textsc{SEN}) as an input, the model defines an intensity function introducing two sets of static embeddings distinguishing between source $\mathbf{w}_u$ and target $\mathbf{z}_v$ node embeddings. Furthermore, each node is assigned its own random effect, distinguishing again the source $\beta_u$ and target $\alpha_v$ roles. The random effects can be parameterized to represent source and target masses through the exponential function. Finally, for each target node of the network, the model further defines an impact function $f_v(t)$ yielding a temporal impact characterization of the nodes' incoming link dynamics which controls the nodes' mass in time as, $\exp{(\alpha_v)}f_v(t)$.}\label{fig:DISEE}
\end{figure*}

\section{Contributions}
We presently extend the Latent Distance Model to account for single-event networks and to accurately characterize paper impact based on citation dynamics. We specify an Inhomogeneous Poisson Point Process for the analysis of \textsc{SEN}s, defining the Single-Event Poisson Process which provides for the first time an appropriate likelihood for the \textsc{SEN} family of temporal networks. We hereby introduce the \textsc{D}ynamic \textsc{I}mpact \textsc{S}ingle-\textsc{E}vent \textsc{E}mbedding Model (\textsc{DISEE}) \cite{pmlr-v238-nakis24a} which characterizes the scientific interactions in terms of a latent distance model in which forces (strength of the interaction) can be reparameterized to be proportional to the product of the masses of the interacting entities. To account for the time-varying impact, the mass of a contribution being used is time-dependent based on flexible parametric representations of scientific impact. The procedure overview is provided in Figure \ref{fig:DISEE}. Analytically our contributions are outlined as:

\begin{itemize}

    \item We, for the first time, derive the single-event Poisson Process (SE-PP). As paper citation networks (and SENs in general) only include a single event we augment the Poisson Process likelihood to have support only for single events forming the single event Poisson Process.

    \item We propose the \textsc{D}ynamic \textsc{I}mpact \textsc{S}ingle-\textsc{E}vent \textsc{E}mbedding Model based on the SE-PP for SENs. We characterize the rate of interaction within a latent distance model such that citations are generated relative to the degree to which a paper cites and a paper is being cited at a given time point interpreted as masses of the citing and cited papers, respectively, augmented by their distance in latent space.
    
    \item We demonstrate how the \textsc{D}ynamic \textsc{I}mpact \textsc{S}ingle-\textsc{E}vent \textsc{E}mbedding Model reconciles conventional impact modeling with latent distance embedding procedures. Specifically, we show how the model enables accurate dynamic characterization of citation impact similar to conventional paper impact modeling procedures while at the same time providing low-dimensional embeddings accounting for the structure of citation networks. We highlight this reconciliation on three real networks covering three distinct fields of science.

\end{itemize}

\section{Experimental design, results, and key findings}
We evaluate how successfully \textsc{DISEE} reconciles traditional impact quantification approaches with latent distance modeling. Specifically, we test the proposal the proposed approach's effectiveness in the link prediction task by comparing it to the classical \textsc{LDM} which is not time-aware and does not quantify temporal impact. We also consider multiple model ablations that are either able to characterize a node's impact or to account for \textsc{GRL}, i.e. define node embeddings, but not both. For the task of link prediction, we remove $20\%$ of network links and we sample an equal amount of non-edges as negative samples and construct the test set. Notably, these negative samples are sampled in a time-aware manner, meaning that we consider only pairs that are possibly to exist as missing links in the network (i.e. we do not consider node pairs where missing citations refer to papers citing future papers, as the target paper did not exist the time when the source paper was published). The link removal is designed in such a way that the residual network stays connected. Analytically, for each network, we do not consider removing links that make up the minimum spanning tree of the graph. For the evaluation, we consider both the Receiver Operator Characteristic and Precision-Recall Area Under Curve scores, as these are metrics not sensitive to the class imbalance between links and non-links. We then continue by evaluating the quality of impact expression of \textsc{DISEE} by visually presenting the inferred impact functions and comparing them against an Impact Function  Model (\textsc{IFM}) which fits an impact function directly on the citation pattern of each paper. Finally, we visualize the model's learned temporal space representing the target papers, accounting for their temporal impact in terms of their mass at a specific time point, and characterizing the different papers' lifespans.

 For the link prediction experiments, the best performance is achieved by model specifications that define an embedding space, i.e. the \textsc{DISEE} and \textsc{LDM} models while the rest of the model ablations defined significantly lower performance. Comparing the two distribution choices for the impact function (\textsc{Truncated Normal} and \textsc{Log Normal}) we observed very similar link prediction scores. We continued by addressing the quality of paper impact characterization based on a target paper's incoming citation dynamics. In such a direction, we further compared the inferred impact functions of the \textsl{DISEE} and \textsc{IFM}, under the \textsc{Truncated} normal and \textsc{Log Normal} distributions, against the true impact dynamics for each one of the corresponding papers. For the \textsc{Truncated} case (Figure \ref{fig:ML_tr}), we observe that \textsl{DISEE} and \textsc{IFM} provide very similar (and in some cases identical) impact functions that capture the underlying citation patterns. In the case of the \textsc{Log-Normal} distribution (Figure \ref{fig:ML_ln}), we witness an agreement between \textsc{DISEE} and \textsc{IFM} models when the paper lifespan does not exceed the $2$ years. For larger lifespans \textsc{DISEE} defines a larger standard deviation than the \textsc{IFM} returning much heavier tails. Both models when compared to the true citation histogram provide much heavier tails when the paper lifespan exceeds the $2$-year threshold. The \textsc{Log-Normal} distribution is not invariant to the scale of the x-axis (contrary to the \textsc{Truncated} normal which is scale-invariant) and this can be potentially a reason for observing this kind of behavior, meaning that the choice of the time resolution is not optimal (this is to be further investigated). Nevertheless, the \textsc{Truncated} normal distribution seems to very accurately represent the true citation dynamics, defining correct distribution tails, but in some cases, the \textsc{Log-Normal} heavier tails may be more appropriate for future impact predictions (as papers stay "alive" longer). Finally, we provide embedding space visualizations of the target (cited) papers, accounting for their temporal impact in terms of their mass at a specific time point, showcasing the evolution of the embedding space for the domain of \textsl{Machine Learning}. These visualizations showed that as the years progress, paper masses reach much larger magnitudes than in the earlier years, defining higher research significance, and accumulating higher citation numbers and impact which can be explained by the increase in published \textsl{Machine Learning} works. Embedding space visualizations are provided in Figures \ref{fig:2023} and \ref{fig:history}. (For more details and the full experiment results please visit the full paper \cite{pmlr-v238-nakis24a}.) 

\section{Conclusion} 

We have proposed the \textsc{D}ynamic \textsc{I}mpact \textsc{S}ingle-\textsc{E}vent \textsc{E}mbedding Model (\textsc{DISEE}), a reconciliation between traditional impact quantification approaches with a Latent Distance Model (\textsc{LDM}). We have focused on Single-Event Networks (\textsc{SEN}s), and more specifically in citation networks, where we for the first time derived Single-Event Poisson Process. Such a process defines an appropriate likelihood allowing for a principled analysis of single-events networks. In order to define powerful ultra-low dimensional network embeddings we turn to the representation power of the directed network version of the \textsc{LDM}. Specifically, for every paper we define static embeddings distinguishing between source and target roles, i.e. we introduced a different position in the latent space for the roles of papers when citing or being cited. In addition, we defined paper random effects that can be reparametrized to represent paper masses, again distinguishing between "being cited" and "citing" masses. For the "being cited" mass we introduced a temporal impact function that characterized the incoming citation dynamics. eligible for impact quantification. The impact function is parameterized through appropriate probability density functions, including the log-normal, as well as, the truncated normal distributions. Through extensive experiments, we showed that the \textsc{DISEE} had the same link prediction performance as the powerful \textsc{LDM}. Furthermore, we showed that the temporal impact characterization was validated by an Impact Function Model \textsc{IFM}. These results, showcase that the \textsc{DISEE} successfully reconciles powerful embedding approaches with citation dynamics impact characterization. Finally, visualizations of the embedding space for target papers provided accurate representations that described the birth and death of papers following their impact lifespans as years pass and science moves forward.

\begin{figure}[!t]
    \centering
     \begin{subfigure}[b]{0.32\textwidth}
        \includegraphics[width=\textwidth]{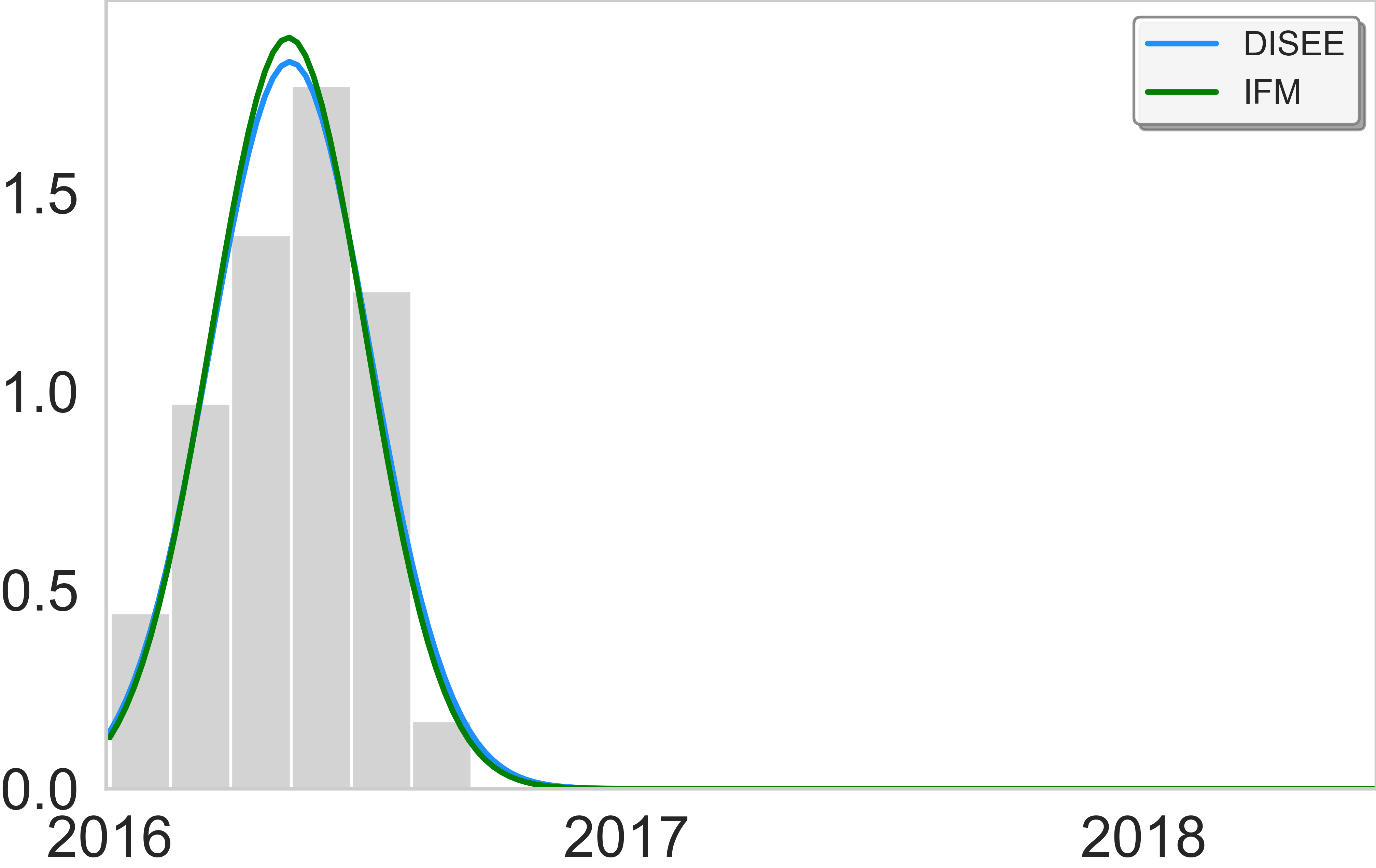}
        \caption{Deep Residual Learning for Image Recognition \cite{ifm1}.}
    \end{subfigure}
    \hfill
    \begin{subfigure}[b]{0.32\textwidth}
        \includegraphics[width=\textwidth]{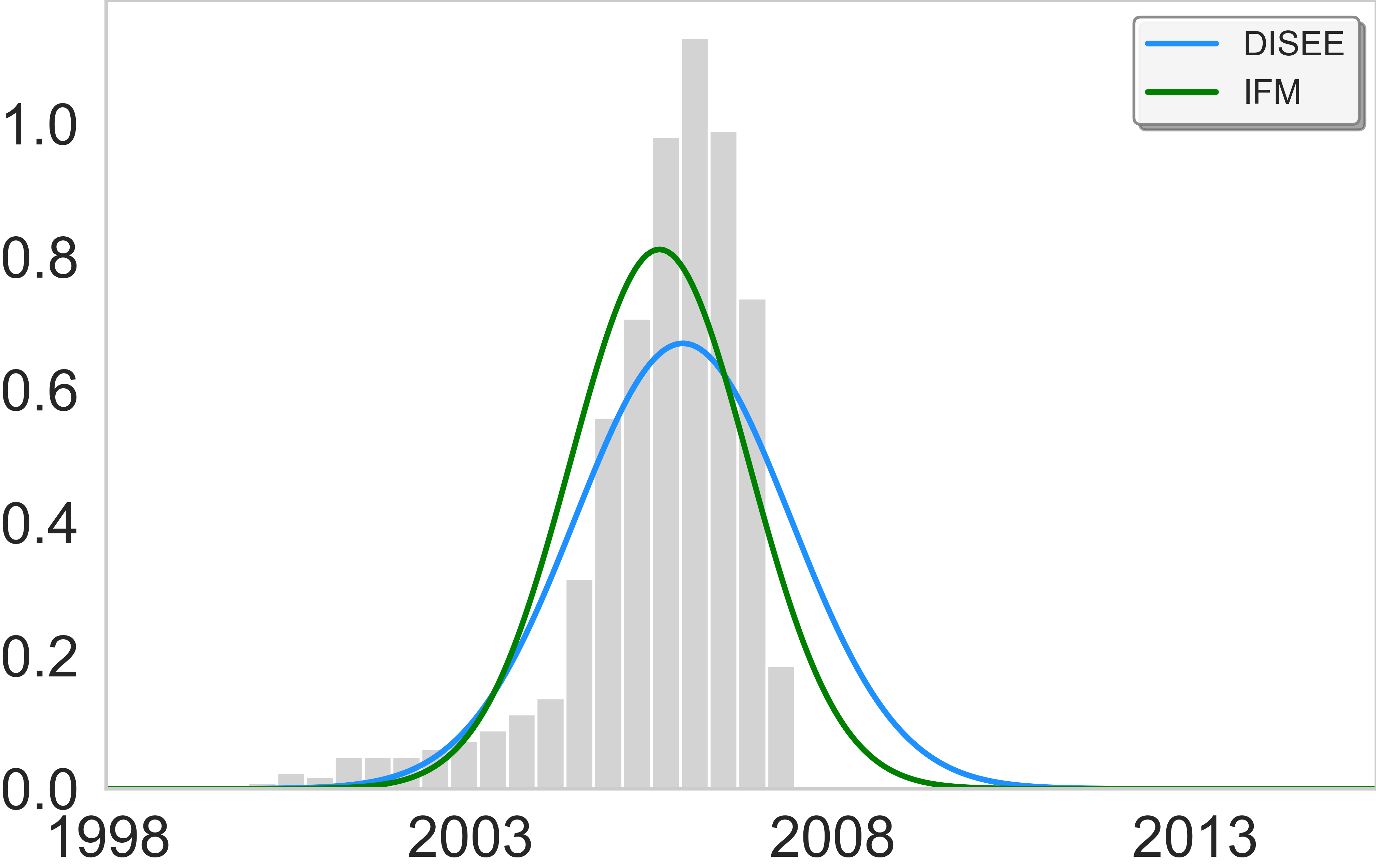}
        \caption{Gradient-based learning applied to document recognition \cite{ifm2}.}
    \end{subfigure}
    \hfill
    \begin{subfigure}[b]{0.32\textwidth}
        \includegraphics[width=\textwidth]{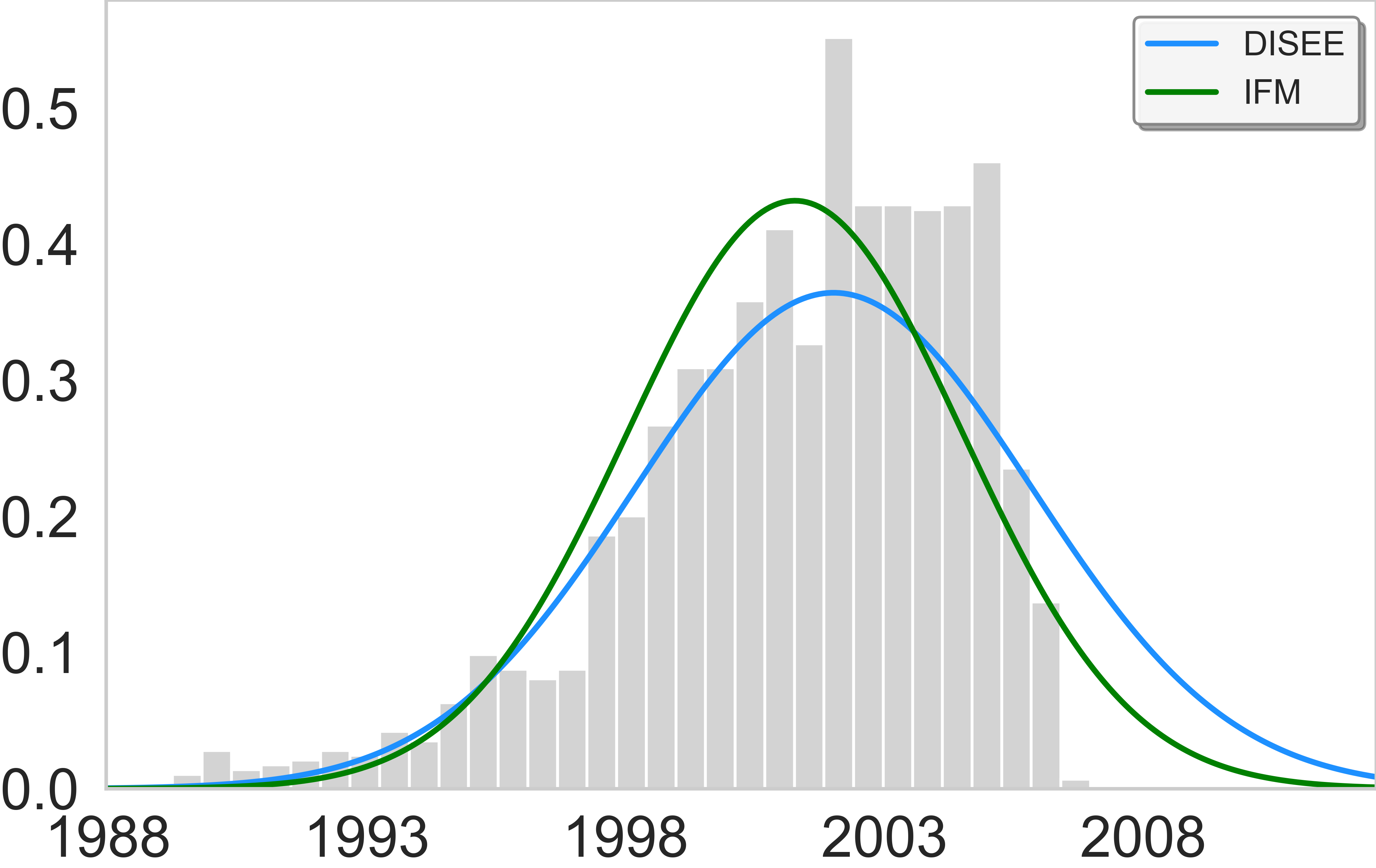}
        \caption{Structural equation modeling in practice \cite{ifm3}.}
    \end{subfigure}
    \hfill
    \begin{subfigure}[b]{0.32\textwidth}
        \includegraphics[width=\textwidth]{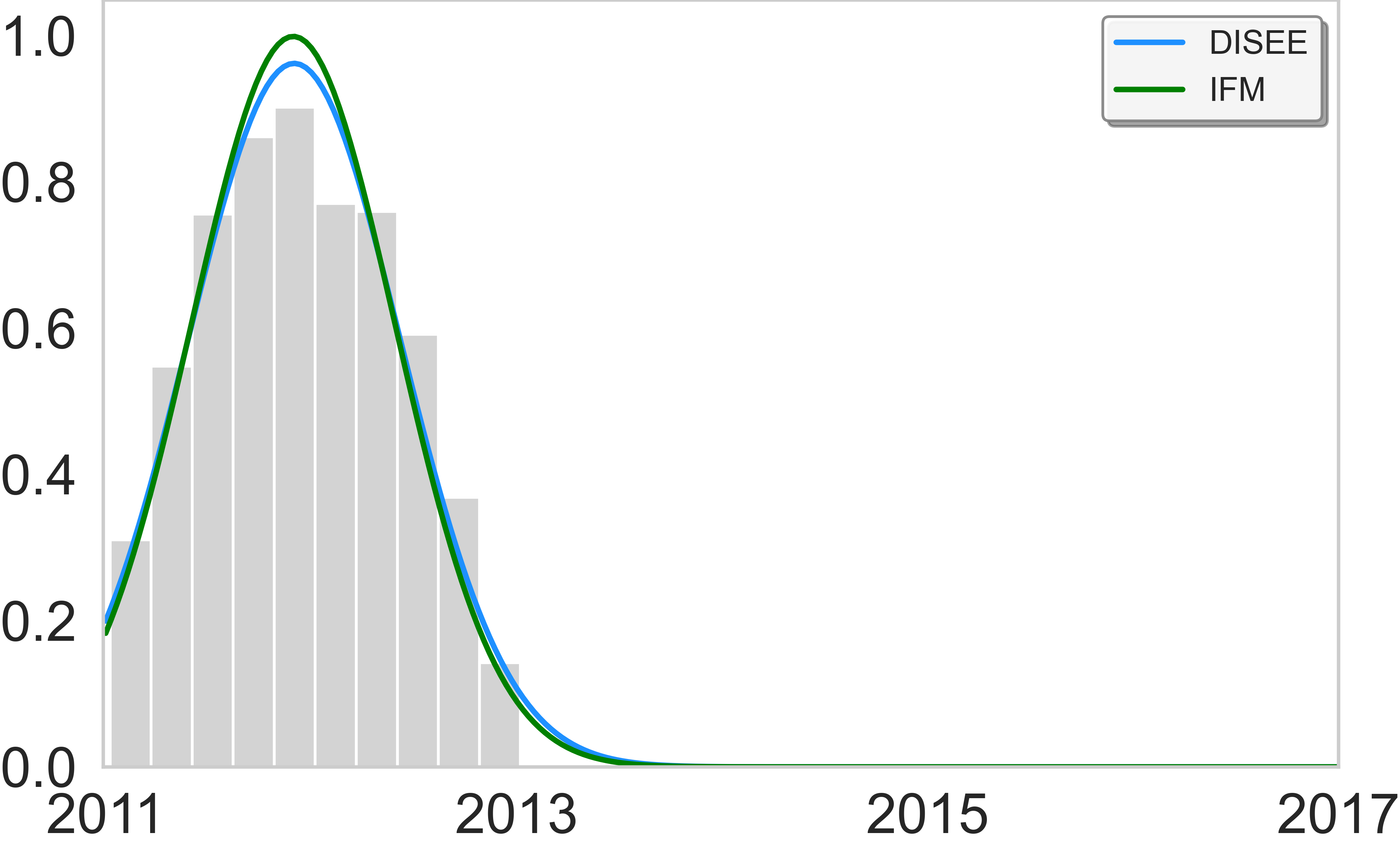}
        \caption{LIBSVM: A library for support vector machines \cite{ifm4}.}
    \end{subfigure}
    \hfill
    \begin{subfigure}[b]{0.32\textwidth}
        \includegraphics[width=\textwidth]{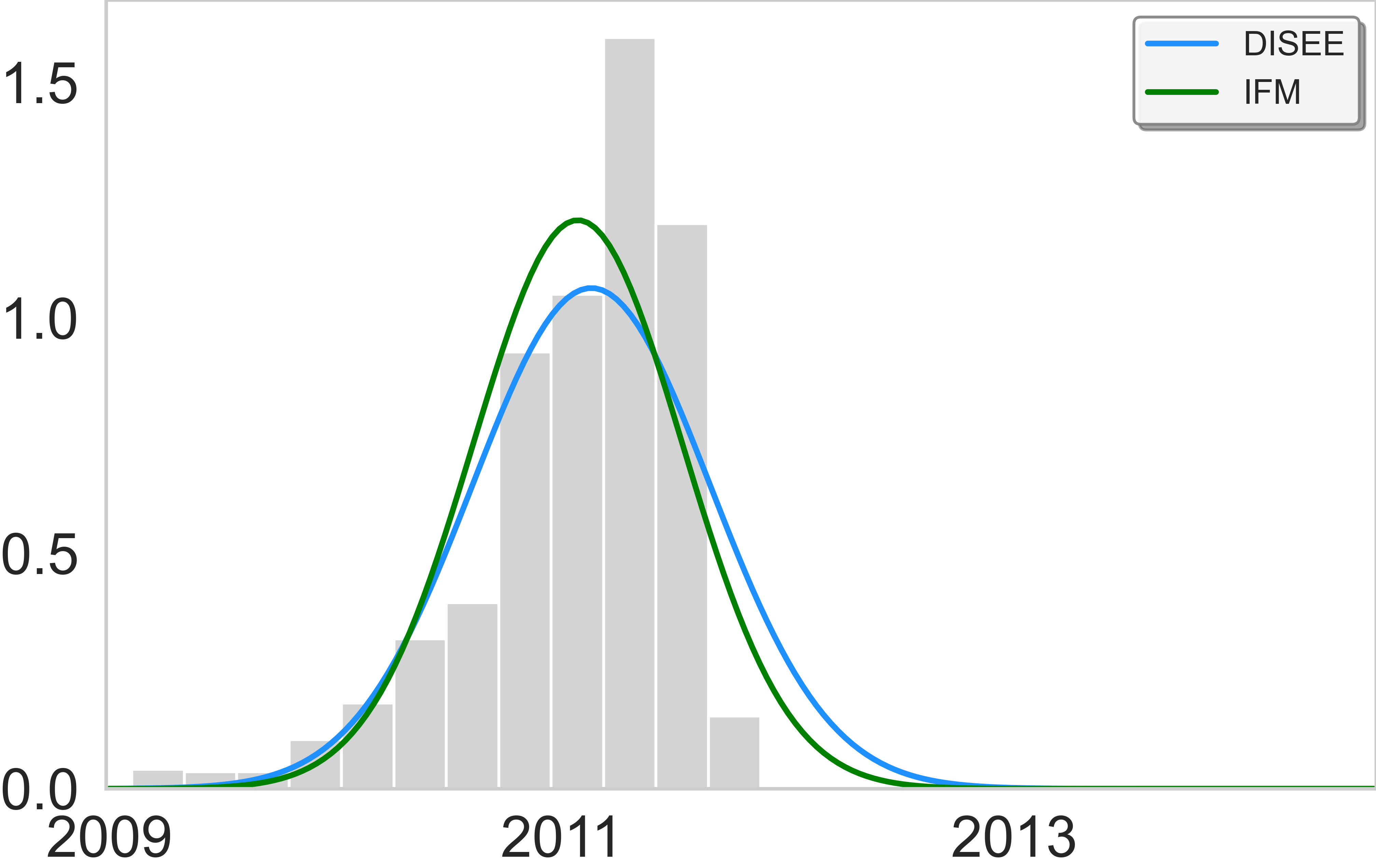}
        \caption{ImageNet: A large-scale hierarchical image database \cite{ifm5}.}
    \end{subfigure}
    \hfill
    \begin{subfigure}[b]{0.32\textwidth}
        \includegraphics[width=\textwidth]{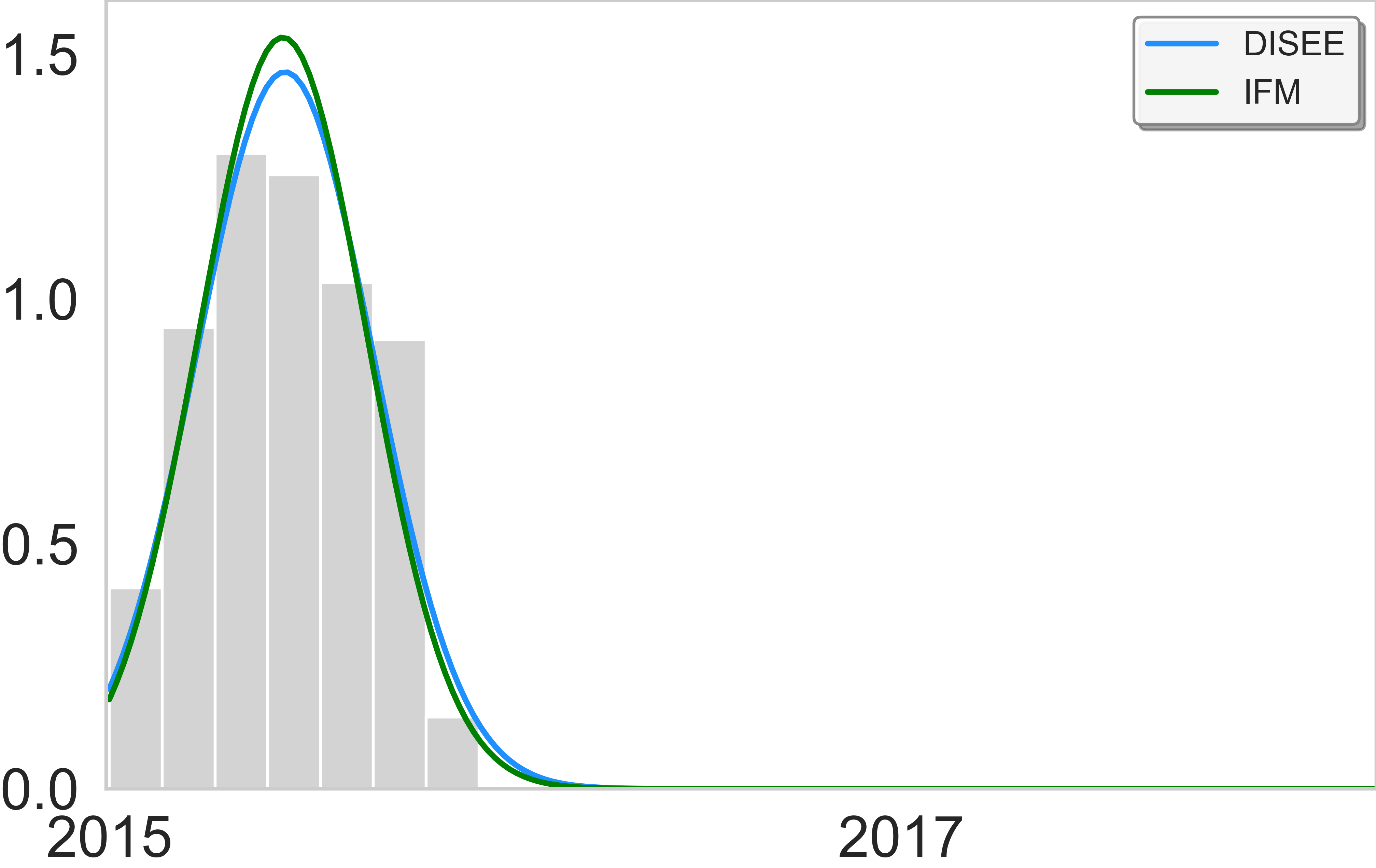}
        \caption{Going deeper with convolutions \cite{ifm6}.}
    \end{subfigure}
    \caption{\textsl{Machine Learning}: \textsc{DISEE Truncated} and \textsc{IFM Trunctated} models inferred impact function visualizations compared to the true citation histogram, for six popular \textsl{Machine Learning} papers with different citation dynamics.}
     \label{fig:ML_tr}
\end{figure}

\begin{figure}[!t]
    \centering
     \begin{subfigure}[b]{0.32\textwidth}
        \includegraphics[width=\textwidth]{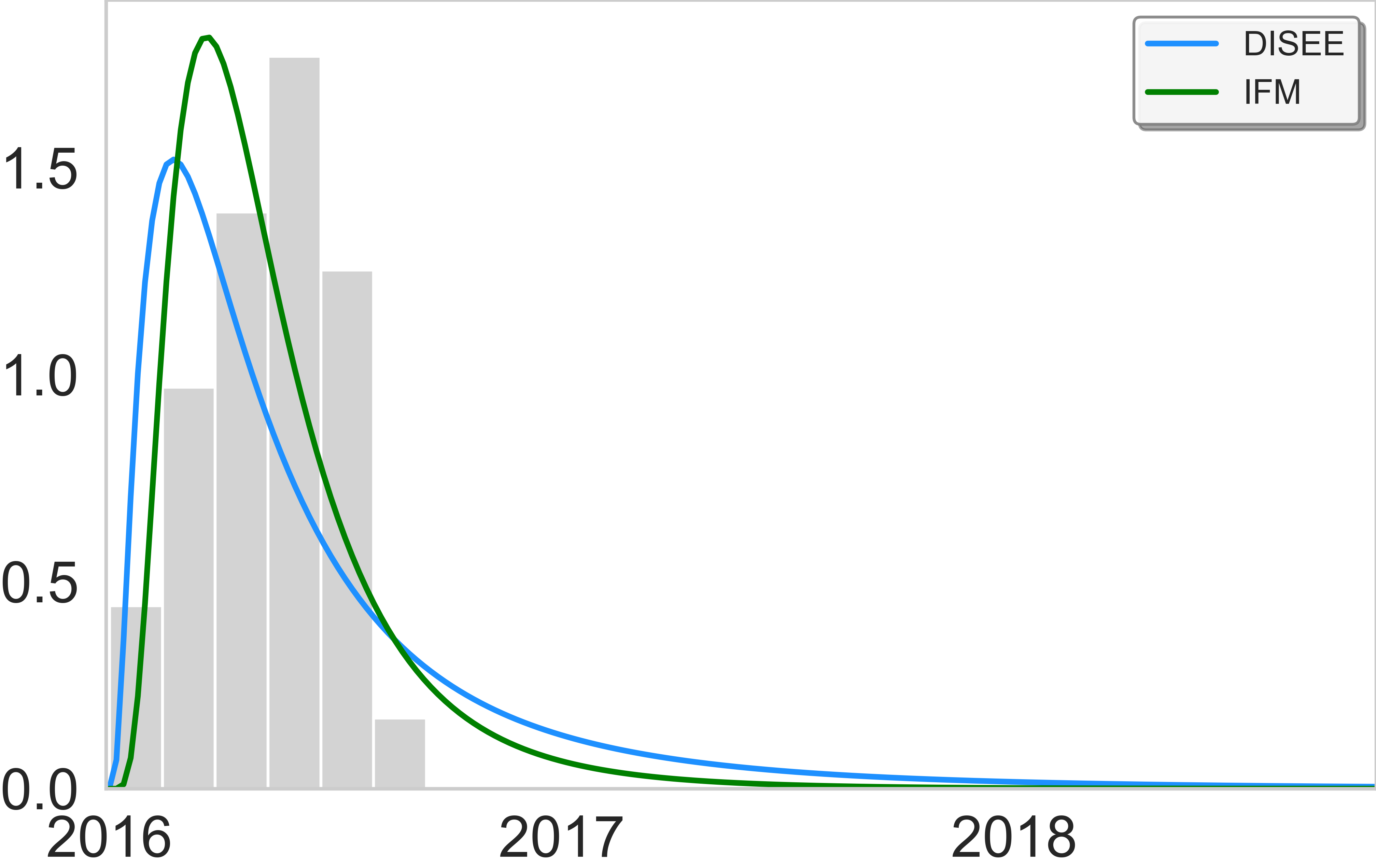}
        \caption{Deep Residual Learning for Image Recognition.}
    \end{subfigure}
    \hfill
    \begin{subfigure}[b]{0.32\textwidth}
        \includegraphics[width=\textwidth]{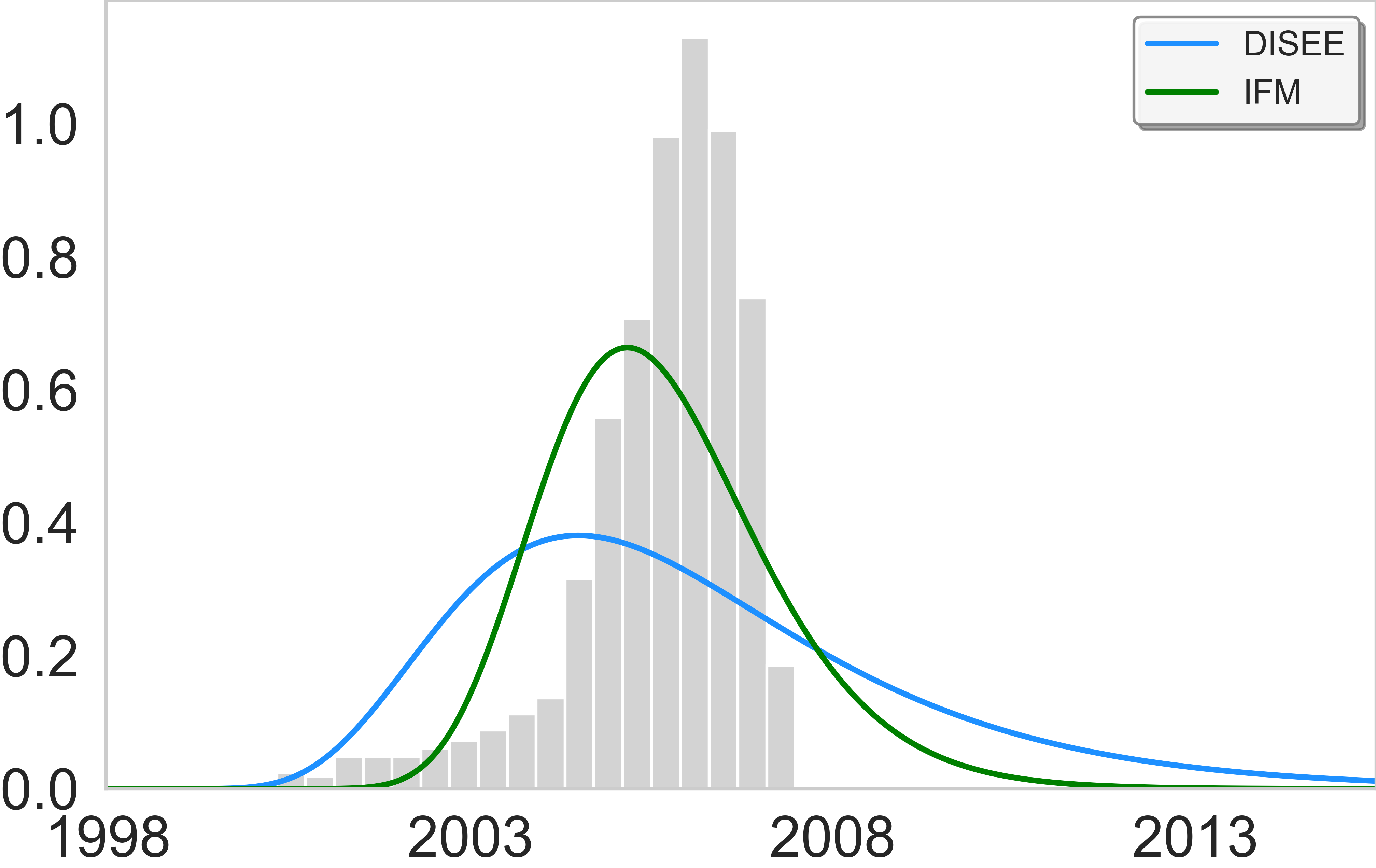}
        \caption{Gradient-based learning applied to document recognition.}
    \end{subfigure}
    \hfill
    \begin{subfigure}[b]{0.32\textwidth}
        \includegraphics[width=\textwidth]{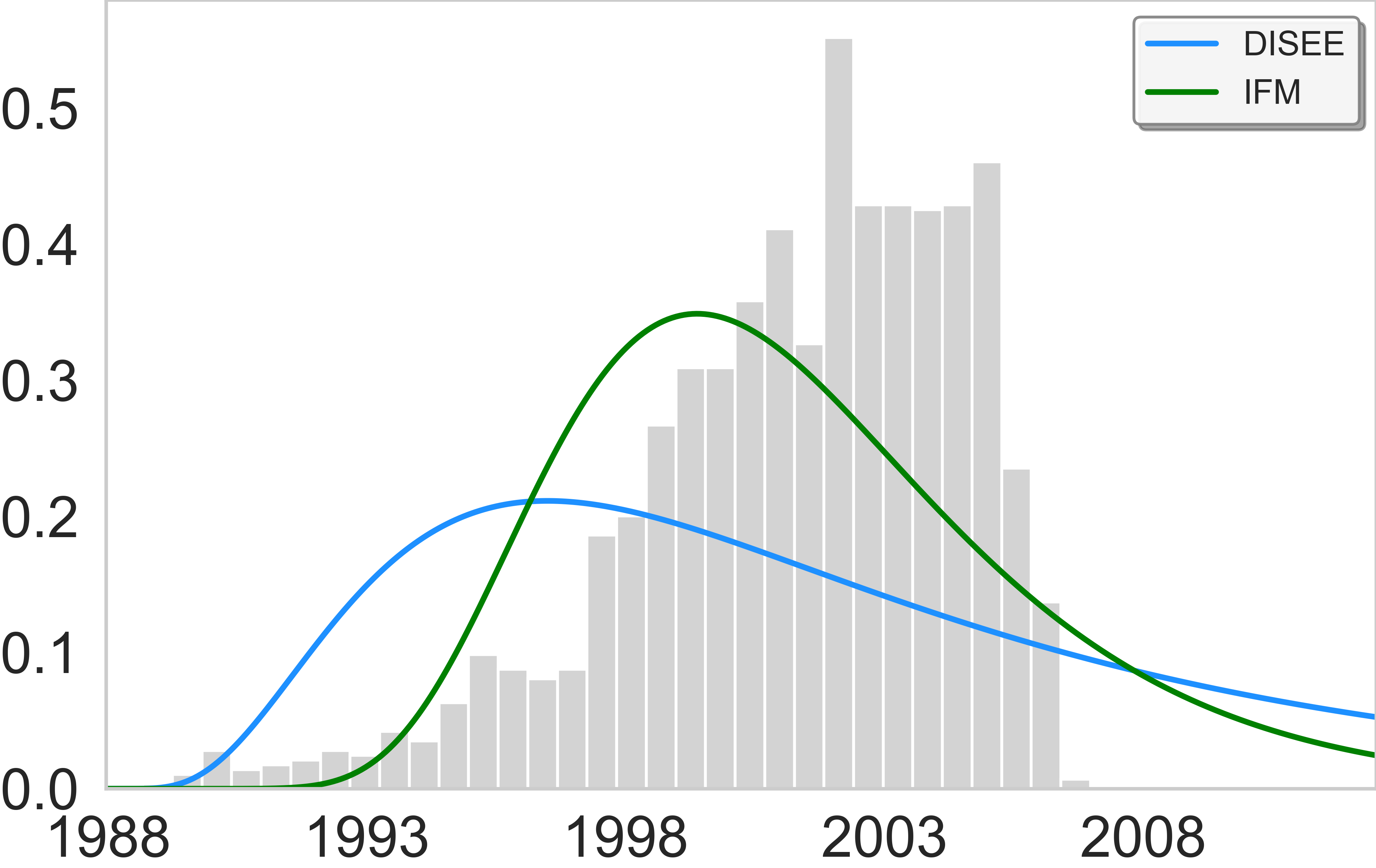}
        \caption{Structural equation modeling in practice.}
    \end{subfigure}
    \hfill
    \begin{subfigure}[b]{0.32\textwidth}
        \includegraphics[width=\textwidth]{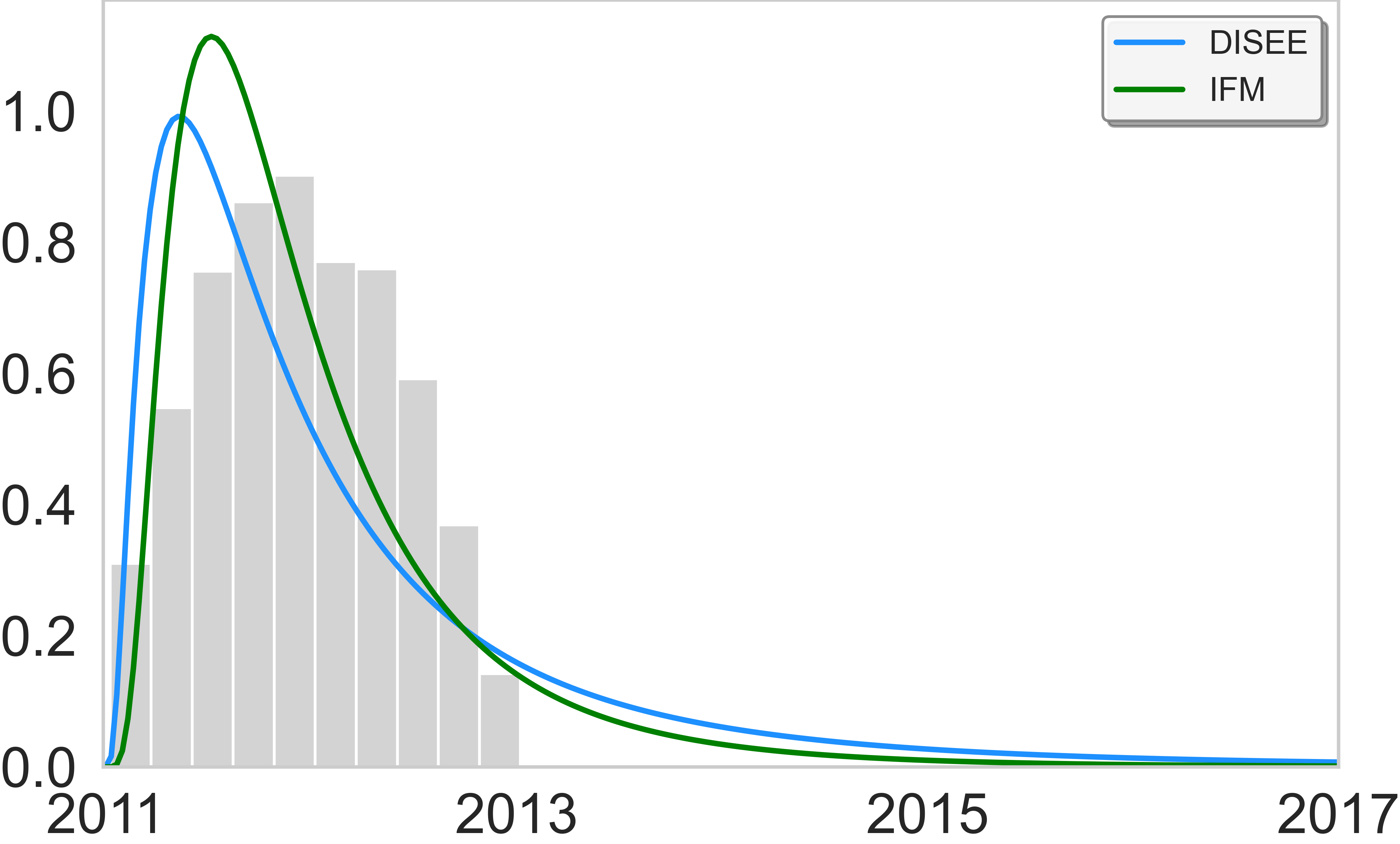}
        \caption{LIBSVM: A library for support vector machines.}
    \end{subfigure}
    \hfill
    \begin{subfigure}[b]{0.32\textwidth}
        \includegraphics[width=\textwidth]{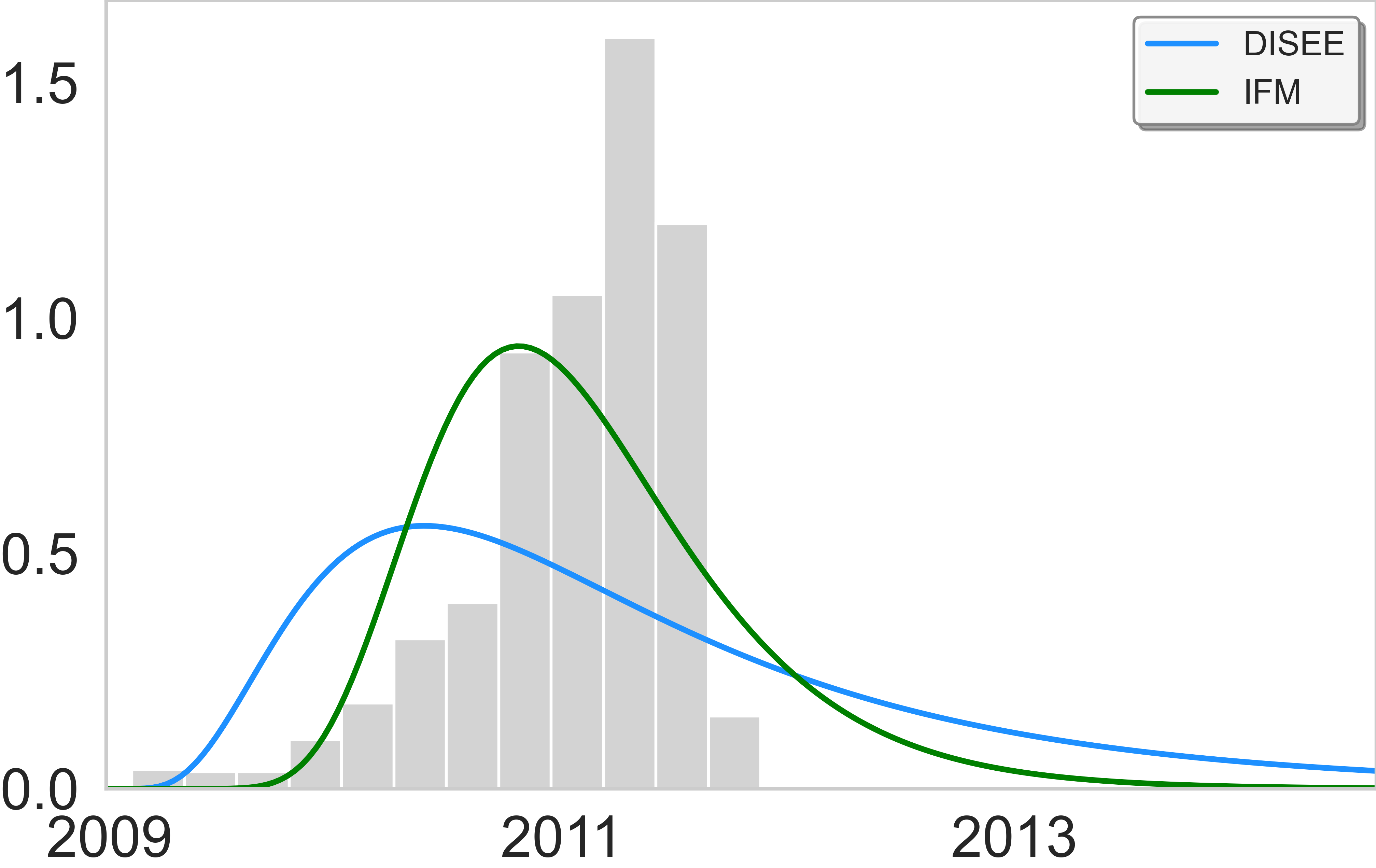}
        \caption{ImageNet: A large-scale hierarchical image database.}
    \end{subfigure}
    \hfill
    \begin{subfigure}[b]{0.32\textwidth}
        \includegraphics[width=\textwidth]{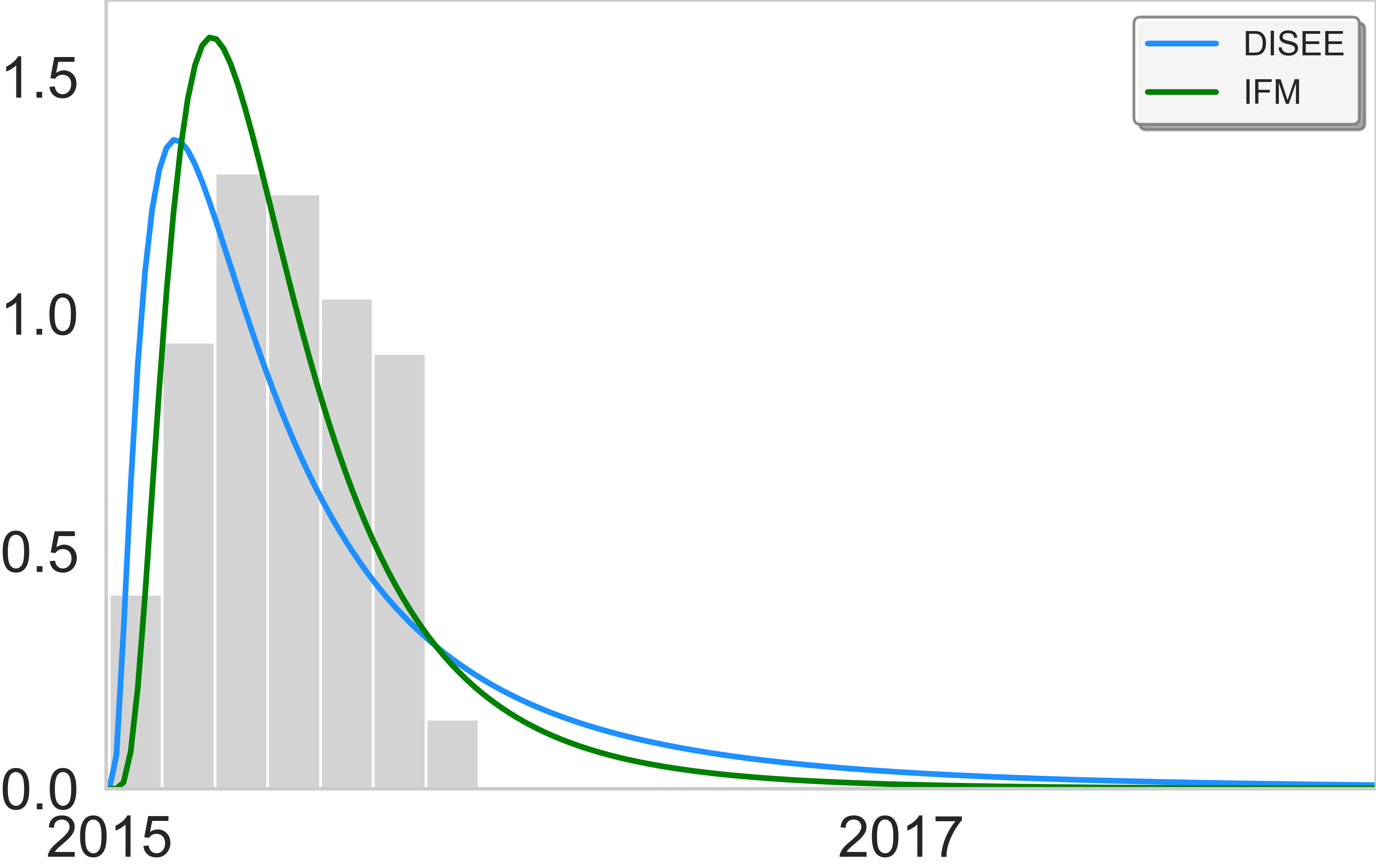}
        \caption{Going deeper with convolutions.}
    \end{subfigure}
    \caption{\textsl{Machine Learning}: \textsc{DISEE Log Normal} and \textsc{IFM Log Normal} models inferred impact function visualizations compared to the true citation histogram, for six popular \textsl{Machine Learning} papers with different citation dynamics.}     
    \label{fig:ML_ln}
\end{figure}

\begin{figure}[!b]
    \centering
    \includegraphics[width=0.9\textwidth]{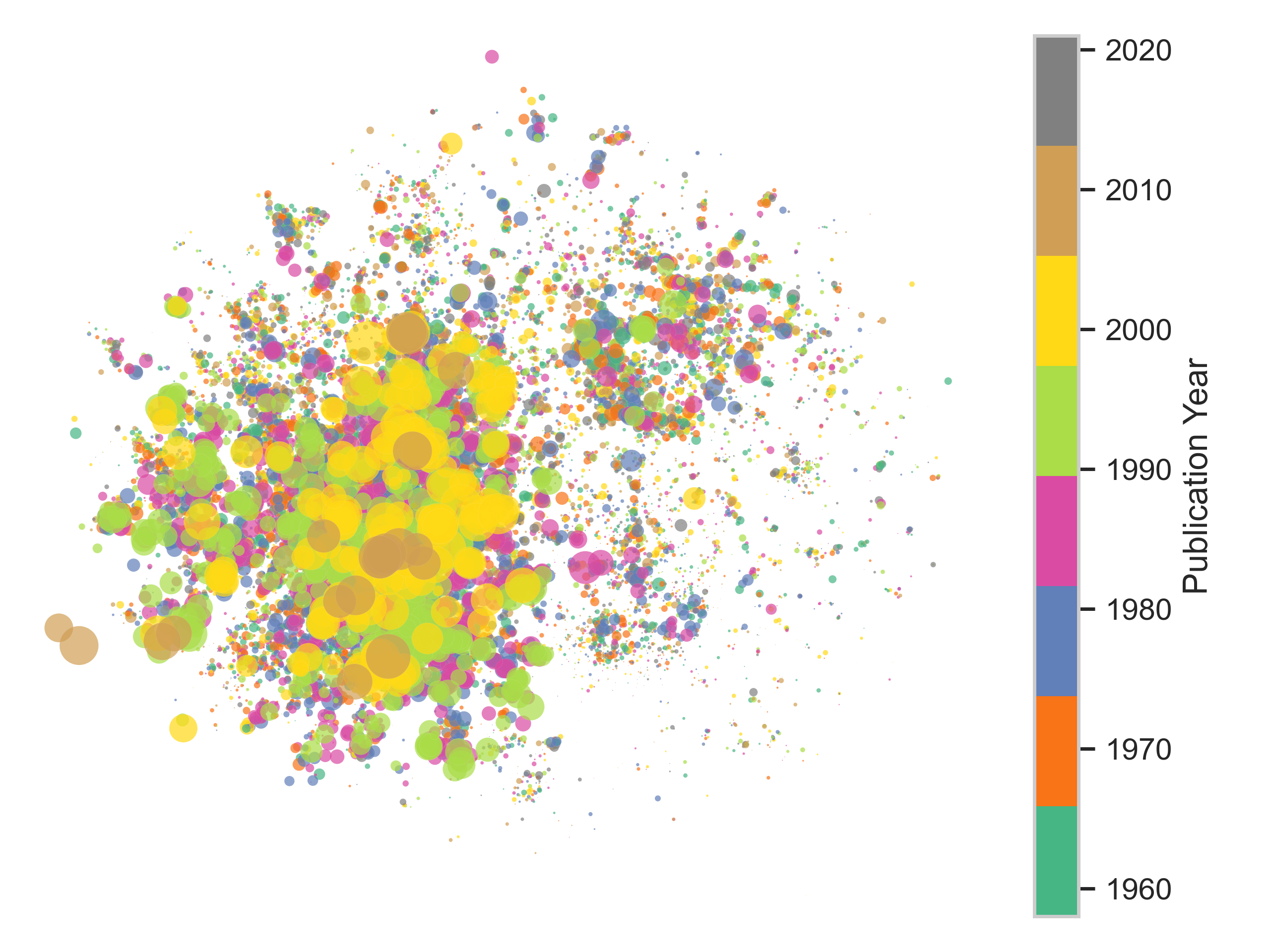}
    \caption{\textsl{Machine Learning}: \textsc{DISEE Truncated} embedding space visualization for all target papers published before the year $2023$. Node sizes are based on each paper's current mass, $f_i(t)*\exp{(\alpha_i)}$, and thus papers with zero mass are not visible denoting the end of their scientific relevance or "lifespan". Nodes are color-coded based on their publication year.}
    \label{fig:2023}
\end{figure}

\begin{figure}[!b]
     \begin{subfigure}{0.24\textwidth}
        \includegraphics[width=1.2\textwidth]{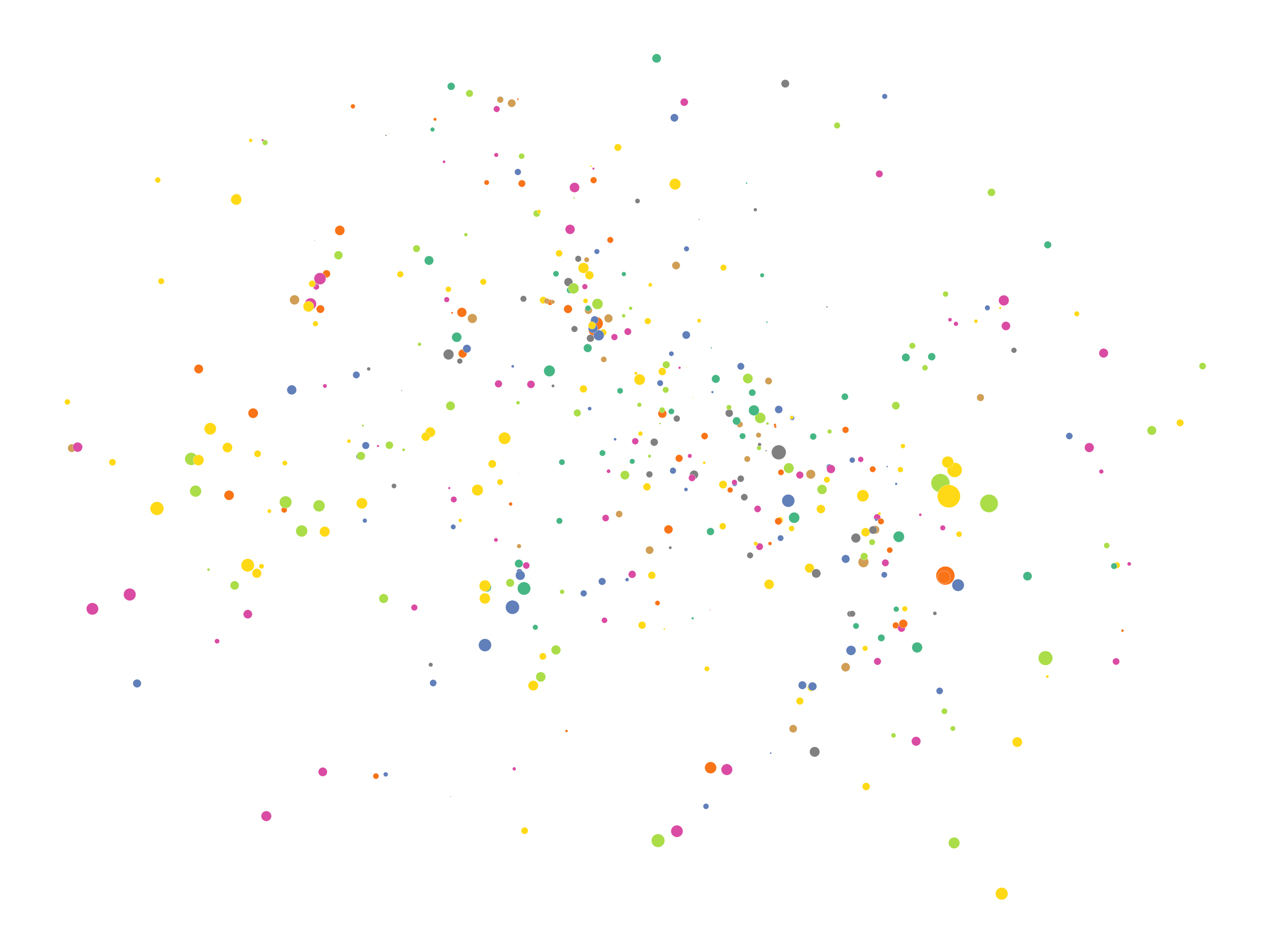}
        \caption{1988}
    \end{subfigure}
    \hfill
    \begin{subfigure}{0.24\textwidth}
        \includegraphics[width=1.2\textwidth]{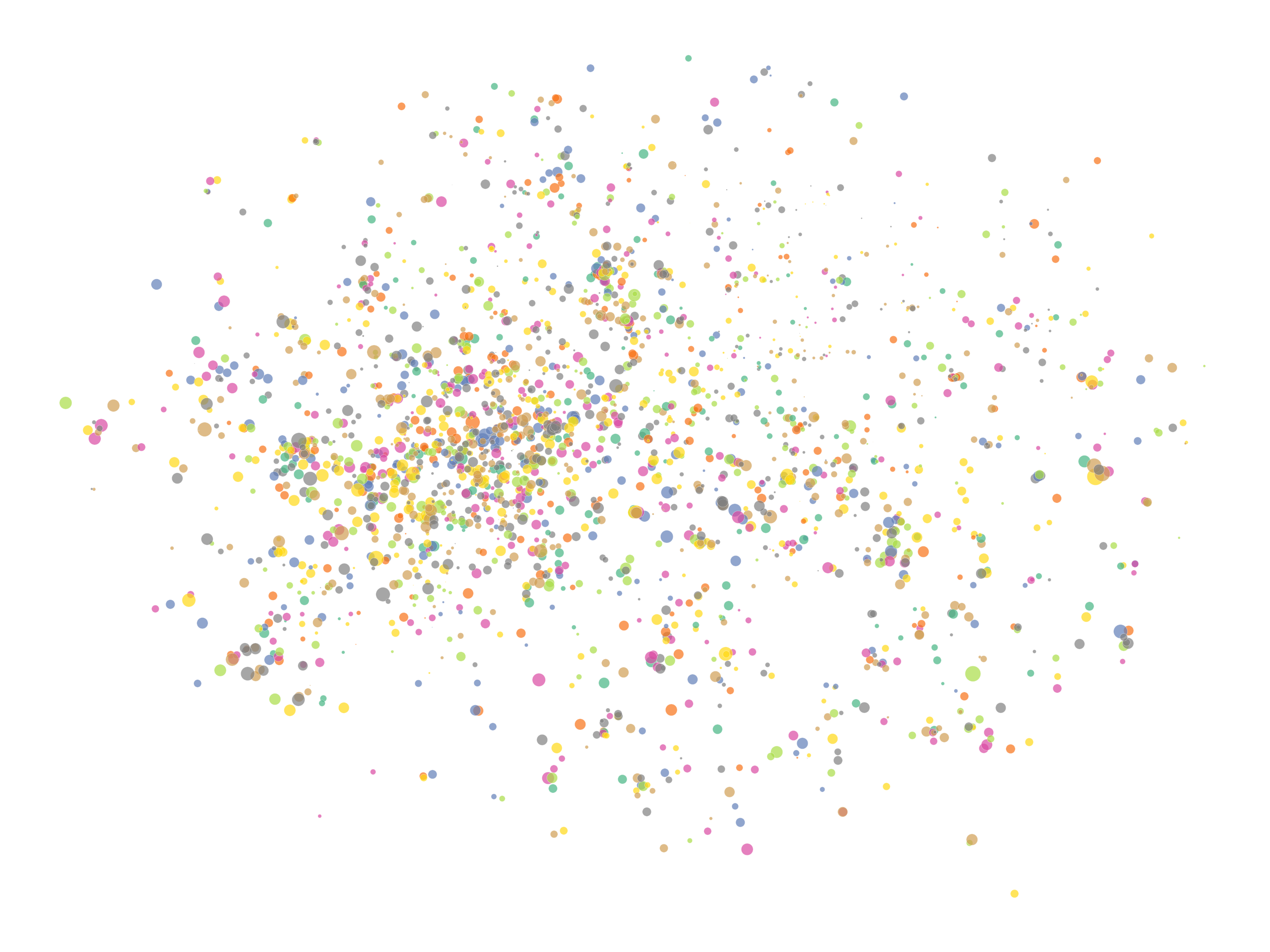}
        \caption{1998}
    \end{subfigure}
    \hfill
    \begin{subfigure}{0.24\textwidth}
        \includegraphics[width=1.2\textwidth]{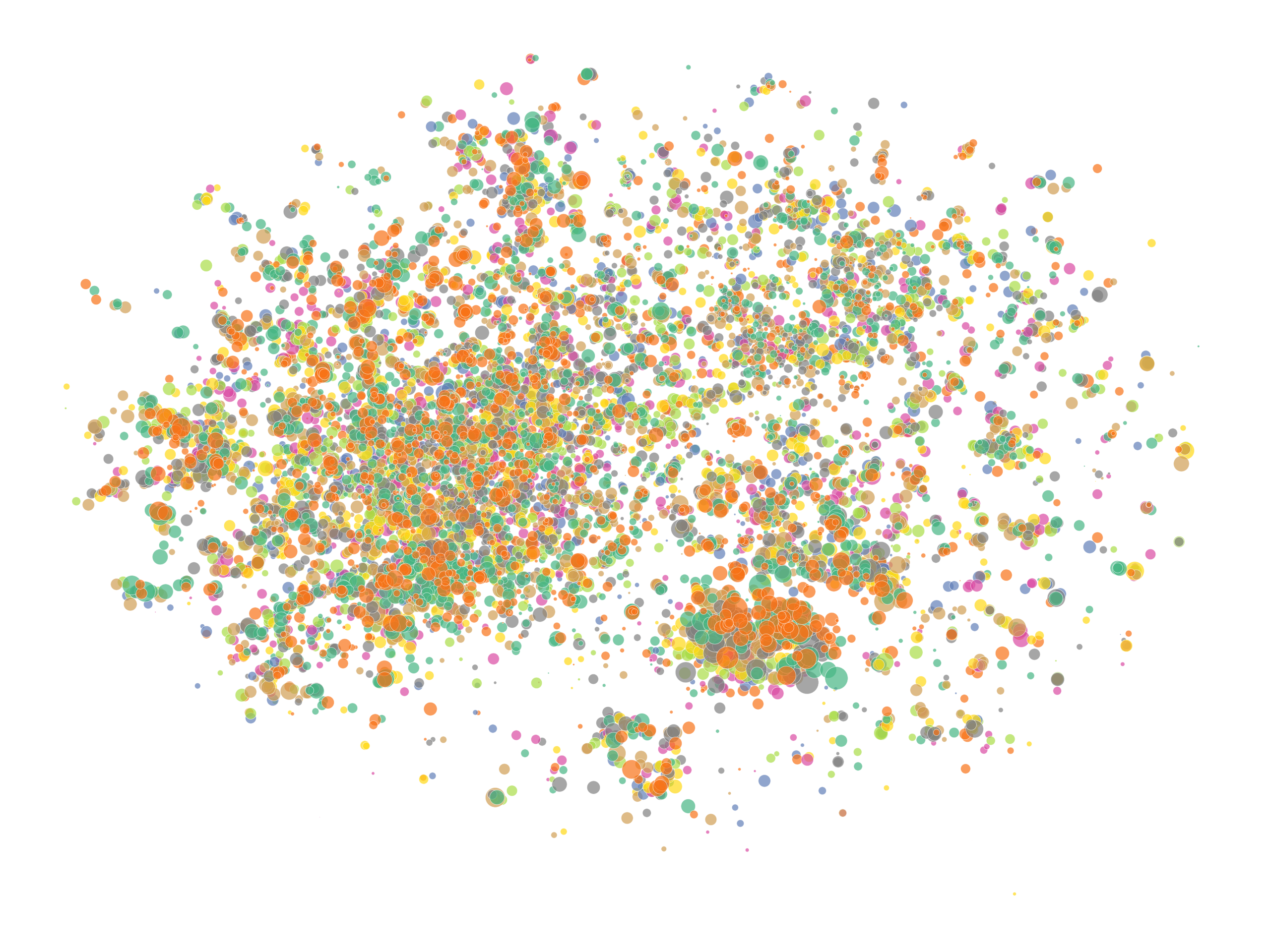}
        \caption{2008}
    \end{subfigure}
    \hfill
    \begin{subfigure}{0.24\textwidth}
        \includegraphics[width=1.2\textwidth]{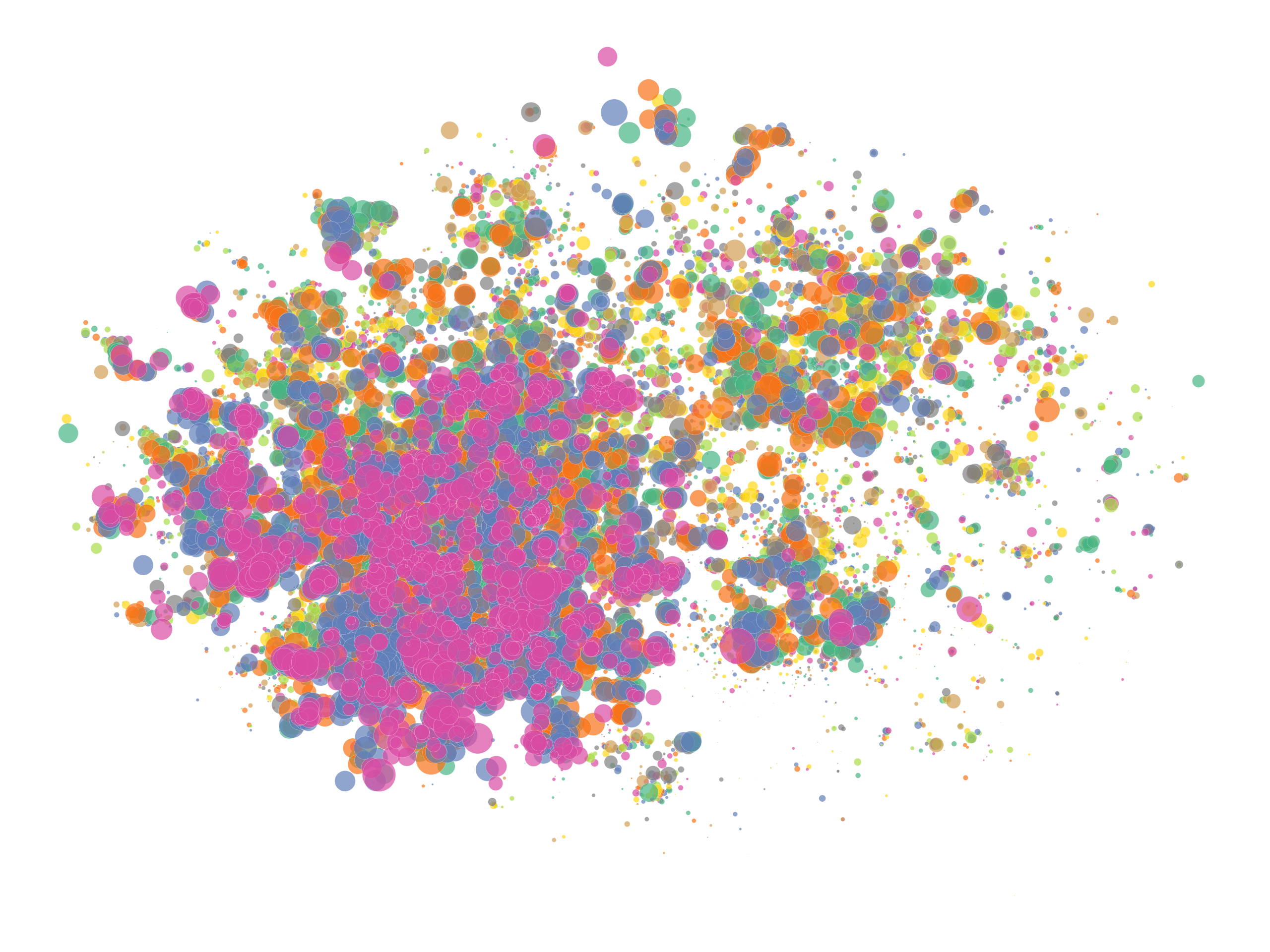}
        \caption{2018}
    \end{subfigure}
 \caption{\textsl{Machine Learning}: \textsc{DISEE Truncated} embedding space evolution throughout the years. Node sizes are based on each paper's mass, $f_i(t)*\exp{(\alpha_i)}$, showcasing how papers reach the end of their scientific relevance or "lifespan" by disappearing from the embedding space as time progresses. Nodes are color-coded based on their publication year.}
 \label{fig:history}
\end{figure}

\part{Discussion and conclusion}
\chapter{Discussion}
 Our work aimed to create novel Graph Representation Learning approaches, and most of all we tried to define what constitutes a fine embedding approach for accurate graph representation. We argue the characteristics of a fine embedding approach should 1) Be interpretable by human perception, similar nodes should be positioned in close proximity in the latent space, i.e. node similarity on the network should be translated into similarity in the latent space (one of the main goals, and intuition behind GRL). 2) Provide insights over the intrinsic structures existing in the network, facilitating interpretation and visualization in a hierarchical/multi-resolutional manner or even extracting pure network nodes and extreme profiles, characterizing network polarization. 3) Visualizations should not depend on heuristic dimensionality reduction approaches but provide accurate low-dimensional representations with maximum $D=3$. 4) Return high performance in downstream tasks such as link prediction/network reconstruction/node classification and community detection. 5) Scale the analysis to massive and large-scale networks, as billion-node graphs become more and more common in real scenarios. 

We demonstrated how the proposed frameworks provide such fine network representations since (1) They operate under the Euclidean distance metric, conveying homophily and transitivity properties, and thus providing an intuitive human perception of both first and high-order node similarity. (2) Naturally characterizing network intrinsic structures via the use of multi-scale hierarchical block structures, or constrained-to-polytopes latent spaces, providing hierarchical community identification, hybrid community memberships, and uncovering of extreme node profiles. 3) All of the frameworks showed very competitive performance under ultra-low dimension of $D=2,3$, providing direct network visualizations but sufficient capacity to enable accurate interpretation of network images. 4) High performance was achieved by all of the proposed methods in multiple downstream tasks, outperforming most of the state-of-the-art baselines while also enabling generative processes contrary to most of the competing methods. 5)  Our methods accounted for the computational costs of modern large-scale networks defining accurate linearithmic approximations of the network likelihood, unbiased random sampling procedures, and case-control inferences.

The first phase of this thesis focused on the Graph Representation Learning of positive integer weighted graphs. Most of all we tried to define what constitutes a fine embedding approach for accurate graph representation. Analytically, we first developed the Hierarchical Block Distance Model (\textsc{HBDM}), a scalable reconciliation of latent distance models and their ability to account for homophily and transitivity with hierarchical representations of network structures. We demonstrated how the proposed \textsc{HBDM} provides favorable network representations by (1) Operating with a Euclidean distance metric providing an intuitive human perception of node similarity. (2) Naturally representing multiscale hierarchical structure based on its block structure and carefully designed clustering procedure optimized in terms of Euclidean distances. (3) Directly and consistently operating in $D=2,3$ with high performance. (4) Performing well on all considered downstream tasks highlighting its ability to account for the underlying network structure. Importantly, the inferred hierarchical structure admits community discovery at multiple scales as highlighted by the inferred dendrograms and ordered adjacency matrices, and naturally extends to the characterization of communities of bipartite networks. Our discoveries highlight the existence and importance of hierarchical multi-scale structures in complex networks. The across hierarchy re-ordered adjacency matrices given by \textsc{HBDM}, manifest sub-communities inside of what already appears as a strongly connected community. This points to how delicate the task of defining communities is and the importance of accounting for communities at multiple scales, as enabled by the \textsc{HBDM}. Importantly, these results generalize for bipartite networks where multi-scale geometric representations, joint hierarchical structures, and community discovery are arduous tasks. In conclusion, we proposed the Hierarchical Block Distance Model, a scalable reconciliation of network embeddings using the latent distance model (\textsc{LDM}) and hierarchical characterizations of structure at multiple scales via a novel clustering framework. Notably, the model mimics the behavior of the \textsc{LDM} where the use of homophily and transitivity is most important while scaling in complexity by $\mathcal{O}(DN\log{N})$. We analyzed thirteen networks from moderate sizes to large-scale with the \textsc{HBDM} having favorable performance when compared to existing scalable embedding procedures. In particular, we observed that the \textsc{HBDM} well predicts links and node classes utilizing a very low embedding dimension of $D=2$ providing accurate network visualizations and characterization of structure at multiple scales. Our results demonstrate that favorable performance can be achieved using ultra-low (i.e. $D=2$) embedding dimensions and a scalable hierarchical representation that accounts for homophily and transitivity.

In the same direction, we have proposed the Hybrid-Membership Latent Distance Model (\textsc{HM-LDM}) that reconciles network embedding and latent community detection. The approach utilizes both the normal and squared Euclidean distance model where the latter integrated the non-negativity-constrained Eigenmodel with the Latent Distance Model. We demonstrated that the model could be constrained to the simplex without losing expressive power. The reduced simplex provides unique representations, ultimately resulting in the hard clustering of nodes to communities when the simplex is sufficiently shrunk. Notably, the proposed \textsc{HM-LDM} combines network homophily and transitivity properties with latent community detection enabling explicit control of soft and hard assignment through the volume of the induced simplex. We observed favorable link prediction performance in regimes in which the \textsc{HM-LDM} provides unique representations while enabling the ordering of the adjacency matrix in terms of prominent latent communities. Finally, we showed the ability of the model to extract valid community structures across multiple networks and showcased how the analysis extends to bipartite networks. Future work should compare the performance of \textsc{HM-LDM} against classical non-embedding methods such as the Degree Corrected Stochastic Block Model (\textsc{DC-SBM})\cite{karrer2011stochastic} or the Mixed Membership Stochastic Block Model (\textsc{MM-SBM}) \cite{JMLR:v9:airoldi08a}. Such a comparison is of particular interest since \textsc{DC-SBM} accounts for degree heterogeneity while \textsc{MM-SBM} for soft assignments, two important properties of \textsc{HM-LDM}.

In the second phase of the thesis, we focused on the analysis of signed integer-weighted networks. In that direction, we proposed the Skellam Latent Distance Model (\textsc{SLDM}) and Signed Latent Relational Distance model (\textsc{SLIM}) to provide easily interpretable network visualization with favorable performance in the link prediction tasks for weighted signed networks. In particular, endowing the model with a space-constrained to polytopes (forming the Signed relational Latent dIstance Model(\textsc{SLIM})) enabled us to characterize distinct aspects in terms of extreme positions in the social networks akin to conventional archetypal analysis but for graph-structured data. The Skellam distribution is considerably beneficial in modeling signed networks, whereas the relational extension of AA can be applied for other likelihood specifications, such as LDMs in general. This work thereby provides a foundation for using likelihoods accommodating weighted signed networks and representations akin to AA in general for analyzing networks.

Later in the second phase, we presented the signed Hybrid-Membership Latent Distance Model (\textsc{sHM-LDM}) reconciling Graph Representation Learning and latent community detection in singed networks. Specifically, we extended a hybrid membership model to account for signed networks and showed that a minimum volume approach could uncover distinct profiles in social networks while ensuring model identifiability. The presented framework was formulated to include an Euclidean as well as a squared Euclidean norm. For the latter, a direct relationship to an Eigenmodel was shown. Furthermore, by controlling the volume of the simplex by the magnitude of $\delta$, a sufficiently reduced simplex leads to unique representations. Notably, the generalization to signed networks facilitated the extraction of distinct network profiles representing positive interactions and animosity. In regimes where the \textsc{sHM-LDM} provide unique representations, we observed favorable link prediction performance and the ability to order the adjacency matrix based on prominent latent communities and distinct profiles. Importantly, the extended \textsc{sHM-LDM} merges homophily and heterophily properties to account for positive and negative ties as present in signed networks, enabling explicit control of soft and hard assignment to extreme node profiles, through the volume of the induced simplex. 

In the third phase, we focused on \textsc{SEN} networks and more specifically on the analysis of citation networks. We proposed a novel likelihood function for the characterization of such single-event networks. Using this likelihood, we defined the Dynamic Impact Single-Event Embedding Model (\textsc{DISEE}) characterizing scientific interactions and impact, in terms of a latent distance model in which forces were reparameterized to be proportional to the product of the masses of the interacting entities. Such a model successfully reconciled static latent distance network embedding approaches with classical dynamic impact assessments of citation networks. Extensive experiments in three real citation networks, showcased \textsc{DISEE} as a powerful link predictor, able to successfully describe papers' impact and relevance lifespans while visualization of the inferred embedding space provided new insights on how different domains of science evolve through time.

Our finding of ultra-low dimensional accurate characterizations of network structures supports the findings in \cite{exact_Emb} in which a logistic PCA model was found to enable exact low-dimensional recovery of multiple real-world networks. Whereas the work of \cite{exact_Emb} focuses on exact network reconstruction we find that generalizable patterns can be well extracted in ultra-low dimensional representations with performance saturating after just $D=8$ dimensions for all networks considered. Whereas \cite{exact_Emb} found that their low-dimensional space did not perform well in classification tasks we observed strong node classification performance by the low-dimensional representations provided by our frameworks. Importantly, in node classification and the \textsc{HBDM}, we observed better performance using KNN as opposed to simple linear classification based on logistic/multinomial regression typically used for node classification. This highlights that whereas most \textsc{GRL} works use linear classifiers there is no guarantee that the embedding space will be linearly separable and performance should therefore be compared to non-linear classifiers as they may provide more favorable performance as observed in this study.

Recent pioneering works \cite{mlg2017_6,NIPS2017_59dfa2df} have drawn significant attention of the research community by questioning the conventional embedding space preference, as also reviewed for the LSM family \cite{LSM_geo}.
It is well known that many real-world networks show power-law degree distribution, or they can consist of latent hierarchical inner structures. Therefore, Euclidean space might not always be appropriate to represent such complex network architectures. It might also require higher-dimensional spaces to show comparable performance in the \textsc{GRL} tasks. The works of \cite{mlg2017_6,NIPS2017_59dfa2df} demonstrated that hyperbolic spaces, such as the Poincare disk model, can provide substantial benefits over the Euclidean space. The presented models, naturally extend to other distance measures and future studies should explore how they can be extended to hierarchical representations and polytopes-defined spaces beyond Euclidean geometry. 

Covariate information plays an important role in the outstanding performance of \textsc{GRL} methods, especially GNNs. In the current \textsc{LSM} literature, side information is accounted for by extra regressors in the logit/log link functions expressing the likelihood of a dyad being connected. Using the Mahalanobis distance imposing a block-diagonal covariance matrix, the proposed frameworks can naturally incorporate covariate information directly into the latent space and notably construct multi-scale structures and polytope representations via the enriched and concatenated embedding of the latent variables and the covariate information. In more detail, we can define a new embedding matrix $\bm{\Bar{Z}}$ as the concatenation over the latent variables and the covariate information for node $i$ as: $\bm{\Bar{z}}_i=[\bm{z}_i;\bm{x}_i]$ and a Mahalanobis correlation matrix as: $\bm{S}=\begin{bmatrix} \bm{I} & \bm{0}\\ \bm{0}^T & \bm{J} \end{bmatrix} \in\mathbb{R}^{(D+R)\times (D+R)}$, where $I \in\mathbb{R}^{D\times D}$ the identity matrix, the zero matrix $\bm{0} \in\mathbb{R}^{D\times R}$ and the covariate coefficient matrix $\bm{J} \in\mathbb{R}^{R\times R}$. In this setting, \textsc{HBDM} is able to construct a covariate information-aware multi-scale latent space by the use of the Mahalanobis distance $d_{ij}=\sqrt{(\bm{\Bar{z}}_i-\bm{\Bar{z}}_j)^T\bm{S}^{-1}(\bm{\Bar{z}}_i-\bm{\Bar{z}}_j)}$. Our analysis presently did not explore side information and this is also why we did not include comparisons to prominent GNN-based approaches as these procedures do not provide favorable performance when only learning from the graph structure itself. As such, we observed (not shown) poor performance of GraphSage \cite{hamilton2017inductive} when only having access to the graph structure in the present setup. Our presented methods, operate on static networks and thus are not naturally inductive models. Nevertheless, potential new emerging nodes can be projected into the inferred latent space by fixing the embeddings of nodes present in the training set while optimizing the new nodes for their locations in the learned latent space. We leave a comparison of such a strategy against naturally inductive models such as GNNs for future work. 

Our frameworks, use the \textsc{LDM} and thus are good at characterizing transitivity and homophily at a node and cluster level, whereas the random effects enable accounting for degree heterogeneity. Notably, our methods suffer from the limitations of the \textsc{LDM} and are thus unable to model stochastic equivalence. Future work should therefore investigate hierarchical structures and polytope representations imposed on more flexible \textsc{GRL} procedures enabling stochastic equivalence and contrast the performance when accounting for stochastic equivalence to the existing methods based on the \textsc{SBM} which as a latent class model is known to express stochastic equivalence \cite{clauset2008hierarchical, roy2007learning, herlau2012detecting,agglo_bayes, herlau2013modeling,Peixoto_2014}. In addition, the optimization for our frameworks is a highly non-convex problem and thus relies on the quality of initialization in terms of convergence speed. In this regard, we use a deterministic initialization based on the normalized Laplacian. In addition, for the signed network models we observed that a maximum likelihood estimation of the model parameters became unstable when the network contained some nodes having only negative interactions. This is a direct consequence of the presence of the distance term ($\exp(+||\cdot||_2)$) for negative interactions, which can lead to overflow during inference. Nevertheless, we adopted a MAP estimation that was found to be stable across all networks. For real signed networks, the generative model created an "excess" of negative links increasing the overall network sparsity. For that, a modified \textsl{SLIM} excluding the regularization over the model parameters was introduced which achieved correct network sparsity (as shown in the main paper). Assuming priors over the model parameters created a bias over the generated network when compared to the ground truth network statistics.

\chapter{Conclusion}

In recent years, there has been a surge in the complexity and volume of data represented as graphs. In this context, we have presented innovative representation learning models, based on the Latent Distance Model formulation, specifically tailored for the examination of networks that involve both signed and unsigned integer weights, as well as, single-event networks. We have successfully presented multiple frameworks able to learn informative node representations, expressing homophily and transitivity properties in unsigned networks while for the case of signed networks, models were generalized to convey the balance theory. The Hierarchical Block Distance Model facilitated the extraction of hierarchical structures present in complex networks while the Hybrid Membership Distance Model accounted for community discovery, explicitly controlling both hard and soft community assignments. Furthermore, the family of Latent Distance Models was extended to the analysis of signed networks via the Skellam Latent Distance Model which was proved to be a powerful link predictor. Constraining the latent space to a polytope yielded the Signed Relational Latent Distance model generalizing Archetypal Analysis to relational data extracting distinct profiles of networks and characterizing network polarization. When the polytope was constrained to the $D$-simplex we obtained the signed Hybrid-Membership Latent Distance Model which a continuously decreasing simplex volume, defined a Minimum Volume approach for Archetypal Analysis yielding also extreme profile identification. Importantly, all proposed frameworks defined scalable optimization approaches via the accurate linearithmic hierarchical approximation of the likelihood (\textsc{HBDM}), unbiased random sampling procedures (\textsc{HM-LDM}, \textsc{sHM-LDM}, \textsc{SLDM}, \textsc{SLIM}), and case-control inferences (\textsc{DISEE}). Our frameworks facilitated informative network visualizations including network hierarchical organization of the adjacency matrix (\textsc{HBDM}), soft and hard community extraction (\textsc{HM-LDM}), informative polytope visualizations for signed networks (\textsc{sHM-LDM}, \textsc{SLIM}), and impact characterization and latent space visualizations of single-event networks (\textsc{DISEE}). Importantly, such valuable visualization analyses were extended to bipartite networks where such a generalization is not trivial. For all of our proposed models, we included extensive experimental evaluations to demonstrate that the proposed approaches generally surpass widely adapted baseline methods in node classification, link prediction, and network reconstruction tasks. Such results were highlighted especially for the ultra-low dimensions of $D=2,3$ where very few of the competing methods were found to be competitive. Importantly, the proposed frameworks were validated in multiple settings and downstream tasks and were found to have the most consistent performance across tasks (in no task their performance was significantly lower than competing baselines). This helps us to characterize embedding approaches relying on the Euclidean metric as the best choice when it comes to defining low-dimensional embeddings that are required to perform multiple tasks. Last but not least, we have successfully shed light on a missing part in the \textsc{GRL} literature which is to extensively position and benchmark the performance of Latent Distance Models for Graph Representation Learning against state-of-the-art baselines, showcasing their superior performance in multiple settings.

\bibliography{main.bib}
\bibliographystyle{ieeetr}


\end{document}